\newcommand{\DocumentType}{thesis} 
\newcommand{\DocumentLanguage}{en} 
\newcommand{\PaperSize}{a4paper} 
\newcommand{\Twosided}{true} 

\documentclass[%
a4paper,
twoside=\Twosided,
openany,
11pt,%
fleqn,
captions=tablesignature,
numbers=noenddot,
headlines=2,
final,%
]{scrbook}

\usepackage{fixltx2e}
\usepackage{ifthen}
\ifthenelse{\equal{\PaperSize}{a4paper}}{
	\usepackage[paper=\PaperSize,twoside=\Twosided,%
	textheight=246mm,%
	textwidth=160mm,%
	heightrounded=true,
	ignoreall=true,
	marginparsep=5pt,
	marginparwidth=10mm,
	hmarginratio=2:1,
	]{geometry}}{}%
\ifthenelse{\equal{\PaperSize}{letterpaper}}{
	\usepackage[paper=\PaperSize,twoside=\Twosided,%
	textheight=9in,%
	textwidth=6.5in,%
	heightrounded=true,
	ignoreheadfoot=false,
	marginparsep=5pt,
	marginparwidth=10mm,
	hmarginratio=3:2,
	]{geometry}}{}%
\ifthenelse{\equal{\DocumentLanguage}{en}}{\usepackage[utf8x]{inputenc}\usepackage[USenglish]{babel}}{}%
\ifthenelse{\equal{\DocumentLanguage}{de}}{\usepackage[utf8x]{inputenc}\usepackage[ngerman]{babel}\usepackage[T1]{fontenc}}{}%
\usepackage[%
headtopline,plainheadtopline,
headsepline,plainheadsepline,%
footsepline,plainfootsepline,%
footbotline,plainfootbotline,%
automark
]{scrlayer-scrpage}
\usepackage{makeidx}
\usepackage[hypcap=false]{caption}
\usepackage{multicol}
\usepackage{setspace}
\usepackage[stable,bottom,hang,splitrule,multiple]{footmisc}

\usepackage{fixltx2e}
\usepackage{ifthen}
\ifthenelse{\equal{\DocumentLanguage}{en}}{\usepackage[utf8x]{inputenc}\usepackage[USenglish]{babel}}{}%
\ifthenelse{\equal{\DocumentLanguage}{de}}{\usepackage[utf8x]{inputenc}\usepackage[ngerman]{babel}\usepackage[T1]{fontenc}}{}%
\usepackage[%
headtopline,plainheadtopline,
headsepline,plainheadsepline,%
footsepline,plainfootsepline,%
footbotline,plainfootbotline,%
automark
]{scrlayer-scrpage}
\usepackage{makeidx}
\usepackage[]{caption}
\usepackage{multicol}
\usepackage{setspace}
\usepackage[stable,bottom,hang,splitrule,multiple]{footmisc}

\usepackage[noadjust]{cite}
\ifthenelse{\equal{\DocumentLanguage}{en}}{\usepackage{varioref}}{}
\ifthenelse{\equal{\DocumentLanguage}{de}}{\usepackage[german]{varioref}}{}
\usepackage{color}
\usepackage{rotating}
\usepackage{gensymb}
\usepackage[right]{eurosym}
\usepackage[normalem]{ulem}
\usepackage{remreset}
\usepackage[%
breaklinks=true,
colorlinks=true,
linkcolor=black,anchorcolor=black,citecolor=black,filecolor=black,%
menucolor=black,urlcolor=black]{hyperref}
\usepackage{paralist}
\usepackage[avantgarde]{quotchap}

\usepackage{amsmath,amssymb,amstext,bm, amsthm}
\usepackage{mathcomp}

\usepackage{graphicx}
\usepackage{subfigure}
\usepackage{placeins}
\usepackage{float}
\usepackage{psfrag}
\usepackage{wrapfig}

\usepackage{hhline}
\usepackage{longtable}
\usepackage{dcolumn}
\usepackage{colortbl}
\usepackage{booktabs}
\usepackage{multirow}
\usepackage{listings}

\usepackage{layout}
\usepackage{xspace}
\usepackage{calc}
\makeindex

\usepackage{framed}
\usepackage[norelsize,ruled,vlined]{algorithm2e}
\setheadwidth{text}
\setfootwidth{text}
\setheadtopline[textwithmarginpar]{0.5pt}
\setheadsepline[textwithmarginpar]{0.5pt}
\setfootsepline[textwithmarginpar]{0.5pt}
\setfootbotline[textwithmarginpar]{0.5pt}

\makeatletter
\renewcommand{\cleardoublepage}{\clearpage\if@twoside\ifodd\c@page\else\thispagestyle{plain}\hbox{}\newpage\if@twocolumn\hbox{}\newpage\fi\fi\fi}
\newcommand{\emptydoublepage}{\clearpage\if@twoside\ifodd\c@page\else\thispagestyle{empty}\hbox{}\newpage\if@twocolumn\hbox{}\newpage\fi\fi\fi}%
\newcommand{\emptypage}{\clearpage\thispagestyle{empty}\hbox{}\newpage\if@twocolumn\hbox{}\newpage\fi}%
\makeatother

\clearscrheadfoot 
\KOMAoptions{headlines=1} 
\ihead[]{}%
\ohead[\ShortTitle]{\footnotesize\headmark}%
\ofoot[--~~\pagemark~~--]{--~~\pagemark~~--}%

\setkomafont{paragraph}{\normalfont\bfseries}

\makeatletter
\renewcommand\l@table{\l@figure}%
\renewcommand\l@lstlisting{\l@figure}%
\renewcommand{\@pnumwidth}{1.85em}
\makeatother%

\makeatletter%
\@removefromreset{footnote}{chapter}%
\makeatother%

\makeatletter%
\renewcommand{\fps@figure}{hbtp}%
\renewcommand{\fps@table}{hbtp}%
\makeatother%

\setlongtables
\addtokomafont{captionlabel}{\footnotesize\itshape}%
\addtokomafont{caption}{\footnotesize\itshape}%
\ifthenelse{\equal{\DocumentLanguage}{en}}{\renewcaptionname{USenglish}{\figurename}{Figure}}{}%
\ifthenelse{\equal{\DocumentLanguage}{de}}{\renewcaptionname{ngerman}{\figurename}{Abbildung}}{}%
\newcommand{\newchapter}[2]{\FloatBarrier\chapter{#1}\label{chp:#2}}
\newcommand{\newchapterNoNr}[2]{\FloatBarrier\chapter*{#1}\label{chp:#2} \addcontentsline{toc}{chapter}{#1}}
\newcommand{\newsection}[2]{\FloatBarrier\vspace{5mm}\section{#1}\label{sec:#2}}%
\newcommand{\newsubsection}[2]{\FloatBarrier\vspace{3mm}\subsection{#1}\label{sec:#2}}%
\newcommand{\newsubsubsection}[2]{\vspace{2mm}\subsubsection{#1}\label{sec:#2}}%
\newcommand{\nxtpar}{\par\medskip}

\newcommand{\openingquote}[2]{\hfill\parbox[t]{0.55\textwidth}{\itshape\raggedleft{"#1"}\\\footnotesize -- #2}\nxtpar}%

\newcommand{\pwd}{.} 

\newcommand{\definition}{=}
\newcommand{\vm}[1]{\ensuremath{\bm{#1}}}

\renewcommand{\log}{\ensuremath{\text{log}}}

\newcommand{\bigO}{\mathcal{O}}

\DeclareMathOperator{\E}{\mathbb{E}}
\DeclareMathOperator*{\argmax}{arg\,max}
\DeclareMathOperator*{\argmin}{arg\,min}
\DeclareMathOperator{\arctanh}{atanh}
\DeclareMathOperator\arcoth{arcoth}
\DeclareMathOperator{\sgn}{sgn}

\renewcommand{\Re}[1]{\ensuremath{\text{Re}\!\left\{#1\right\}}}
\renewcommand{\Im}[1]{\ensuremath{\text{Im}\!\left\{#1\right\}}}

\newcommand{\REAL}{\ensuremath{\mathbb{R}}}
\newcommand{\REALPos}{\mathbb{R}_+^*}
\newcommand{\INTEGER}{\ensuremath{\mathbb{Z}}}
\newcommand{\INTEGERPos}{\mathbb{Z}_+^*}
\newcommand{\COMPLEX}{\ensuremath{\mathbb{C}}}
\newcommand{\POLY}{\ensuremath{\mathbb{K}}}

\newcommand{\stationaryPoint}{\circ}				

\newcommand{\sampleSpace}[1]{\mathcal{\uppercase{#1}}}   		
\newcommand{\productSpace}[2]{\mathcal{\uppercase{#1}}^{#2}}
\newcommand{\RV}[2][]				
{\ifthenelse{\equal{#1}{}}
	{X_{#2}} {\uppercase{#1}_{#2}}}			
\newcommand{\RVSet}[2]  
{\ifthenelse{\equal{#1}{}}
	{\mathbf{X}_{#2}} {\mathbf{\uppercase{#1}}_{#2}}}	
\newcommand{\RVval}[2][]
{\ifthenelse{\equal{#1}{}}
	{x_{#2}} {\lowercase{#1}_{#2}}}			
\newcommand{\RVvalSet}[2][]				
{\ifthenelse{\equal{#1}{}}
	{\mathbf{x}_{#2}} {\mathbf{\lowercase{#1}}_{#2}}}	
\newcommand{\pmf}[1]{P_{#1}}				
\newcommand{\marginals}[1]{\pmf{#1}}			
\newcommand{\joint}{\pmf{\setOfNodes}(\RVvalSet[]{})}	
\newcommand{\jointApprox}{\pmfApprox{\setOfNodes{}}(\RVvalSet{})}
\newcommand{\pmfApprox}[1]{\tilde{P}_{#1}}		
\newcommand{\pseudomarginals}[1][]			
{\ifthenelse{\equal{#1}{}}
	{{\tilde{P}_{{{B}}}}} {{\tilde{P}_{{{B}[#1]}}}}}	
\newcommand{\pseudomarginalsMinGlobal}[1][]{\ifthenelse{\equal{#1}{}}{\pseudomarginals^*}{\pseudomarginals[#1]^{*}}}
\newcommand{\pseudomarginalsMinLocal}[1][]{\ifthenelse{\equal{#1}{}}{\pseudomarginals^{m}}{\pseudomarginals[#1]^{m}}}
\newcommand{\pseudomarginalsMinLocalNeg}[1][]{\ifthenelse{\equal{#1}{}}{\pseudomarginals^{n}}{\pseudomarginals[#1]^{n}}}
\newcommand{\pseudomarginalsExact}{P_{B}}
\newcommand{\singleExact}[1]{\pmf{\RV{#1}}}
\newcommand{\singleExactVal}[2]{\singleExact{#1}(\RV{#1}=#2)}
\newcommand{\singleApprox}[1]{\pmfApprox{\RV{#1}}}

\newcommand{\pairwiseExact}[2]{\pmf{\RV{#1},\RV{#2}}}

\newcommand{\pairwiseApprox}[2]{\pmfApprox{\RV{#1},\RV{#2}}}
\newcommand{\pairwiseShort}[2]{\Phi(\RVval{#1},\RVval{#2})}
\newcommand{\pairwiseSBP}[3]{\Phi_{#3}(\RVval{#1},\RVval{#2})}
\newcommand{\localShort}[1]{\Phi(\RVval{#1})}
\newcommand{\localSBP}[2]{\Phi_{#2}(\RVval{#1})}

\newcommand{\MPolytope}{\mathbb{M}}			
\newcommand{\LPolytope}{\mathbb{L}}			

\newcommand{\graph}{\mathcal{G}}		
\newcommand{\setOfNodes}[1][]
{\ifthenelse{\equal{#1}{}}
	{\mathbf{X}}{\mathbf{\uppercase{#1}}}}	
\newcommand{\setOfEdges}{\mathbf{E}}		
\newcommand{\edge}[2]{(#1,#2)}			
\newcommand{\neighbors}[1]{\partial({#1})}	
\newcommand{\neighborsWO}[2]			
{\{\neighbors{#1} \backslash \RV{#2}\}}
\newcommand{\nodeDegree}[1]{d_{#1}}		
\newcommand{\averageDegree}[1][]{\hat{d}_{#1}}	
\renewcommand{\path}{\mathcal{P}}		
\newcommand{\ugm}[1][]{\mathcal{U}_{#1}}	
\newcommand{\clique}[1]{{C}_{#1}}		
\newcommand{\setOfCliques}{\mathbf{C}}		
\newcommand{\setOfPotentials}{ \Psi}		
\newcommand{\setOfPotentialsSBP}[1]{\Psi_{#1}}
\newcommand{\potential}[1]{\Phi_{#1}}
\newcommand{\potentialCliquePairwise}[2]{
	\Phi_{C_{(#1,#2)}}(\RVval{#1},\RVval{#2})}
\newcommand{\pairwise}[3]{ 
	\Phi_{\RV[#1]{#2},\RV[#1]{#3}} (\RVval[#1]{#2},\RVval[#1]{#3})}
\newcommand{\local}[2]{\Phi_{\RV[#1]{#2}}(\RVval[#1]{#2}) }

\newcommand{\partitionFunction}{\mathcal{Z}}
\newcommand{\partitionBethe}{\mathcal{Z}_B}
\newcommand{\partitionBetheLocalMin}[1]{\partitionBethe^{#1}}

\newcommand{\coupling}[2]{J_{#1#2}}
\newcommand{\field}[1]{\theta_{#1}}
\newcommand{\parameter}{(\field{},\coupling{}{})}
\newcommand{\energy}[1]{E\big( #1 \big)}
\newcommand{\entropy}[1]{S\big( #1 \big)}
\newcommand{\FGibbs}[1] {\ifthenelse{\equal{#1}{}}{\mathcal{F_G}}{\mathcal{F_G}(#1)}}
\newcommand{\FB}{\mathcal{F_B}}
\newcommand{\EB}{E_{\mathcal{B}}}
\newcommand{\SB}{S_{\mathcal{B}}}
\newcommand{\FBGlobalMin}{\FB^*}
\newcommand{\FBStationary}{\FB^{\circ}}
\newcommand{\FBLocalMin}[1]{\FB^{#1}}
\newcommand{\meanMinLocal}[1]{m_{#1}}

\newcommand{\msg}[4][]						
{\ifthenelse{\equal{#4}{}} 
	{\mu^{#1}_{#2 #3}(x_{#3})} {\mu^{#1}_{#2 #3}(\RV{#3}=#4)}}
\newcommand{\msgReparam}[3][]{\nu^{#1}_{#2 #3}}
\newcommand{\fpMsg}[2]{\msg[\stationaryPoint]{#1}{#2}{}}	
\newcommand{\msgNorm}[3][]{\alpha^{#1}_{#2 #3}}
\newcommand{\setOfMessages}[1][]  {\vm{\mu}^{#1}}	
\newcommand{\setOfMsgReparam}[1][]{\vm{\nu}^{#1}}
\newcommand{\fpSetOfMessages}[1][]
{\ifthenelse{\equal{#1}{}} 
	{\vm{\mu}^{\stationaryPoint}}
	{\vm{\mu}^{(#1)}}}
\newcommand{\fpSetOfMessagesReparam}[1][]
{\ifthenelse{\equal{#1}{}} 
	{\vm{\nu}^{\stationaryPoint}}
	{\vm{\nu}^{(#1)}}}
\newcommand{\setOfNorm}[1][]{\vm{\alpha}^{#1}}		
\newcommand{\fpSetOfNorm}{\setOfNorm[\stationaryPoint]}

\newcommand{\BP}{\mathcal{BP}}
\newcommand{\mapBP}[2][]{\mathcal{BP}_{#1}\left(#2\right)}
\newcommand{\BPD}{\text{BP}_{\text{D}}}
\newcommand{\SBP}{\text{SBP}}
\newcommand{\BPVariant}[1]{\mathcal{BP}_{\text{#1}}}

\newcommand{\mean}[1]{m_{#1}}
\newcommand{\correlation}[2]{\chi_{#1 #2}}
\newcommand{\meanAvg}{\langle \mean{} \rangle}
\newcommand{\meanAvgApprox}{\langle \tilde{\mean{}} \rangle}
\newcommand{\mse}[1][]{\text{MSE}}
\newcommand{\mseb}{\text{MSE}_{\mathcal{B}}}

\newcommand{\fp}[1]{#1^{\stationaryPoint}}			
\newcommand{\stateVec}[1]{\vm{#1}}			 	
\newcommand{\map}{\mathcal{F}}					
\newcommand{\gradient}{\nabla}
\newcommand{\Hessian}{\nabla^2}

\newcommand{\equationSystem}[1]{\vm{#1}}
\newcommand{\eq}[1][]{f_{#1}}					
\newcommand{\setOfBPEq}{\setOfEq (\setOfMessages, \setOfNorm )}	
\newcommand{\BPEq}[1][]{\eq[#1]  (\setOfMessages, \setOfNorm )}	
\newcommand{\setOfStartEq}{\startSys (\setOfMessages, \setOfNorm )}
\newcommand{\coeff}[1]{a_{#1}}

\newcommand{\Jacobian}[1][]{\mathcal{F}_{#1}'(\fpSetOfMessagesReparam)}
\newcommand{\JacobianGeneral}[1]{\mathcal{F}'(#1)}
\newcommand{\Spectrum}[2][]{\Lambda{#1}\big(#2\big)}
\newcommand{\Radius}[2][]{\rho_{#1}\big(#2\big)}
\newcommand{\ev}[1]{\lambda_{#1}}

\newcommand{\startSys}{\equationSystem{Q}}
\newcommand{\setOfEq}{\equationSystem{F}}				
\newcommand{\homotopy}{\equationSystem{H}}
\newcommand{\variety}{\mathbf{V}}
\newcommand{\ideal}[1]{\langle{#1}_1,\ldots,{#1}_s\rangle}
\newcommand{\polytope}[1]{S_{#1}}
\newcommand{\polytopeLifted}[1]{\hat{S}_{#1}}

\newcommand{\allConfigurations}{\RVvalSet{} \in \sampleSpace{X}^N}	
\newcommand{\partitionBetheWithP}{\partitionBethe(\pseudomarginals)}
\newcommand{\realSol}{\variety_{\REALPos}(\setOfEq)}
\newcommand{\decisionP}{\mathcal{D}_P}
\newcommand{\decisionZ}{\mathcal{D}_Z}

\newcommand{\scaling}{\zeta}		
\newcommand{\iteration}[1]{{#1}}	

\newcommand{\FRegion}{(F)}
\newcommand{\AFRegion}{(AF)}
\newcommand{\PRegion}{(P)}

\newcommand{\eqSys}{\setOfBPEq}
\newcommand{\fb}{\FB}

\newcommand{\zbApproximate}{\partitionBethe}

\newcommand{\ScaleMatrix}{\boldsymbol{W}}
\newcommand{\cavityField}[3][]{h^{#1}_{#2#3}}

\newcommand{\patch}[1]{\graph_{#1}}
\newcommand{\patchRV}[1]{\setOfNodes_{{#1}}}
\newcommand{\patchEdges}[1]{\setOfEdges_{{#1}}}
\newcommand{\setOfStableSol}{\mathbf{S}}

\newcommand{\setOfMinimaSol}{\mathbf{M}}
\newcommand{\setOfAllSol}{\mathbf{T}}
\newcommand{\fixedPointTuple}[1]{\big(\partitionBethe^{#1},\pseudomarginals^{#1}  \big)}
\newcommand{\Region}{(II)}
\newcommand{\EPartition}[1]{E_{\partitionFunction}(#1)}
\newcommand{\EMarginal}[1]{E_{P}(#1)}

\newcommand{\effectiveField}[1]{\tilde{\field{#1}}}
\newcommand{\mismatch}[2]{Q_i({#1},{#2})}

\newcommand{\startp}{\pseudomarginalsMinLocal(\scaling=0)}
\newcommand{\terminalp}{\pseudomarginalsMinLocal(\scaling=1)}

\newcommand{\pbs}[1]{\let\temp=\\#1\let\\=\temp}%

\definecolor{bkred}{rgb}{0.9,0,0}
\definecolor{bkgreen}{rgb}{0,0.67,0}
\definecolor{bkblue}{rgb}{0,0,0.75}

\newcommand{\DocumentTitle}{Understanding the Behavior \\ of Belief Propagation}
\newcommand{\DocumentSubtitle}{Convergence Properties, Approximation Quality, and Solution Space Analysis}
\newcommand{\ShortTitle}{Understanding the Behavior of Belief Propagation} 
\newcommand{\DocumentAuthor}{Dipl.-Ing. Christian Knoll, BSc.}
\newcommand{\DocumentDate}{Graz, November 8, 2019}
\newcommand{\ThesisType}{PhD Thesis}
\newcommand{\Organizations}{Signal Processing and Speech Communications Laboratory \\ Graz University of Technology, Austria} 
\newcommand{\Supervisors}{Assoc.Prof. Dipl.-Ing. Dr.mont. Franz Pernkopf} 
\newcommand{\Assessors}{Assoc.Prof. Dipl.-Ing. Dr.mont. Franz Pernkopf \\ Dr. Adrian Weller}
\newcommand{\SpecialNote}{}

\newcommand{\ConfidNote}{November 8, 2019}

\newboolean{OptDraftMode}
\newboolean{DisplayContentBoxes}
\newcommand{\ContentBox}[2]{\ifthenelse{\boolean{DisplayContentBoxes}}{\FloatBarrier\nxtpar\colorbox{yellow}{\parbox{\textwidth}{\footnotesize#1\par\hrulefill\par Number of pages: #2}}\nxtpar}{}}

\ifthenelse{\boolean{OptDraftMode}}{}{}
\theoremstyle{plain}
\newtheorem{lm}{Lemma}
\newtheorem{thm}{Theorem}
\newtheorem{prop}{Proposition}
\newtheorem{cor}{Corollary}[thm]

\theoremstyle{definition}
\newtheorem{defn}{Definition}[section]
\newtheorem{ex}{Example}

\renewcommand{\emptyset}{\text{\O}}
\newenvironment{example}%
{\begin{leftbar}\begin{ex}}%
		{\end{ex}\end{leftbar}}
{\begin{leftbar}}%
	{\end{leftbar}}

\begin{document}
\begin{titlepage}
	\vspace*{-1cm}
	\hfill
	\begin{minipage}{4cm}
		\includegraphics[width=40mm]{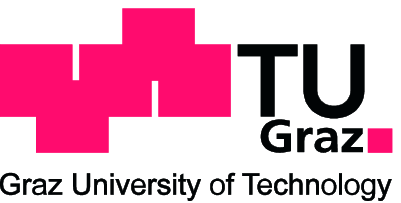}
	\end{minipage}\\
	\vspace*{2.2cm}
	\begin{center}
		\hrulefill
		\textsc{\large~\ThesisType~}
		\hspace{-2mm} \hrulefill \\[5.5mm]
		\parbox{\textwidth}{\centering\scshape\huge\linespread{0.9}\selectfont\DocumentTitle}
		\\[4mm] \hrulefill \\[2mm]
		{\centering\parbox{0.9\textwidth}{\centering\Large\DocumentSubtitle}} \\[2.0cm] \par
		\ifthenelse{\equal{\DocumentLanguage}{en}}{
			conducted at the \\ \Organizations
			\\[1cm]\par by \\ \DocumentAuthor
			\\[2.0cm]\par Supervisors: \\ \Supervisors
			\\[1cm]\par Assessors/Examiners: \\ \Assessors
			\\[1.75cm]\par
		}{}
		\ifthenelse{\equal{\DocumentLanguage}{de}}{
			durchgeführt am \\ \Organizations
			\\[1cm]\par von \\ \DocumentAuthor
			\\[2.5cm]\par Betreuer: \\ \Supervisors
			\\[1cm]\par Begutachter: \\ \Assessors
			\\[1.75cm]\par
		}{}
		\vfill{\flushright \DocumentDate}\\
		{\scriptsize\vspace*{3mm}\SpecialNote\vspace*{-8mm}}
	\end{center}
\end{titlepage}\emptydoublepage

\ifthenelse{\equal{\DocumentType}{thesis}}{\pagestyle{empty}\pagenumbering{roman}}{}



\ifthenelse{\equal{\DocumentType}{thesis}}{\emptydoublepage \thispagestyle{empty} \vspace*{1cm}
	
	\ifthenelse{\equal{\DocumentLanguage}{en}}{
		\begin{center}\Large\bfseries Statutory Declaration\end{center}\vspace*{1cm}
		\noindent I declare that I have authored this thesis independently, that I have not used other than the declared sources$/$resources, and that I have explicitly marked all material which has been quoted either literally or by content from the used sources.
		The text document uploaded to \mbox{TUGRAZonline} is identical to the present doctoral dissertation.
		\par\vspace*{4cm}
		\centerline{
			\begin{tabular}{m{1.5cm}cm{1.5cm}m{3cm}m{1.5cm}cm{1.5cm}}
				\cline{1-3} \cline{5-7}
				& date & & & & (signature) &\\
		\end{tabular}}
	}
	
	\ifthenelse{\equal{\DocumentLanguage}{de}}{
		\begin{center}\Large\bfseries Eidesstattliche Erklärung\end{center}\vspace*{1cm}
		Ich erkläre an Eides statt, dass ich die vorliegende Arbeit selbstständig verfasst, andere als die angegebenen Quellen$/$Hilfsmittel nicht benutzt, und die den benutzten Quellen wörtlich und inhaltlich entnommene Stellen als solche kenntlich gemacht habe.
		\par\vspace*{4cm}
		\centerline{
			\begin{tabular}{m{1.5cm}cm{1.5cm}m{3cm}m{1.5cm}cm{1.5cm}}
				\cline{1-3} \cline{5-7}
				& Graz, am & & & & (Unterschrift) &\\
		\end{tabular}}
	}
	
}{}

\emptydoublepage
\tableofcontents

\ifthenelse{\equal{\DocumentType}{thesis}}{\emptydoublepage\pagestyle{scrheadings}\pagenumbering{arabic}\mainmatter}

\ifthenelse{\equal{\DocumentType}{thesis}}{\setcounter{page}{9}}{}

\emptydoublepage

\renewcommand{\pwd}{chapter0}
\newchapterNoNr{Abstract}{abstract}
Probabilistic graphical models are a powerful concept for modeling high-dimensional distributions.
Besides modeling distributions, probabilistic graphical models also provide an elegant framework for performing statistical inference;
because of the high-dimensional nature, however, one must often use approximate methods for this purpose.

Belief propagation performs approximate inference, is efficient, and looks back on a long success-story.
Yet, in most cases, belief propagation lacks any performance and convergence guarantees.
Many realistic problems are presented by graphical models with loops, however, in which case belief propagation is neither guaranteed to provide accurate estimates nor that it converges at all.

This thesis investigates how the model parameters influence the performance of belief propagation.
We are particularly interested in their influence on (i) the number of fixed points, (ii) the convergence properties, and (iii) the approximation quality.
For this purpose, we take a different perspective on belief propagation and realize that the fixed points define a set of polynomial equations -- albeit a large one.
Solving polynomial equations of this size is problematic; 
nonetheless, we present  the numerical polynomial homotopy continuation method that is capable of solving the fixed point equations, the solutions of which are the fixed points of belief propagation.

The solutions to the fixed point equations give us knowledge of the whole solution space and serve as a stepping stone for analyzing belief propagation's properties.
In particular, we observe a large variety of marginal accuracy across all fixed points.
This, to some degree, explains the large discrepancy in the performance of belief propagation.
Another important aspect of belief propagation's fixed points is their stability, that is if belief propagation can -- at least in principle -- converge to a given fixed point.
Existing stability analyses were limited to models without local potentials for the lack of knowing the solution space.
The capability to solve the fixed point equations thus allows us to extend the stability analysis to a wide range of models.
In doing so, we obtain novel insights into how the model parameters and the model size affect the stability.
In particular, we find that strong pairwise potentials degrade the performance, whereas strong local potentials enhance the performance.

Moreover, our theoretical findings inspire a simple, yet powerful, modification of belief propagation.
We present self-guided belief propagation that starts from a simple model (for which belief propagation obtains the exact solutions) and iteratively adapts it to the desired model.
As the model is modified, self-guided belief propagation keeps track of the solution;
this way, it improves upon standard belief propagation and it obtains the best possible solution for attractive models with unidirectional local potentials.
For more general models, we empirically show that self-guided belief propagation maintains its favorable properties, converges more often, and is superior in terms of marginal accuracy.

Finally, we question whether the global minimum of the Bethe free energy provides the most accurate marginals.
This is a common conjecture that inspired a range of methods that aim to minimize the Bethe free energy.
In the past, the studied models were either too simplistic or too complex as to the true nature of this relationship.
Therefore, we must first introduce a novel class of models -- termed patch potential models -- which are simple enough so that we can compute all fixed points.
Yet, patch potential models are complex enough to possess a rich and non-trivial solution space.
A study of this solution space proves this conjecture wrong and, additionally, explains the nature of the difference between accurate marginals and good approximations of the free energy.

  \emptydoublepage
\newchapterNoNr{Acknowledgments}{acknowledgments}
There are a lot of people who influenced this thesis and made working on it much more enjoyable.
Here I want to express my gratitude for all of you.

First of all, I must thank my supervisor Franz Pernkopf for being such a great advisor.
You somehow always managed to find time for stimulating discussions, both about research and about life in general.
Most importantly, however, you not only granted me the freedom to pursue my own ideas but you also believed in me when I struggled to do so.

Thank you, Bernhard and Gernot for supervising my Master's thesis.
It is because of you I had such a good experience back then which made me seriously consider pursuing a PhD at this lab.

I want to thank all members of the SPSC I had the pleasure to meet during my time here -- you are responsible for creating an awesome environment to work in.
There are several people I want to thank in particular:
Robert and Sebastian for the warm welcome I received when starting here, for the thought-provoking discussions, and for showing me how to survive in academia.
Josef for being the best room-mate I can think of, for sharing the white-board with me, and for introducing me to space-exploration (as well as for discussing all kinds of research questions).
Wolfgang for always having an open door and an open mind (and a good supply of sweets).
Johannes for our endless discussions about research, teaching, and -- most importantly -- everything else.
Thomas for making sure that our kitchen is always well-stocked with the essentials.
Stefan, Erik, Jamilla, Michi, Martin T., Elmar, Johanna and Alex for regularly testing the capacity of our coffee-machine -- you certainly enriched the time I spent standing in the queue.
I am pretty sure I forgot to mention some of you -- please forgive me and do not take this personally!

Looking back at how I got here, I must thank those who probably shaped me the most during my whole life: A big thank you to my parents and my sister.
Thanks for your support, for always being honest, and for listening to me whenever I need someone to talk to.

Finally, I must thank my own small family.
Thank you, Jakob and Hanna, for constantly reminding me that there are things in life far more important than my research.
And, of course, I own my deepest gratitude to you, Kati!
Thank you for your endless support; 
thank you for sometimes understanding me better than I do myself; 
and thank you for learning me how to appreciate and enjoy even the little success-stories in life.

\emptydoublepage

\newchapterNoNr{Notational Conventions}{notation}
\begin{table}
		\begin{tabular}{p{0.3\linewidth} p{0.7\linewidth}}
			\textbf{General Notation} & \\
			\toprule
			$\REALPos$ \hspace*{5cm }& positive real numbers, excluding zero \\
			$\INTEGERPos$ & positive integer numbers, excluding zero\\
			$\gradient$ \hspace*{5cm }& gradient \\
			$\Hessian$ \hspace*{5cm }&  Hessian  \\
			$D(\pmfApprox{\setOfNodes{}}||\pmf{\setOfNodes{}})$ & Kullback Leibler divergence between $\pmfApprox{\setOfNodes{}}$ and $\pmf{\setOfNodes{}}$\\
			$\Re{\cdot}$, $\Im{\cdot}$ & real and imaginary part \\
			$i$ & imaginary unit\\
			$|\cdot|$ & magnitude or absolute value\\
			$\oplus$ & exclusive disjunction\\
			$\sgn (\cdot)$ & sign function\\
			$n\pmod{m}$ & n modulo m\\
			$\mathbf{1}$ & indicator function\\
			$\log(\cdot)$ & natural logarithm\\
			$||\cdot||_p$ & $l_p$-norm\\
			$\mathcal{N}(\mu,\sigma^2)$ & Gaussian distribution with mean $\mu$ and variance $\sigma^2$\\
			$\mathcal{U}(a,b)$ & uniform distribution on $(a,b)$
		\end{tabular}
	\end{table}
	\begin{table}
		\begin{tabular}{p{0.3\linewidth} p{0.7\linewidth}}
			\textbf{Probability} & \\
			\toprule
			$\RV{}, \RV[y]{},\ldots $ \hspace*{5cm }& random variables \\
			$\sampleSpace{X}, \sampleSpace{Y},\ldots$ \hspace*{5cm }& range of a discrete random variable \\
			$\RVval[x]{}, \RVval[y]{},\ldots $ \hspace*{5cm }& value of random variable \\
			$\setOfNodes{}$ \hspace*{5cm }& set of random variables \\
			$\RVvalSet{}$ \hspace*{5cm }& values for a set of random variables (configuration) \\
			$\sampleSpace{X}^N$ \hspace*{5cm }& product space, range of $\setOfNodes$\\
			$\pmf{\RV{}}(\RVval{})$ \hspace*{5cm }& probability distribution of $\RV{}$ \\
			$\pmfApprox{X}(x)$ & approximation of the probability distribution of $\RV{}$\\
			$\pseudomarginals$ \hspace*{5cm }& pseudomarginals (set of all singleton and pairwise marginals) \\
			$Z_i$, $Z_{ij}$ \hspace*{5cm }& normalization terms for approximated marginals \\
			$\pmf{\RV{}| \RV[y]{}}(\RVval{}|\RVval[y]{})$ \hspace*{5cm }& conditional probability distribution of $\RV{}$ given    $\RV[y]{}$\\
			
			$\pmf{\setOfNodes{}}(\RVvalSet{})$ \hspace*{5cm }& Joint probability distribution of $\setOfNodes{}$ \\
			$\E(x)$ & expected value of $\RV{}$\\
			$\mean{i} $ \hspace*{5cm }& mean of $\RV{i}$\\
			$\meanAvg$ & expected mean of $\setOfNodes$\\
			$\correlation{i}{j}$ & correlation between $\RV{i}$ and $\RV{j}$\\

		\end{tabular}
	\end{table}

	\begin{table}
		\begin{tabular}{p{0.3\linewidth} p{0.7\linewidth}}
			\multicolumn{2}{l}{\textbf{Graphs}} \\
			\toprule
			$\graph$ \hspace*{5cm }& undirected graph \\
			$\RV{1}, \RV{2}, \ldots$ \hspace*{5cm }& nodes \\
			$\setOfNodes{}$ \hspace*{5cm }& set of nodes \\
			$N$ \hspace*{5cm }& number of nodes \\
			$\edge{i}{j}$ \hspace*{5cm }& edge between $\RV{i}$ and $\RV{j}$ \\
			$m$ & number of the edge $\edge{i}{j}$ under some ordering\\
			$\setOfEdges{}$ \hspace*{5cm }& set of edges \\
			
			$\neighbors{i}$ \hspace*{5cm }& neighbors of $\RV{i}$ \\
			$\nodeDegree{i}$ \hspace*{5cm }& degree of $\RV{i}$ \\
			$\averageDegree[\graph]$ \hspace*{5cm }& average degree of $\graph$ \\
			$\graph'$ \hspace*{5cm }& subgraph of $\graph$ induced by $\setOfNodes'$\\
			$\vm{A}$ \hspace*{5cm }& adjacency matrix \\

		\end{tabular}
	\end{table}
	
	\begin{table}
		\begin{tabular}{p{0.3\linewidth} p{0.7\linewidth}}
			\multicolumn{2}{l}{\textbf{Graphical Models}} \\
			\toprule
			$\ugm$ \hspace*{5cm }& undirected graphical model \\
			$\clique{i}$ \hspace*{5cm }& clique \\
			$\clique{(i,j,\ldots)}$ \hspace*{5cm }& clique consisting of $\RV{i},\RV{j},\ldots$ \\
			$\setOfCliques$ \hspace*{5cm }& set of cliques \\
			$\potential{\clique{{i}}}(\RVvalSet{\clique{i}})$ \hspace*{5cm }& clique-potential \\
			$\setOfPotentials$ \hspace*{5cm }& set of potentials \\
			$\local{x}{i}$, $\localShort{i}$ \hspace*{5cm }& local potential \\
			$\pairwise{x}{i}{j}$, $\pairwiseShort{i}{j}$ \hspace*{5cm }& pairwise potential \\
			$\partitionFunction$ \hspace*{5cm }& partition function \\ 
			$E(\RVvalSet{})$ \hspace*{5cm }& energy of the configuration $\RVvalSet{}$ \\
			$\coupling{i}{j}$ \hspace*{5cm }& coupling \\
			$\field{i}$ \hspace*{5cm }& local field \\
			$\sigma_{i}$ \hspace*{5cm }& spin \\
			$\bm{\sigma}$ \hspace*{5cm}& configuration of spins\\
			$\beta$ \hspace*{5cm }& inverse temperature \\
			$H$ \hspace*{5cm }& external magnetic field \\
			$J_{\path}$\hspace*{5cm }& product of couplings along path $\path$\\
		\end{tabular}
	\end{table}
	\begin{table}
		\begin{tabular}{p{0.3\linewidth} p{0.7\linewidth}}
			\multicolumn{2}{l}{\textbf{Dynamical Systems and Equation-Systems}} \\
			\toprule
			$\stateVec{x}$ \hspace*{5cm }& state vector \\
			$\stateVec{x}(n), \stateVec{x}^n$ & state vector $\stateVec{x}$ at time index $n$\\
			$\fp{\stateVec{x}}$ \hspace*{5cm }& fixed point of $\stateVec{x}$\\
			$U(\stateVec{x})$ & $\epsilon$-neighborhood of $\stateVec{x}$\\
			$\map(\cdot)$ \hspace*{5cm }& discrete-time map \\
			$\vm{B}$ & system matrix \\
			$\JacobianGeneral{\stateVec{x}}$ & Jacobian matrix\\
			$\ev{i}$ & eigenvalue\\
			$\ev{max}$& eigenvalue with the largest magnitude\\
			$\Spectrum{\vm{A}}$ & set of all eigenvalues for the matrix $\vm{A}$\\
			$\Radius{\vm{A}}$ & spectral radius of the matrix $\vm{A}$\\

			$\eq[i](\stateVec{x})$ \hspace*{5cm }& function or polynomial equation \\
			$\coeff{k}$ \hspace*{5cm }& polynomial coefficient \\
			$\POLY[x_1,\ldots,x_n]$ \hspace*{5cm }& polynomial ring \\
			
			$\setOfEq(\stateVec{x}), \setOfEq(x_1,\ldots, x_n)$ \hspace*{5cm }& system of polynomial equations\\
			
			$\variety(\setOfEq)$ \hspace*{5cm }& variety (set of solutions) of $\setOfEq(\stateVec{x})$ \\
			$\variety_{\REAL}(\setOfEq), \variety_{\REALPos}(\setOfEq)$ \hspace*{5cm }& variety of $\setOfEq(\stateVec{x})$ for the real and the positive real numbers\\
			
			$\polytope{i}$ & convex hull of the exponent vectors of $\eq[i]$\\
			$V(\polytope{i})$ & volume of the polytope $\polytope{i}$\\
			$M(\polytope{i},\polytope{j})$ & mixed volume of $\polytope{i}$ and $\polytope{j}$\\
			$\omega_i$ & lifting for the polynomial $\eq[i]$ \\ 
			$\omega_i(a)$ & lifting function for the exponent-vector $a$\\
			$\ideal{f}$ \hspace*{5cm }& ideal of $\setOfEq$ \\
			$d_t$ & total degree\\
			BKK & Bernshtein-Kushnirenko-Khovanskii bound\\
			$\homotopy(\stateVec{x},t)$ & homotopy\\
			$\startSys(\stateVec{x})$ & start system\\
			$t$ & goes from $0$ to $1$ to deform the homotopy\\
			$\gamma$ & random complex number\\
			
		\end{tabular}
	\end{table}
	\begin{table}
		\begin{tabular}{p{0.3\linewidth} p{0.7\linewidth}}
			\multicolumn{2}{l}{\textbf{Belief Propagation, Variational Approximations}} \\
			\toprule
			$\msg[n]{i}{j}{}$ \hspace*{5cm }&  message from $\RV{i}$ to $\RV{j}$\\
			$\setOfMessages[n]$ & set of messages\\
			$n$ \hspace*{5cm }&  iteration index\\
			$\msgNorm[n]{i}{j}$ \hspace*{5cm }&  message normalization for $\msg[n]{i}{j}{}$\\
			$\fpMsg{i}{j}{}$ & fixed point message\\
			$\setOfMessages[\stationaryPoint]$ & fixed point of belief propagation\\
			
			$\energy{\pmfApprox{\setOfNodes{}}}$ \hspace*{5cm }&  average energy\\
			$\entropy{\pmfApprox{\setOfNodes{}}}$ \hspace*{5cm }&  entropy\\
			$\mathcal{F_H}, \FGibbs{}, \FB$ \hspace*{5cm }& Helmholtz-, Gibbs-, and Bethe- free energy  \\
			$\EB(\pseudomarginals)$ \hspace*{5cm }& average (Bethe) energy \\
			$\SB(\pseudomarginals)$ \hspace*{5cm }& Bethe entropy \\
			$\MPolytope(\graph), \MPolytope$\hspace*{5cm }& marginal polytope \\
			$\LPolytope(\graph), \LPolytope$ \hspace*{5cm }& local polytope \\
			$\partitionBethe$ \hspace*{5cm }& Bethe partition function \\
			$\FBGlobalMin, \FBLocalMin{m}, \FBStationary$ \hspace*{5cm }& global minimum, local minimum, and stationary point of $\FB$ \\
			$\pseudomarginals^{*}, \pseudomarginals^{m}, \pseudomarginals^{\circ}$ & pseudomarginals at the global minimum, a local minimum, and a stationary point of $\FB$\\
			$\setOfStableSol$ & set of stable belief propagation fixed points\\
			$\setOfMinimaSol$ & set of fixed points corresponding to local minima of $\FB$\\
			$\setOfAllSol$    & set of all fixed points\\
			$\EPartition{m} $ \hspace*{5cm }& error of the partition function approximation \\
			$\EMarginal{m} $ \hspace*{5cm }& error of the approximated marginals \\
			$\BP\big(\cdot\big) $ \hspace*{5cm }& mapping induced by belief propagation \\
			$\decisionP$ \hspace*{5cm }& evaluation function for the pseudomarginals \\
			$ \decisionZ $ \hspace*{5cm }& evaluation function for the Bethe partition function \\
			$r_{ij}^{n}$ \hspace*{5cm }& message-residual\\
			$\vm{r}^n$ & set of residuals\\
			$\epsilon$ & damping factor of belief propagation with damping\\
			$\setOfBPEq$ & fixed point equations of belief propagation\\
			$\msgReparam[n]{i}{j}$ & reparameterized messages \\
			$\vm{\nu}$ & set of reparameterized messages \\
			$h_{ij}$ & cavity field\\
			$\pseudomarginals^{T}$ & weighted combination of all fixed points\\
			$\pseudomarginals^{M}$ & weighted combination of all fixed points  belonging to local minima\\
			$\pseudomarginals^{MAX}$ & fixed point maximizing the partition function\\
			$J_{A}, J_{C}, J_{C} (\graph,\theta)$ & critical values of coupling strength at the onset of phase-transitions\\
			$\kappa_{\theta}$ & field-dependent scaling term for Jacobian matrix\\
		\end{tabular}
	\end{table}

	\begin{table}
		\begin{tabular}{p{0.3\linewidth} p{0.7\linewidth}}
			\multicolumn{2}{l}{\textbf{Error Correcting Codes}} \\
			\toprule
			$\epsilon$ \hspace*{5cm }& error-probability \\
			$Y_i$ & variable node\\
			$f_a$ & factor node\\
			$r_{Ai}^n(y_i)$& message from factor to variable\\
			$q_{iA}^n(y_i)$& message from variable to factor\\
			$\alpha_{iA}^n$& normalization term\\
		\end{tabular}
	\end{table}

	\begin{table}
		\begin{tabular}{p{0.3\linewidth} p{0.7\linewidth}}
			\multicolumn{2}{l}{\textbf{Self-Guided Belief Propagation}} \\
			\toprule
			$\scaling_k$ \hspace*{5cm }& scaling term \\
			$K$ & number of models considered by self-guided belief propagation (length of $\{\scaling_k\}$)\\
			$\ugm[K]$ & undirected graphical model for $\scaling_k$\\
			$\setOfPotentials_k$ & set of potentials for $\scaling_k$\\
			$\pairwiseSBP{i}{j}{k},\localSBP{i}{k}$ & pairwise and local potentials for $\scaling_k$\\
			$\setOfMessages[\circ]_\iteration{k}$ & fixed point of BP for the model $\ugm[k]$\\
			$c(\scaling)$ & solution path\\
			$N_{BP}$ & maximum number of belief propagation iterations\\
		\end{tabular}
	\end{table}
	
	\begin{table}
		\begin{tabular}{p{0.3\linewidth} p{0.7\linewidth}}
			\multicolumn{2}{l}{\textbf{Patch Potential Models}} \\
			\toprule
			$\patch{i}$ \hspace*{5cm }& patch of the graph $\graph$\\
			$P^m_{\mu}$ & fraction of runs converging to the $m$\textsuperscript{th} fixed point\\
			$\pseudomarginals^{r},\pseudomarginals^{q}$ & fixed points with all marginals biased to one state\\
			$\pseudomarginals^{p}$ & state-preserving fixed point\\
			$\pseudomarginals^{u},\pseudomarginals^{v},\ldots$ & all other fixed points\\
			$\effectiveField{i}$ & effective field for $\RV{i}$\\
			$Q_i(k,l)$ & mismatch between $\pseudomarginals^{k}$ and $\pseudomarginals^{l}$\\
			$\setOfEdges_P$ & set of all boundary edges\\
			$\setOfEdges_C$ & set of edges between variables that favor different states\\
			$N_f$ & number of flipped variables\\
			$N_c$ & number of state-preserving variables\\
			$\Delta \SB$ & difference of the Bethe entropy between two fixed points\\
		\end{tabular}
	\end{table}

  \emptydoublepage
\newchapter{Introduction}{intro}
\openingquote{Uncertainty is an uncomfortable position. \\But certainty is an absurd one.}{Voltaire}
\renewcommand{\pwd}{}

\newsection{Motivation}{intro:motivation}
	Every realistic domain, in accordance with the opening quote, contains some degree of uncertainty:
	not only is every human imperfect -- with each of our decisions being subject to uncertainty --
	but also has every artificial system only access to partial information.
	When reasoning under such uncertainties, one must take all available information into account and infer the quantity of interest.
	This task of probabilistic reasoning is one of the central problems of statistical inference\footnote{
		One may argue that statistical inference provides the common ground for a myriad of scientific fields ranging from mathematics over empirical sciences to philosophy~\cite{efron2016computer}. We comply with this perspective and provide a broad context to the rather specific topic of this thesis.}
	and, additionally, lies at the center of every decision-making process.

	Most inference-problems of practical relevance belong to the realm of multivariate statistics and involve many interacting variables.
	As the number of variables increases, the corresponding distributions become increasingly hard to grasp and thus require compact representations; particularly if one wishes to retain some interpretability.
	Fortunately, such a representation exists in the form of probabilistic graphical models;
	these models (as suggested by their name) rely on the proven capabilities of graphs\footnote{ 
		Graphs naturally emerge in a wide range of problems that include applications in social sciences, biology, physics, and communications. 
		Graph theory constitutes a concept with many persuasive properties that warrant their wide-spread usage.
		First, graphs lend themselves for visual representations that reveal the underlying problem-structure, often hidden initially, to the human  observer.
		Second, manipulations and computations on graphs are well established and provide an extensive tool-box to tackle a problem once formulated as graph.}
	and represent complex interactions in an intuitive and expressive way.

	We will study two of the most fundamental problems of inference.
	These are:
	computing the \emph{marginal distribution} and evaluating the \emph{partition function}.
	Both problems suffer from the increased complexity when working with many variables.
	Assistance comes in the form of graphical models once again;
	not only do they represent the problems efficiently, but they also warrant efficient inference algorithms that exploit the graph structure and thus facilitate the task of inference.
	Note, however, that graphical models are not a panacea and while inference becomes tractable for some models (e.g., trees) it remains NP-hard for graphs with loops.
	
	Besides being of theoretical interest, such loopy graphs arise in a  multitude of practical applications; ranging from  statistical signal-, speech-, and image-processing, to statistical physics, medical diagnosis systems, and error-correcting codes.
	For all these important problems, exact inference methods, however, are destined to fail, which  substantiates the need for efficient approximation methods.

	One particularly prominent approximation method focuses on local interactions and infers the global model-behavior thereof.
	Consider, for example, a large group of friends that want to celebrate a party and need to find a suitable date.
	If all friends come together and discuss their preferences in one large negotiation, they would find an optimal date;
	but this seems overly complicated and impracticable. 
	Instead, one could hope to find an acceptable date that works for most by negotiating in a distributed manner:
	assume that everybody discusses his preferences only with his closest friends (and that these sub-groups sufficiently overlap), then a reasonable agreement on the date will be found that will suit at least the majority of the group.
	Note how this focus on local interactions comes relatively intuitive when dealing with large, complex systems.
	The graphical model's structure is exploited precisely in such a way by belief propagation that performs  local interactions -- in the form of exchanging messages between neighboring nodes -- to approximate the marginal distribution and the partition function.\\

	Let us briefly summarize the main strength of belief propagation (BP); 
	it is the ability to efficiently perform approximate inference on models where (because of loops) exact inference
	becomes infeasible.
	BP often works remarkably well in this context, although theoretical results fail to explain BP's empirical success.
	Yet, portraying BP as an outright success-story would not correspond to the truth either -- and indeed,
	while BP often performs approximate inference in a very efficient manner it completely fails to do so other times.
	For this lack of reliability, BP is sometimes confronted with skepticism.
	Instead of only hoping for a reasonable performance of BP, we would, ideally, like to have some performance guidelines established.
	Stating whether a given model is well-suited for the application of BP, such guidelines must focus on the following two aspects.

	First, one needs to understand the underlying reasons for the failure of BP.
	Although we still lack a rigorous understanding of why BP fails, it is empirically well-established that BP fails because:
	(i) multiple solutions exist with varying accuracy; 
	(ii) one or all fixed points are unstable so that BP does not converge.
	
	Second, one needs to derive performance guarantees that take the model specifications into account.
	If we want to better understand BP we will not only need to understand how the model specifications influence (i) the convergence properties and (ii) the approximation quality but also how both properties relate to each other.
	Note that both properties are directly related to each other for simple models (e.g., graphs with a single loop or small grid graphs), i.e., the better the approximation quality the faster BP converges~\cite{weiss2000correctness,ihler07}.
	Such a relation, however, does not generalize to more complex models~\cite{weller2013approximating}.
	
	To summarize, it will be important to understand how specific models affect the performance of BP and if BP can be expected to perform well.
	This will also increase the reliability of BP.\\
	

	To a large degree, we owe our current understanding of BP to concepts from statistical physics.
	As it turns out, there is a fundamental connection between many concepts in computer science and in physics (cf.~\cite{mezard2009, welling2003approximate, tatikonda2002}).
	Most notably, the fixed points of BP are in a one-to-one correspondence with the stationary points of the Bethe free energy.

	But, although the Bethe free energy can provide many insights, it remains an intricate function that is hard to analyze.
	Also, most theoretical results on the Bethe free energy only hold for restricted model classes,
	where typically all variables sit on a regular grid and the model is specified by relatively few parameters.
	How those insights carry over to more general models is an open question.

	Moreover,  the Bethe free energy fails to reveal whether a given fixed point is stable and under which conditions BP converges to it.
	The Bethe free energy also fails to explain why and how certain modifications of BP (e.g., scheduling or damping) often help to achieve convergence.\\

	The quest for a better understanding of BP is thus ongoing, with the hope that new insights will recognize certain model classes for which it is safe to utilize BP, i.e., for which BP converges fast while maintaining the desired accuracy.
	
	The overarching aim of this thesis is to extend the current understanding of BP, the reason for this being twofold: 
	first, to theoretically understand for which problems and applications BP can be expected to perform well;
	and second, to utilize those theoretical insights and modify BP to enhance its capabilities.
	
	We specifically address the question of how knowledge of the solution space can advance our current understanding of BP.
	As the complete set of solutions is generally not available, we must first develop a way to obtain all fixed points for a given model.
	We then start with the analysis of relatively small and simple models and, by successively building upon our obtained insights, extend our analysis to increasingly more complex models.
	Finally, we take a comprehensive view at the solution space and study the relation between the convergence properties, approximation quality, and the number of fixed points.
	This approach will provide several insights into the behavior of BP, extend the current theoretical understanding and  open the door for practical considerations that enhance the performance of BP.
	In particular, our findings suggest a modification of BP that enforces convergence toward accurate fixed points, thus improving the approximation quality.
	
\newsection{Five Relevant PhD Theses}{intro:5thesis}

As discussed, BP touches a diverse set of scientific fields. 
It is the common ground between all those fields that provides a solid foundation for the analysis of BP.
The work presented in this thesis therefore builds upon a vast body of literature.
The following five PhD theses cover a wide range of the recent developments and are highly relevant for the current thesis.
Therefore, this section contains a brief overview of how they shaped the current thesis in particular.\\

The connection between exponential representations of distributions, information geometry, and approximate inference methods on probabilistic graphical models is revealed in the thesis of \textbf{Martin Wainwright}, submitted at the Massachusetts Institute of Technology in 2002~\cite{wainwright_thesis}.
His thesis introduces the concept of reparameterization.
This concept casts BP as one specific instance in a more general class of message passing algorithms (termed tree-based reparameterization (TRP)).
Moreover, TRP suggests solving a particular sequence of simpler sub-problems over spanning trees in the graph~\cite{wainwright2003tree-scheduling} in order to enhance the convergence properties.
Additionally, important insights with respect to the marginal accuracy are obtained in the form of an exact expression for the marginal error, that -- although being infeasible to evaluate in general -- suggests computable bounds on the marginal error.
The proposed bounds rely on the approximation of the log-partition function which advocates a close connection between  both quantities;
this connection nicely connects to the present work where we inspect this relationship in great detail in Chapter~\ref{chp:accuracyBP}.
Another powerful concept is the consideration of BP as an optimization problem (over the local polytope).
In the conclusion, the author proposes tracing the evolution of the pseudomarginals while relaxing the marginal- to the local polytope with the prospect of practical and theoretical consequences; 
a similar evolution of the pseudomarginals lies at the core of Chapter~\ref{chp:selfguided}.

The overarching aim of the thesis of \textbf{Joris Marten Mooij}, submitted at the Radboud Universiteit Nijmegen in 2008~\cite{mooij_thesis}, is to understand and improve belief propagation, which closely resembles the aim of the current thesis.
In a nutshell, his thesis investigates the relationship between accuracy, uniqueness, and convergence properties of belief propagation's fixed points.
One particular insightful contribution was the consideration of belief propagation as a dynamical system.
This led to conditions for stability and uniqueness of a fixed point, as well as insights into the relation between those properties.
The analysis, however, was restricted to vanishing local potentials.
We adhere to the spirit of considering belief propagation as a dynamical system and extend the analysis to models with arbitrary parameters in Chapter~\ref{chp:solutionsBP}, providing novel theoretical insights.

The application of belief propagation to inference problems arising in the context of sensor networks is the main motivation in the thesis of \textbf{Alexander Ihler}~\cite{ihler_thesis}, submitted at the Massachusetts Institute of Technology in 2005. 
His thesis focused on studying the fundamental limitations of belief propagation.
Studying the vulnerability of the approximation quality with respect to errors in the messages was one particularly important aspect with the prospect of validating whether approximating the messages is a viable option.
Notably, this leads to a couple of interesting theoretical results on the accuracy of belief propagation in general, and the introduction of error-bounds on the marginal accuracy specifically.
While these bounds are only valid for cycle-free graphs,
they predict the performance of belief propagation reasonably well in the presence of loops as well.
This is one of the few results that analyze the error in the marginals, while much of the literature focuses on the approximation error of the partition function.
We agree with the author on the importance of assessing the marginal accuracy and will particularly focus on how the model parameters influence the marginal accuracy (cf. Chapter~\ref{chp:solutionsBP} and~\ref{chp:accuracyBP}).

The thesis~\cite{weller_thesis} of \textbf{Adrian Weller}, submitted at the Columbia University in 2014, focuses on the variational interpretation of belief propagation and the Bethe approximation in particular.
His work provides interesting insights into the differences between log-partition function estimates, singleton marginals, and pairwise marginals for both the (non-convex) Bethe approximation and convex variational approaches. 
In particular, it becomes evident that different approximation methods do not affect the marginals and the log-partition function in the same way;
the accuracy of the log-partition function may for example remain the same whereas the accuracy of the marginals increases.
This observations raises the question of how both quantities are related, an important question that will be the main focus of Chapter~\ref{chp:accuracyBP}.
Besides the theoretical relevance of the obtained insights, his thesis further proposes an approximation of the Bethe function, which extends preceding work~\cite{shin2012complexity} by approximating the global minimum.
This approximation method converges in polynomial runtime for attractive models and serves as an important comparison for our proposed method in Chapter~\ref{chp:selfguided}.

A direct relationship between belief propagation and graph geometry is established in the thesis of \textbf{Yusuke Watanabe}~\cite{watanabe_thesis}, submitted at The Graduate University for Advanced Studies, SOKENDAI in 2010; 
this relationship provides novel insights into the behavior of belief propagation. 
In particular, the \emph{graph-zeta function} relates the convergence properties of belief propagation to the shape of the Bethe free energy.\footnote{
	In addition to demonstrating the wide range of fields that play an important role for our current understanding of belief propagation, the graph-zeta function also serves as a foundation for recent developments in spectral clustering~\cite{saade2014spectral}.}
Moreover, the results impose some general properties on the solution space. 
These properties are of relevance for belief propagation as well and, for example, demonstrate that the overall number of fixed points is always odd.
Computing the number of fixed points is also of central interest throughout the current thesis, and Chapter~\ref{chp:solutionsBP} and~\ref{chp:accuracyBP} in particular.

\newsection{Contribution and Outline}{intro:outline}
A large part of the contributions to this thesis has previously been published; a list of the corresponding publications is presented below. 
Several parts, however, have been significantly reworked and restructured in order to align nicely with the structure of the thesis.

\begin{itemize}
	\item \cite{knoll_scheduling}
	Christian Knoll, Michael Rath, Sebastian Tschiatschek, and Franz Pernkopf.
	\newblock Message scheduling methods for belief propagation.
	\newblock In {\em Proceedings of ECML PKDD}, pages 295--310. Springer, 2015.
	
	\item \cite{knoll_ws_fixedpoints} 
	Christian Knoll, Franz Pernkopf, Dhagash Mehta, and Tianran Chen.
	\newblock Fixed point solutions of belief propagation.
	\newblock In {\em NIPS-Workshop: Advances in Approximate Bayesian Inference},
	2016.
	\item \cite{knoll_stability}
	Christian Knoll and Franz Pernkopf.
	\newblock On loopy belief propagation -- local stability analysis for
	non-vanishing fields.
	\newblock In {\em Proceedings of {UAI}}, 2017.
	\item \cite{knoll_fixedpoints}
	Christian Knoll, Dhagash Mehta, Tianran Chen, and Franz Pernkopf.
	\newblock Fixed points of belief propagation -- an analysis via polynomial
	homotopy continuation.
	\newblock {\em IEEE Transactions on Pattern Analysis and Machine Intelligence},
	2018.
	\item \cite{knoll_sbp} 
	Christian Knoll, Florian Kulmer, and Franz Pernkopf.
	\newblock Self-guided belief propagation--a homotopy continuation method.
	\newblock {\em arXiv preprint arXiv:1812.01339}, 2018.
	\item \cite{knoll_accuracy} 
	Christian Knoll and Franz Pernkopf.
	\newblock Belief propagation: Accurate marginals or accurate partition function
	-- where is the difference?
	\newblock In {\em Proceedings of {UAI}}, 2019.
\end{itemize}

It is the very nature of most PhD theses to pinpoint and raise multiple questions before providing some -- hopefully insightful -- answers to them.
As is often the case, this thesis tackles very specific problems in an already highly specialized field.
One aspect, however, that widens the scope and that made working on this thesis particularly interesting is that
belief propagation is applied in various scientific fields ranging from information theory and signal processing to statistical physics and artificial intelligence.
Our study of belief propagation brings insights from all these fields together and benefits from a particularly large toolbox, fueled by such diverse inputs.
%
%
While it is interesting to delve into each particular field and emerge oneself in all the associated details, it is nearly impossible to hide one's past.
Therefore, and for the purpose of a consistent notation, we settle for the language used in the statistics and machine learning community, point at relations to associated scientific fields if appropriate, and draw from them if beneficial.

Writing a coherent thesis, even after completing the research tasks, remains an extensive task that should ideally serve the interested reader.
Therefore, this thesis -- rather than being a conglomerate of the presented results -- offers a thorough introduction and sticks to one coherent story.
In doing so, we identify an inherent structure that lends itself to a segmentation into two major parts.\\

The first part (Chapter 2 - 4) provides all relevant background and serves as the preparation for the subsequent chapters.

The whole thesis resides in the context of probabilistic graphical models, which are introduced in \textbf{Chapter 2}.
We briefly discuss the relevant background from probability- and graph-theory and introduce pairwise graphical models that are the main focus of this work.
Moreover, we define the problem of inference  with a particular focus on efficient exact methods.

Belief propagation (BP) is introduced as a method of approximate inference  in \textbf{Chapter 3}.
After discussing some of the most serious issues of BP, we describe some of the underlying reasons for failure of BP and introduce alternative characterizations of BP.
In particular, we introduce BP as a variational method and connect the properties of BP with the energy landscape of the Bethe free energy.
Additionally, we  cast BP as a dynamical system. 
It turns out that this characterizations provides a general framework, encompassing a wide range of approximate inference methods.
Moreover, the consideration as a dynamical system suggests various ways to enhance the properties of BP.

\textbf{Chapter 4} provides the reader with the most important background of dynamical system theory and algebraic geometry.
While this chapter does not contain novel insights, it provides us with the relevant tools for analyzing BP in detail.
In particular, 
the most prominent methods for solving system of equations are introduced and discussed.\\

%

The second part (Chapter 5-7) builds upon the knowledge developed so far, presents the major results and extends the current understanding of BP's behavior.
%

The consideration of BP as a dynamical system provides the foundation of the results developed in \textbf{Chapter 5}.
We apply tools from dynamical systems theory, analyze the solution space of BP, and gain new insights into the behavior of BP;
the focus on the whole solution space instead of just a single fixed point reveals how the number of fixed points, the approximation quality of the individual fixed points, and the convergence properties are related to each other.
One major issue is that, despite being conceptually straightforward, it is problematic to find the set of \emph{all} fixed points. 
We present how the set of all BP fixed points can be computed by using the numerical polynomial homotopy continuation (NPHC) method.
This allows us to assess and compare the accuracy of the individual BP fixed points and weighted combinations thereof.
Moreover, the knowledge of all fixed points allows us to extend the local stability analysis -- previously restricted to models with vanishing local potentials -- to more general models.
This generalization also explains the role of the local potentials and reveals how strong local potentials enhance the convergence properties.

%

The focus of \textbf{Chapter 6} is to exploit the theoretical finding 
that local and pairwise potentials play an opposing role regarding the performance of BP.
In this chapter we propose one way to account for this observation and to enhance the performance of BP.
%
%
By modifying BP according to a homotopy continuation method we account for this observation and incorporate the pairwise potentials only gradually.
This procedure 
is deterministic, converges to a uniquely defined fixed point, and thus resolves the dependence on a well-chosen initialization.
Experiments on a wide range of models reveal that the proposed method increases the performance without increasing the computational burden;
in particular we exemplify that the results are at least as accurate as BP, if BP converges, 
and that accurate results are often obtained, even if BP fails to converge.
This empirical analysis is further supplemented with a theoretical analysis that proves optimality -- i.e., the proposed method converges to the global minimum of the Bethe free energy that constitutes the fixed point with the most accurate marginals -- for restricted models that have all local potentials favoring the same state.

All models considered so far adhere to the common conjecture that accuracy with respect to the marginals and with respect to the partition function are  interchangeable.
This conjecture suggests that the most accurate marginals of BP are to be found at the global minimum of the Bethe free energy.
In \textbf{Chapter 7} we aim to validate this assumption.
We therefore introduce \emph{patch potential models} that simplify the analysis significantly.
The parameter space of patch potential models is split into multiple regions with fundamentally different properties.
Elaborating on one specific region we first exemplify why there is no strict relationship between the accuracy of the marginals and the partition function and then provide sufficient and necessary conditions for a fixed point to be optimal  with respect to approximating  both.

Finally, \textbf{Chapter 8} concludes this thesis, summarizes the results obtained, and discusses the extent to which our understanding of BP has changed.
Furthermore, the most pressing questions left unanswered are indicated, hence providing a pointer to potential future research directions.

  \emptydoublepage
\newchapter{Background}{background}
\openingquote{A mind is like a parachute. \\It doesn't work if it is not open.}{ Frank Zappa}
\renewcommand{\pwd}{background}

This chapter introduces the necessary background and lays the foundation for the subsequent chapters.
The focus of this thesis lies at problems that arise in the context of probabilistic graphical models. Accordingly, we begin with a brief introduction to probability theory, in Section~\ref{sec:background:probability}, and to graph theory, in Section~\ref{sec:background:graphs}, before we finally unite both concepts in probabilistic graphical models, in Section~\ref{sec:background:pgm}.
Section~\ref{sec:background:models} discusses one specific class of probabilistic graphical models, \emph{undirected binary pairwise models}, that are studied extensively throughout the thesis.
We finally devote Section~\ref{sec:background:inference} to the tasks summarized under the notion of \emph{inference}.

\newsection{Probability Theory}{background:probability}
We begin with a brief review of probability theory and introduce the most relevant notions. 
Note that we will restrict our focus to discrete random variables throughout this thesis.
Let us therefore consider a discrete random variable $\RV{}$ that maps from the sample space $\Omega$ to a discrete finite set  $\mathcal{X}$, i.e., $X:\Omega \rightarrow \mathcal{X}$.

We denote the probability mass function that assigns a probability to the generic value $\RVval{} \in \sampleSpace{X}$ by the shorthand notation $\pmf{\RV{}}(\RVval{} )\definition\pmf{\RV{}}(X=\RVval{})$. 
Likewise, let $\RV{}$ and $\RV[Y]{}$ be two discrete random variables; then, with slight abuse of notation we define the joint distribution according to
\begin{align}
	\pmf{\RV[x]{},\RV[Y]{}}(\RVval[x]{},\RVval[y]{}) \definition \pmf{\RV[X]{},\RV[Y]{} }\left((\RV[X]{}=\RVval[x]{})\cap (\RV[Y]{}=\RVval[y]{})\right).
\end{align}
Now let us consider a set of $N$ discrete random variables $\setOfNodes = \{X_1,X_2,\ldots,X_N \}$ with the joint distribution $\pmf{\setOfNodes}(\mathbf{x})$ 
where the range of $\setOfNodes$ is the product space $\mathcal{X}^N \definition \sampleSpace{X}_1 \times \cdots \times \sampleSpace{X}_N$.
We further denote the configuration for a given set of random variables by $\RVvalSet{} = \{x_1,x_2,\ldots,x_N \}$.
The marginal probability distribution for a subset  $\mathbf{Y}  \subset \setOfNodes$ is obtained by summing out all variables $\RV{i}$ that are not in $\mathbf{Y}$ according to
\begin{align}
	\pmf{\mathbf{Y}}(\mathbf{y}) \definition \sum_{\RV{i} \in \{\mathbf{X}\backslash\mathbf{Y}\}} \sum_{\RVval{i} \in \sampleSpace{X}_i} \pmf{\setOfNodes}(\mathbf{x}).
	\label{eq:marginals}
\end{align}

(Conditional) statistical independence between random variables is a particularly important property when dealing with probabilistic graphical models. We say that $\setOfNodes$ is conditionally independent of $\setOfNodes[Y]$ given $\setOfNodes[Z]$ whenever $ \pmf{\setOfNodes[X],\setOfNodes[Y]|\setOfNodes[Z]}(\mathbf{x},\mathbf{y}|\mathbf{z}) = \pmf{\mathbf{X}|\setOfNodes[Z]}(\mathbf{x}|\mathbf{z}) \pmf{\mathbf{Y}|\setOfNodes[Z]}(\mathbf{y}|\mathbf{z})$ holds for all configurations $\RVvalSet{} \in\sampleSpace{x}^N$, $\RVvalSet[y]{} \in\sampleSpace{Y}^N$, and $\RVvalSet[z]{}\in\sampleSpace{z}^N$.

Distributions are often characterized by their expectations.
For a discrete random variable $\RV{}$ we define its expectation under the distribution $\pmf{\RV{}}(x)$ according to
\begin{align}
	\E(x) \definition \sum_{x\in\mathcal{X}} x \cdot \pmf{\RV{}}(x).
\end{align}


\newsection{Graph Theory}{background:graphs}
One of the most persuasive properties of probabilistic graphical models is the representation of joint probability distributions by using graphs. 
This section gives a self-contained introduction to the basics of graph theory.
We refer the interested reader to one of the many available books on graph theory for a more in-depth treatment (e.g., \cite{korte,diestel}).

Let us consider a \emph{graph} $\graph = (\setOfNodes,  \setOfEdges)$ with a set of \emph{nodes} (or vertices)  $\setOfNodes = \{\RV{1},\dots,\RV{N}\}$ and a set of \emph{edges} $\setOfEdges$. 
A graph is either directed or undirected and consists of directed or undirected edges respectively.
We denote an undirected edge by $\edge{i}{j} \in \setOfEdges$ (or equivalently by $\edge{j}{i}$) if it joins two nodes $\RV{i} \in \setOfNodes$ and $\RV{j} \in \setOfNodes$.
Note that, for the remainder of this thesis, we will only consider undirected graphs and further restrict our focus to \emph{simple graphs} -- 
that come without parallel edges (that would join the same pair of nodes twice) and without self-loops (i.e., $i\neq j$). 

Let us consider an edge $\edge{i}{j} \in \setOfEdges$, then $\RV{j}$ is a \emph{neighbor} of $\RV{i}$ and vice versa; the set of neighbors for any variable $\RV{i} \in \setOfNodes$ is defined by 
\begin{align}
	\neighbors{i} \definition \{\RV{j} \in \setOfNodes : (i,j) \in \setOfEdges \}.
\end{align}
The \emph{degree} of a node $\RV{i}$ is the number of incident edges and is consequently defined by the cardinality of the neighbor-set, i.e., by
\begin{align}
	\nodeDegree{i} \definition |\neighbors{i}|.
\end{align}
We further denote the \emph{average degree} of $\graph$ by $\averageDegree[\graph] \definition \frac{2|\setOfEdges|}{N}$.
A \emph{path} between two nodes $\RV{i},\RV{k} \in \setOfNodes$ is a graph $\path = (\{\RV{i},\ldots \RV{k} \}, \{(i,i+1),\ldots, (k-1,k) \} )$, where $\RV{i}$ and $\RV{k}$ are the end-nodes of $\path$. 
A path goes from $\RV{i}$ to $\RV{k}$ if there is an edge-progression that connects those two nodes.
\begin{defn}[Loop]
	A loop is a path from a node $\RV{i}$ back to the same node along a sequence of edges, i.e.,  a path $\path = (\{\RV{i},\RV{i+1},\ldots, \RV{k}\},\{\edge{i}{i+1},\ldots,\edge{k-1}{k}\})$ exists for $k=i$.
	
\end{defn}
Depending on the notational conventions, loops are sometimes also referred to as cycles or circuits.
Note that every loop has to contain at least three edges, since we do not allow for self-loops.

A graph is said to be \emph{connected} if, for every pair of nodes $\RV{i}$ and $\RV{j}$, a path exists that connects $\RV{i}$ to $\RV{j}$.
We only consider connected graphs in this thesis.\\

In the context of this thesis, there are certain graph types that are particularly relevant; 
most notably these include the following:
\begin{defn}[Tree]
	A connected graph that does not contain any loops is a tree.
\end{defn}

\begin{defn}[Complete Graph]
	A complete graph $\graph = (\setOfNodes, \setOfEdges)$ has every pair of nodes $\RV{i}$ and $\RV{j}$ connected by an edge $\edge{i}{j}$; i.e., $\graph$ is fully connected with  
	$\setOfEdges = \{\edge{i}{j}: \RV{i}, \RV{j} \in \setOfNodes, i \neq j\}$.
	It follows that every node has an equal degree $\nodeDegree{i} = N-1$.
	See Figure~\ref{fig:complete} for an illustration of complete graphs of size one to four.
\end{defn}
\begin{figure}
	\centering
	\subfigure{\raisebox{17mm}{\includegraphics{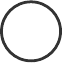} \vspace*{20cm} }} \hspace{1cm}
	\subfigure{\raisebox{17mm}{\includegraphics{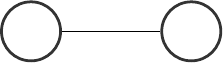} }} \hspace{1cm}
	\subfigure{{\includegraphics{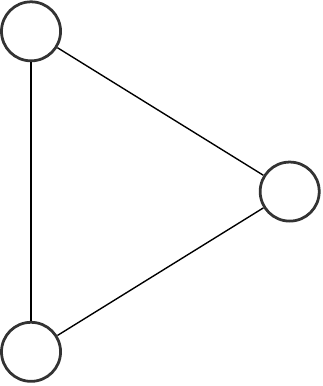} }} \hspace{1cm}
	\subfigure{{\includegraphics{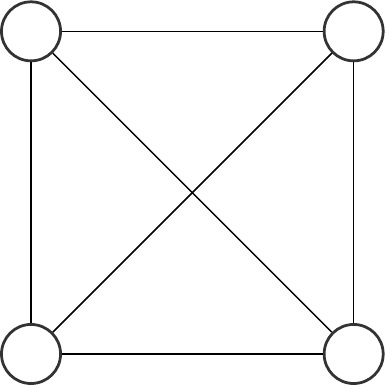} }}
	\caption{Complete graphs of size $N=1$ to $N=4$.}
	\label{fig:complete}
\end{figure}

\begin{defn}[Grid Graph]
	An $m \times n$ grid graph, or lattice graph, is a graph with $N=mn$ nodes that has all edges aligned along the square lattice. See Figure~\ref{fig:grid} for an illustration of a two-dimensional grid graph.
\end{defn}
\begin{figure}
	\centering
	\includegraphics{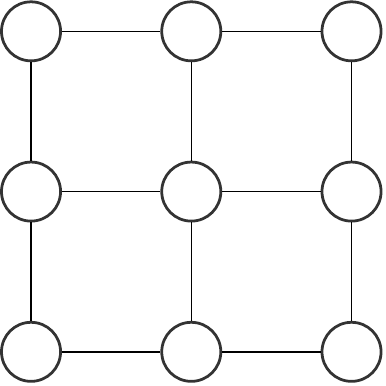}
	\caption{Two-dimensional grid graph of size $N=3\times3$}
	\label{fig:grid}
\end{figure}

It will often be necessary to consider only a part of the graph; we call this a subgraph.
\begin{defn} [Subgraph]
	Let $\graph = (\setOfNodes, \setOfEdges)$ be a graph and let $\setOfNodes' \subseteq \setOfNodes$.
	Then, the subgraph (or graph-component) $\graph' = (\setOfNodes', \setOfEdges')$ is induced by $\setOfNodes'$, where $\setOfEdges' = \{\edge{i}{j}\in \setOfEdges: \RV{i}, \RV{j} \in \setOfNodes'\}$.
\end{defn}
Complete subgraphs are of particular relevance and are often referred to as cliques.
\begin{defn} [Clique]
	Let $\graph' = (\setOfNodes', \setOfEdges')$ be a subgraph. 
	If $\graph'$ is complete we call it a clique and refer to it as $\clique{i}$.
	If any adjacent node $\RV{j} \in \setOfNodes \backslash \setOfNodes' : \edge{i}{j} \in \setOfEdges, \RV{i} \in \setOfNodes'$ exists that, by adding, would render $\clique{i}$ not complete anymore, $\clique{i}$ is a maximal clique.
	%
\end{defn}
See  Figure~\ref{fig:complete} that depicts the maximal cliques of size one to four.
Note that any graph $\graph=(\setOfNodes, \setOfEdges)$ is complete if and only if $\graph$ is a maximal clique.

\begin{defn}[Regular Graph]
	Let $\graph = (\setOfNodes, \setOfEdges)$ be a graph where all nodes $\RV{i} \in \setOfNodes$ have equal degree
	$\nodeDegree{i} = d$.
	Then, we refer to $\graph$ as a d-regular graph.
\end{defn}
It follows that the average degree for a regular graph equals the degree of all individual nodes, i.e., $\averageDegree[\graph] = \nodeDegree{i}$.

\begin{defn}[Bipartite Graph]
	A bipartite Graph $\graph = (\setOfNodes, \setOfEdges)$ is a graph that decomposes into two disjoint subgraphs $\setOfNodes[Y]$ and $\setOfNodes[Z]$ where $\setOfNodes = \setOfNodes[Y] \cup \setOfNodes[Z]$ and $\setOfNodes[Y] \cap \setOfNodes[Z] = \emptyset$ so that every edge connects those two subgraphs; i.e., for all $\RV[y]{i} \in \setOfNodes[Y]$ we have $\neighbors{\RV[Y]{i}} \in \setOfNodes[Z]$ and vice versa.
\end{defn}

\begin{defn}[Adjacency Matrix]
	The adjacency matrix $\vm{A}$ represents the connections of a finite graph. 
	$\vm{A}$ is a $0-1$ matrix with rows and columns indexed by the set of nodes so that $a_{ij}=1$ if and only if $\edge{i}{j}\in \setOfEdges$.
	Note that $\vm{A}$ is symmetric for undirected graphs.
\end{defn}


\newsection{Probabilistic Graphical Models}{background:pgm}
Probabilistic graphical models provide a compact representation of joint distributions and are particularly well suited for representing distributions with many random variables.
The straightforward specification of a distribution in $N$ random variables requires one to define 
and store the probabilities of all $|\sampleSpace{X}|^N$ configurations; 
considering the fact that typical problems often have hundreds of random variables renders such an approach impracticable. 
The representation of the joint distribution by the means of a graph on the other hand is intuitive and comes with the advantage of being interpretable by humans.
Moreover, and of even greater relevance, the graph exploits the statistical dependencies of the distribution and makes them explicit.
This is required if confronted with (even moderately) high-dimensional joint distributions.

Probabilistic graphical models come in different forms that represent the statistical dependencies in slightly different ways.
These include the most prominent types such as factor graphs, Bayesian networks, and Markov random fields. We will, however, restrict our focus to Markov random fields (which are also termed undirected graphical models).

\newsubsection{Undirected Graphical Models (Markov Random Fields)}{background:pgm:ugm}
A probabilistic graphical model $\ugm \definition (\graph,\setOfPotentials)$ 
consists of an undirected graph $\graph = (\setOfNodes, \setOfEdges)$ and a set of $K$ potentials (sometimes called clique potentials, or compatibility functions) $\setOfPotentials$
that defines the joint distribution $\pmf{\setOfNodes}(\RVvalSet{})$. 

There is a certain elegance to graphical models;
one that relates the structural properties of the graph to the properties of the associated joint distribution.
Specifically, a one-to-one correspondence between the set of nodes $\setOfNodes = \{\RV{1},\ldots,\RV{N}\}$ and the set of random variables\footnote{
	We consider only random variables with a discrete alphabet $\sampleSpace{X}$ although the framework would generalize to continuous random variables with $\sampleSpace{X}=\mathbb{R}$ as well.}
holds. 
Additionally, each edge $\edge{i}{j}$ represents the existence of a statistical dependency between $\RV
{i}$ and $\RV{j}$.

The correspondence between edges and statistical dependencies manifests itself into the global Markov property.
Let $\setOfNodes[u] \subset \setOfNodes[x]$ be a subset of nodes;
we say that $\setOfNodes[u]$ separates two disjoint (sets of) nodes $\setOfNodes[v]$ and $\setOfNodes[w]$ (i.e., $\setOfNodes[v] \cap \setOfNodes[w] = \emptyset$) if every path connecting the sets $\setOfNodes[V]$ and $\setOfNodes[W]$ contains at least one variable  $\RV[u]{i} \in \setOfNodes[u]$.
Removing the subgraph induced by $\setOfNodes[U]$ from $\graph$ thus eliminates all paths between $\setOfNodes[v]$ and $\setOfNodes[w]$.
Let us assume that $\setOfNodes[u]$ separates $\setOfNodes[v]$ and $\setOfNodes[w]$ as discussed for $\graph$;
then a joint distribution $\pmf{\setOfNodes}(\RVvalSet{})$ satisfies the global Markov property with respect to $\graph$ if
\begin{align}
	\pmf{\setOfNodes[v],\setOfNodes[w]|\setOfNodes[u]}(\RVvalSet[v]{}, \RVvalSet[w]{}| \RVvalSet[u]{}) = \pmf{\setOfNodes[v] | \setOfNodes[u]}(\RVvalSet[v]{} | \RVvalSet[u]{}) \pmf{\setOfNodes[w]|\setOfNodes[u]}(\RVvalSet[w]{}| \RVvalSet[u]{}).
	\label{eq:independencies}
\end{align}
See Figure~\ref{fig:independence} for a visualization of a graph that satisfies~\eqref{eq:independencies}.
\begin{figure}
	\begin{center}
		\includegraphics[width=0.7\linewidth]{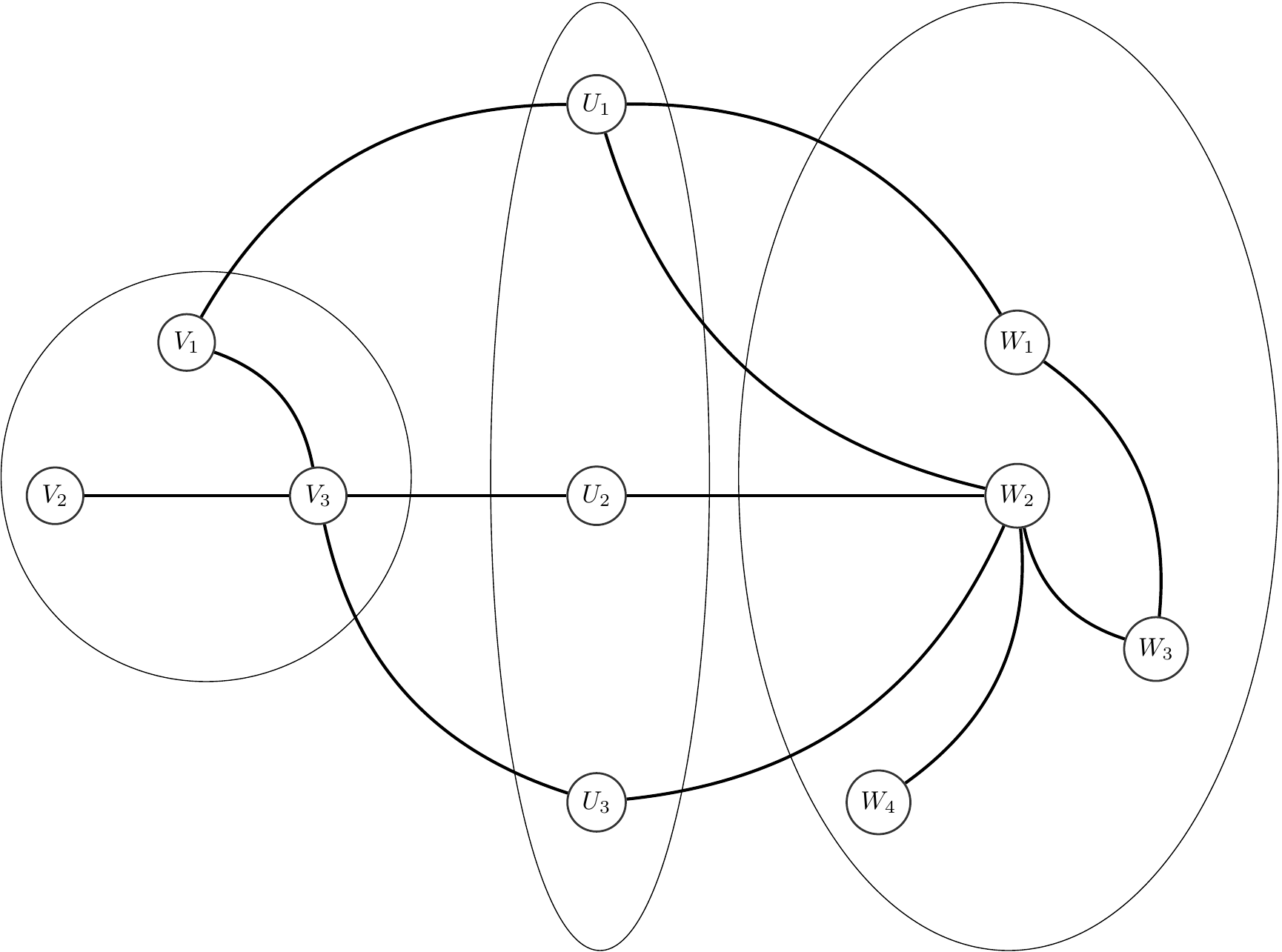}
	\end{center}
	\caption{Graphical representation of the conditional independencies implied by~\eqref{eq:independencies}. The set of nodes $\setOfNodes[u]$ separates $\setOfNodes[v]$ from $\setOfNodes[w]$, i.e., the after removing the subgraph induced by $\setOfNodes[u]$ from $\graph$ two separate subgraphs remain.}
	\label{fig:independence}
\end{figure}

The global Markov property entails strong restrictions on the factorization properties of the joint distribution.
Specifically, according to the \emph{Hammersley-Clifford-Theorem}~\cite{hammersley1971}, the joint distribution factorizes into a set of \emph{potentials} $\setOfPotentials = \{\potential{\clique{1}},\ldots, \potential{\clique{K}}\}$ specified over the set of cliques $\setOfCliques \definition \{\clique{1},\ldots,\clique{K}\}$ of the graph. That is
\begin{align}
	\pmf{\setOfNodes}(\RVvalSet{}) = \frac{1}{\partitionFunction} \prod_{\clique{i}\in\setOfCliques}^{} \potential{\clique{i}}(\RVvalSet[x]{\clique{i}}), 
	\label{eq:joint_over_cliques}
\end{align}
where the potentials $ \potential{\clique{i}}(\RVvalSet[x]{\clique{i}})$ only depend on the values of random variables  belonging to the given clique, i.e., 
$\RVvalSet[x]{\clique{i}} = \{\RVval{i}:\RV{i} \in \clique{i}\}$. 
Although the potentials need to be non-negative and may remind us of conditional probabilities, 
it is important to stress that this is not the case and that potentials do not necessarily have a specific probabilistic interpretation (cf.~\cite[Example 4.2]{koller2009}).
The flexibility one has in assigning values to the potential function comes at a cost, however:
the product over all potential functions may require a normalization function (partition function) $\partitionFunction$ to ensure that the product in~\eqref{eq:joint_over_cliques} constitutes a valid distribution.\footnote{
	The partition function is evaluated by computing the unnormalized sum over all states and is usually denoted by the letter $\partitionFunction$ to denote its origin in the German word \emph{Zustandssumme}.}

The product in~\eqref{eq:joint_over_cliques} can, in principle, consist of potentials specified only over the max-cliques of the graphs.
This may, however, hide the factorization properties of the underlying problem~\cite[p.108]{koller2009} and it is thus often preferable to consider only smaller (sub)-cliques as e.g., of size two for \emph{pairwise models}.
Pairwise models are particularly well-suited for a theoretical analysis because of their simplicity and will therefore be the main focus of this thesis.
Because of their central role, we devote the following section to the introduction of binary pairwise models.

\newsection{Binary Pairwise Models}{background:models}
Throughout the main part of this thesis, we will focus on one particular type of undirected models:
these are \emph{binary pairwise graphical models}.
Binary pairwise models admit a relatively simple treatment while being rich enough to represent a wide class of problems.
Indeed, most practical problems permit a representation as a binary pairwise graphical models, which highlights that the class is much less restrictive as it may seem at first.

We will first provide a formal introduction in Section~\ref{sec:background:models:description}
and then discuss both the limitations and the relevance of binary pairwise graphical models in Section~\ref{sec:background:models:relevance}.
Finally, we introduce the exponential representation in Section~\ref{sec:background:models:terminology} and clarify the relation to the Ising model in Section~\ref{sec:background:models:ising}.

\newsubsection{Model Description}{background:models:description}
\emph{Binary} models are graphical models where every random variable $\RV{i} \in \setOfNodes$ takes values from the binary alphabet, e.g., according to $\RVval{i} \in \sampleSpace{X} = \{-1,+1\}$.
\emph{Pairwise} models have potential functions that consist of two random variables at most; i.e., 
potentials are only associated with cliques $\clique{i}\in\setOfCliques$ of size $|\setOfNodes_{\clique{i}}|\leq2$. 
We resort to an even finer-grained representation 
that separately specifies the potentials over the nodes and edges.
Complying to this representation, we then define the joint distribution according to
\begin{align}
	\joint  &= \frac{1}{\partitionFunction} \prod_{\edge{i}{j} \in \setOfEdges} \pairwise{x}{i}{j} \prod_{\RV{i} \in \setOfNodes} \local{x}{i}, \label{eq:joint_pairwise}
\end{align}
where we introduced two different potential-types.
These are: 
the pairwise potentials $\pairwise{x}{i}{j}$ that act on the edges $\edge{i}{j} \in \setOfEdges$ and the local potentials $\local{x}{i}$ that act on the nodes $\RV{i} \in \setOfNodes$. 
Note that each pairwise potential is only considered once as $\edge{i}{j}=\edge{j}{i}$.
We will prefer to denote the local and pairwise potentials by their shorthand notation $\localShort{i}$ and $\pairwiseShort{i}{j}$, unless the actual values of the associated random variables are of immediate relevance.

Note that the factorization in~\eqref{eq:joint_over_cliques} considers the set of all cliques (with varying sizes), whereas the factorization in~\eqref{eq:joint_pairwise} considers cliques of size two at most.
Therefore, note that every connected pair of variables is also a clique (albeit possibly only a subset of a larger clique).
This means that both forms become identical  -- even in the presence of larger cliques -- whenever all higher-order potentials are  trivial.\footnote{A trivial potential 
	does not influence the values of the joint distribution. One way of making a potential trivial is by assigning identical values, as for example $\potential{\clique{}}(\RVvalSet[x]{\clique{}}) = 1$, for all possible configurations $\RVvalSet[x]{\clique{}}$.}

The representation in~\eqref{eq:joint_pairwise} has yet another equivalent factorization  
\begin{align} 
	\joint = \frac{1}{\partitionFunction} \prod_{\clique{(i,j)}\in\setOfCliques} \potential{\clique{(i,j)}}(\RVval{i},\RVval{j}) \nonumber
\end{align}
that incorporates all local potentials into the pairwise ones.
Such a compact representation clearly reduces the overall number of potentials,  but may obscure the underlying structure.
This often makes the model much less intuitive to interpret.
We will now discuss and compare the different ways of factorizing the joint distribution by means of the following example.

\begin{example}[Factorization into Potentials of Varying Size]$ $\newline 
	Consider the undirected graphical model $\ugm = (\graph,\setOfPotentials)$ depicted in Figure~\ref{fig:cliques}.
	\begin{center}
		\includegraphics[width=0.25\linewidth]{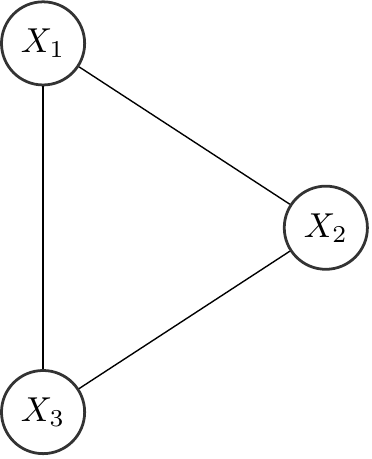}
		\captionof{figure}{Complete graph with three random variables with its potentials specified over the set of edges and the set of nodes according to~\eqref{eq:joint_pairwise}.}
		\label{fig:cliques}
	\end{center}
	We consider a joint distribution that is specified according to~\eqref{eq:joint_pairwise} in terms of local potentials $\localShort{i}$ and pairwise potentials $\pairwiseShort{i}{j}$ so that
	\begin{align}
		\joint \!= \!\frac{1}{\partitionFunction}\pairwiseShort{1}{2}\pairwiseShort{2}{3}\pairwiseShort{3}{1}\localShort{1}\localShort{2}\localShort{3}.
		\label{eq:example:factorization:definition}
	\end{align}
	
	As discussed above, one possible factorization stems from incorporating the local potentials into the pairwise ones.
	The specific model in Figure~\ref{fig:cliques} provides one obvious way of doing so; 
	therefore, let us define the clique potentials according to
	\begin{align}
		\potential{\clique{(i,j)}}(\RVval{i},\RVval{j}) = \pairwiseShort{i}{j} \localShort{j}
		\label{eq:example:factorization:pairwise}
	\end{align}
	so that
	\begin{align}
		\joint = \frac{1}{\partitionFunction}\potential{\clique{(1,2)}}(\RVval{1},\RVval{2}) \potential{\clique{(2,3)}}(\RVval{2},\RVval{3}) \potential{\clique{(3,1)}}(\RVval{3},\RVval{1}).
		\label{eq:example:factorization:pairwise:joint}
	\end{align}
	Note that~\eqref{eq:example:factorization:pairwise} is just one possible way of constructing clique potentials from the local and the pairwise ones.
	One only has to make sure that every  potential is considered exactly once, as the joint distribution in~\eqref{eq:example:factorization:pairwise:joint} would not conform with its original definition in~\eqref{eq:example:factorization:definition} otherwise.
	
	Finally, it is always possible to factorize the joint distribution over the maximum cliques.
	Note that, in this example, $\graph$ is already a maximum clique, which makes it possible to express the joint distribution by a single clique potential.
	This potential of the maximum clique incorporates all pairwise- and singleton- potentials according to
	\begin{align}
		\potential{\clique{(1,2,3)}}(\RVval{1},\RVval{2},\RVval{3}) = \pairwiseShort{1}{2}\pairwiseShort{2}{3}\pairwiseShort{3}{1}\localShort{1}\localShort{2}\localShort{3}
	\end{align}
	so that
	\begin{align}
		\joint = \frac{1}{\partitionFunction} \potential{\clique{(1,2,3)}}(\RVval{1},\RVval{2},\RVval{3}).
	\end{align}
	\label{ex:factorization}
\end{example}

The above example highlights the influence of the particular factorization on the overall number of potentials.
Higher-order cliques reduce the overall number of potentials but, simultaneously, tend to conceal the underlying  structure of the joint distribution~\cite[Chapter~4.2]{koller2009}).
Similarly, pairwise graphical models enforce one specific factorization and may thus hide the ``true'' factorization of the underlying distribution as well.
%
Factor graphs, on the other hand, provide a flexible representation that makes this underlying factorization explicit.
We will, however, except for the application to error-correcting codes, restrict ourselves to pairwise graphical models for the remainder of this thesis.

\newsubsection{Generality of Binary Pairwise Models}{background:models:relevance}
The simplicity of binary pairwise models is compelling, although the exclusion of other models -- not representable by binary pairwise models -- may seem rather restrictive.
Despite this impression, binary pairwise models provide a general framework.
That is, most graphical models can be directly converted into a binary pairwise model~\cite{yedidia2001, weiss2000correctness}.

The conversion from general undirected models to binary pairwise models becomes immediately apparent for models with strictly positive potentials, in which case, a simple reduction to the binary pairwise case exists~\cite{eaton2013model}.
Note, however, that such a conversion potentially scales up the problem immensely and is therefore not always ideal from a practical point of view~\cite{mackay2001conversation}.

From a  theoretical point of view, however, it is  beneficial to study binary pairwise models;
these models are rich enough to exhibit complex behavior -- thus demonstrating many interesting aspects, though they still admit a simplified treatment -- thus avoiding many technical subtleties.
Moreover, the study of binary pairwise models carries great relevance; not only because they are still far from being fully understood but also because theoretical insights carry over to more general models (by conversion of one model-class to the other).

Besides their generality, binary pairwise models are important in their own right as they arise in various applications.

\newsubsection{Exponential Representation and Model Parametrization}{background:models:terminology}
So far we have not considered the actual specification of the potentials.
We will now introduce one specific parameterization of binary pairwise models, show why it belongs to an exponential family, and define different model-classes with fundamentally different behavior.
Besides, we will fix our naming convention and briefly point to alternative ones.
Note that,  because of its immediate connection, the terminology is often rooted in the physical interpretations of the quantities.

Many relevant problems have a non-zero probability for all configurations $\RVvalSet{}$, such that the joint distribution factorizes into strictly positive potentials.
This allows us to assign some \emph{energy} $E(\RVvalSet{})$ to every configuration, so that the joint distribution from~\eqref{eq:joint_pairwise} can be expressed in its exponential form according to
\begin{align}
	\joint = \frac{1}{\partitionFunction} \cdot e^{-E(\RVvalSet{})}. \label{eq:exponential}
\end{align}

For the remainder of this thesis we shall use a minimal representation of binary pairwise models (cf.~\cite[Section 3.3]{wainwright2008graphical}) that defines the energy in terms  of couplings $\coupling{i}{j}\in \REAL$, that act on the edges  $\edge{i}{j} \in \setOfEdges$ and local fields $\field{i}\in \REAL$, that act on the nodes $\RV{i} \in \setOfNodes[x]{}$. 
Note that we drop the subscripts and write $\coupling{i}{j} = \coupling{}{}$ or $\field{i} = \field{}$ whenever the parameters are identical for all edges or for all nodes respectively.\footnote{If all nodes have the same value $\field{i} = \field{}$ we will sometimes refer to $\field{}$ as the external field.} 

Let the local and pairwise potentials of state $\RVval{i} \in \{-1,+1\}$ be 
$\localShort{i} = \exp (\field{i}\RVval{i})$ and $\pairwiseShort{i}{j}= \exp(\coupling{i}{j}\RVval{i}\RVval{j})$.
Then, if we plug these potentials into~\eqref{eq:joint_pairwise} we end up with a joint distribution that has its energy (cf.~\eqref{eq:exponential}) given by
\begin{align}
	E(\RVvalSet{}) = -\sum_{\edge{i}{j} \in \setOfEdges} \coupling{i}{j} \RVval{i}\RVval{j} - \sum_{\RV{i} \in \setOfNodes} \field{i} \RVval{i}.
	\label{eq:binary_pw_energy}
\end{align}
The energy (and consequently the joint distribution) depends not only on $\RVvalSet{}$ but on $\coupling{i}{j}$ and $\field{i}$ as well; 
we will, however, only make this dependence explicit by $E(\RVvalSet{}, \coupling{}{}, \field{})$ if of immediate relevance.

In the literature, one distinguishes two different types of interactions between variables: 
if a coupling is positive ($\coupling{i}{j}> 0$) then the associated edge $\edge{i}{j}$ is \emph{attractive}; 
if a coupling is negative ($\coupling{i}{j} < 0$) then the associated edge $\edge{i}{j}$ is \emph{repulsive}.
In accordance with this naming-convention, we call a model $\ugm$ attractive if it contains only attractive edges (these models are also known as ferromagnetic models~\cite{mezard2009} or log-supermodular models~\cite{ruozzi2012bethe}); 
we call it repulsive (or antiferromagnetic) if it contains only repulsive edges; 
and we call a model \emph{general} if it contains both types of edges.

The particular representation of~\eqref{eq:binary_pw_energy} is also known as Hopfield network (cf.~\cite{hopfield1982neural} and \cite[Chapter 42]{mackay2003}) or Boltzmann machine (cf.~\cite{hinton1986learning,welling2003approximate} and~\cite[Chapter 43]{mackay2003}) in the machine learning community and as Ising model (cf.~\cite{history_ising} and~\cite[Chapter 31]{mackay2003}) in the physics literature.
The Ising model has been studied for a long time in statistical physics. 
For its particular relevance, we will  devote the subsequent section to the Ising model, discuss its most important properties, and provide a brief historical outline of its development.

\newsubsection{The Ising Model}{background:models:ising}
The study of pairwise models in the form of~\eqref{eq:exponential} actually dates back for more than a century.\footnote{
	Distributions of this form were first introduced by Ludwig Boltzmann and by Josiah Willard Gibbs and provide the foundation of statistical physics. In the physics literature, one often refers to distributions of the form~\eqref{eq:exponential} as Boltzmann-distributions or Gibbs-measures.}
In the field of statistical physics, one studies the macroscopic behavior of systems with a large number of interacting components (e.g., atoms or molecules) in the thermodynamic limit of infinitely many components. 

One concept of central relevance are phase transitions; these are points in the parameter space where the partition function becomes non-analytic.
Phase transitions trigger fundamental changes in the system behavior, as for example the change from a fluid to gas.
In computer science, we often encounter algorithms that have many ``components'' interacting with each other (e.g., as in message passing algorithms);
such algorithms often exhibit rapid performance-drops for certain points in the parameter space, which clearly remind us of phase transitions.\\

One extensively studied model in statistical physics is the Ising model: 
it consists of $N$ atoms in a configuration $\bm{\sigma} = (\sigma_1,\ldots, \sigma_N)$ with spins $\sigma_i \in \{-1,+1\}$ that lie on a $d$-dimensional lattice and that are magnetically coupled.\footnote{ 
	Here we use the standard notation for the Ising model, although the correspondence to the notation in this thesis becomes obvious by comparing~\eqref{eq:binary_pw_energy} to~\eqref{eq:e_ising}.}
As a model for magnetic bodies, the Ising model's phase transitions describe the change from the paramagnetic to the ferromagnetic region among others.
The Ising model is one of the simplest models that is sufficiently rich to exhibit phase transitions; therein lies its appeal.

The interaction between two neighboring spins is specified by the coupling strength $J$ that is either ferromagnetic (i.e., positive) or anti-ferromagnetic (i.e, negative). 
Two neighbors then energetically favor the same state in case of ferromagnetic interactions and the opposite state in case of anti-ferromagnetic interactions.
%
The energy of a configuration $\sigma$ is
\begin{align}
	E(\sigma) = -\beta J \sum_{\edge{i}{j} \in \setOfEdges} \sigma_i \sigma_j + H \sum_{i=1}^N \sigma_i,
	\label{eq:e_ising}
\end{align}
where $H$ is the external magnetic field and $\beta = \frac{1}{k_B T}$ is the inverse temperature with $k_B$ being the Boltzmann constant. \\

The appreciation of the Ising model had its ups and downs. 
After being dumped as a too simplified model with no physical usefulness, it took some years before the Ising model received its well-deserved attention.
We outline some of the cornerstones in this development below and refer the interested reader to the surprisingly exciting read~\cite{history_ising} for a review of the major events in the history of the Ising model.

Wilhelm Lenz proposed the Ising model, as a simplified model for the interactions inside magnetic bodies, with the aim of developing a better understanding of the underlying properties.
One of his students, Ernst Ising, solved the one-dimensional case (by applying the transfer-matrix method) in his thesis.
The result, rather surprisingly, revealed the inexistence of phase-transitions in the one-dimensional case.
Being unable to solve it in higher dimensions, he concluded that the Ising model is generally incapable of experiencing phase transitions.
This conjecture, however, was in stark contrast to the ferromagnetic theory developed by Pierre Currie, which led to the disregard of the Ising model.

%
Finally, this contradiction was resolved by Lars Onsanger who solved the Ising model on the two-dimensional grid in a mathematical ``tour-de-force'' (as the Onsanger solution is nowadays often referred to) 
and revealed the existence of phase transitions in two dimensions.
Note that the exact solution of the Ising model is only known in these two cases and it remains an open problem to compute the exact solution for higher-dimensional cases.

The classical Ising model, with its energy defined according to~\eqref{eq:e_ising}, provides a powerful generalization:
the \emph{spin glass} model has the (local) fields $\field{i}$ and the couplings $\coupling{i}{j}$ take potentially different values for all nodes and edges and has its energy defined according to
\begin{align}
	E(\sigma) = -\beta  \sum_{\edge{i}{j} \in \setOfEdges} \coupling{i}{j}\sigma_i \sigma_j +  \sum_{i=1}^N \field{i}\sigma_i.
	\label{eq:e_spin_glass}
\end{align}
What makes these spin glass models so interesting is that we have a much poorer understanding of them as opposed to the classical Ising model.
The lack of understanding is mainly because of \emph{frustrations}~\cite{toulouse}.
These are configurations that have  some pairs of random variables $\RV{i},\RV{j}$ energetically favor states that contradict the coupling of the associated edge $\edge{i}{j}$ (cf. Example~\ref{ex:frustrations}).
To make this more precise we introduce the product of couplings along some cycle $\path$, i.e.,
\begin{align}
	J_{\path} \definition \prod\limits_{\path} \coupling{i}{j}.
	\label{eq:frustrations}
\end{align}
Then, a graph is frustrated whenever it contains a cycle for which $J_{\path}$ equates to a negative number.
For the classical Ising model where all couplings take the same value, the existence of frustrations depends only on the graph structure (i.e., if cycles of odd-length exist).
For spin glasses, however, the existence of frustrations depends on both the graph structure and the parameters.
This is one of the main reasons for studying spin glasses as frustrations can lead to a complex solution space with potentially many different solutions.\\

So far, we have seen that binary pairwise models constitute a very general class of graphical models and fit nicely into the framework of statistical physics.
Although this emphasizes the relevance of binary pairwise models, we barely mentioned the purpose of introducing graphical models.
We will now shift our focus to the problems we intend to solve and thus reveal the elegance and the advantages of graphical models.

\newsection{Inference}{background:inference}
The task of inference plays an important role in many scientific fields and it is a prerequisite for probabilistic reasoning.
In a nutshell, inference deals with drawing statistical conclusions about a certain subset of random variables, given some (possibly noisy) observations~\cite{pernkopf2014pgm}.

Applications include, but are not limited to, computer vision and speech processing where the observations are corrupted versions of the image or the speech-signal and one is interested in finding the most probable explanation.
Ideally, one would hope that the best explanation for the given observations matches the original image or signal.
Another important application is found in the context of error-correcting codes, where the observation is the received codeword and one is interested in obtaining the sent codeword.\\

To make the task of inference more precise we consider a joint distribution $\joint$ over a set of nodes $\setOfNodes[x]$ that consists of two disjoint subsets, i.e, $\setOfNodes[X] = \setOfNodes[Y] \cup \setOfNodes[O]$ and $\setOfNodes[y] \cap \setOfNodes[o] = \emptyset$.
The observed variables are denoted by $\setOfNodes[O]$ and the unobserved ones are denoted by $\setOfNodes[Y]$.
%
Without loss of generality, we will restrict ourselves to problems without observed variables $\setOfNodes[O] = \emptyset$ throughout this thesis.
Now we present the following three central problems of probabilistic inference.
\begin{itemize}
	\item \emph{Maximum a posterior (MAP) inference} estimates the mode of the distribution with the aim of identifying the joint assignment $\RVvalSet[y]{}^*$ that maximizes the probability according to 
	$\RVvalSet[y]{}^* = \argmax_{\RVvalSet[y]{} \in \productSpace{y}{|\mathbf{Y}|}} \pmf{\setOfNodes[y]|\setOfNodes[o]}(\RVvalSet[y]{}|\RVvalSet[o]{})$.
	Note that, strictly speaking, this is the most probable explanation (MPE) estimate, whereas the MAP estimate deals with the more general problem of maximizing the probability for any subset of $\setOfNodes[y]$; it is quite common, however, to neglect this distinction despite MPE inference being the easiest instance~\cite[Chapter~2.3]{koller2007introduction}.
	We abide to this convention and, provided $\setOfNodes[o] = \emptyset$, aim to obtain 
	\begin{align}
		\RVvalSet[]{}^* = \argmax_{\RVvalSet[]{} \in \productSpace{x}{N}} \joint.
	\end{align}
	
	\item \emph{Marginal inference} is the task of computing the marginal distribution for a subset of random variables
	$\RVSet{y}{} \subset \setOfNodes$, i.e., to compute 
	$\pmf{\RVSet{y}{}}(\mathbf{y})$ according to~\eqref{eq:marginals}. 
	Note that we will sometimes use the following shorthand notation for this double-summation where
	\begin{align}
		\pmf{\RVSet{y}{}}(\mathbf{y}) = \sum_{\mathbf{X}\backslash\RVSet{y}{}} \pmf{\setOfNodes}(\mathbf{x})  = \sum_{\mathbf{O}} \pmf{\setOfNodes}(\mathbf{x})    
		=\sum_{\RV{i} \in \mathbf{O}} \sum_{\RVval{i} \in \sampleSpace{X}} \pmf{\setOfNodes}(\mathbf{x}).
	\end{align}
	We will be particularly interested in the singleton marginals $\pmf{\RV{i}}(\RVval{i})$ for single random variables and the pairwise marginals $\pmf{\RV{i},\RV{j}}(\RVval{i},\RVval{j})$ for pairs of random variables.
	
	For binary models it is often more convenient to work with the expectations, which are the mean $\mean{i}$ (or magnetization) and the correlation $\correlation{i}{j}$, instead of considering the singleton marginals $\singleExact{i}(\RVval{i})$ and the pairwise marginals $\pairwiseExact{i}{j}(\RVval{i},\RVval{j})$ explicitly, where
	\begin{align}
		\mean{i} = \E(\RV{i}) = \singleExactVal{i}{1} - \singleExactVal{i}{-1},
		\label{eq:mean}
	\end{align}
	\begin{align}
		\correlation{i}{j} = \E (\RV{i}\RV{j}).
		\label{eq:correlation}
	\end{align}

	\item Another important problem is to evaluate the \emph{partition function} 
	\begin{align}
		\partitionFunction = \sum_{\allConfigurations} \prod_{\clique{i}\in\setOfCliques}^{} \potential{\clique{i}}(\RVvalSet[x]{\clique{i}}),
	\end{align}
	which is the normalization coefficient of the joint distribution in~\eqref{eq:joint_over_cliques}.
\end{itemize}
We will primarily focus on the problems of computing marginal distributions and evaluating the partition function in this thesis. 
Note that these two problems are in fact closely related  as the marginal distribution  $\pmf{\RV{i}}(\RVval{i})$
equals the ratio between a partial partition function $\partitionFunction(\RV{i})$ and the partition function:
\begin{align}
	\pmf{\RV{i}}(\RVval{i}) &= \sum_{\mathbf{X}\backslash\RV{i}} \pmf{\setOfNodes}(\mathbf{x}) = \sum_{\mathbf{X}\backslash\RV{i}} 
	\frac{1}{\partitionFunction} \prod_{\clique{i}\in\setOfCliques}^{} \potential{\clique{i}}(\RVvalSet[x]{\clique{i}}), \nonumber \\
	& = \frac{1}{\partitionFunction} \sum_{\mathbf{X}\backslash\RV{i}}   
	\prod_{\clique{i}\in\setOfCliques}^{} \potential{\clique{i}}(\RVvalSet[x]{\clique{i}}), \nonumber \\
	& = \frac{\sum\limits_{\mathbf{X}\backslash\RV{i}}    \prod\limits_{\clique{i}\in\setOfCliques}^{} \potential{\clique{i}}(\RVvalSet[x]{\clique{i}})}
	{\sum\limits_{\allConfigurations}    \prod\limits_{\clique{i}\in\setOfCliques}^{} \potential{\clique{i}}(\RVvalSet[x]{\clique{i}})}
	\definition \frac{\partitionFunction(\RV{i})}{\partitionFunction}.\label{eq:marginals:practice}
\end{align}
The computation of the marginals seems relatively straightforward according to~\eqref{eq:marginals:practice}.
Yet, there is one fundamental problem that prohibits the application of~\eqref{eq:marginals:practice} to practical problems, which essentially boils down to the overall number of variables involved.
Even if the joint distribution $ \pmf{\setOfNodes}(\mathbf{x})$ is known -- 
neglecting the fact that the memory complexity of storing $\pmf{\setOfNodes}(\mathbf{x})$ is 
exponential in the number of variables -- computing $ \pmf{\RV{i}}(\RVval{i})$ is problematic.
More specifically, consider a joint distribution  $ \pmf{\setOfNodes}(\mathbf{x})$ specified over $N$ random variables with $k=|\sampleSpace{x}|$ states. 
Then, the sum in the numerator of~\eqref{eq:marginals:practice} 
is evaluated $N-1$ times and goes over $k$ terms each time; 
i.e., in total computing a marginal distribution would require the summation over $k^{N-1}$ terms in total.

This drastically limits the problem size for which the marginals can be evaluated in practice.
On the positive side, however, we have already encountered the compact representation of probabilistic graphical models.
We will subsequently show how to utilize this representation and how this opens the door for efficient inference methods.
Before considering arbitrary models, we focus on models with certain graph structures that lend themselves to particularly elegant ways of performing inference. 
Ideally one should aim to exploit this representation beyond that and extend it to the central problems of inference.

\newsubsection{Exact Inference: Efficient Methods}{background:pgm:exact}
\newsubsubsection{Tree-Structured Graphs}{background:pgm:exact:tree}
Chains and trees posses Markov properties that, if exploited properly, give rise to efficient inference methods.
The statistical dependencies imposed by tree-structured models admit inference methods that,
instead of manipulating the joint distribution directly, recursively perform local computations;
this reduces the computational complexity immensely.
Various inference methods were independently introduced in different fields (cf. Section~\ref{sec:bp:intro} for a brief overview) that all rely on the very same principle to perform efficient inference.
The basic principle is to perform a set of local computations, often interpreted as messages between random variables;
this also explains the term \emph{message passing algorithms} that is often used to unite all those algorithms.
The following example demonstrates the underlying principles and highlights the efficiency of message passing algorithms.

\begin{example}[Exact Inference on a Tree]$ $\newline
	Consider the undirected graphical model $\ugm = (\graph,\setOfPotentials)$ depicted in Figure~\ref{fig:message_passing}.
	First, we express the joint distribution as the product over all pairwise clique potentials according to~\eqref{eq:joint_over_cliques} so that
	\begin{align}
		\pmf{\setOfNodes}(\RVvalSet{}) &=  \frac{1}{\partitionFunction} \prod_{\clique{i}\in\setOfCliques}^{} \potential{\clique{i}}(\RVvalSet[x]{\clique{i}}) \nonumber\\
		& = \frac{1}{\partitionFunction} \potentialCliquePairwise{1}{4} \potentialCliquePairwise{2}{4} \potentialCliquePairwise{3}{4} \potentialCliquePairwise{4}{5} \cdot\nonumber \\
		& \hspace{1cm}\potentialCliquePairwise{5}{6} \potentialCliquePairwise{6}{7}.
	\end{align}
	Note that the factorization in terms of pairwise cliques equals the factorization in terms of maximum cliques for tree-structured models for the lack of loops.
	
	Second, we compute the singleton marginals for one specific random variable, e.g., $\RV{5}$, by summing over all other variables such that   
	\begin{align}
		\pmf{\RV{5}}({\RVval{5}}) &= \frac{1}{\partitionFunction}  \sum_{x_1\in\sampleSpace{X}} \cdots \sum_{x_4\in\sampleSpace{X}} \sum_{x_6\in\sampleSpace{X}} \sum_{x_7\in\sampleSpace{X}}.
		\prod_{\clique{i}\in\setOfCliques}^{} \potential{\clique{i}}(\RVvalSet[x]{\clique{i}})
		\label{eq:ex:total_summation}
	\end{align}
	A closer look at the clique potentials reveals the benefits of reordering of the summations.
	The total amount of summations becomes much more manageable if we make use of the commutative and the distributive law and rewrite~\eqref{eq:ex:total_summation} according to
	
	\begin{align}
		\pmf{\RV{5}}({\RVval{5}}) =
		&\bigg(\! \sum_{x_4\in\sampleSpace{X}} \potentialCliquePairwise{4}{5} \nonumber \cdot\\
		&\underbrace{\hspace*{5pt} \underbrace{\Big(\! \sum_{x_1\in\sampleSpace{X}} \potentialCliquePairwise{1}{4} \!\Big)}_{\msg{1}{4}{}}
			\underbrace{\Big(\! \sum_{x_2\in\sampleSpace{X}} \potentialCliquePairwise{2}{4} \!\Big)}_{\msg{2}{4}{}}
			\underbrace{\Big(\! \sum_{x_3\in\sampleSpace{X}} \potentialCliquePairwise{3}{4} \! \Big)}_{\msg{3}{4}{}}\!\! \bigg)}_{\msg{4}{5}{}} \nonumber \cdot\\
		& \underbrace{\bigg(\! \sum_{x_6\in\sampleSpace{X}} \potentialCliquePairwise{5}{6}
			\underbrace{\Big(\! \sum_{x_7\in\sampleSpace{X}} \potentialCliquePairwise{6}{7}  \!\Big)}_{\msg{7}{6}{}} \!\!\bigg)}_{\msg{6}{5}{}}.
		\label{eq:ex:message_passing}
	\end{align}
	Note that we introduced the powerful notion of messages in~\eqref{eq:ex:message_passing}; 
	messages $\msg{i}{j}{}$ provide a compact notational convention and express the consequence of summing over the respective variable.
	We say that $\msg{i}{j}{}$ is passed along the edge from $\RV{i}$ to $\RV{j}$ (see Figure~\ref{fig:message_passing}).
	
	The desired marginal distribution $\pmf{\RV{5}}({\RVval{5}})$ is then given by the normalized product of all incoming messages. 
	We consequently realize from~\eqref{eq:ex:message_passing} that
	\begin{align}
		\pmf{\RV{5}}({\RVval{5}}) = \frac{1}{\partitionFunction}\msg{4}{5}{} \msg{6}{5}{}.
		\label{eq:ex:marginal}
	\end{align}

	\begin{center}
		\includegraphics[width=0.5\linewidth]{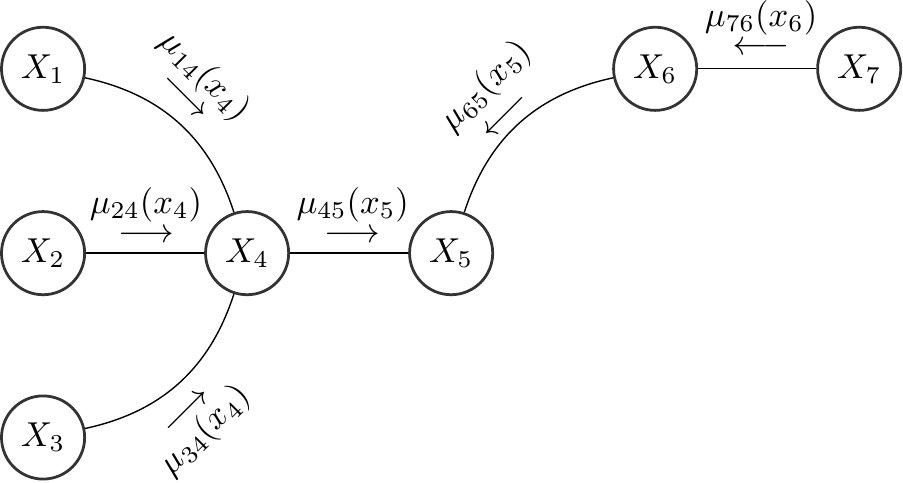}
		\captionof{figure}{Tree-structured model that illustrates the concept of message passing with the purpose of computing the marginal distribution $\pmf{\RV{5}}(\RVval{5})$. All messages of~\eqref{eq:ex:message_passing} are illustrated next to the associated edges.}
		\label{fig:message_passing}
	\end{center}
	Note how the messages $\msg{4}{5}{}$ and $\msg{6}{5}{}$ incorporate all messages of the sub-trees rooted in $\RV{4}$ and $\RV{6}$.
	Consequently, knowledge of only the two messages $\msg{4}{5}{}$ and $\msg{6}{5}{}$ is sufficient to compute the marginal $\pmf{\RV{5}}(\RVval{5})$ according to~\eqref{eq:ex:marginal}.
	
	The reordering of the summations reduces the overall amount of required summations notably.
	Comparison of~\eqref{eq:ex:total_summation} and~\eqref{eq:ex:message_passing} 
	reveals that the computational complexity reduces from $\bigO(|\sampleSpace{X}|^N)$ to $\bigO(|\sampleSpace{X}|^2)$.
	
	\label{ex:message_passing}
\end{example}

Remember how $\msg{4}{5}{}$ accounted for all messages that come from the subgraph induced by $\setOfNodes'=\{\RV{1},\RV{2},\RV{3}, \RV{4}\}$. 
This suggests a recursive rule for computing the messages without the need for explicitly rearranging the summations.
In particular, we can compute the message $\msg{i}{j}{}$ by taking the product of all incoming messages, except the one from $\RV{j}$, times the pairwise potential $\potentialCliquePairwise{i}{j}$ so that
\begin{align}
	\msg{i}{j}{} \definition \sum \limits_{\RVval{i} \in \sampleSpace{X}} \potentialCliquePairwise{i}{j} \prod \limits_{\RV[x]{k} \in \{\neighbors{i} \backslash \RV[x]{j}\}}  \msg{k}{i}{}.
	\label{eq:message_passing_tree}
\end{align}
The message $\msg{i}{j}{}$ is a vector over all states $\RVval{j} \in \sampleSpace{X}$. 
One can interpret this message as the belief of node $\RV{i}$ about the relative probabilities that $\RV{j}$ is in state $\RVval{j}$ given all the information\footnote{
	We are using the term information here only figuratively in the sense that all necessary summations for evaluating $\pmf{\RV{i}}(\RVval{i})$ are subsumed in $\msg{j}{i}{}$. This does not adhere to the formal definition of entropy as an information-theoretic measure.} 
that is available to $\RV{i}$, except from $\RV{j}$.

The notion of messages has a further advantage:
if we wish to compute the marginals $\pmf{\RV{i}}({\RVval{i}})$ for multiple nodes in the graphical model we can, instead of performing the procedure multiple times, reuse the already computed messages.
It suffices to pass the messages back and forth throughout the whole graph only once. 
For a chain this means propagating the messages along both directions.\footnote{In the particular case of a chain the messages actually equal the Chapman Kolmogorov equations.}
If we wish to apply this procedure to a tree we must first fix some ordering, i.e., pick one node as root, before then passing the messages upwards from the leaves and then pass the messages down again.

The representation of the summations in~\eqref{eq:ex:message_passing} in terms of messages captures an important property of tree-structured graphical models.
For every $\RV{i}\in\setOfNodes$ it holds that any pair of its neighbors $\{\RV{k},\RV{l}\}  \in \neighbors{i}$ is conditionally independent given $\RV{i}$, i.e., 
\begin{align}
	\pmf{\RV{k},\RV{l}|\RV{i}}(\RVval{k},\RVval{l}|\RVval{i}) = \pmf{\RV{k}|\RV{i}}(\RVval{k}|\RVval{i}) \pmf{\RV{l}|\RV{i}}(\RVval{l}|\RVval{i}).
	\label{eq:independence_tree_structured}
\end{align}
These statistical independence statements further impose one important property on the graph.
Namely, that removing the edge $\edge{i}{k}$ creates two disconnected subgraphs $\graph^{(i)} = \big(\setOfNodes^{(i)},\setOfEdges^{(i)}\big):\RV{i} \in\setOfNodes^{(i)}$ and $\graph^{(k)} = \big(\setOfNodes^{(k)},\setOfEdges^{(k)}\big):\RV{k} \in\setOfNodes^{(k)}$.
Note that the existence of any alternative path between $\graph^{(i)}$ and $\graph^{(k)}$, i.e., the existence of loops, would violate~\eqref{eq:independence_tree_structured}.

\subsubsection{Graphs with Loops}
So far we have witnessed the capabilities of message passing algorithms for chains and tree-structured models; one might ask how this concept generalizes and if it is able to cope with loops.
This is a critical question for two reasons:
efficient inference methods are essential when confronted with large and loopy models;
and many problems of practical relevance correspond to loopy models.
Every general-purpose inference method must, therefore, be able to cope with loops.
%

Unfortunately, this is not the case for message passing algorithms.
Loopy graphs violate the independence statements in~\eqref{eq:independence_tree_structured} and thus prevent a straightforward generalization of message passing to loopy graphs.


As discussed, excluding models with loops is not a viable option either.
Nonetheless, exploiting the Markov properties seemed promising.
Fortunately, we can have both.
That is, an algorithm exists that is both capable of accounting for loops and capable of exploiting the Markov properties.
These seemingly conflicting intentions are satisfied by modifying the graph and carefully constructing certain subgraphs until we finally end up with a tree again.

The failure of message passing is demonstrated on a loopy graph in Example~\ref{ex:junction_tree}.
Subsequently, we show how the graph needs to be modified to allow for efficient inference.

\newpage
\begin{example}[Exact Inference on Loopy Graphs]$ $\newline
	We create a loopy graph by adding the edge $\edge{5}{7}$ to the model of Example~\ref{ex:message_passing} (cf. Figure~\ref{fig:tree_with_loop}~(a)).
	Let us focus on the subgraph induced by $\setOfNodes'=\{\RV{5}$, $\RV{6}$, $\RV{7}\}$.
	We first observe that the conditional independence statements of~\eqref{eq:independence_tree_structured} are violated as 
	\begin{align}
		\pmf{\RV{5},\RV{7}|\RV{6}}(\RVval{5},\RVval{7}|\RVval{6}) \neq \pmf{\RV{5}|\RV{7}}(\RVval{5}|\RVval{7}) \pmf{\RV{6}|\RV{7}}(\RVval{6}|\RVval{7});
	\end{align}
	and second that information flows from  $\RV{7}$  towards $\RV{5}$ (and the remainder of the graph) via two different paths:
	once via  $\path_1 = (\{\RV{5}, \RV{7} \}, \{(7,5)\})$ and once via 
	$\path_2 = (\{\RV{5},\RV{6}, \RV{7} \}, \{\edge{7}{6}, \edge{6}{5}\} )$.
	As a result, the message $\msg{5}{7}{}$, if computed according to~\eqref{eq:message_passing_tree}, incorporates the belief stemming from $\RV{7}$ via one path and erroneously propagates it back to $\RV{7}$ via the other path.
	
	\begin{center}
		\includegraphics[width=0.75\linewidth]{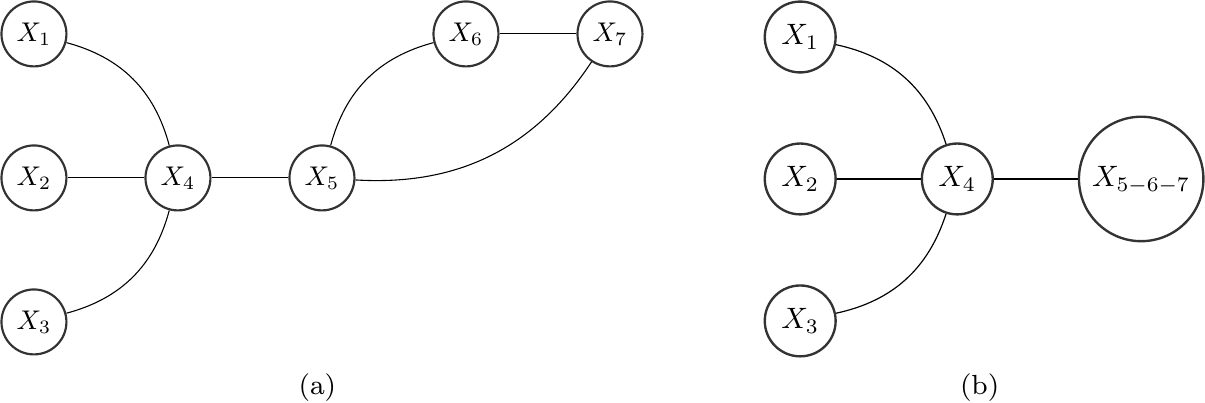}
		\captionof{figure}{  (a) Model from Example~\ref{ex:message_passing} with the additional edge $\edge{5}{7}$ that creates a loop.
			(b) Junction tree for the same model; note that variable grouping renders the graph tree-structured again.}
		\label{fig:tree_with_loop}
	\end{center}
	The graphical model needs to be tree-structured in order to satisfy the formal requirements for message passing to work;
	this, however, is not the case as seen in Figure~\ref{fig:tree_with_loop}(a).
	We can, however, form super-nodes, consisting of multiple variables, until no more loops are present in the graph to leverage the power of message passing.
	For our current example one can simply group the variables $\RV{5}$, $\RV{6}$, and $\RV{7}$ together into $\RV{5-6-7}$ and define the pairwise potential between $\RV{4}$ and this super-node according to
	\begin{align}
		\Phi_{C_{(4,5-6-7)}}(\RVval{4},\RVval{5},\RVval{6},\RVval{7}) = \potentialCliquePairwise{4}{5} \potentialCliquePairwise{5}{6} \potentialCliquePairwise{6}{7}\potentialCliquePairwise{5}{7}.
	\end{align}
	The resulting graph is depicted in Figure~\ref{fig:tree_with_loop} (b).
	Now that we have created a tree-structured model again (cf. Figure~\ref{fig:tree_with_loop}(b)), we can simply perform message passing on this modified graph and compute the marginals accordingly.
	Note how the new node $\RV{5-6-7}$ has an increased number of states according to $|\sampleSpace{x}|^{|\clique{5-6-7}|} = |\sampleSpace{x}|^{3}$ now.
	Despite the exponential growth of the state space, grouping variables together has its merits.
	In particular, it suggests an optimal way of rearranging the variables in the computation of the marginals that maintains the Markov properties despite the existence of loops.

	\label{ex:junction_tree}
\end{example}
We have seen in Example~\ref{ex:junction_tree} that performing inference on loopy graphs increases the computational complexity.
This is an immediate consequence of considering higher-order cliques in the construction of the \emph{junction tree} (i.e., the modified graph).
In general -- and for models with multiple intertwined loops in particular -- however, there is not just one but many possibilities of grouping variables together.
Consequently, one should construct the junction tree such that the largest clique is kept as small as possible.
If one wants to keep the complexity in check, it is thus of utmost importance to carefully construct the junction.\\

%

One might ask whether a modified graph (i.e., a clique- or junction-tree) without loops always exists -- and how it is constructed.
Indeed, for every graphical model such a junction tree exists and the construction of an optimal junction tree -- with as small cliques as possible -- involves three major steps.
These steps define the \emph{junction tree algorithm}~\cite{lauritzen-junction-tree} that we will review below.
We do not give an exhaustive introduction with all the details but rather aim to provide an overview that highlights the overall idea of the required operations.

First, note that every graph $\graph$ has an associated junction tree if and only if $\graph$ is triangulated~\cite{lauritzen-junction-tree}. 
A graph is triangulated (or chordal) if every cycle of length four or greater has a chord (i.e., an edge that joins two nodes of the cycle). 
Usually, a graph possesses multiple valid triangulated graphs, the choice of which has a major influence on the overall computational complexity~\cite[Section 10.4]{koller2009}.
Finding an optimal, that is a minimal, triangulation is NP-hard on its own but various heuristics exist that find reasonably good triangulation.
Most of these heuristics rely on a particular form of variable elimination, whereas it depends on the given graph which method performs best~\cite[Section 9.4.3.2]{koller2009}.

Second, the triangulated graph provides the basis for the construction of the junction tree.
The nodes in the junction tree correspond to maximal cliques in the triangulated graph;
if a variable is present in two cliques, the cliques are joined by an edge in the junction tree.
A triangulated graph usually admits multiple junction trees of varying sizes. 
The construction of the junction tree thus has a major influence on the overall efficiency.
Typically, one assigns a weight to each edge that corresponds to the number of variables included in both cliques and subsequently constructs a maximum spanning tree~\cite[Section 10.4.2]{koller2009}.

Finally, once the junction tree is constructed an efficient inference method, for example message passing, can be applied to the junction tree to yield the marginal distributions over all cliques.
The singleton marginals are then obtained by direct summation over a clique containing the desired variable.\\

The junction tree algorithm is straightforward in principle and provides a graph-based approach that exploits the factorization properties for loopy graphs.
Although the junction tree significantly reduces the complexity of exact inference, it is still of limited practical use.
This stems from working with the maximal cliques of the triangulated graph
and the exponential growth of the state space with the clique-size.
The applicability of the junction tree algorithm is consequently limited by the size of its largest clique~\cite{bodlaender}.

Nonetheless, there are two reasons that justify the introduction of the junction tree:
First, despite its limitation to problems with relatively small tree-width, the junction tree is efficient in some sense; 
in particular, no general exact inference method exists that is computationally more efficient than the junction tree~\cite[Section 8.4.6]{bishop}. The junction tree thus serves as the method of choice for estimating the ground truth when assessing approximate inference methods in the subsequent chapters.
Second, the junction tree algorithm reveals how the inherent properties of loopy graphs increase the complexity of exact inference; 
this further emphasizes the need for efficient approximation methods.

In the subsequent chapter, we will explicitly focus on loopy graphs and show how to tackle the related complexity issues.
We take different points of view on performing approximate inference, extend the underlying concept of message passing to more general graphs, and discuss how these methods cope with the existence of loops.
Ultimately, this search for efficient methods culminates in a wide range of available approximate inference methods.

  \emptydoublepage
\newchapter{Approximate Inference: Belief Propagation }{bp}
\renewcommand{\pwd}{approximate_inference}
\openingquote{The stars bend like slaves to laws not \\decreed for them by human intelligence,\\ but gleaned from them.}{Ludwig E. Boltzmann} 
Owing to belief propagation's specific relevance for this thesis, we present belief propagation (BP) in detail.
This includes the classical definition as a message passing algorithm in Section~\ref{sec:bp:preliminaries}, as well as the variational interpretation in Section~\ref{sec:bp:variational} that highlights the connection to methods  from statistical physics.
Moreover, we cast BP as a dynamical system and elaborate on the resulting implications in Section~\ref{sec:bp:map}.
The perspective of dynamical systems gives rise to a whole group of message passing algorithms for approximate inference, 
where every particular instance comes with its own pitfalls and limitations.
We briefly discuss some popular variants of BP and describe how each one enhances the performance over the standard implementation in Section~\ref{sec:intro:bp:improving:init}-~\ref{sec:intro:bp:improving:evaluation}.

Much of this chapter summarizes established textbook knowledge (cf.~\cite{koller2009,murphy1999loopy,mezard2009})
although we review some of the most recent developments as well.
The representation of BP as a dynamical system is rather obvious and has thus been considered multiple times (e.g., in~\cite{mooij2005properties,ruffer2010belief,tan2006belief});
here, we emphasize the added value of doing so and show how casting BP as a dynamical system provides a unifying framework. 
The presentation of BP variants in Section~\ref{sec:intro:bp:improving:update} contains results developed in collaboration with Michael Rath, Sebastian Tschiatschek, and Franz Pernkopf~\cite{knoll_scheduling}.

\newsection{Motivation}{bp:intro}
In Chapter~\ref{chp:background} we have seen how exact inference methods suffer from the existence of loops.
In fact, exact inference is NP-hard~\cite{cooper1990}, unless the probabilistic graphical model is appropriately restricted, which, however, would rule out many models of practical relevance.
This highlights the need for efficient approximate inference methods.
Even approximate inference, however, is NP-hard if a certain accuracy is required~\cite{roth1996,dagum1993}.
Note, however, that this pessimistic statement does not render approximate inference completely useless, but only indicates that specific models do exist for which achieving the desired accuracy is intractable.

If one wants to perform approximate inference, one can choose from a wide range of different methods.
Before one can make a well-informed choice and select an appropriate variant it is necessary to have a good understanding of a given method's capabilities and limitations.
Hence, developing this understanding is of central importance.
Knowledge of the limitations and failure modes further has the advantage of suggesting ways to improve upon.

We will focus on one specific class of approximate inference methods.
That is the class of message passing algorithms.
As discussed in Section~\ref{sec:background:pgm:exact} message passing algorithms efficiently perform exact inference on tree-structured models;
the existence of loops, however, has a severe effect on message passing and impedes the straightforward generalization (cf. Example~\ref{ex:junction_tree}).
Nonetheless, the elegance and simplicity of performing only local operations remain indisputable and it is tempting, therefore, to ignore the existence of loops, apply the same principles, and hope for a reasonable outcome.

Indeed, various approximate inference methods work according to this very principle.
In fact, because of their appealing simplicity, similar concepts were independently introduced multiple times in different fields:
Judea Pearl  introduced message passing algorithms in the machine learning- and statistics-community~\cite{pearl1988} for tree-structured graphs, terming it \emph{belief propagation}.
He already advocated the extension to loopy graphs as an approximate method~\cite[Section 4.4]{pearl1988} -- nowadays often referred to as \emph{loopy belief propagation} to emphasize the approximate nature of the method.\footnote{ 
	We will overload the terminology and, for the sake of brevity, always refer to it as belief propagation;
	irrespective whether the graph contains loops or not.}
In the information theory community, message passing algorithms were first introduced in the PhD Thesis of Robert Gallager~\cite{gallager} for decoding low-density-parity-check codes.
The capabilities of message passing -- known as the sum-product-algorithm -- were, however, largely overlooked until the success of Turbo codes~\cite{berrou1993turbo}.
Relying on the very same principles, this renewed interest finally put the original work into perspective and ultimately led to its well-deserved recognition.
Very similar concepts are also applied with success in the signal-processing community as in the Kalman filter~\cite{kalman} for state estimation, or in the Viterbi algorithm~\cite{viterbi1967error} for hidden Markov models.
Probably the first explicit application of local operations for approximating the global behavior has to be attributed to the physics community:
Hans Bethe and Rudolf Peierls introduced this concept to the field of statistical physics with the aim of understanding the behavior of large, otherwise incomprehensible, interacting systems~\cite{bethe,peierls}.
This approximation is known as the Bethe-Peierls approximation or the Cavity method.\footnote{ 
	Nowadays, the connections between the coding- and the artificial intelligence community~\cite{kschischang2001factor} as well as the connections to the physics community~\cite{yedidia2005} are well established and may seem rather obvious in hindsight. Looking at the original literature, however, it becomes clear that, because of the different formalism, these findings were indeed rather surprising at first.}

\newsection{Preliminaries}{bp:preliminaries}
Here in this section, we will finally define belief propagation.
Let us consider a binary pairwise graphical model $\ugm = (\graph,\setOfPotentials)$.
First, we define messages that are passed along the edges, where
the message from  $\RV{i}$ to $\RV{j}$ is denoted by $\msg[n]{i}{j}{}$ with $n\in\INTEGER$ numbering the current iteration.\footnote{ 
	Two neighboring nodes are joined by a single edge by definition. Note, that information has to flow into both directions and that two distinct messages are passed along opposing directions over every edge.}
The messages are updated according to the recursive rule:
\begin{align}
	\msg[n+1]{i}{j}{} \propto \sum \limits_{\RVval{i} \in \sampleSpace{X}} \pairwiseShort{i}{j} \localShort{i} \prod \limits_{\RV[x]{k} \in \{\neighbors{i} \backslash \RV[x]{j}\}}  \msg[n]{k}{i}{}.
	\label{eq:update}
\end{align}
To compute the messages, BP collects all messages sent to $\RV{i}$, except from $\RV{j}$ and multiplies this product with the local potential $ \localShort{i}$ and the pairwise potential $\pairwiseShort{i}{j}$. Finally, the sum over all states $\RVval{i}\in\sampleSpace{X}$ is sent.
In practice, the messages require some form of normalization~\cite{ihler2005loopy}; we will normalize the messages by $\msgNorm[n]{i}{j} \in \REALPos$ so that $\sum_{\RVval{j} \in \sampleSpace{X}} \msg[n]{i}{j}{} = 1$, which gives the messages a probabilistic interpretation. 
Note that all messages are consequently restricted to $\msg{i}{j}{} \in [0,1]$ and the application of the update rule ~\eqref{eq:update} does not change this.
\begin{lm}\label{lm:msg}
	Normalized messages, sent from node $X_i$ to $X_j$ over $\edge{i}{j}\in\setOfEdges$, represent probabilities and remain so under successive application of the BP update equation -- provided all messages are initialized to be positive.
\end{lm}
\begin{proof}
	Positive potentials in \eqref{eq:update} guarantee that all messages remain positive at every iteration.
	Consequently, a normalization term $\msgNorm[n]{i}{j}$ exists so that $\sum\limits_{\RVval{j} \in \sampleSpace{X}} \msg[n+1]{i}{j}{} = 1$.
\end{proof}

The update equation in~\eqref{eq:update} closely resembles the recursive definition of the messages in Section~\ref{sec:background:pgm:exact}, except for one fundamental difference.
For tree-structured graphs the recursive definition of the messages results from simply reorganizing the summations.
Consequently, the messages converge once all messages were passed up and down the tree.
For loopy graphs, however, such a reorganization is not possible -- since the recursive definition of a message, say $\msg[]{i}{j}{}$, contains the message $\msg[]{j}{i}{}$ itself.
For the very same reason, the notion of passing the messages up and down the tree is simply not possible anymore. 
Instead, one can neglect the existence of loops, update the message according to the recursive definition, accept the approximate nature of this approach, and hope for the messages to converge to the   \emph{fixed point messages} $\fpMsg{i}{j}{}$.


After convergence of BP, one can approximate the marginals similar as for the tree-structured models.
Note that the marginals were defined over any possible subset of random variables in~\eqref{eq:marginals}, though we will only consider the singleton marginals $\marginals{\RV{i}}(\RVval{i})$ and the pairwise marginals $\marginals{\RV{i},\RV{j}}(\RVval{i}, \RVval{j})$ in this thesis. 
In general, the marginals over any subset of random variables are computed by the product of incoming messages times the associated potentials.
Accordingly, we approximate the singleton and pairwise marginals by
%
\begin{align}
	\pmfApprox{\RV{i}}(\RVval{i}) = \frac{1}{Z_i} \localShort{i} \prod_{\RV{k} \in \neighbors{i}} \fpMsg{k}{i}{},
	\label{eq:marginals:single}
\end{align}
\begin{align}
	\pmfApprox{\RV{i},\RV{j}}(\RVval{i}, \RVval{j}) = \frac{1}{Z_{ij}} \localShort{i} \localShort{j} \pairwiseShort{i}{j}  
	\prod_{\RV{k} \in \{\neighbors{i} \backslash \RV{j}\}}  \fpMsg{k}{i}{}   
	\prod_{\RV{l} \in \{\neighbors{j} \backslash \RV{i}\}}  \fpMsg{l}{j}{},
	\label{eq:marginals:pw}
\end{align}
where  $Z_i,Z_{ij} \in \REALPos$ guarantee that all probabilities sum to one. 
We denote the set of all approximated singleton- and pairwise marginals by
\begin{align}
	\pseudomarginals \definition 
	\{\pmfApprox{\RV{i}}(\RVval{i}),\pmfApprox{\RV{i},\RV{j}} (\RVval{i},\RVval{j}): 
	\RV{i} \in \setOfNodes, \edge{i}{j} \in \setOfEdges\},
	\label{eq:pseudomarginals}
\end{align}
and refer to $\pseudomarginals$ as the pseudomarginals.\footnote{
	Note that the pseudomarginals are also often called \emph{beliefs}, hence the name belief propagation.}   
This naming-convention should highlight the fact that, in general, BP only approximates the marginals.
In fact, one may end up with pseudomarginals that do not correspond to any valid distribution at all.

\begin{example}[Unrealizable Pseudomarginals]$ $\newline
	Let us consider the pairwise graphical model in Figure~\ref{fig:ex:frustration} where $\setOfNodes{}$ is a set of three binary random variables with $\RVval{i} \in \sampleSpace{x}=\{-1,1\}$.
	
	The model is frustrated and contains two attractive edges with $\coupling{i}{j}=1$ (depicted by solid lines) and one repulsive edge with $\coupling{i}{j}=-1$ (depicted by the dashed line).

	Let us specify the pairwise potentials so that 
	\begin{align}
		\pairwise{x}{1}{2} &= \begin{bmatrix} 
			\exp(J) & \exp(-J)  \nonumber\\
			\exp(-J) & \exp(J)
		\end{bmatrix},\\
		\pairwise{x}{2}{3} &= \begin{bmatrix} 
			\exp(J) & \exp(-J) \nonumber\\
			\exp(-J) & \exp(J)
		\end{bmatrix},\\
		\pairwise{x}{3}{1} &= \begin{bmatrix} 
			\exp(-J) & \exp(J) \nonumber\\
			\exp(J) & \exp(-J)
		\end{bmatrix}.
	\end{align}
	
	\begin{center}
		\includegraphics[width=0.2\textwidth]{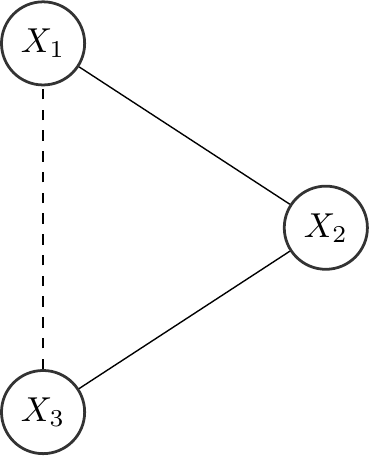}
		\captionof{figure}{Example of a frustrated model. The solid edges are attractive and the dashed edge is repulsive; note that the product of couplings along all edges (cf.~\eqref{eq:frustrations}) equates to a negative number, which characterizes the existence of frustrations.
			While $\RV{2}$ tries to pull $\RV{1}$ and $\RV{3}$ towards the same state, both variables are pushed apart by each other because of the repulsive edge.}
		\label{fig:ex:frustration}
	\end{center}
	
	We further set all local potentials to identical values and choose $\field{}=0$  so that
	\begin{align}
		\local{x}{i} = \begin{bmatrix} 
			\exp(\theta) \\
			\exp(-\theta)
		\end{bmatrix}
		= \begin{bmatrix} 
			1 \\
			1
		\end{bmatrix}.
	\end{align}
	Belief propagation converges to a unique fixed point defined by $\fpMsg{i}{j} = 0.5 $ for all $x_j\in \sampleSpace{X}$. 
	It is well-established that this particular fixed point exists and does not depend on the pairwise potentials if $\field{}=0$ (cf. Section~\ref{sec:stabilityBP:empirical_analysis:vanishing}).
	The fact that these messages $\fpMsg{i}{j} = 0.5 $ constitute a fixed point becomes immediately obvious if we plug in all potentials into the update equation~\eqref{eq:update}.
	
	The pairwise marginals are subsequently approximated by $\pairwiseApprox{i}{j}(x_i,x_j)$; in particular we have
	\begin{align}
		\pairwiseApprox{1}{2}(\RVval{1},\RVval{2}) &= \begin{bmatrix} 
			0.44 & 0.06 \\
			0.06 & 0.44
		\end{bmatrix}\label{eq:unrealizable:pairwise:1},\\
		\pairwiseApprox{2}{3}(\RVval{2},\RVval{3}) &= \begin{bmatrix} 
			0.44 & 0.06 \\
			0.06 & 0.44
		\end{bmatrix}\label{eq:unrealizable:pairwise:2},\\
		\pairwiseApprox{3}{1}(\RVval{3},\RVval{1}) &= \begin{bmatrix} 
			0.06 & 0.44 \\
			0.44 & 0.06
		\end{bmatrix}.\label{eq:unrealizable:pairwise:3}
	\end{align}
	Now we will proof by contradiction that the pairwise marginals from above are unrealizable;
	i.e., one cannot specify a joint distribution over three random variables that has its pairwise marginals correspond to~\eqref{eq:unrealizable:pairwise:1}~-~\eqref{eq:unrealizable:pairwise:3}.
	To make this more precise, note that by~\eqref{eq:unrealizable:pairwise:1} and by the sum-rule we have
	$\pmfApprox{\setOfNodes}(-1,+1,-1) + \pmfApprox{\setOfNodes}(-1,+1,+1) = 0.06$, so that the joint probability for both configurations must satisfy $\pmfApprox{\setOfNodes}(-1,+1,-1) \stackrel{!}{\leq}{0.06}$ and $\pmfApprox{\setOfNodes}(-1,+1,+1) \stackrel{!}{\leq} 0.06$.
	
	Accordingly,~\eqref{eq:unrealizable:pairwise:1}~-~\eqref{eq:unrealizable:pairwise:3} reveal the joint probabilities for all configurations $\RVvalSet{} \in |\sampleSpace{X}|^N$.
	In this example all joint probabilities must satisfy
	\begin{align}
		\pmfApprox{\setOfNodes}(\RVvalSet{}) \stackrel{!}{\leq} 0.06.
	\end{align}
	The total probability is thus bounded from above according to
	\begin{align}
		\sum_{\RVvalSet{}}\pmfApprox{\setOfNodes}(\RVvalSet{}) \leq 2^N\cdot 0.06 = 0.48,
	\end{align}
	which violates the fundamental rules of probability.
	Consequently, the pseudomarginals obtained by BP cannot belong to a valid joint distribution.
	\label{ex:frustrations}
\end{example}

The above example demonstrated why the pseudomarginals may fail to represent a valid distribution.
Still, unrealizable pseudomarginals need not pose a serious problem.
In fact, we usually resort to BP whenever 
evaluating the exact marginals is not possible, so that approximating the marginals reasonably well is often sufficient  (even if they are unrealizable).
The main motivation for considering BP was the hope for its efficient nature to carry over to loopy graphs.
This is unfortunately not always the case and, in the presence of loops, BP -- besides only approximating the marginals --  may  fail to converge altogether.

Nonetheless, despite all these shortcomings, BP often works surprisingly well, even for models that contain many loops.
So far we have depicted BP as a pure heuristic without any theoretical motivation.
This is unsatisfactory, 
and the situation was exactly like this for many years until a close relation to variational methods was revealed.
The variational perspective of BP is appealing for several reasons:
it relates to a similar problem, well studied in statistical physics;
it provides a more lucid justification of BP;
and it explains the approximate nature of BP.
It is for all those reasons that the variational perspective spurred much of the research on BP in the last decade and significantly improved the theoretical understanding of BP. 

\newsection{Variational Interpretation of BP -- A Physics Perspective}{bp:variational}
Ultimately, the aim of approximate inference is to approximate some unknown complex distribution in a way that simplifies answering probabilistic queries, as for example estimating the marginals.
One particularly powerful way of approximating distributions is available in the form of variational methods.
The variational free energy approach~\cite{jordan1999introduction} was first formalized in the context of statistical mechanics by Richard Feynman~\cite[Section 3.4]{feynman}.
Although variational methods originate from the physics literature, they are not limited to problems arising in this context and have become a reliable tool in the machine learning community.
For our purpose, we focus on the Bethe free energy $\FB$ that bridges the gap from BP to variational methods.

We show how the Bethe free energy emerges quite naturally if we approximate $\joint$ according to the variational free energy principle and follow the excellent deductive presentation of~\cite{yedidia2005} or~\cite[Chapter~33]{mackay2003}.
An exhaustive treatment of variational representations in the context of probabilistic graphical models that elucidates how a variety of approximate inference methods can be understood in terms of their variational representations is presented in the excellent review article~\cite{wainwright2008graphical}.
Despite its origin, we refrain, however, from discussing the insights of specific relevance for problems in physics and refer the interested reader to some of the many well-written books~\cite{georgii, huang} for a more in-depth treatment of the Bethe approximation from the physics perspective.\\

When approximating a joint distribution $\joint$, we first introduce a \emph{trial distribution} $\jointApprox$.
The aim is to manipulate the trial distribution so that the mismatch between both distributions is reduced.
Therefore, we quantify the mismatch by the Kullback-Leibler divergence, defined according to 
\begin{align}
	D(\pmfApprox{\setOfNodes{}}||\pmf{\setOfNodes{}}) \definition \sum_{\allConfigurations} \pmfApprox{\setOfNodes{}}(\RVvalSet{}) \log \frac{\pmfApprox{\setOfNodes{}}(\RVvalSet{})}{\pmf{\setOfNodes{}}(\RVvalSet{})}, \label{eq:KullbackLeibler}
\end{align}
where, with slight abuse of notation, $\log(\cdot)$ corresponds to  the natural logarithm. 
Note that the Kullback-Leibler is a non-symmetric function and not a proper distance measure. 
It does, however, satisfy the Gibbs inequality, i.e., it satisfies $D(\pmfApprox{\setOfNodes{}}||\pmf{\setOfNodes{}}) \geq 0$ with equality if and only if  $\pmfApprox{\setOfNodes{}}(\RVvalSet{}) = \pmf{\setOfNodes{}}(\RVvalSet{})$.\footnote{ 
	Different proofs can be found in the literature for this property of the Kullback-Leibler divergence that, e.g., rely on the log-sum or the Jensen's inequality~\cite[p.7]{mezard2009}.} 

Let $\joint$ belong to an exponential family, defined by~\eqref{eq:exponential}, and let $\pmfApprox{\setOfNodes{}}(\RVvalSet{})$ be an arbitrary trial distribution.
Then we can rewrite the Kullback-Leibler divergence according to 
%
\begin{align}
	D(\pmfApprox{\setOfNodes{}}||\pmf{\setOfNodes{}}) &= \sum_{\allConfigurations} \pmfApprox{\setOfNodes{}}(\RVvalSet{})  \log \pmfApprox{\setOfNodes{}}(\RVvalSet{}) + 
	\sum_{\allConfigurations} \pmfApprox{\setOfNodes{}}(\RVvalSet{})  \energy{\RVvalSet{}}+ 
	\sum_{\allConfigurations} \pmfApprox{\setOfNodes{}}(\RVvalSet{})  \log \partitionFunction.
	\label{eq:kullback_leibler_2}
\end{align}
Overloading our notation, we further define the average energy $\energy{\pmfApprox{\setOfNodes{}}}\definition \E_{\setOfNodes}(E(\RVvalSet{}))$ and the entropy $\entropy{\pmfApprox{\setOfNodes{}}}\definition \E_{\setOfNodes{}} (-\log(\pmfApprox{\setOfNodes}(\RVvalSet{})) )$ so that 
\begin{align}
	\energy{\pmfApprox{\setOfNodes{}}} & = \sum_{\allConfigurations} \pmfApprox{\setOfNodes{}}(\RVvalSet{})E(\RVvalSet{}) \label{eq:gibbs_energy}\\
	\entropy{\pmfApprox{\setOfNodes{}}}&= - \sum_{\allConfigurations} \pmfApprox{\setOfNodes{}}(\RVvalSet{}) \log \pmfApprox{\setOfNodes{}} (\RVvalSet{}). \label{eq:gibbs_entropy}
\end{align}
We can now plug the average energy and the entropy into the definition of the Kullback-Leibler divergence and express~\eqref{eq:kullback_leibler_2} in a simplified way according to
\begin{align}
	D(\pmfApprox{\setOfNodes{}}||\pmf{\setOfNodes{}}) & = \log \partitionFunction + \energy{\pmfApprox{\setOfNodes{}}} - \entropy{\pmfApprox{\setOfNodes{}}}.
	\label{eq:kullback_leibler_3}
\end{align}
Let us define two important quantities, namely the Helmholtz free energy (or the negative log-partition function)
\begin{align}
	\mathcal{F_H} \definition -\log \partitionFunction,
	\label{eq:helmholtz}
\end{align}
and the Gibbs free energy, which is given by
\begin{align}
	\FGibbs{\pmfApprox{\setOfNodes{}}} &\definition \energy{\pmfApprox{\setOfNodes{}}} - \entropy{\pmfApprox{\setOfNodes{}}}.
	\label{eq:f_gibbs}
\end{align}
Then, we can finally rearrange terms, plug the Gibbs free energy into~\eqref{eq:kullback_leibler_3} and express it in a sensible way that suggests how to estimate $\mathcal{F_H}$ (and thus the partition function $\partitionFunction$. 
\begin{align}
	\FGibbs{\pmfApprox{\setOfNodes{}}}  &= \mathcal{F_H} + D(\pmfApprox{\setOfNodes{}}||\pmf{\setOfNodes{}}).
\end{align}

The properties of the Kullback-Leibler divergence imply that $\FGibbs{} \geq \mathcal{F_H}$ with equality if and only if $\pmfApprox{\setOfNodes{}}(\RVvalSet{}) = \pmf{\setOfNodes{}}(\RVvalSet{})$. 
This is a pleasant formalism that opens the door for variational approaches; 
in particular minimizing $\FGibbs{\pmfApprox{\setOfNodes{}}}$ provides $\mathcal{F_H}$ and additionally, at its minimum, where the Kullback-Leibler divergence vanishes, the trial distribution converges to the exact one $\pmf{\setOfNodes{}}(\RVvalSet{})$.
More formally, for the set of all valid trial distribution $\pmfApprox{\setOfNodes}(\RVvalSet{})$ over $\sampleSpace{X}^N$ we have
\begin{align}
	\pmf{\setOfNodes}(\RVvalSet{}) = \argmin_{\pmfApprox{\setOfNodes{}}(\RVvalSet{})} \FGibbs{\pmfApprox{\setOfNodes{}}}.
	\label{eq:minimization_fg}
\end{align}
Note that this constitutes the central problems of inference, i.e., evaluating the partition function and computing marginal distributions (cf. Section~\ref{sec:background:inference}).
The formulation as an optimization problem provides an elegant way of solving those inference problems but -- in its current form -- is of limited relevance. 
Evaluation of the Gibbs free energy alone becomes infeasible with increasing model size as  both terms in the Gibbs free energy require a summation over exponentially many terms (cf.\eqref{eq:gibbs_energy} and~\eqref{eq:gibbs_entropy}).
The importance of~\eqref{eq:minimization_fg}, however, lies in the fact that 
a whole family of tractable approximation methods results from restricting the possible choices of trial distributions.\\

One popular approximation method that follows this idea  is the \emph{mean field} method.
In its simplest form, one assumes that the trial distribution is defined over independent random variables.
Consequently, the joint distribution factorizes into the product of the singleton marginals according to
\begin{align}
	\pmfApprox{\setOfNodes{}}(\RVvalSet{}) = \prod_{\RV{i} \in \setOfNodes} \pmfApprox{\RV{i}}(\RVval{i}).
	\label{eq:mean_field}
\end{align}
Restricting ourselves to trial distributions of the form~\eqref{eq:mean_field}, the minimization of the Gibbs free energy suddenly becomes tractable.
The minimizer of~\eqref{eq:minimization_fg} then approximates the marginals.
Some problems exist for which the mean field approximation even becomes asymptotically exact, as for example for infinite size Ising grid graphs with $\averageDegree[\graph]=\infty$~\cite[p.80]{mezard2009}.
The approximation tends to require a high average degree to work well.
In general, however, this is not the case and the mean field method often provides poor approximations.
This is mainly because of statistical dependencies that are present in $\joint$, which the mean field method fails to account for.

It is an obvious next step to enhance the approximation quality by accounting for the correlations between pairs of random variables.
One prevalent way comes in the form of the Bethe approximation that is of particular relevance for this work.
Not only does the Bethe free energy $\FB$ restrict the class of trial distributions but it inherently approximates the Gibbs free energy as well.
The latter is a result of the average energy $\EB(\pseudomarginals)$ and the Bethe entropy $\SB(\pseudomarginals)$ that are evaluated over the pseudomarginals $\pseudomarginals$ instead of the joint distribution according to
\begin{align}
	\EB(\pseudomarginals) \definition& -  \sum_{\setOfNodes_m \in \{\RV{i}\in \setOfNodes\} \cup\{\RV{i},\RV{j} :\edge{i}{j}\in\setOfEdges\}
	} \pmfApprox{\setOfNodes_m}(\RVvalSet{m})\cdot \ln \potential{\setOfNodes_m}(\RVvalSet{m}) \nonumber\\
	=& - \sum_{\RV{i}\in\setOfNodes}\sum_{x_i\in\sampleSpace{x}}\pmfApprox{\RV{i}}(x_i)\ln\localShort{i}  -\sum_{\edge{i}{j} \in \setOfEdges} \sum_{x_i,x_j\in\sampleSpace{X}^2}\ \pmfApprox{\RV{i},\RV{j}}(x_i,x_j) \ln {\pairwiseShort{i}{j}}  \label{eq:energy} \\
	\SB(\pseudomarginals) \definition& - \sum_{\edge{i}{j} \in \setOfEdges} \sum_{x_i,x_j\in\sampleSpace{x}^2}\pmfApprox{\RV{i},\RV{j}}(x_i,x_j) \ln \pairwiseApprox{i}{j}(x_i,x_j)\nonumber\\ 
	&+ \sum_{\RV{i}\in\setOfNodes}\big(\nodeDegree{i}-1\big)  \sum_{x_i\in\sampleSpace{X}^2}\! \pmfApprox{\RV{i}}(x_i)\ln\pmfApprox{\RV{i}}(x_i).  \label{eq:entropy}
\end{align}
This subsequently defines the Bethe free energy in accordance with~\eqref{eq:f_gibbs} so that
\begin{align}
	\FB(\pseudomarginals) =& \EB(\pseudomarginals) - \SB(\pseudomarginals) \\
	=& \sum_{\edge{i}{j} \in \setOfEdges} \sum_{x_i,x_j}\pmfApprox{\RV{i},\RV{j}}(x_i,x_j) \ln \frac{\pairwiseApprox{i}{j}(x_i,x_j)}{\pairwiseShort{i}{j}}
	- \sum_{\RV{i}}\sum_{x_i}\pmfApprox{\RV{i}}(x_i)\ln\localShort{i}\nonumber \\
	&-\sum_{\RV{i}}\big(\nodeDegree{i}-1\big)  \sum_{x_i}\! \pmfApprox{\RV{i}}(x_i)\ln\pmfApprox{\RV{i}}(x_i). 
	\label{eq:f_bethe}
\end{align}

In comparison to the Gibbs free energy, the Bethe free energy reduces the computational burden drastically;
solely because $\FB$ is evaluated over the pseudomarginals, i.e., the singleton and pairwise marginals, instead of the full joint distribution.
Just as the minimum of $\FGibbs{}$ provides the partition function and the marginals, one would hope that -- as $\FB$ approximate $\FGibbs{}$ -- minimizing $\FB$ will provide approximations to the partition function and the marginals.

In order to minimize $\FGibbs{}$ efficiently, we had to restrict the choice of trial distributions.
A similar restriction is required for minimizing $\FB$;
the most obvious class of trial distributions are realizable pseudomarginals that adhere to the sum-rule of probability.
This constraints the pseudomarginals so that they correspond to some valid distribution $\joint$; 
we refer to this set as the marginal polytope.
\newpage
\begin{defn}[Marginal Polytope]
	The marginal polytope is the set of all pseudomarginals that are jointly realizable by a valid joint distribution $\joint$; i.e., 
	\begin{align}
		\MPolytope(\graph) \definition  \bigg\{ \pmfApprox{\RV{i}} (x_i), \pmfApprox{\RV{i},\RV{j}}(x_i,x_j)\hspace{5pt} : \hspace{5pt}  
		&\pmfApprox{\RV{i}}(x_i)  = \sum_{\setOfNodes \backslash \RV{i}}  \joint : \RV{i} \in \setOfNodes, \nonumber\\
		&\pmfApprox{\RV{i},\RV{j}}(x_i,x_j) = \sum_{\setOfNodes \backslash \{\RV{i},\RV{j}\}}\joint :\edge{i}{j} \in \setOfEdges \bigg\}.
	\end{align}
	Note that the number of constraints in the marginal polytope depends on the structure of the graph as 
	the singleton marginals and the pairwise marginals are only defined over the set of all nodes and the set of all edges respectively.
	We will refrain from making the dependence on the graph explicit, however, and only refer to the marginal polytope as $\MPolytope$.
\end{defn}

The constrained minimization of the Bethe free energy, i.e., $\min_{\MPolytope} \FB(\pseudomarginals)$, seems like a standard optimization problem.
It is, but the huge amount of constraints in $\MPolytope$ impedes the optimization in practice~\cite[Section 22.3]{murphy}.
This issue is usually dealt with by relaxing the set of constraints.
Accordingly we may relax the requirement from \emph{globally realizable marginals} to \emph{locally consistent marginals}.
This makes the dependence on $\joint$ superfluous and only requires two properties to be satisfied: 
all singleton and pairwise marginals have to be properly normalized, and the singleton marginals must be consistent with the pairwise marginals.
These constraints then define the local polytope $\LPolytope$.
\begin{defn} [Local Polytope]
	The local polytope is the set of all pseudomarginals that are locally consistent, i.e., 
	\begin{align}
		\LPolytope(\graph) = \bigg\{ \pmfApprox{\RV{i}} (x_i), \pmfApprox{\RV{i},\RV{j}}(x_i,x_j)  :  
		& \sum_{x_i\in\sampleSpace{x}} \pmfApprox{\RV{i}} (x_i)  = 1, \nonumber\\
		& \pmfApprox{\RV{i}}(x_i) =\sum_{x_j\in\sampleSpace{x}} \pmfApprox{\RV{i},\RV{j}} (x_i,x_j) \bigg\}
		\label{eq:local}
	\end{align}
	As for the marginal polytope, we will also refrain from making the dependence on the graph structure explicit and only refer to the local polytope by $\LPolytope$.
	Further note that $\LPolytope$ contains fewer constraints than $\MPolytope$ and thus provides an outer bound on the marginal polytope (with equality for tree-structured models (cf.~\cite[Prop. 4.1]{wainwright2008graphical}).
\end{defn}

It is worth pointing out a few differences that set the Bethe approximation apart from other variational methods.
On the one hand, we would like to stress that the Bethe approximation approximates the marginals well in many cases.
On the other hand, there are some conceptional shortcomings:
First, the minimization must be performed over the local polytope;
the obtained marginals are consequently only locally consistent and, in general, do not belong to a valid joint distribution.
Second, except for certain sub-classes of graphical models, the Bethe free energy does not upper bound (nor lower bound) the original objective.
This seems to fundamentally jeopardize the main assumption that the minimum of $\FB(\pseudomarginals)$ should be somewhat close to the minimum of $\FGibbs{\pmfApprox{\setOfNodes{}}}$.

To summarize, the Bethe approximation has the potential to work well and may thus serve as an efficient approximation method
but, at the same time, its properties indicate that performance guarantees will often be hard to come by.
This highlights the need for developing a good understanding of the Bethe approximation and its shortcomings.
The Bethe approximation is well-studied in statistical physics.
Note, however, that the models studied in physics differ from the ones studied in computer science in one important aspect.
Whereas physicists usually study the behavior in the thermodynamic limit, i.e., for graphs of infinite size,\footnote{ 
	This allows one to take limits and admits an analytical treatment in some cases.}
in computer science, we are confronted with models of finite size.
Although we expect a similar behavior to a certain degree, we will see that finite-size effects will often have a notable effect on the properties of $\FB$.

\newsubsection{Energy Landscape of the Bethe Free Energy}{bp:variational:energy_landscape}
Let us now focus on the practical aspects of minimizing the Bethe free energy directly.
In particular, we will take a closer look at how to approximate the marginals and the partition function in this setting.
Similar as the partition function corresponds to  the Helmholtz free energy (cf.~\eqref{eq:helmholtz}), we 
introduce the \emph{Bethe partition function} $\partitionBethe$ that corresponds to the Bethe free energy according to
\begin{align}
	\partitionBethe = \exp(-\FB).
	\label{eq:bethe_partition_fb}
\end{align}
Just as $\FB$ approximates $\mathcal{F_H}$, $\partitionBethe$ approximates $\partitionFunction$.
Note that there is literally no difference between both quantities and instead of minimizing $\FB$ one could equivalently maximize $\partitionBethe$.
We stick to the notion of minimizing the Bethe free energy, primarily for historical reasons.

Let us recap the two key-ingredients of the Bethe approximation again:
we must replace the Gibbs free energy by the Bethe free energy and 
replace the constraints of the marginal polytope by the ones of the local polytope.
This gives us the global minimum of the constrained Bethe free energy according to
\begin{align}
	\FBGlobalMin = \min_{\LPolytope} \FB(\pseudomarginals),
	\label{eq:minimization_fb}
\end{align}
where the pseudomarginals are specified by the associated minimizer
\begin{align}
	\pseudomarginals^* = \argmin_{\LPolytope} \FB(\pseudomarginals).
	\label{eq:pseudomarginals_global}
\end{align}
Although the relaxation to $\LPolytope$ reduces the complexity, it concurrently alters the energy landscape with drastic consequences.
As opposed to the convex Gibbs free energy, the Bethe free energy is generally a non-convex function (cf. Example~\ref{ex:energy_landscape}).

We will denote all stationary points of the constrained Bethe free energy by
\begin{align}
	\FBStationary \in \big\{\FB: \gradient \FB(\pseudomarginals) =0    \big\}.
\end{align}
Finally, local minima of the constrained Bethe free energy are particularly relevant.
Let us express the Hessian  of the Bethe free energy by $\Hessian{\FB(\pseudomarginals)}$. Then we denote local minima explicitly by
\begin{align}
	\FBLocalMin{m} \in \big\{\FB: \gradient \FB(\pseudomarginals) =0, \Hessian{\FB(\pseudomarginals)} \quad \text{is positive definite}\big\} .
\end{align}
Note, that the different types of stationary points consequently satisfy $\FBGlobalMin  \subseteq \{\FBLocalMin{m} \} \subseteq \{\FBStationary \}$.

\newsubsection{The Bethe Approximation and Belief Propagation}{bp:variational:correspondence}
One might ask where the added value of the variational approach lies exactly?
The main reason is the immediate connection between the stationary points of the Bethe free energy and the fixed points of BP.
The rich history of studying the Bethe free energy thus benefits our understanding of BP directly, as the theoretical properties carry over through this relationship.

\newsubsubsection{Explicit Connection}{bp:variational:correspondence:explicit}
There is a fundamental connection between the stationary points of the Bethe free energy and the BP fixed points $\fpSetOfMessages$ (and the associated pseudomarginals $\pseudomarginals^{\circ}$).
In a nutshell, the stationary points are in a one-to-one correspondence with the fixed points.
We take advantage of the fact that $\FB$ is defined over the pseudomarginals (cf.~\eqref{eq:minimization_fb}) and express the relation between the stationary points $\FBStationary$ and the  pseudomarginals according to
\begin{align}
	\FBStationary = \FB(\pseudomarginals^{\circ}).
\end{align}

This correspondence becomes apparent when specifying the minimization of $\FB$ explicitly and including the constraints of the local polytope as a Lagrangian.
Taking the derivatives and setting them to zero then yields the BP update equations~\cite{yedidia2001}.
Consequently, at stationary points of $\FB$ (where all partial derivatives are zero), all BP messages remain unaffected by the update rule (which constitutes a BP fixed point).
In fact, every fixed point of BP is an interior stationary point of the constrained Bethe free energy~\cite{yedidia2005}.
Note that the nature of this correspondence further reveals a subtle, yet important, detail of BP: 
updating a message forces one gradient at a time to zero while keeping all other variables fixed, which albeit often going downwards is not a gradient descent step per se.
The variational principle, however, suggests that we should specifically consider minima of $\FB$;
luckily, the BP updates still tend to proceed in a sensible way towards minima of $\FB$~\cite{heskes2003stable, aji2000generalized}.

\newsubsubsection{Implications}{bp:variational:correspondence:implications}
Stable fixed points of BP (see. Section.~\ref{sec:solving:stability} for a thorough discussion on stability) are of particular relevance in practice.
Let us index all stable fixed points $\pseudomarginals^{(s)}$ by $s=1,\ldots,S$.
Every stable fixed point $(s)$ then has an associated local minimum $\FBLocalMin{(s)}$, an associated partition function $\partitionBetheLocalMin{(s)}$, and associated pseudomarginals $\pseudomarginals^{(s)}$; we denote the set of all $S$ stable fixed points by
\begin{align}
	\setOfStableSol = \Big\{ \fixedPointTuple{(1)},\ldots, \fixedPointTuple{(S)} \big\}.
\end{align}
Likewise, we consider the set of all fixed points that constitute minima of the Bethe free energy
\begin{align}
	\setOfMinimaSol = \Big\{ \fixedPointTuple{1},\ldots, \fixedPointTuple{M} \big\},
\end{align}
and the total set of fixed points
\begin{align}
	\setOfAllSol = \Big\{ \fixedPointTuple{1},\ldots, \fixedPointTuple{T} \big\}.
\end{align}
Further note that all stable fixed points of BP must be minima of $\FB$.
In the presence of frustrated cycles, however, minima may be unstable as well~\cite{heskes2003stable,murphy1999loopy}.
The different types of fixed points thus relate to each other according to
$\setOfStableSol \subseteq \setOfMinimaSol \subseteq \setOfAllSol$.

One important question is to understand under which conditions a fixed point is unique, i.e., when $|\setOfAllSol| =1$.
It seems appealing to utilize the connection between BP and $\FB$ for that purpose.
Besides for very simple models, however, conditions for convexity of $\FB$ are hard to come by, and established conditions are often far from necessary~\cite{pakzad2002belief,heskes2004uniqueness}.
A good overview of different conditions for convexity of $\FB$ is presented in~\cite{mooij2007sufficient}.

Besides these results for uniqueness, not many results aim to characterize the expected number of fixed points.
One work that pursues this direction is~\cite{watanabe}.
Note that -- at least as long as some proper form of message normalization is used~\cite{martin2011} -- the number of fixed points is always finite (cf. Theorem 7 and Lemma 2 in~\cite{watanabe}).   
Another important insight is that the number of fixed points is always odd.

We will later characterize the number of fixed points and asses their stability in Chapter~\ref{chp:solutionsBP} and in Chapter~\ref{chp:accuracyBP} for a range of models.
As of now, we will show the correspondence between stationary points of $\FB$ and fixed points of BP for one exemplary model.

\begin{example}[Energy Landscape of $\FB$ for an Attractive Ising Model]$ $\newline
	Let us take a closer look at the relationship between the Bethe free energy and BP by means of an attractive Ising model, specified on an infinite-size two-dimensional grid graph.
	Let all variables have identical local potentials, specified by $\field{i} = \field{} > 0$, and all edges have identical pairwise potentials, specified by $\coupling{i}{j} = \coupling{}{} > 0$.
	We illustrate a slice of the Bethe- and the Gibbs- free energy along the marginals of one variable $\pmf{\RV{i}}(\RVval{i})$ for a model with weak couplings (Figure~\ref{fig:f_bethe_illustrative_1}) and for a model with strong couplings (Figure~\ref{fig:f_bethe_illustrative_2})
	Note that the Bethe free energy upper bounds the Gibbs free energy for attractive models (cf. Section~\ref{sec:bp:evaluation}).
	
	\begin{center}
		\includegraphics[width=0.5\textwidth]{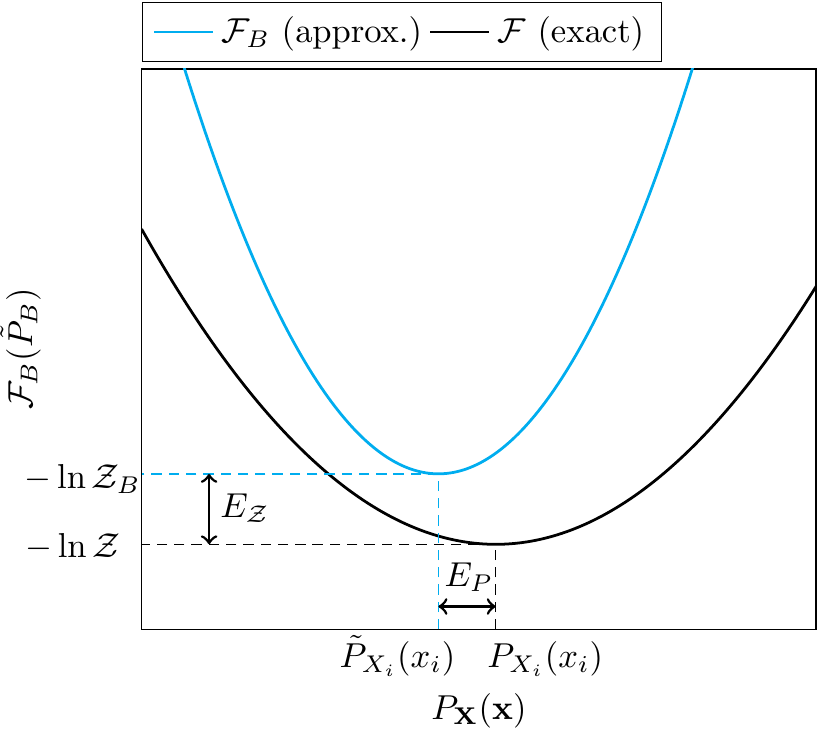}
		\captionof{figure}{
			Schematic illustration of the energy landscape that shows a slice along $\pmf{\RV{i}}(\RVval{i})$ for an attractive model with weak couplings. The exact and the approximate results are depicted for both the partition function and the marginals.}
		\label{fig:f_bethe_illustrative_1}
	\end{center}
	If the couplings are sufficiently small, $\FB$ is convex and BP has a unique and stable fixed point that has the whole message-space as a region of attraction (cf. Section~\ref{sec:solutionsBP:accuracy_selected_models}).
	BP will thus always converge independently of its initialization.
	
	\begin{center}
		\includegraphics[width=0.5\textwidth]{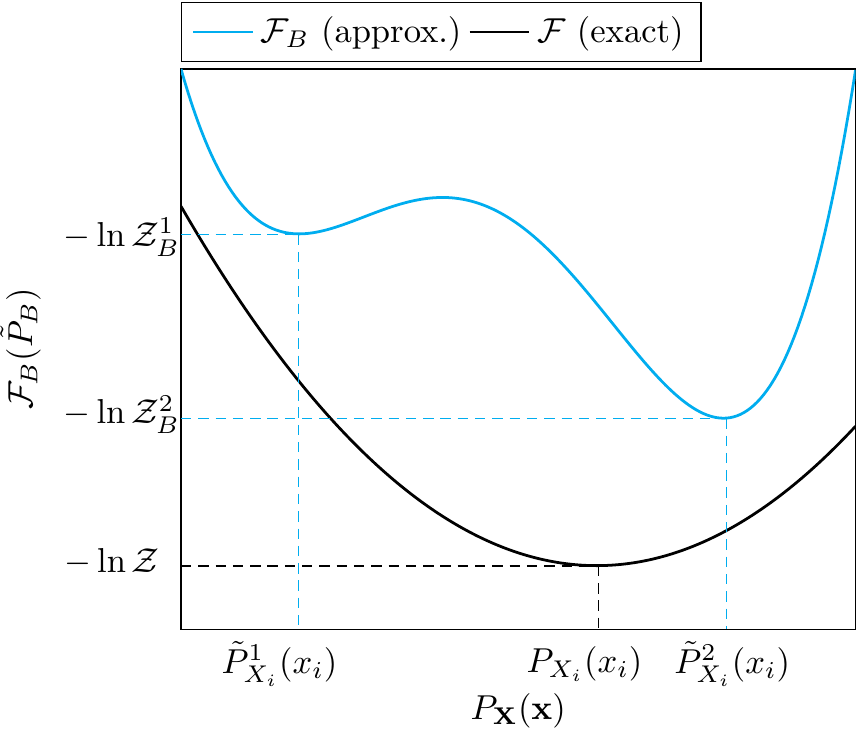}
		\captionof{figure}{ Schematic  illustration of the energy landscape for an attractive model with strong couplings. The Bethe free energy is non-convex and BP may converge to one of the local minima.
			Note how $\pmfApprox{\RV{i}}^2(\RVval{i})$ is more accurate than $\pmfApprox{\RV{i}}^1(\RVval{i})$.}
		\label{fig:f_bethe_illustrative_2}
	\end{center}
	\label{ex:energy_landscape}

	Let us now consider a model with strong couplings.
	Note how $\FB$ becomes non-convex and exhibits multiple stationary points that correspond to BP fixed points.
	We will later show in Chapter~\ref{chp:solutionsBP} that both local minima are stable fixed points and that the message initialization will determine to which fixed point BP converges.
	This has a major influence on the overall quality of the approximated quantities as the difference to the exact marginals $\pmf{\RV{i}}(\RVval{i})$ (i.e., the values minimizing the Gibbs free energy) and thus the accuracy varies considerably between different fixed points.

\end{example}

\newsubsection{Other Variational Approaches}{bp:variational:variants}
The one-to-one correspondence between stationary points of $\FB$ and fixed points of BP led to an improved understanding of BP and paved the way for methods that minimize the Bethe free energy directly~\cite{welling2001belief,welling2003approximate,cccp2003yuille}.
The minimization of $\FB$, however, remains non-trivial and usually requires additional considerations to render it viable. 

The consideration of the Bethe free energy opens the door for provable convergent algorithms.
Belief optimization~\cite{welling2001belief} minimizes $\FB$ by following the negative gradient and is guaranteed to converge.
The minimization takes place along the edges of the local polytope, which guarantees that the marginalization constraints in~\eqref{eq:local} are satisfied throughout.
Alternatively, one can decompose the non-convex $\FB$ into a convex and a concave problem (this decomposition is in general not unique)
and optimize  both objectives in an alternating fashion.
This procedure, known as the constrained convex-concave procedure~\cite{cccp2003yuille},, is guaranteed to obtain a stationary point of $\FB$.
Although both methods converge to  stationary points of $\FB$, two major limitations remain:
first, the approximation quality may vary considerably between different stationary points and only obtaining some (local) extremum may yield sub-optimal solutions;
and second, despite the appeal of convergence guarantees, run-time guarantees are often just as important in practice.

To counteract some of these problems, one can also relax the Bethe free energy and come up with well-behaved surrogates that can be minimized efficiently.
Convex surrogates seem to be specifically well-suited for this task and have thus received considerable attention.
One can, for example, upper bound the objective by a convex function as in tree-reweighted belief propagation (TRW)~\cite{wainwright2003tree-convex, wainwright2005new, kolmogorov2006convergent} that enforces a concave entropy-term as a combination of tree-entropies.
Note, however, that there are multiple ways of coming up with convex surrogates~\cite{meltzer2009convergent,globerson2007convergent, hazan2008convergent}.
One comprehensive overview that unifies many different approaches in terms of the chosen counting number is presented in~\cite{meshi2009convexifying}.
Overall these convex versions are well-behaved and can be optimized efficiently.
A trade-off between convergence-properties and accuracy, however, persists and -- if it can be minimized -- the Bethe approximation often outperforms its convex surrogates in terms of accuracy~\cite{meshi2009convexifying, weller2013approximating}.

This observation led to a renewed focus on trying to minimize the Bethe free energy efficiently (i.e., in polynomial run-time).
The efficient minimization of $\FB$ becomes possible if we impose certain properties on the graphical model (in terms of its structure or parameters) and consider $\epsilon$-approximating the stationary points.
In particular, this includes sparse models~\cite{shin2012complexity}, where a projection scheme in the minimization task allows for a fully polynomial-time approximation on graphs with $\max(\nodeDegree{i}) = \mathcal{O}(\log N)$.
For attractive models (not necessarily sparse), this algorithm is further improved in~\cite{weller2013approximating} so that it obtains the global minimum of the Bethe free energy.
If both properties are fulfilled, i.e., for locally tree-like attractive models the Bethe approximation is exact and can be optimized efficiently~\cite{dembo2010ising}.\\

From a completely different point of view, the approximation quality can also be enhanced by better approximations of $\FGibbs{}$.
The concept of the Bethe approximation, that only accounts for the singleton- and pairwise marginals, generalizes to the Kikuchi method that accounts for the marginals of larger cliques as well.
The same principle -- i.e.,  considering larger cliques --  generalizes BP as well; this method is accordingly termed as generalized belief propagation~\cite{yedidia2001}.
Both the accuracy and the convergence properties improve by adopting the computations to larger cliques.
Although larger cliques inevitably increase the computational complexity,  the principles of  generalized belief propagation are flexible enough to allow moving freely along this trade-off.

Gauge transformations~\cite{chertkov2006loop,chertkov2008belief} are somewhat similar in that they work with the exact partition function $\partitionFunction$ directly while simplifying its estimation. 
This is achieved by expressing the exact partition function via a loop-series expansion of $\partitionBethe$. 
Interestingly, all terms in the loop-series correspond to fixed points of BP, which suggests one way of computing them.
The number of terms in the expansion, however, may be large for models with many loops so that computing all terms is often not an option.
Nonetheless, gauge transformations provide a principled approach to improve upon BP by at least accounting for some terms in the loop-series.\\

To conclude this section, a variety of variational approaches are available that aim to find the sweet spot between accuracy and complexity when performing approximate inference.
Most of the above methods build upon the fundamental principles of the Bethe approximation, that -- although providing a considerable simplification over the Gibbs free energy -- remains problematic to minimize in practice.

\newsection{Approximation Quality}{bp:evaluation}
So far we have already discussed various approximate inference methods.
Different methods will perform differently and, depending on the given model, achieve varying accuracy.
If we want to evaluate and compare approximate inference methods, it is important to measure the accuracy of the approximation.
Having said that, approximate inference methods are usually applied to problems that forbid the computation of the exact solution.
This prohibits measuring the accuracy by comparison to the exact solution and highlights the need for some qualitative measures of the expected performance for a given problem. 
Essentially, one would like to provide model-specific performance guarantees and bounds on the approximation error.
Few bounds on the approximation quality are, however, established and some bounds require considerable computational resources on their own,
not to mention that most bounds are often relatively loose.

In this section, we outline how the error in the approximation of the partition function and the marginals will be measured throughout the thesis.
Moreover, we will discuss some of the available error-bounds for both quantities.


\newsubsubsection{Partition Function}{bp:evaluation:partition}
The error of the partition function is usually evaluated in terms of the relative error between the log-partition functions  (i.e., the negative value of the Bethe free energy) according to
\begin{align}
	\EPartition{m} = \frac{|\log \partitionBethe^m-\log\partitionFunction|}{\log \partitionFunction} = \frac{|\FGibbs{}-\FB^m|}{-\FGibbs{}}, \label{eq:error_partition}
\end{align}
where $\partitionBethe^m = \partitionBethe(\pseudomarginals^m)$ is the Bethe partition function of the $m^{th}$ fixed point~\cite{gomez2007truncating}.
Note that \eqref{eq:error_partition} quantifies the error in the partition function and in the free energy.


Existing bounds on the partition function usually combine an upper bound~\cite{wainwright2005new, jaakkola1997recursive}
with some lower bound as e.g., the naive mean field~\cite{wainwright2008graphical}. Other bounds are based on the loop-series expansions~\cite{willsky2008loop} or the non-backtracking operator~\cite{saade2014spectral}.

One intriguing detail of the Bethe approximation is that it only approximates the partition function but neither provides an upper nor a lower bound of it.
For the important class of attractive models, however, the Bethe partition function does indeed lower bound the partition function, i.e., $\partitionBethe < \partitionFunction$~\cite{ruozzi2012bethe}.
This has the important consequence that minimizing $\FB$ is always optimal with respect to the approximation quality of the partition function for attractive models, since $\argmin_{\partitionBethe^m}(\EPartition{m}) = \exp(-\min_{\LPolytope} \FB(\pseudomarginals^m))$.

An alternative, arguably more intuitive, proof that reveals how the Bethe partition function lower bounds the true partition function relies on the concept of clamping~\cite{weller2014clamping}.
The main idea is to condition (i.e., to clamp) on a variable taking one specific value at a time and to evaluate the partition function as a sum over all sub-partition functions.
This brings the advantage of working directly on the Bethe free energy and obviates the need for relying on additional concepts such as graph covers~\cite{vontobel2013counting, ruozzi2013bethebound} or loop-series expansions~\cite{willsky2008loop}.
Moreover, for attractive models clamping always improves the approximation quality of the partition function over BP.
This is particularly true if selecting the clamped variables wisely, as for example according to some heuristics~\cite{weller2016uprooting}.
Clamping also improves the upper and lower bounds of $\partitionFunction$ as it can be utilized to improve the partition function estimates of TRW and the mean field method~\cite{weller2016clamping}.\\

Closely related to BP, but non-iterative, is the mini-bucket elimination scheme~\cite{dechter2003mini}.
Mini-bucket elimination also provides upper and lower bounds on the partition function and offers a trade-off between the accuracy of the bound and computational efficiency based on the clique-size considered.
One can further generalize this concept based on H\"olders inequality~\cite{liu2011bounding} and efficiently compute bounds with good quality, where the improvement on the estimated bounds is because of incorporating the concept of TRW. 
The method of gauge transformations further generalized this concept and allows one to obtain tighter bounds on the partition function with similar computational effort~\cite{ahn2018gauged}.\\

\newsubsubsection{Marginals}{bp:evaluation:marginals}
We measure the error of the singleton marginals by the mean squared error (MSE), where the error of the approximated marginals at the $m^{th}$ fixed point is given by 
\begin{align}
	\EMarginal{m} &= \frac{1}{N} \sum_{\RV{i}\in\setOfNodes} ||\pmf{\RV{i}}(\RVval{i}) -\pmfApprox{\RV{i}}^m(\RVval{i})||_2^2 \label{eq:error_mse}\\
	&= \frac{2}{N}\sum_{\RV{i}\in\setOfNodes} \big(\pmf{\RV{i}}(\RVval{i}=1) -\pmfApprox{\RV{i}}^m(\RVval{i}=1)\big)^2. 
	\label{eq:error_marginal}
\end{align}
The simplification in~\eqref{eq:error_marginal} is because of symmetry properties for binary random variables.
Note that one can  replace the squared $l_2$-norm in~\eqref{eq:error_mse} by any other norm if desired; 
in particular the $l_{\infty}$-norm is considered by some authors to measure the worst-case error.

The expected mean describes the response of the system to the field $\field{}$~\cite{mezard2009} according to
\begin{align}
	\meanAvg = \E(m) = \frac{1}{N} \sum_{\RV{i}\in\setOfNodes} \mean{i} = \frac{1}{N} \sum_{\RV{i}\in\setOfNodes} \E(\RV{i}),
	\label{eq:mean_expected}
\end{align}
where we parameterize binary random variables by their mean (cf.~\eqref{eq:mean}).
Note that the difference between the expected mean of the exact marginals $\meanAvg$ and of the approximated marginals $\meanAvgApprox$ is identical to the sum over all marginal errors
%
\begin{align}
	\meanAvg - \meanAvgApprox = \frac{2}{N} \sum_{\RV{i}\in\setOfNodes} \singleExact{i}(\RVval{i}=1)- \singleApprox{i}(\RVval{i}=1).
	\label{eq:mean_difference}
\end{align}
Note that we consider binary random variables with $\sampleSpace{x}=\{-1,1\}$; 
if we would consider  $\sampleSpace{x'}=\{0,1\}$ instead, the expected mean would change according to $\langle m' \rangle = \frac{1}{N} \sum_{i=1}^N \singleExact{i}(\RVval{i}=1) = \frac{1}{2}(\meanAvg+1)$.
\\

Some methods quantify and bound the approximation  error of the marginals instead of $\EPartition{m}$ by computing a confidence interval (i.e., an upper- and lower bound) on the exact marginals.\footnote{ 
	Note that upper and lower bounds of the partition function can also be related to upper and lower bounds of the marginals~\cite{weller2013approximating}.}
Bound propagation~\cite{leisink2003bound} computes bounds for small clusters of nodes; 
similar to BP, these bounds are then propagated throughout the graph and provide bounds on the singleton marginals after convergence.
The computational complexity is determined by the tree-width (similar as for the Junction tree algorithm), thus limiting the applicability to models with few loops.

Similar in principle, but more efficient in terms of computational complexity, is the recursive propagation of the bounds over sub-trees of the model~\cite{mooij2009bounds}.
The computational complexity is determined by the support of the random variables.
This renders the method applicable to models with high connectivity as well.

From a completely different perspective, TRP provides bounds on the approximation error as well~\cite{wainwright2003tree-scheduling}.
TRP considers spanning trees to approximate the marginals, which suggests computing bounds on the marginals over spanning trees as well.
Regarding the computational complexity, there is one particularly intricate ingredient involved in bounding the error over the spanning trees, with the problem being the requirement  for the log-partition function that cannot be computed exactly for complex models.

Some of these issues are resolved by computing the bounds over self-avoiding walk trees instead~\cite{ihler07}.
If we compare the quality of the bounds we observe that the proposed methods in~\cite{mooij2009bounds} and~\cite{wainwright2003tree-scheduling} are of similar quality.
The computation over self-avoiding walk trees impacts the quality of the obtained bounds:
for models with weak interactions (where BP converges fast) the tightest bounds are provided by~\cite{ihler07}; for models with strong interactions (where BP performs worse) the confidence intervals become wider than the ones provided by~\cite{wainwright2003tree-scheduling}.

All those methods aim to estimate the expected performance of BP in terms of the marginal error and while this often estimates the performance well, all presented methods have one particular problem in common.
Increasing the coupling strength degrades the quality of the estimated bounds.

\newsection{BP as a Dynamical System}{bp:map}
It is often convenient to consider BP as a dynamical system, i.e., as a discrete time map (see Definition~\ref{defn:map} for a formal definition). Let $  \setOfMessages[n] \definition \{ \msg[n]{i}{j}{}, \msg[n]{j}{i}{} : \edge{i}{j} \in \setOfEdges \}$ be the set of all messages at iteration $n$, then we denote the mapping induced by the update equations of BP in \eqref{eq:update} according to
\begin{align}
	\setOfMessages[n+1] = \BP(\setOfMessages[n]).
	\label{eq:BP_dynamical_sys}
\end{align}
In accordance with the notion of fixed point messages $\fpMsg{i}{j}{}$ that remain unaffected under the application of BP,  we refer to the set of all fixed point messages by $\fpSetOfMessages$.
We will often refer to $\fpSetOfMessages$ simply as fixed point and write
\begin{align}
	\fpSetOfMessages= \BP^{\stationaryPoint}(\setOfMessages),
\end{align}
where $\BP^{\stationaryPoint}$ updates the messages until convergence.

Remember that BP is neither guaranteed to provide accurate marginals nor guaranteed to converge.
Failure of BP (to provide accurate results) can be attributed to the existence of multiple (stable) fixed points or the existence of unstable fixed points (for which the messages oscillate far away from any fixed point~\cite{weiss2000correctness, mooij2007sufficient, ihler2005loopy}).
Thus, if we want to make BP more robust, we should modify BP such that it converges towards an accurate stable fixed point.
Potential modifications include variants that minimize $\FB$ directly or minimize alternative objectives as e.g., convex surrogates (cf. Section~\ref{sec:bp:variational:variants}).
From a completely different perspective, one can modify the update procedure directly -- ignoring the effect on the variational function -- so that the overall performance improves.
In particular, the consideration of BP as a dynamical system comes with a flexible formalism that suggests multiple ways of modifying and enhancing BP.


\newsubsection{Enhancing Belief Propagation's Properties}{bp:map:improving}
BP, as presented in Section~\ref{sec:bp:intro}, is only one specific instance in a broad class of algorithms that fit under the umbrella of more general message passing methods. 
We will see that, although similar in nature, the choice of the respective message passing algorithm is not a mere formality but influences the behavior significantly -- and is therefore of great practical relevance.
The family of message passing algorithms is defined by a set of key-ingredients, where
all message passing algorithms proceed along with the following three steps:

\begin{itemize}
	\item \textbf{Initialization:} Each message $\msg{i}{j}{}\in\setOfMessages$ is assigned an initial value $\msg[1]{i}{j}{}$.
	\item \textbf{Update:} An update function is defined that acts on the set of messages $\map: \setOfMessages[n] \rightarrow \setOfMessages[n+1]$.
	This update function usually consists of one function per message $f_{ij}: \setOfMessages[n] \rightarrow \msg[n+1]{i}{j}{}$.
	\item \textbf{Evaluation:} An evaluation or decision function maps the set of messages to the values of interest. 
	In particular, we are interested in approximating the marginals (by the pseudomarginals) and in approximating the partition function (by the Bethe partition function);
	thus we consider evaluation functions for the pseudomarginals $\decisionP: \setOfMessages \rightarrow \pseudomarginals$ and for the Bethe partition function $\decisionZ: \setOfMessages \rightarrow \partitionBethe$.
\end{itemize}
This rather general formulation of message passing algorithms grants a certain degree of freedom in specifying an algorithm.
Together, those aspects define the overall behavior of BP and every aspect may be tuned individually to enhance its performance.
Not a single best algorithm exists though and,
depending on the application, one or the other variant may perform better.

We will now discuss some established ways of enhancing BP's performance.
We particularly focus on how all those variants result from tuning some of the key-step mentioned above.
There is an undeniable elegance in providing such a unifying view; 
even more important though is the fact that this unifying view also suggests novel ways of enhancing the performance of BP. 
We will also highlight how many contributions of the later chapters conform to this view as well.

\newsubsection{Initialization}{intro:bp:improving:init}
The choice of the initialization is no formality but, in case of multiple fixed points, determines which fixed point a given algorithm will converge to  and as such it plays an important role in analyzing BPs behavior.
Although only practical for problems of modest size, one can for example explore the solution space exhaustively by using a sufficiently wide range of initial values.

Note that BP often performs significantly better then worst-case analyses suggest.
It seems, that for some reason, a wide range of initial message values converges to a good fixed point.
This further raises the question of how one should initialize message passing algorithms.
One common choice, that often works well, is to initialize all messages to the same value $\msg[1]{i}{j}{} = \frac{1}{|\sampleSpace{X}|}$;
it remains unclear, however,  where the advantages of uniform initialization stem from.
In general, it would be desirable to identify initial message values that are optimal for certain problem-classes.
We will later pursue this idea and provide a heuristic that obviates the need for choosing initial values and guides BP towards accurate fixed points (cf. self-guided belief propagation in Chapter~\ref{chp:selfguided}).

\newsubsection{Update}{intro:bp:improving:update}
The choice of the update function is one of the key-aspects of message passing algorithms and has an enormous influence on the overall behavior of a particular algorithm.
In particular, it is altering the update function that often leads to algorithms with favorable properties.
The search for enhanced message passing algorithms thus spurred the development of alternative update functions, which ultimately led to the proposal of various algorithms.
We will discuss some successful variants and group these modifications into two different categories.

First, instead of applying the update function $\map$ to all messages at once, one can update only a single message at each iteration.
Effectively, this procedure determines a schedule according to which the messages are updated; we consequently term this concept \emph{scheduling}.
Note that scheduling preserves all update functions and only changes the order in which they are applied.

On the other hand, one can, of course, try to improve the convergence properties by altering the update functions themselves.
Although one can think of numerous ways to alter the update functions we intend to consider variants that 
perform at least as good as plain BP.
In particular, this includes two variants that will be considered in this thesis; these are:
\emph{damping} and \emph{self-guided belief propagation}.

\newsubsubsection{Scheduling}{intro:bp:improving:scheduling}
It is widely accepted that scheduling the messages and updating them asynchronously enhances BP.
In particular, this helps to achieve convergence more often and in fewer iterations.
Up until now, we considered BP without scheduling where all messages are updated according to
\begin{align}
	\setOfMessages[n+1] &= \BP(\setOfMessages[n]) \nonumber \\
	& = \big( \eq[1](\setOfMessages[n]),\ldots,\eq[m](\setOfMessages[n]),\ldots,\eq[2|\setOfEdges|](\setOfMessages[n]) \big).
\end{align}
Note that we imposed some ordering $o:\edge{i}{j}\rightarrow m$ on the messages and refer to the update functions accordingly, i.e., $\eq[m]$ computes the m\textsuperscript{th} message. 
There are $2|\setOfEdges|$ messages in total as they are passed along both directions of the edges.

For an asynchronous version of BP only a single message $\msg{i}{j}{}=\msg{m}{}{}$ is updated per iteration according to
\begin{align}
	\fpSetOfMessages[n+1] = \big( \msg[n]{1}{}{},\ldots,\eq[m](\setOfMessages[n]),\ldots,\msg[n]{2|\setOfEdges|}{}{} \big).
\end{align}

The update-order, i.e., which message to compute, has a major influence on the overall performance.
One common choice is a \emph{fixed order}, either according to a round-robin scheme or a random ordering.
In both cases, the sorting of the messages follows a particular ordering and the messages are selected for an update according to
\begin{align}
	m=n \pmod{2|\setOfEdges|}. 
\end{align}
Such a fixed update-order already improves the performance of BP significantly, yet it leaves considerable room for improvement.

One popular way of scheduling the messages relies on the fact that, for a tree, BP is guaranteed to converge and to provide the exact marginals.
\emph{Tree based reparameterization} (TRP)~\cite{wainwright2002tree} computes a set of spanning trees so that the union of all spanning trees includes all edges of the original graph.
Then, one applies BP to one tree at a time.
One says that this calibrates the marginal along the spanning tree, i.e., the marginals would be exact if the graph would contain no more edges.
Of course, performing BP on another spanning tree will alter some of the marginals again, but at least TRP enforces consistency through parts of the graph.
The choice and the order of the spanning trees impact the overall performance of TRP and can be tailored in a model-specific way.
Intuitively, one can argue that the focus on spanning trees enhances the overall convergence properties by enforcing a brisk exchange of information in the graph.\\

All scheduling methods discussed so far rely on a fixed update-order that may or may not depend on the  graph.
Instead of carefully selecting an update-order, it is, however, much more  flexible to resort to an \emph{adaptive} update-order, where the current message values determine which message to update next.
This obviates the need for carefully tuning the update-order to the current graph and further improves the convergence properties because of the inherent flexibility in adaptive methods.

The most basic adaptive scheduling method,
\emph{residual belief propagation} (RBP)~\cite{elidan2012residual}, already improves the convergence properties a lot.
RBP updates the messages along the least calibrated edge, i.e., for which the message varies the most.
The underlying assumption is that strongly varying messages contain more information and are thus of predominant importance for the convergence of BP. 
To make this more precise we introduce the residual $r_{ij}$ to measure the message-distances between two successive iterations~\footnote{Ultimately one is interested in the distance to the fixed point, if it exists, $\lim\limits_{n \rightarrow \infty}{\mu_{ij}^n(x_j)}$. However, since $\lim\limits_{n \rightarrow \infty}{\mu_{ij}^n(x_j)}$ is not known, the time variation of the messages offers a valid surrogate (cf.~\cite{elidan2012residual}).} 
\begin{align}
	r_{ij}^{n} &= ||\msg[n+1]{i}{j}{} - \msg[n]{i}{j}{} ||_{\infty} \nonumber \\
	&= \max_{\RVval{j}} |\msg[n+1]{i}{j}{} - \msg[n]{i}{j}{} |.
\end{align}
Let $\vm{r}^n =\{r_{ij}^n : \edge{i}{j} \in \setOfEdges\} $ be the set of all residuals at iteration $n$, then the index that maximizes the residual 
\begin{align}
	m = \argmax_m \vm{r}^{n}
	\label{eq:rbp}
\end{align}
identifies the message that is selected for an update.
Overall this significantly reduces the number of message updates until convergence and the overall convergence time.
Computing all residuals, however, poses some computational overhead.
We can reduce this overhead and further reduce the convergence time without degrading the accuracy much by computing an upper bound on the residual~\cite{sutton2012improved}. \\

Although adaptive scheduling increases the number of models for which BP converges, some models remain problematic.
Closer inspection reveals that only small parts of the graph fail to converge, whereas large parts (almost) converge in a couple of iterations.
RBP thus updates the same subset of messages ad infinitum.
This observation inspired the design of scheduling methods that detect such problematic behavior, specifically counteract the oscillations, and enforce convergence~\cite{knoll_scheduling}.

\emph{Noise injection belief propagation} (NIBP) applies RBP and actively checks for oscillations:
if RBP converges no modifications are undertaken;
if, however, oscillations occur, NIBP \emph{injects} Gaussian noise to the selected message.
The underlying intuition is that the injected noise will propagate from the most influential part of the graph (i.e., the message selected by RBP) throughout the whole graph and, by introducing a relevant change to the overall model, lead to convergence.
Let $\msg{i}{j}{}$ be the message selected according to~\eqref{eq:rbp}; if this message oscillates, i.e., if $\msg[n]{i}{j}{} = \msg[n-l]{i}{j}{}$ for any $l\in\REAL$, Gaussian noise is added so that
\begin{align}
	\msg[n+1]{i}{j}{} = \eq[ij](\setOfMessages[n]) + \mathcal{N}(0,\sigma^2).
	\label{eq:NIBP}
\end{align}

Alternatively, instead of actively detecting oscillations, one can question the underlying assumption of RBP.
Messages, despite varying significantly, often take identical values repeatedly;
arguably, these messages carry not much information.
\emph{Weight decay belief propagation} (WBP) penalizes this behavior and dampens the residual of messages that were already updated.
Consequently, WBP increases the relevance of the remaining messages and thus further refines the parameterization of the overall graph.
More precisely, WBP divides the residual of a messages by the number of times this message has already been scheduled and reformulates~\eqref{eq:rbp} according to
\begin{align}
	m = \argmax_m \frac{r_{ij}^{n}}{\sum_{n'=1}^n \mathbf{1}_{ij}^{n'}},
	\label{eq:wbp}
\end{align}
where the indicator function $\mathbf{1}_{ij}^{n'} = 1$ if and only if $r_{ij}^{n'} = \argmax_m \vm{r}^{n'}$.
Detailed pseudocode for NIBP and WBP is presented in Appendix~\ref{sec:appendix:pseudocode}.

\begin{figure}
	\centering
	\includegraphics[width=0.5\textwidth]{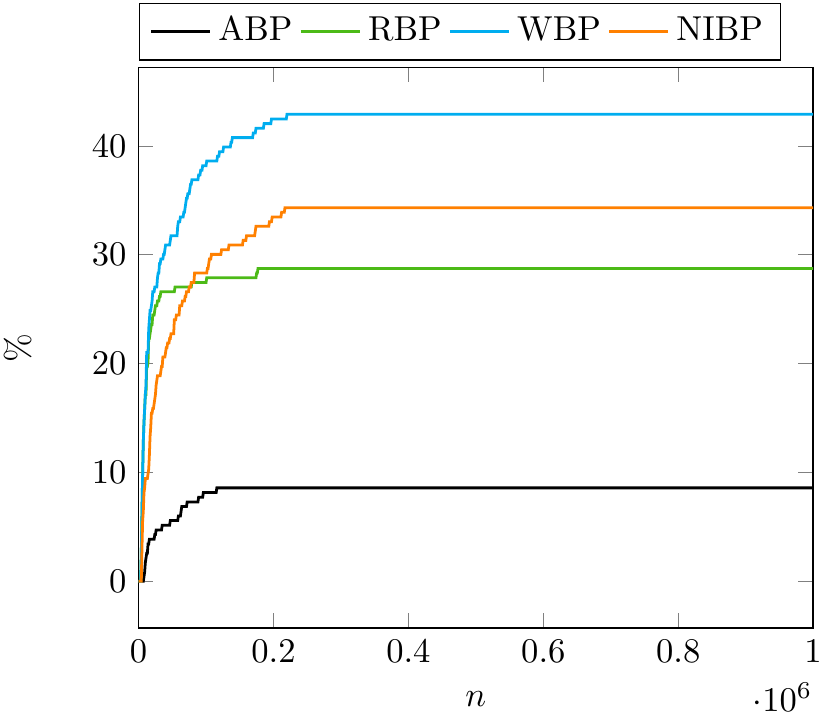}
	\caption{Percentage of converged runs $\%$ over the number of iterations $n$ for: ABP (black), RBP (green), NIBP (blue), and WBP (orange). The results are averaged over 233 grid graphs of size $13\times13$ with random potentials specified by $(\coupling{i}{j},\field{i}) \sim \mathcal{U}(-6.5,6.5)$ in accordance with~\cite{elidan2012residual}.}
	\label{fig:convergence}
\end{figure}

When comparing the convergence properties of the discussed scheduling methods on grid graphs with random Ising potentials, a clear picture emerges. 
We present one exemplary experiment in Figure~\ref{fig:convergence} and refer to the paper~\cite{knoll_scheduling} for more exhaustive experiments.

Although scheduling with a round-robin schedule, ABP (black), improves the convergence properties over BP without scheduling, ABP does not perform as well as more sophisticated methods.
In particular, adaptive scheduling enhances the convergence properties notably, demonstrating its conceptual advantage.
Evaluating the overall convergence behavior,  NIBP (orange) converges for more models than any other method.
WBP (blue) still converges more often than RBP (green), although it requires more iterations to converge.
Damping of the residual slows the convergence as messages that require multiple updates to converge will now be selected less frequently.

Although the positive influence of adaptive scheduling methods on the convergence properties is indisputable, the influence on the marginal accuracy is less obvious.
In practice, however, we are not only interested in good convergence properties but require accurate marginals as well. A comprehensive comparison of different scheduling methods in terms of marginal accuracy is presented in~\cite{knoll_scheduling}. 
We summarize the major insights here:
considering only graphs 
for which a simple round-robin schedule (i.e., ABP) achieves convergence, ABP approximates the marginals well for all models, whereas RBP and NIBP fail to do so.
It may seem that advanced scheduling methods pay the price for improving the convergence properties and obtain less accurate marginals.
It is thus particularly remarkable that WBP maintains the marginal accuracy of ABP although it converges for a wide range of models.
Of all considered methods, WBP is thus superior.

\newsubsubsection{Alternative Update Functions}{intro:bp:improving:update:alternatives}
When the BP messages oscillate (often just between two values) applying a damping term has proven to be a successful strategy to improve the convergence properties~\cite{murphy1999loopy}. 
That is, one replaces the messages with a weighted average of the last messages so that
\begin{align}
	\BPVariant{D} (\setOfMessages) = (1-\epsilon) \BP(\setOfMessages) +\epsilon \setOfMessages,
	\label{eq:bp_damping}
\end{align}
where $\epsilon \in [0,1)$.
Remember that we intend to change the update function without changing the fixed points.
Complying with our requirement, fixed points of $\BPVariant{D} (\setOfMessages)$ must, therefore, be fixed points of BP without damping as well. 
This is the case as fixed points of BP must satisfy $\fpSetOfMessages = \BP(\fpSetOfMessages)$, which holds per definition so that 
\begin{align}
	\BPVariant{D} (\fpSetOfMessages) = (1-\epsilon) \BP(\fpSetOfMessages) +\epsilon \fpSetOfMessages = \BP(\fpSetOfMessages). 
	\label{eq:equal_stability}
\end{align}
Note that, 
although damping does not affect the fixed points, the local stability of the individual fixed points may change (cf. Section~\ref{sec:stabilityBP:linearization}).\\

In Chapter~\ref{chp:selfguided} we will propose self-guided belief propagation that relies on an alternative update function.
While damping modifies the update function in a static manner, SBP modifies it in a more flexible, time-dependent, way.
The underlying idea is to consider the potentials as a function of the iteration $\setOfPotentials(n)$, which enhances the performance across all areas (i.e., convergence properties and accuracy).

\newsubsection{Evaluation}{intro:bp:improving:evaluation}
BP often performs better than suggested by worst-case analyses~\cite[Section 22.3]{mezard2009}.
This observation comes not too surprising, especially when considering the existence of multiple fixed points.
We would indeed be rather surprised if all fixed points have similar accuracy.
Rather than a disadvantage, the difference in the accuracy can be seen as a chance;
modified evaluation functions may exploit this difference and select or combine fixed points to enhance the approximation quality.
Certainly, modified evaluation functions are not a remedy to all problems.
But modified evaluation functions are an important, flexible, concept to account for the existence of multiple fixed points and to mitigate some of the problems arising in loopy models.\\

Here we discuss two popular modifications that take advantage of the variability in accuracy across multiple fixed points.
Taking the set of all fixed points $\setOfAllSol$ (or a subset thereof) into account, we modify the evaluation function $\decisionP:\{\setOfMessages[\stationaryPoint]\}\rightarrow \pseudomarginals$ in two ways.
Either to select a single fixed point or to compute a weighted combination of multiple fixed points.
Note that we will consider and evaluate both modifications in Chapter~\ref{chp:solutionsBP} and theoretically analyze them in Chapter~\ref{chp:accuracyBP}.
The modifications of $\decisionP$ become relevant in the presence of multiple fixed points.
BP may converge to any fixed point as it does, in general, not favor a particular one over the other.
Even if all fixed points are available, however, it remains unclear how one should select the fixed point that approximates the marginals best.

Variational methods  minimize their objective and aim to find the global minimum, although -- for non-convex $\FB$ -- this cannot be done efficiently.
Given a set of fixed points, we should thus select the fixed point that minimizes $\FB$ (and maximizes $\partitionBethe$), which is the global minimum of $\FB$.
We denote the associated pseudomarginals by
\begin{align}
	\pseudomarginals^{MAX}= \argmax_{\pseudomarginals\in \setOfAllSol} \partitionBetheWithP =  \argmin_{\LPolytope} \FB(\pseudomarginals).
\end{align}
Note that the complete set of fixed points $\setOfAllSol$ is generally not available and the maximization can thus only be performed over the set of available fixed points.

Recently an interest emerged in methods that obtain and combine multiple fixed points~\cite{batra2012diverse, braunstein2005survey, srinivasa2016survey}.
The combination of all fixed points $\setOfMinimaSol$ lies at the heart of the replica symmetry breaking (RSB) assumption.
This assumption emerged in statistical physics from the study of spin glass models and describes how the exact marginals decompose into a weighted sum of different contributions, which turn out to correspond to minima of $\FB$.\footnote{
	In physics one deals with the decomposition of the Gibbs measure (i.e., the joint distribution) into a weighted combination of Bethe measures (that correspond to BP fixed points).} 
If $\FB$ is non-convex, the RSB theory describes the decomposition of the exact solution into a convex combination of marginals that are weighted by their associated partition function so that
\begin{align}
	\pmf{\RV{i}}(\RVval{i}) = \frac{1}{\sum_m \partitionBethe^m }\sum_{m=1}^M \partitionBethe^m \singleApprox{i}^m(\RVval{i}).
	\label{eq:rsb}
\end{align}
This representation can be attributed to~\cite{mezard1987spin}
and, rather than a theorem, it is a set of postulates.
One underlying assumption is that the system does exhibit multiple fixed points (unique fixed points would falsely imply exact marginals otherwise);
an accessible introduction to the RSB theory and all underlying assumptions can be found in~\cite[Chapter~19]{mezard2009}.
Despite its non-rigorous flavor,~\eqref{eq:rsb} has been verified for a wide range of problems (e.g., random SAT problems and spin glasses).
In particular, many state-of-the-art solvers for combinatorial problems rely on the RSB theory~\cite{ravanbakhsh2015perturbed}.
Further note that, even if the local polytope is a strict outer bound on the marginal polytope, a convex combination of $\singleApprox{i}^m(x_i)$ (that correspond to edge points of $\LPolytope$~\cite{yedidia2005}) may consequently end up on an edge point of $\MPolytope$ (which it must if it is exact).

The RSB theory states how to form the exact solution from pseudomarginals that were obtained by efficient  approximate inference methods.
This seems to contradict all results on the computational complexity of exact inference.
Closer inspection resolves this apparent contradiction:
the number of fixed points may grow exponentially with the model-size so that obtaining all fixed points is just as problematic as computing the marginals exactly.
Indeed, this is one of the reasons that limit the practical applicability of the RSB assumption.

For the special case of constrained satisfaction problems, survey propagation~\cite{braunstein2005survey} is one efficient way to evaluate~\eqref{eq:rsb}.
The extension to more general models, however, remains somewhat elusive~\cite{srinivasa2016survey, ravanbakhsh2014revisiting}.

Nonetheless, besides being relevant from a theoretical point of view, we will also identify certain models for which only relatively few fixed points exist, which substantiates the practical relevance of the RSB assumption as well.

\newsection{Conclusion}{intro:bp:questions}
This chapter introduced BP as a message passing algorithm for approximate inference.
We further presented two alternative perspectives that hone our understanding of BP.
Namely, we discussed
(i) the variational interpretation that relates all fixed points of BP to stationary points of the Bethe free energy and 
(ii) the consideration as dynamical systems that suggest various ways to enhance BP.
Both concepts arguably provided insights into the solution space of BP, explained the behavior of BP, and thus benefited our understanding of BP.
Besides discussing the properties of BP from different perspectives, one major purpose of the current chapter was to provide an introduction for the second part of this thesis (i.e., Chapter \ref{chp:solutionsBP} - \ref{chp:accuracyBP}).
Our discussion is far from being exhaustive and left many questions open.
Below we list a couple of open questions worth pursuing that will be addressed in the remainder of the thesis and that will  provide interesting insights into the nature of BP.\\

One of the most pressing needs when considering the variational interpretation is a thorough analysis of the energy landscape of the Bethe free energy (specifically for problems where $\FB$ is non-convex).
Since every local minimum serves as a potential fixed point it would be of great relevance to characterize the solution space and to obtain all possible fixed points.

While the knowledge of all fixed points is valuable in its own, this does not reveal all properties of relevance for BP.
In particular, the energy landscape does not explain the convergence properties and it remains an open, albeit important, question under which conditions one can expect BP to converge.\\

Furthermore, we defined the set of all stable fixed points $\setOfStableSol$ in Section~\ref{sec:bp:variational:correspondence}.
Whether this set is available is questionable as the full set of fixed points is generally not available.
It is consequently an important subject to come up with message passing algorithms that tend to converge towards the most accurate fixed point.\\

Following the discussion in Section~\ref{sec:bp:evaluation}, it is crucial to assign some qualitative measurements to the fixed points after convergence of BP.
This is particularly important as the exact solution is not available in practice.
It remains an ongoing research topic, however, to determine useful accuracy-bounds that can be evaluated efficiently.
Moreover, it is not obvious at all how different measurements, e.g, bounds on the marginal accuracy or the partition function, relate to each other.

  \emptydoublepage
\newchapter{Discrete-Time Dynamical Systems: Solution~Space~Analysis}{solving}
\openingquote{Inside a broken clock\\
Splashing the wine\\ With all the rain dogs\\
Taxi, we'd rather walk...} {Tom Waits}

\renewcommand{\pwd}{dynamical_sys}

This chapter provides a brief, high-level, introduction  to (nonlinear) dynamical systems and  further introduces tools from computational mathematics, required in the subsequent chapters.
Already in Section~\ref{sec:bp:map}, we have witnessed some of the benefits that result from considering BP as a nonlinear dynamical system.
Framing BP as discrete time map allows us to draw from the rich history of dynamical systems (cf.~\cite{Teschl2003,Scheinerman1996}) and will be particularly beneficial in studying BP's convergence properties.
Before we discuss how to analyze the convergence properties of a discrete-time map, a proper definition of the terminology is required.
Subsequently, we express the essential steps in analyzing discrete-time maps, present multiple ways of doing so, and compare the individual advantages.  

%
We begin this chapter with the definition of discrete time maps in Section~\ref{sec:solving:map}.
We then illustrate how the problem of computing the fixed points reduces to solving a polynomial system in Section~\ref{sec:solving:fixed_points}, before we discuss and compare different approaches for solving polynomial systems in Section~\ref{sec:solving:comparison}.
Finally, we introduce the notion of stability and show how to assess the stability of a fixed point  in Section~\ref{sec:solving:stability}.

\newsection{Discrete Time Map}{solving:map}
Let us first introduce the notation of a discrete time map now:
\begin{defn}[Discrete Time Map]
	A discrete time map is defined by an $n$-dimensional state vector $\stateVec{x} \in \REAL^n$ and a function $\map:\REAL^n \rightarrow \REAL^n$ that acts on the state vector according to 
	$\stateVec{x}(n+1) = \map \big( \stateVec{x}(n) \big)$, where $n\in \INTEGERPos$ is the time index. 
	Note that the function $\map$ is a composition of $n$ functions $f_i: \REAL^n \rightarrow \REAL$.
	We will denote the value of the state vector at time $n$ by $\stateVec{x}^{n} \definition \stateVec{x}(n)$ so that
	\begin{align}
		\stateVec{x}^{n+1} = \map\left( \stateVec{x}^{n} \right). \label{eq:defn:map}
	\end{align}
	A  discrete-time map is further said to be \emph{nonlinear} if some (or all) functions $f_i(\stateVec{x})$ are nonlinear.
	\label{defn:map}
\end{defn}

One of the overarching goals in the field of dynamical systems is to understand how the system responds and evolves for given initial conditions in order to gain knowledge about the global behavior.
In particular, this knowledge can be gained by first obtaining all fixed points of a map, for which the state vector $\stateVec{x}$ remains unaffected by the application of~\eqref{eq:defn:map}.
More formally, a \emph{fixed point} of the map $\map$ is a state vector $\fp{\stateVec{x}} \in \REAL^n$ that maps to itself so that
\begin{align}
	\fp{\stateVec{x}} = \map(\fp{\stateVec{x}}).
	\label{eq:defn:fp}
\end{align}
The problem of obtaining all fixed points is formally introduced in Section~\ref{sec:solving:fixed_points} before we compare different approaches in Section~\ref{sec:solving:comparison}.
The knowledge of all fixed points alone, however, does not provide sufficient knowledge about the global behavior of a system and is only part of the whole picture.
According to the definition in~\eqref{eq:defn:fp}, fixed points are not affected by~\eqref{eq:defn:map}, though it is not revealed if and under which conditions $\stateVec{x}$ will converge to a specific fixed point.
We are, however, interested in determining which fixed points can be obtained through the repetitive application of~\eqref{eq:defn:map}, i.e., whether a fixed point is stable (cf. Section~\ref{sec:solving:stability}).

\newsection{Finding Fixed Points of a Discrete Time Map}{solving:fixed_points}
Fixed points of a discrete-time map are state vectors that remain unaffected by the map; 
the fixed points consequently correspond to the solutions of the \emph{fixed point equations}  $\fp{\stateVec{x}} - \map(\fp{\stateVec{x}}) = 0$ per definition.
Thus, computing the fixed points boils down to solving a system of polynomial equations, i.e., of nonlinear algebraic equations.
However, albeit straightforward in principle, solving the fixed point equations is by no means straightforward to do in practice;
in particular if confronted with a nonlinear system. 

Systems of polynomial equations arise quite naturally in many engineering applications and it is a classical and important problem in computational mathematics to solve them.
But, despite their relevance and the long research-history, it is still an active field of research. 
Today, solving systems of polynomial equations brings together many different branches of mathematics 
(such as topology, numerical mathematics, geometry, and algebra)
that are joined under the term \emph{algebraic geometry}.
However, the body of the corresponding literature is vast and -- as different fields play together -- requires deep knowledge in many subjects. 
Moreover, the topic is unfortunately mainly studied from a theoretical point of view, which often prevents the application to practical problems and poses a notable hurdle for non-experts.

Although many approaches for solving polynomial systems seem promising at first, it is a daunting task to understand their properties and capabilities enough as to select an appropriate method, well-suited for the task at hand.
Indeed, the ability to deal with large systems of polynomial equations was a major contribution and a prerequisite for many results that are obtained in Chapter~\ref{chp:solutionsBP}.
Although it is possible to read those sections and to understand the ultimate insights, this thesis should be as self-contained as possible. 
Therefore we give a brief (and far from complete) summary of the most prominent ways for solving polynomial systems subsequently.
In particular, we will discuss numerical methods (Section~\ref{sec:solving:numerical}), symbolical methods (Section~\ref{sec:solving:symbolic}), and numerical polynomial homotopy continuation methods (Section~\ref{sec:solving:nphc}).

\newsubsection{Systems of Polynomial Equations}{solving:problem_formulation}
Before we delve into the subtleties that come with the problem of solving systems of polynomial equations and before we compare some fundamentally different approaches, we have to provide some definitions first.
Note that our notation is adapted from~\cite{cox1992ideals}, 
where a comprehensive overview is provided and the underlying concepts for solving polynomial systems are described in great detail.

\newpage
\begin{defn}[Polynomial]\label{def:map}
	A polynomial $\eq$   in the variables $\stateVec{x} = \{x_1,\ldots,x_n\}$, and with coefficients $\coeff{1},\ldots,\coeff{K}$  in $\COMPLEX^n$, is a function of the form
	\begin{align}
		f(x_1,\ldots, x_n) \definition \sum_{k=0}^{K} \coeff{k}x_1^{k_1}\cdots x_n^{k_n}.
		\label{eq:polynomial}
	\end{align}
	We say that the \emph{monomial} $\stateVec{x}^k = \prod_{i=1}^{n}x_i^{k_i}$ is of degree $|k| = \sum_{i=1}^{n}{k_i}$ and that $\eq$ is of degree $\max\limits_{0\leq k\leq K} |k|$. 
\end{defn}

One fundamental concept in commutative algebra is the notion of a polynomial ring $\mathbb{K}[x_1,\ldots,x_n]$.
It consists of all possible polynomials (constructed as in~\eqref{eq:polynomial}) in the variables $\{x_1,\ldots,x_n\}$ with coefficients $\coeff{i} \in \mathbb{K}$ (see~\cite[Appendix A]{cox1992ideals} for a more formal definition).

We will only work with polynomials 
that have their variables defined over the complex numbers, i.e., $\stateVec{x} \in \COMPLEX^n$ in this thesis.
Note that we are actually only interested in polynomials with real variables, i.e., with $\stateVec{x} \in \REAL^n$.
Some methods do, however, 
require the definition over complex variables.
Moreover,  
numerical inaccuracies may introduce a non-zero, albeit negligible, imaginary part to an originally real solution.
Considering only $\stateVec{x} \in \REAL^n$ leads to the failure of obtaining these solutions.
Also, considering the complex domain does no harm as we can later easily choose the real solutions.

Now let us consider a \emph{system of polynomial equations} that we intend to solve:
\begin{defn}[Polynomial System]
	A polynomial system, or set of polynomial equations, $\setOfEq$ consists of $s$ polynomials $\eq[1],\ldots, \eq[s]$ in the variables $\{x_1,\ldots, x_n\}$ and is a function of the form 
	\begin{align}
		\setOfEq(x_1,\ldots,x_n) \definition 
		\begin{cases}
			f_1(x_1,\ldots,x_n)\\
			\hspace{1.25cm}\vdots \\
			f_s(x_1,\ldots,x_n).
		\end{cases}
	\end{align}
\end{defn}

We are interested in obtaining the set of solutions for which all equations equate to zero.
This defines a \emph{variety}, a classical object in algebraic geometry.

\begin{defn}[Set of Solutions (Variety)]
	A variety, or the set of solutions, $\variety(\setOfEq)\subset \COMPLEX^n$, is the set of variables $\stateVec{x}$ for which all equations equate to zero, i.e., 
	\begin{align}
		\variety(\setOfEq) \definition \{\stateVec{x} \in \COMPLEX^n : \eq[i](\stateVec{x}) = 0 \quad \text{for all} \quad \eq[i](\stateVec{x}) \in \setOfEq(\stateVec{x})\}.
	\end{align}
\end{defn}
Note that the variety is defined over the complex numbers as well;
we further define the set of solutions over the real and the strictly positive real numbers $\variety_{\REALPos}(\setOfEq)\subset \variety_{\REAL}(\setOfEq) \subset \variety(\setOfEq)$.

\newsection{Solving Polynomial Systems: A Comparison of Methods}{solving:comparison}
We are now well prepared to finally present how systems of polynomial equations are actually solved.
A great variety of approaches have been developed over the years; 
we present some of the most well-known approaches and introduce the underlying concepts to highlight some of their advantages and disadvantages.
It shall be noted that this is by no means an exhaustive overview. 
Instead of discussing all methods in full detail, which would be beyond the scope of this thesis, we rather aim to outline the underlying concepts for each method.

Every method comes with its specific properties, that have to be considered and should ideally be aligned with the requirements of a given problem;
in particular, one should answer the following questions before selecting one particular method to solve the polynomial system (cf.~\cite[p.68]{sommese2005numerical}):
\begin{itemize}
	\item Do we require the set of all solutions $\variety(\setOfEq)$, or just a single possible solution?
	\item What are the specifications of the problem under consideration (sparsity, problem size, etc.) and which methods are well-suited to exploit those properties?
	\item Is the method already available in some software-package? 
	\item Are there any hyper-parameters that influence the efficiency of the method and how much experience and effort is required in adapting the parameters to the problem under consideration?
\end{itemize}

The subsequently presented methods belong to -- and are the more prominent examples of -- three fundamentally different approaches:
numerical solvers, symbolic methods, and numerical polynomial homotopy continuation (NPHC) methods.
We will evaluate and compare the different approaches with respect to the above-mentioned properties.
Note that although the most relevant properties for the subsequent chapters will be pointed out, it would be presumptuous to speak of an all-encompassing list of advantages and disadvantages. 

\newsubsection{Numerical Methods}{solving:numerical}
One basic and well-established method for solving systems of nonlinear equations is the Newton-Raphson method (often just called Newton's method) which is an iterative solver that progressively refines an initial guess to reach a solution (see~\cite[pp.30]{cox2006using} for a brief introduction, or some classical books on numerical analysis, e.g.~\cite[Chapter~2.3]{burden_numerical}).

The underlying idea is to approximate the considered system by a first-order Taylor approximation in an initial guess, and to successively refine this guess by proceeding along a series of first-order Taylor approximations.

Newton's method is a powerful tool that often converges quickly to the solution, particularly if one has a good initial guess.
Theoretical results on the convergence rate of such iterative solver are relatively well established (cf.~\cite[Chapter~2.4]{burden_numerical}) and have two particularly relevant implications for our work:
First, an initial guess, sufficiently close to a solution, is required or, otherwise, the iterative solver may diverge of even exhibit chaotic behavior.
Second, many different initializations are required to obtain multiple solutions, and it consequently remains problematic to obtain the full set of solutions with these methods.

\newsubsection{Symbolic Methods}{solving:symbolic}
From a completely different point of view, symbolic methods \cite{cox1992ideals,cox2006using}
(e.g. Gr\"obner basis method, Border basis method, Wu's method, and method of sparse resultant) rely on symbolic manipulation of the polynomial system and successive elimination of variables to obtain a simpler but equivalent form.
In a sense, these methods generalize the Gaussian elimination method from linear systems into the nonlinear settings.

Before we discuss the details of symbolic methods, we need to introduce some fundamental notions from  commutative algebra.
One particularly relevant concept is the notion of an ideal.
\begin{defn}[Ideal]
	Consider a polynomial system $\setOfEq(x_1,\ldots,x_n)$ in a polynomial ring $\POLY[x_1,\ldots,x_n]$ and some other polynomial system $\equationSystem{H}(x_1,\ldots,x_n)$ in the same polynomial ring.
	The ideal $\ideal{f}$ is closed under addition and multiplication and defines the set of all polynomial systems that satisfy
	\begin{align}
		\ideal{f} \definition \bigg\{\sum_{i=1}^s h_i f_i :\equationSystem{H} \in \POLY[x_1,\ldots,x_n]\bigg\}.
		\label{eq:ideal}
	\end{align}
	We say that the ideal $\ideal{f}$ is generated by the basis $\setOfEq({\stateVec{x}})$.
\end{defn}
The above definition of the ideal implies that any element of the ideal equates to zero if 
$\setOfEq(x_1,\ldots,x_n) = 0$ (cf.~\cite[pp.30]{cox1992ideals}).
This is an immediate consequence of~\eqref{eq:ideal} and it has important implications for obtaining the variety $\variety(\setOfEq)$.
Not that, in fact, the Hilbert's basis theorem states that every ideal $\ideal{f}$ is generated by a finite set of equations.
This implies that the variety of the ideal 
\begin{align}
	\variety(\ideal{f}) = \big\{\stateVec{x}\in\POLY[\stateVec{x}] :\eq(\stateVec{x}) = 0  \quad \text{for all} \quad \eq(\stateVec{x}) \in \ideal{f}\big\} \nonumber
\end{align}
must equal the variety of its generator, i.e., $\variety(\setOfEq) = \variety(\ideal{f})$ (cf.~\cite[Proposition 9]{cox1992ideals}).
This has one particularly important consequence, namely that two polynomial systems $\equationSystem{F}(\stateVec{x})$ and $\equationSystem{G}(\stateVec{x})$ that generate the same ideal, i.e., for which $\ideal{f} = \ideal{g}$, 
will also have the same variety, i.e., $\variety(\setOfEq) = \variety(\equationSystem{G})$.

Note that every ideal has not just one but many different generators; 
this is a crucial prerequisite for the concept of symbolic methods.
Moreover, this allows us to draw an analogy to linear algebra: in this analogy, the ideal corresponds to some subspace (that is closed under addition and multiplication) that is spanned by some vectors (that correspond to the polynomials in the generator).

Symbolic methods take a fundamentally different approach to solving systems of polynomial equations  $\setOfEq(\stateVec{x})$ than numerical methods.
Whereas numerical methods aim to estimate the solution directly, symbolic methods focus on computing an alternative generator $\equationSystem{G}(\stateVec{x})$ with ``good'' properties.
The overarching aim is to construct $\equationSystem{G}(\stateVec{x})$ so that both generators have the same ideal $\ideal{g}=\ideal{f}$ -- and thus identical solutions, and to concurrently impose certain properties on $\equationSystem{G}(\stateVec{x})$ that facilitate computing the set of solutions.
In particular, we hope that $\equationSystem{G}(\stateVec{x})$ consists of equations that have a low degree and are in an upper triangle form (to take up the analogy with linear algebra, the aim is to find a more intuitive description of the subspace, e.g., by orthonormal vectors).

\newsubsubsection{Gr\"obner Basis}{solving:symbolic:groebner}
One particularly successful way to 
construct a good basis is to compute the (reduced) Gr\"obner basis. 
Together with the Buchberger algorithm that computes such a basis, Gr\"obner bases were introduced in the Thesis of Bruno Buchberger~\cite{buchberger_thesis, buchberger_en}.

Gr\"obner bases tend to consist of equations of low degree that are -- if possible - in upper triangle form. A formal and more comprehensive treatment of the Gr\"obner basis method can be found in many textbooks (e.g., in~\cite{cox1992ideals}); here we shall be content with illustrating the underlying concepts exemplary.

\begin{example}[Gr\"obner Basis for Solving System of Equations]$ $ \newline
	Let us solve a simple polynomial system $\setOfEq(\stateVec{x})$ with three variables in three equations by utilizing the Gr\"obner basis method.
	Therefore, consider
	\begin{align}
		f_1(\stateVec{x}) &= x_1^2+x_2^2-1\\
		f_2(\stateVec{x}) &= (x_1-1)^2+x_2^2-1\\
		f_3(\stateVec{x}) &= x_1^2+x_2^2+x_3^2-2.
	\end{align}
	We then compute the Gr\"obner basis using Wolfram Mathematica and obtain an alternative basis according to
	\begin{align}
		g_1(\stateVec{x}) &= x_3^2-1\\
		g_2(\stateVec{x}) &= 4x_2^2-3\\
		g_3(\stateVec{x}) &= 2x_1-1.
	\end{align}
	Note that all equations $g_i(\stateVec{x}) \in \equationSystem{G}(\stateVec{x})$ 
	depend only on a single variable and have low degree, thus satisfying all desired requirements.
	One can consequently obtain all four solutions from $\equationSystem{G}(\stateVec{x})$ in a straight-forward manner; these are $(x_1,x_2,x_3)=(\frac{1}{2}, \pm \frac{\sqrt{3}}{2}, \pm 1)$.
\end{example}

The Gr\"obner basis method obviously obtains the full set of solutions (as compared to only a single one by numerical methods).
Furthermore, methods that obtain the Gr\"obner basis are particularly attractive for solving systems of polynomial equations as many algorithms are available and included in many software packages.

The application of the Gr\"obner basis method to large-scale problems with many variables demands a problem-dependent adaption of the algorithm, however, that requires good knowledge of algorithmic details. 
Even more problematic is the fact that Gr\"obner bases are unstable under small changes in the coefficients, which limits its application to problems with rational coefficients.
We exemplify this numerical instability in Example~\ref{ex:instability_groebner}.

\begin{example}[Instability of Gr\"obner Basis]$ $\newline
	The Gr\"obner Basis is numerically unstable with respect to small changes in the parameters.
	Let us exemplify this by the following system of equation $\setOfEq(\stateVec{x})$ (inspired by~\cite[Example 6.4.1]{kreuzer2005computational}):
	we consider two ellipses defined by the following two polynomials in two variables
	\begin{align}
		f_1(\stateVec{x}) &= \frac{1}{4}x_1^2+x_2^2-1\\
		f_2(\stateVec{x}) &= x_1^2+\frac{1}{4}x_2^2-1.  
	\end{align}
	Again, the Gr\"obner basis provides a well-behaved  system of polynomial equations
	\begin{align}
		g_1(\stateVec{x}) &= 5x_1^2-4\\
		g_2(\stateVec{x}) &= 5x_2^2-4,
	\end{align}
	that yields all four solutions $(x_1,x_2) = (\pm \frac{2}{\sqrt{5}}, \pm \frac{2}{\sqrt{5}})$ immediately.
	
	Now, assume we rotate the ellipses slightly by changing the underlying system of equations to $\equationSystem{\tilde{F}}(\stateVec{x})$ with
	\begin{align}
		\tilde{f}_1(\stateVec{x}) &= \frac{1}{4}x_1^2+x_2^2-1+0.001 x_1 x_2\\
		\tilde{f}_2(\stateVec{x}) &= x_1^2+\frac{1}{4}x_2^2-1+0.001 x_1 x_2.
	\end{align}
	Note that this minor rotation shifts the four solutions change only slightly (cf. Figure~\ref{fig:instability_groebner}); 
	in fact, the solutions to $\equationSystem{\tilde{F}}(\stateVec{x})$ reside within a radius of $5\cdot10^{-4}$ of the solutions to $\equationSystem{F}(\stateVec{x})$.
	On the other hand, however, a dramatic change occurs in the corresponding Gr\"obner basis, both in the coefficients and in the monomials: 
	\begin{align}
		\tilde{g}_1(\stateVec{x}) &= 1250x_2^3-1000x_2+x_1\\
		\tilde{g}_2(\stateVec{x}) &= x_2^4- 1.6x_2^2+0.64.
	\end{align}
	Not only do the Gr\"obner bases $\equationSystem{G}(\stateVec{x})$ and $\equationSystem{\tilde{G}}(\stateVec{x})$ differ in their structure, but -- even more problematic -- computation of $\variety(\equationSystem{\tilde{G}})$ now suffers from numerical issues.
	The large coefficients in $\tilde{g}_1(\stateVec{x})$ introduce an extreme sensitivity with respect to numerical accuracy  and thus render the computation of the solutions by backward substitution highly problematic.
	%
	\begin{center}
		\includegraphics[width=0.75\linewidth]{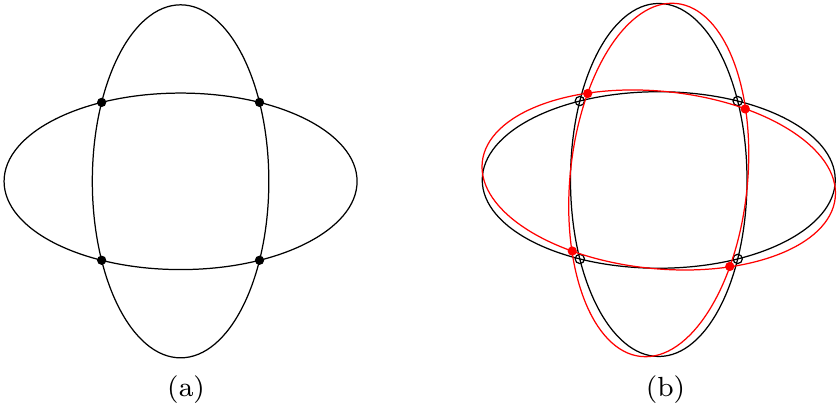}
		\captionof{figure}{(a) Two ellipses defined by $\equationSystem{F}(\stateVec{x})$ (illustrated in black), where the solutions are depicted by solid black dots;
			(b) two rotated ellipses defined by $\equationSystem{\tilde{F}}(\stateVec{x})$ (illustrated in red), where the solutions are depicted by solid red dots (with the solutions of $\equationSystem{F}(\stateVec{x})$ depicted by black circles).
			Note that the solutions of both equation systems are similar.}
		\label{fig:instability_groebner}
	\end{center}
	\label{ex:instability_groebner}
\end{example}


According to the implicit function theorem, the solutions of $\equationSystem{\tilde{F}}(\stateVec{x})$ will change only slightly if the coefficients are modified  by an infinitesimally small value;
the Gr\"obner basis, however, may experience a dramatic shift.
In fact, the result will  also differ depending on the representation (decimal, or fraction of two integers) of the coefficients in $\equationSystem{\tilde{F}}(\stateVec{x})$.
This poses a drastic problem for any polynomial system with non-rational coefficients as the representation of the coefficients will have a major influence on the final result.

\newsubsubsection{Other Symbolic Methods}{solving:symbolic:other}
The limitation of the Gr\"obner basis approach to problems with rational coefficients and the numerical instabilities pose a considerable problem for many applications. 
Consequently, to extend the applicability of symbolic methods, much attention was focused on circumventing these issues and extending the range of solvable polynomial systems.

All symbolic methods adhere to one common characteristic; that is the search for an alternative well-behaved generator.
There are various ways of obtaining this generator.
One promising approach relates the computation of the generator (non-linear in the polynomial ring) to a similar problem in the quotient ring.
The  consideration in the quotient ring renders the problem a linear one that can be solved efficiently with well-established methods from linear algebra, e.g., by an eigendecomposition.
This fruitful connection was first recognized in the work of Winfried Auzinger and Hans J. Stetter.\footnote{ 
	Stetter who was already about to retire found interest in the topic and prolonged his active research time by another 10 years; the results of which are summarized in his book~\cite{stetter2004numerical}.}
Their work already revealed many expedient properties that come with working on the quotient ring~\cite{auzinger1988elimination, auzinger1989study}.
This was only a first step, however, one that avoided the actual construction of the quotient ring~\cite{stetter2004numerical}.
Nonetheless, the underlying ideas proofed to be worth further pursuing and inspired a range of results.
The inherent connection between the representation in the quotient ring and the problem of solving algebraic equations was first made explicit in~\cite{moller1993systems}.
This connection finally led to the first working algorithm, the M\"oller-Stetter method~\cite{moller1995stetter}. An excellent overview of the underlying concepts and the initial developments can be found in~\cite{stetter2004numerical}.

%

After laying out the foundation, many promising advances were proposed, often encompassed under the name of Border basis.
This development is still ongoing and the advances are scattered over many publications.
We refer to some of the excellent comprehensive papers available that provide a good overview of the recent advances and may serve as a good starting point for the interested reader.
The beneficial properties of Border basis are outlined and discussed in the tutorial paper~\cite{mourrain2007pythagore};
likewise~\cite{kehrein2005algebraist} summarizes the major properties and provides some accessible examples that show how to compute the Border basis.
In a nutshell, the Border basis method combines the advantages of the Gr\"obner basis method,
while avoiding some of the most relevant disadvantages (e.g., numerical instabilities).
Unfortunately, however, much of the results remain of academic nature and have not yet found their way into existing software packages, thus severely limiting the potential application to practical problems.

We would like to stress that the toolbox of symbolic methods is of course not limited to Gr\"obner- and Border basis but contains many more methods, all with their own set of advantages and drawbacks. 
Nonetheless, we hope that this section provides a gentle introduction to some of the more established methods;
a more detailed treatment of symbolic methods is available in a number of books (and the references therein)~\cite{cox1992ideals, cox2006using, kreuzer2000computational, kreuzer2005computational, dickenstein2005}.

As a final statement, we conclude by stressing the overall drawback of symbolic methods:
it is often expensive to apply these methods to problems of high degree.
This is particularly critical as the sparsity in polynomial systems cannot be utilized in a straightforward manner so that even small -- but high dimensional -- systems suffer from this drawback.

\newsubsection{Numerical Polynomial Homotopy Continuation (NPHC) Method}{solving:nphc}

Another important approach for solving a system of polynomial equations is the \emph{numerical polynomial homotopy continuation} (NPHC) method~\cite{li2003solving,sommese2005numerical}.
The NPHC method performs multiple stages in order to compute the solutions of the \emph{target system} $\setOfEq(\stateVec{x})$.

First, the target system is inspected and a root count is computed that provides an upper bound on the number of solutions.
Second, a closely related \emph{start system} $\startSys(\stateVec{x})$ is created that is trivial to solve and has precisely as many solutions as suggested by the root count.
Third, the start system is continuously deformed into the target system.
Finally, with appropriate construction, the trivial solutions of the start system also vary continuously under this deformation forming \emph{solution paths} that connect to the desired solutions of the target system.

For instance, one basic form of a homotopy is given by
\begin{align}
	\homotopy(\stateVec{x},t) = (1-t)\startSys(\stateVec{x}) + \gamma t \setOfEq(\stateVec{x}) = 0,
	\label{eq:homotopy}
\end{align}
for $t \in (0,1]$ and with $\gamma \in \COMPLEX$.
Clearly, at $t = 0$ the homotopy reduces to the start system $\startSys(\stateVec{x})$ and at $t=1$ it reduces to the target system $\setOfEq(\stateVec{x})$.
As $t$ varies continuously from 0 to 1, the homotopy represents a deformation from the start system to the target system and the NPHC method tracks the solutions.

\begin{example}[NPHC method for Solving an Equation]$ $\newline
	Although the NPHC method is particularly suited for large systems of equations with many variables, we consider one minimalistic example for illustrative purposes.
	Let us consider the following target system consisting of a single equation $f(x)$ in a single complex variable $x$  
	\begin{align}
		\setOfEq(\stateVec{x}) = x^2 +ix -2 = 0.
		\label{eq:eq:nphc1}
	\end{align}
	The root count for~\eqref{eq:eq:nphc1} is straightforward to estimate and just looking at the system reveals the existence of two solutions.
	This consequently requires a start systems with two solutions as well; one of the simplest system of equations with two solutions is for example given by
	\begin{align}
		\startSys(\stateVec{x}) = (x-1)(x+1) = 0.
		\label{eq:eq:nphc2}
	\end{align}
	We can immediately recognize both initial solutions of the start system  $x_s = (-1,+1)$.
	
	The homotopy $\homotopy(x,t)$ is finally given by
	\begin{align}
		\homotopy(x,t) = (1-t)(x-1)(x+1) + \gamma t (x^2 +ix -2) = 0,
	\end{align}
	where $\gamma = \exp(i\theta)$ with a random angle $\theta\in[0,2\pi)$.
	A couple of solution paths, that emerge for different values of $\theta$, are illustrated in Figure~\ref{fig:example_nphc}.
	One can see that the angle $\theta$ influences the shape of the solution path.
	Every chosen value $\theta$ has a well-behaved solution path that connects the initial solutions of the start system $(x_{s,0}, x_{s,1})$ to the desired solutions of the target system $(x_{t,0}, x_{t,1})$.
	Using a predictor-corrector method to proceed along the solution path, the NPHC method then obtains both solutions of the target system 
	$x_t = (\frac{\sqrt{7}-i}{2}, - \frac{\sqrt{7}-i}{2})$.
	\begin{center}
		\includegraphics[width=0.5\linewidth]{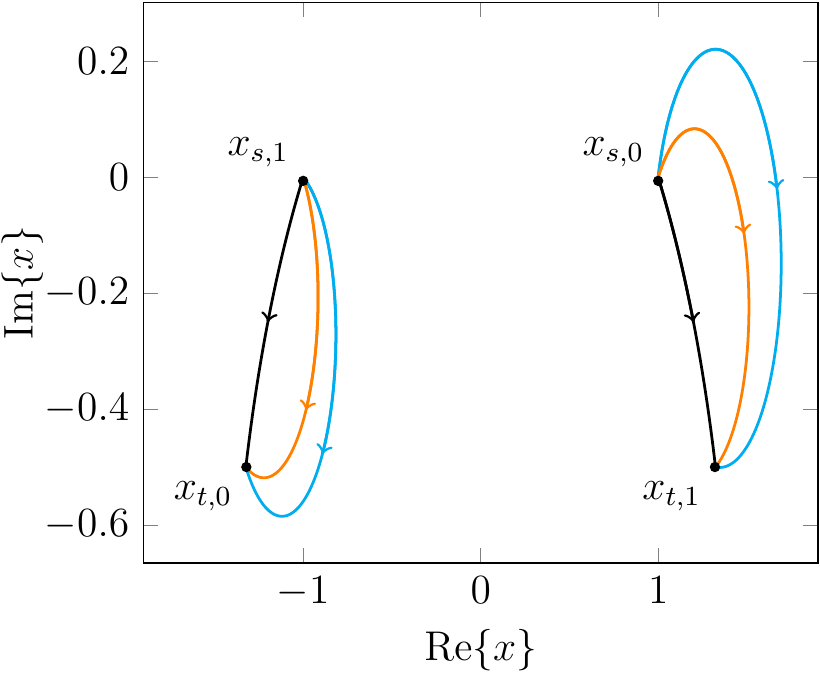}
		\captionof{figure}{Solution paths for $\theta \in \{0, 1, \pi/2\}$.
			All solution paths lead from the solutions of the start system $x_{s,i}$ to the solutions of the target system $x_{t,i}$.}
		\label{fig:example_nphc}
	\end{center}
\end{example}

Although the NPHC method seemingly provides a simple method to obtain the desired solutions we have spared some of the subtleties that may arise; there are two ingredients in particular that require careful attention and were not mentioned in detail so far. 
These are the \emph{path-tracking algorithm} that keeps track of the solution as $t$ increases, the computation of the \emph{root count}, and the creation of the \emph{start system}.

\newsubsubsection{Tracking the Solution}{tracking}
Because of its importance, the process of tracking the solution paths with increasing $t$ is central to the NPHC method.
It is thus important to consider robust and proficient path-tracking algorithms.
We will present the key-steps of such path-tracking algorithms in the context of the NPHC method, closely following~\cite[Section 2.3]{sommese2005numerical}, and refer to~\cite{allgower2003numerical} for a more  exhaustive description of path-tracking in general.

A generic approach  must proceed along the solution paths, defined by the continuous homotopy $\homotopy(\stateVec{x},t)$ with known initial values $\homotopy(\stateVec{x},0) = 0$, using a predictor-correction method.
Such an iterative algorithm proceeds according to the following three steps.
In the prediction step, one fits a function to the current values and predicts the values of $\stateVec{x}$ at the next step.
This can be done either by extrapolating a linear (or higher-order) function based on the last couple of points, or by linearizing in the last point and proceeding along the tangent direction. 
Note that the latter approach is well known as Euler's method.

In the correction step, one corrects the predicted value, using a numerical method as e.g., Newton's method.
Despite its shortcomings discussed before, Newton's method is well suited for the correction step as the prediction step will estimate a value close enough for Newton's method to converge.
In general, the accuracy of the prediction will increase with a reduced step-size; 
if necessary, it is thus always possible to enforce convergence by adaptive reduction of the step-size.

Before repeating the above procedure and predicting the next value again,
the step-size should ideally adapt according to the current  correction step. 
If the preceding prediction was very accurate it is often safe to assume that the step-size can be enlarged, whereas it is preferable to reduce the step-size otherwise.
Overall, the step-size significantly impacts the overall performance; 
a step-size too small introduces unnecessary many iterations and a step-size too large may lead to failure in the correction step.

These steps, prediction, correction, and step-size adaption are then repeated until the solutions of the target system are obtained for $t=1$. \\

An important prerequisite for path-tracking algorithms is the existence of a well-behaved path. 
Fortunately, the introduction of some generic hyper-parameter, as for example the random complex number $\gamma$ in~\eqref{eq:homotopy}, produces such well-behaved paths with probability one~\cite[Lemma 7.1.3]{sommese2005numerical}.\footnote{ The main argument of the proof is that for all possible values of $\gamma \in \COMPLEX$, except for a finite amount of distinct values, the solution paths are well-behaved. Picking the value $\gamma$ at random will therefore result in a well-behaved path with probability one.}

\newsubsubsection{Start System}{startsys}
Another crucial detail in the NPHC method is the computation of the root-count and the creation of the start system.
There are various ways of computing this bound that differ in their complexity.

Loose bounds are often straightforward to compute but may lead to -- potentially many -- solutions in the start system that do not correspond to any solution in the target system.
Still, one has to track all emerging solution paths, before finally recognizing the superfluous ones that diverge as the homotopy resembles the target system.
Besides unnecessarily wasting computational resources this may have even more serious implications; 
a root count  too large may render solving a problem infeasible because of the sheer amount of paths that need to be tracked.

Various methods have been proposed to bound the  number of solutions (cf.~\cite[Section 8.1]{sommese2005numerical}) with more sophisticated methods generally reducing the number of paths to be tracked.
Nonetheless, it is safe to assume that the computational effort required for estimating the bound increases with aiming for a tighter bound.
Moreover, closely related to the estimation of the root count, the creation of the start system becomes more intricate as well.

For some problems, the consideration of more sophisticated methods scales nicely with the problem and the reduction in the number of paths outweighs the additional burden of creating an appropriate start-system.
For other problems, however,  estimating a more accurate root count requires significantly more computational resources than the second task, of tracking all paths to their solutions, altogether.

The optimal choice of method -- that introduces as few paths as possible and spends as few resources for the creation of the start system as needed -- is thus vital for the overall performance of the NPHC method.
Unfortunately, general guidelines do not exist and choosing a reasonable method for a given problem is usually based on experience. \\

\newsubsubsection{Properties of the NPHC method}{solving:nphc:proeprties}
The NPHC method takes the structure and the sparsity of the target system $\setOfEq(\stateVec{x})$ into account.
If the number of solutions is small, a well-chosen start system will also have only a few solutions thus implicitly exploiting the sparsity of the equation system.

The fundamental concept is that every solution of the target system has one corresponding solution residing in the start system. 
Looking at~\eqref{eq:homotopy} it becomes apparent that, as $t$ goes from 0 to 1, every solution path is completely independent of all others.
This property admits parallel implementations that track all solution paths independently after the creation of the start system.
The NPHC method is thus, in principle, suited for solving much larger systems than symbolic methods that cannot utilize parallelism in such a way.

Moreover, the NPHC method is implemented and available in a couple of software packages that often have interfaces with established software packages as, e.g., Maple, Matlab, or Python.
Existing software packages, only naming a few, include \textsf{PHCpack}~\cite{phc_pack}, \textsf{bertini}~\cite{bertini}, and \textsf{Hom4PS-3}~\cite{chen2014hom4ps}.
The existence of such established software packages makes the NPHC method attractive from a practical perspective and lowers the entrance hurdle.

One fundamental disadvantage of the NPHC method is the need for working in the complex domain, even though only the real solutions are of interest.
Although the real solutions are contained in the set of complex solutions, working in the complex domain may lead to some problems:
On the one hand, the NPHC method may fail if some of the solutions are not zero-dimensional, i.e., the solutions are continua instead of points.
Restricting ourselves to the real line, these solutions could be zero-dimensional; yet the NPHC method will inevitably fail for its need to work in the complex domain.
On the other hand, from a more practical perspective, the consideration of the complex field may increase the computational requirements significantly. 
This is most obvious if considering a problem with only a few real solutions but a huge amount of complex solutions that induce an equally huge amount of solution paths to be tracked.
As it is not possible to determine the complex solutions upfront, all solution paths have to be tracked before neglecting most of them only after the termination of the NPHC method.

\newsection{Stability Analysis of Fixed Points}{solving:stability}
As discussed at the beginning of this chapter, the knowledge of the fixed points alone is not sufficient in getting an adequate understanding of the considered map.
All fixed points remain constant under repeated application of the considered map per definition, but may be different in their nature.
To account for that difference we classify the fixed points as stable or unstable, corresponding to the long-term behavior of the map.

We start by getting  some intuition behind the meaning of stable and unstable fixed points.
Therefore, consider a dynamical system description of how a ball moves through a surface with valleys and peaks. 
By neglecting all kinds of forces that occur in real-world, except for gravity, we get a particularly simple model: 
along the slope, the ball will always accelerate downwards and decelerate upwards.
Fixed points are consequently flat points on the surface with zero slope.
Assume that ball sits on the top of a peak; 
then, moving the ball slightly will throw it out of equilibrium 
as it starts rolling towards another fixed point: we  refer to the peak as an unstable fixed point.
If, on the other hand, the ball resides at the bottom of a valley, the ball will return back to the fixed point after moving it slightly: we refer to the valley as a stable fixed point.

In analogy, the stability of a fixed point can be defined by the trajectory that the state vector describes after setting the state vector to some value close to a fixed point and whether repeated iteration of the map brings it back to the considered fixed point or not.

Let us make this notion of  stability more precise now: 
A fixed point $\fp{\stateVec{x}}$ is \emph{(asymptotic) stable} if a neighborhood $U(\fp{\stateVec{x}})$ exists such that any $\stateVec{x}^{1} \in U(\fp{\stateVec{x}})$ converges to $\fp{\stateVec{x}}$; i.e., for every (small) value of $\epsilon > 0$ there is a maximum number of iterations $N$ so that $\stateVec{x}^{n}$ is $\epsilon$-close to $\fp{\stateVec{x}}$ for $n>N$.
A fixed point $\fp{\stateVec{x}}$ is \emph{unstable} if $\stateVec{x}$ diverges from the fixed point for some $\stateVec{x}^{1} \in U(\fp{\stateVec{x}})$.\footnote{Note that we do not discuss marginally stable fixed points, i.e., fixed points around which the state vector does neither converge nor diverge. Nonlinear systems require a more involved analysis as considered here if marginally stable fixed points are of relevance.}

\newsubsection{Stability Analysis for a Nonlinear Discrete Time Map}{solving:stability:linearization}
We are now going to discuss a simple, yet powerful, method to assess the stability of a fixed point $\fp{\stateVec{x}}$ under a specific discrete-time map $\map(\cdot)$.
Therefore, we only require some tools from linear algebra.
Note that we can express every linear map as a linear system of equations, and write it in matrix form so that  $\map(\stateVec{x}) =  \vm{B}\stateVec{x}$ where the $i\textsuperscript{th}$ equation is given by
\begin{align}
	\eq[i](\stateVec{x}) = \sum_{k=0}^{n} a_{i,k}x_k.
	\label{eq:linear_sys_matrix}
\end{align}
This representation already suggests how to assess the stability of a fixed point.
First of all, according to~\eqref{eq:defn:fp} a fixed point remains unaffected under multiplication by $\vm{B}$.
Consequently, a fixed point $\fp{\stateVec{x}}$ is stable if repeated application of $\vm{B}$ brings the state vector $\stateVec{x} \neq \fp{\stateVec{x}}$ to the fixed point;
accordingly, we call a fixed stable whenever
\begin{align}
	\fp{\stateVec{x}} = \vm{B}^n \stateVec{x}
	\label{eq:stability_linear}
\end{align}
for all values $n>N$ with $n,N \in \INTEGER$.

Taking a closer look at~\eqref{eq:stability_linear} reveals how to assess the stability:
the effect of repeated multiplication by $\vm{B}$  directly relates to the eigenvalues~\cite[Chapter~6]{strang}.
It is a fundamental property of linear algebra that a state vector converges to $\fp{\stateVec{x}}$ if all eigenvalues $\ev{i}$ have a magnitude strictly smaller than one.
Accordingly, we conclude that the linear system $\map(\stateVec{x})$ exhibits a stable fixed point if its associated system-matrix $\vm{B}$ has all eigenvalues $ |\ev{i}| < 1$.\\


The above analysis crucially relies on the matrix representation of the discrete-time map.
Since non-linear maps do not permit this matrix representation it is not directly possible to generalize the stability analysis.
Despite this limitation, we may just approximate the nonlinear map $\map(\cdot)$ by a linear one in every fixed point and analyze the linearized system instead.
Fortunately, the Hartman-Grobman theorem~\cite{hartman, grobman} (see~\cite[pp.264]{Teschl2003} for a proof) allows us to do just that.   
More precisely, the theorem permits the simplification of treating the nonlinear system as a linear one in its fixed points (at least in most cases).

A one-dimensional system is linearized by taking the derivative; 
a multivariate system is linearized by computing all first-order partial derivatives and collecting them in the Jacobian according to
\begin{align}
	\JacobianGeneral{\stateVec{x}} \definition  \begin{bmatrix}
		\frac{\partial f_1}{\partial x_1} &  \cdots & \frac{\partial f_1}{\partial x_n}\\
		\vdots &  \ddots & \vdots\\\
		\frac{\partial f_s}{\partial x_1}    &  \cdots     &\frac{\partial f_s}{\partial x_n} 
	\end{bmatrix}
\end{align}

Then, in accordance with linear system theory, we inspect the eigenvalues of the Jacobian.
Therefore, let us denote the spectrum of the Jacobian, i.e., the set of all eigenvalues by
\begin{align}
	\Spectrum{\JacobianGeneral{\stateVec{x}}} \definition \{ \ev{1},\ldots,\ev{n}\},
\end{align}
and the spectral radius, i.e., the maximum magnitude of all eigenvalue by 
\begin{align}
	\Radius {\JacobianGeneral{\stateVec{x}}} \definition \max_{1\leq i\leq n} |\ev{i}|.
\end{align}

This means that -- whenever $\Radius {\JacobianGeneral{\stateVec{x}}} \neq 1$ -- it is possible for nonlinear systems to infer the stability of a fixed point by looking at the spectral radius of the linearized system.
A fixed point $\fp{\stateVec{x}}$ is stable if all eigenvalues of the Jacobian have absolute value strictly smaller than one and lie inside the unit circle, i.e., if 
\begin{align}
	\Radius {\JacobianGeneral{\fp{\stateVec{x}}}} < 1.
\end{align}
The system $\map(\cdot)$ consequently converges to $\fp{\stateVec{x}}$ if initialized sufficiently close enough.
A fixed point is unstable with respect to $\map(\cdot)$ if at least one eigenvalue exists outside the unit circle so that $\Radius {\JacobianGeneral{\fp{\stateVec{x}}}} > 1$.

The concept of linearization now allows us to assess the stability of a given fixed point.
Note that this still excludes any discussion about the region of attraction;
just because a fixed point is stable under a given map does not imply that every possible initialization will converge to the fixed point.
Nonetheless, the notion of local stability has its merits and tells us whether a map 
can, in principle, converge to a fixed point (if it is stable) or if it can never converge to a fixed point (if it is unstable).

  \emptydoublepage
\newchapter{Solution Space of Belief Propagation: Number~of~Fixed~Points~and~Their~Stability}{solutionsBP}
\openingquote{
Travel makes one modest.\\ You see what a tiny place\\ you occupy in the world.}{Gustave Flaubert}

In this chapter, we represent BP as a dynamical system and analyze its solution space in order to gain a deeper understanding of BP's properties. 
Choosing from the variety of methods presented in Chapter~\ref{chp:solving} for solving the fixed point equations of BP, the NPHC method proves to be capable of providing the full set of fixed points.
This is the first time that this is achieved for finite-size models with non-vanishing local potentials, which reveals a fundamental connection between the Bethe free energy and the accuracy of the marginals. 
The notion of a local stability analysis relies on the tools presented in Section~\ref{sec:solving:fixed_points} and was first applied to BP -- although only for models without local potentials --  in~\cite{mooij2005properties}.
The NPHC method makes the full set of fixed points available and consequently admits an extension of the local stability analysis to more general models.
This extension has far-reaching implications.
For example, it was a common conjecture that strong local potentials positively influence the stability; we finally prove this assumption true.
The empirical results further inspire the derivation of theoretical results that explain the influence of the potentials and the graph size on the convergence properties.

We begin this chapter with a brief summary of existing results on the number of fixed points in Section~\ref{sec:solutionsBP:overview}.
In Section~\ref{sec:solutionsBP:solving} we show how to formulate the fixed point equations for BP and how to use the NPHC method for solving them.
We then reparameterize the messages and discuss the subtleties of applying a stability analysis to BP in Section~\ref{sec:stabilityBP:stability}.
Finally, we specify a range of models in Section~\ref{sec:solutionsBP:fps_selected_models}, before we compute all fixed points and evaluate their accuracy in Section~\ref{sec:solutionsBP:accuracy_selected_models} as well as their stability in Section~\ref{sec:stabilityBP:empirical_analysis}.
We then follow our empirical observations and analyze how the model-size and the potentials affect the stability in Section~\ref{sec:stabilityBP:theoretical_analysis}.
In Section~\ref{sec:solutionsBP:coding} we compute all fixed points and discuss the limitations of BP in the context of error-correcting codes.

Large portions of this chapter have been previously published but were considerably modified to fit into one coherent chapter:
The computation of the full set of fixed points and the evaluation of the accuracy have been published in~\cite{knoll_fixedpoints} and in~\cite{knoll_ws_fixedpoints}.
Shifting the focus away from symbolic methods was triggered by a fruitful discussion with Michael Kerber; 
the application of polyhedral homotopy methods finally proved to be successful with the aid of Dhagash Metha and Tianran Chan. 
In particular the interpretation of the results and the assessment of the accuracy of the BP fixed points are due to the present author.
The local stability analysis is a result of joint work with Franz Pernkopf and is published in~\cite{knoll_stability}.

\renewcommand{\pwd}{solutionsBP}

\newcommand{\nb}[1]{\partial(X_{#1})}

\newsection{Motivation}{solutionsBP:introduction}
The previous chapters introduced BP as an \emph{efficient} method for approximate inference.
For arbitrary models with many loops, however, neither guarantees for convergence nor bounds on the approximation error are established.
It is precisely the existence of many loops, however, that renders exact inference intractable and requires the utilization of approximate inference methods (as for example BP).
A better understanding of BP's capabilities and limitations would therefore be of great relevance for loopy models.
In this chapter, we aim to analyze BP on a range of models in great detail; 
we further aim to establish theoretical properties on the performance of BP and explain why BP often works surprisingly well on dense and loopy graphs, but other times fails to converge or give accurate results.

If multiple fixed points exist, the expected behavior of BP depends on the number of fixed points and the individual fixed point's behavior.
It is thus quite natural to strive for answers to the following questions:
\begin{itemize}
	\item Can we specify model classes -- by restricting, either individually or jointly, the structure and the parameters of probabilistic graphical models -- for which a unique fixed point exists?
	\item Which models admit at least one fixed point to which BP converges for any choice of initial message values (or just if initialized close enough)?
	Does uniqueness of a fixed point imply that BP converges? 
\end{itemize}

These questions have obviously been addressed before: 
sufficient conditions for uniqueness of fixed points were proposed by accounting for both the potentials as well as the graph structure~\cite{heskes2004uniqueness,mooij2007sufficient}.
A direct relationship exists between accuracy and convergence rate for graphs with a single loop~\cite{weiss2000correctness} and for small grid graphs~\cite{ihler07}.
In contrast, graphs do exist that feature surprisingly accurate fixed points although BP fails to converge~\cite{weller2013approximating}.
We shall thus gain a better understanding of the behavior of BP by developing a precise relation among the
\emph{number of fixed points} (and conditions for uniqueness), the \emph{approximation accuracy}, and the \emph{convergence properties}.\\

The consideration of BP as a discrete-time map suggests one way to compute the solution space of BP; the analysis of which reveals deep insights into the behavior of BP and (at least partially) answers the questions raised above.
Applying the tools prepared in Chapter~\ref{chp:solving}, one has to perform the following three steps for a given model:
(i) obtain the set of all possible fixed points;
(ii) assess and compare the accuracy for all fixed points; and
(iii) analyze the stability of all fixed points.

First, we must obtain all fixed points: 
BP is unsuited for this task of obtaining \emph{all} fixed points as it provides only a \emph{single} fixed point and, obviously, does not provide unstable fixed points.
If the unstable fixed points correspond to local minima of the Bethe free energy they can be obtained by methods that minimize $\FB$ directly (cf. Section~\ref{sec:bp:variational:variants}). 
Nonetheless, those methods fail to obtain fixed points that correspond to local maxima and are not even guaranteed to obtain all local minima.
In order to find all fixed points we reformulate the fixed point equations as a system of polynomial equations that we solve directly.

Second, after all fixed points are obtained we assess their accuracy in a straightforward manner; 
this is possible since we keep the models small enough to admit the application of exact inference methods for comparison.

Third, the concept of local stability is well established for discrete-time maps in general, as well as for BP in particular (see~\cite{mooij2005properties} for the special case of binary pairwise models without local potentials). The major obstacle that prohibits the stability analysis for more general models is the need for all fixed points, which are not known in general.
We have already obtained all fixed points in the first step, so that we can analyze the stability of all fixed points by linearization of BP now.

Finally, we consolidate the insights of all three steps and shed some light onto the relation between the number of fixed points, the approximation accuracy, and the convergence properties.
In particular, we will provide novel insights and find answers to the following questions:
\begin{itemize}
	\item Convergence properties and accuracy relate to each other for small models. Does this generalize to models of arbitrary size as well?
	\item Does BP favor a particular fixed point if multiple fixed points are present? Is the most accurate fixed point always stable?
	\item What can be said about the accuracy if multiple fixed points exist? Do accurate fixed points still exist or does the existence of multiple fixed points imply failure of BP to provide accurate results? Can we enhance the approximation quality by a (suitable) combination of multiple fixed points?
	\item Under which circumstances does damping help to enforce convergence of BP?
	\item What is the influence of the model parameters: How does the graph structure (number of variables and connectivity) influence the number of fixed points and their stability for finite-size graphs? How do the local potentials influence the accuracy and the stability of fixed points? -- It is a common conjecture (cf.~\cite{mooij2005properties}) that models with weak local potentials perform worst; why does the presence of strong local potentials enhance the performance of BP?
\end{itemize}

\newsection{Fixed Points of Belief Propagation}{solutionsBP:overview}
Let us recall the iterative message update equations as discussed in Section~\ref{sec:bp:intro}:
\begin{align}
	\msg[n+1]{i}{j}{}  = \msgNorm[n]{i}{j} \sum \limits_{\RVval{i} \in \sampleSpace{X}} \pairwise{x}{i}{j} \local{}{i} \prod \limits_{\RV[x]{k} \in \{\neighbors{i} \backslash \RV[x]{j}\}}  \msg[n]{k}{i}{}.
	\label{eq:solutions:update}
\end{align}
The collection of all update equations $\setOfMessages[n+1] = \BP(\setOfMessages[n])$ defines a discrete-time map.
We can thus rely on the  standard recipe for analyzing discrete-time maps and analyze the solution space induced by~\eqref{eq:solutions:update} as discussed in Chapter~\ref{chp:solving}.
We begin with computing and counting the fixed points $\fpSetOfMessages$ for some reasonable form of normalization.

One alternative way to compute and count the fixed points is to  utilize the correspondence to stationary points of the Bethe free energy and to study this energy landscape instead. 
Several methods (cf. Section~\ref{sec:bp:variational:variants}) are available that minimize $\fb$ and,  in the presence of multiple BP fixed points (i.e., for non-convex $\fb$), either converge to a local or the global minimum.
None of these methods, however, is guaranteed to capture all fixed points.
Moreover, every method that relies on minimizing $\fb$ inevitably fails to account for (unstable) fixed points that correspond to local maxima of $\fb$.\footnote{
	The trivial (paramagnetic) solution  $\fpMsg{i}{j} = \frac{1}{|\sampleSpace{X}|} = \frac{1}{2}$ for all $\msg{i}{j}{} \in \setOfMessages$ is known in the special case of models with $\field{}=0$ and, depending on the coupling strength, it is either a local minimum or a local maximum of $\fb$~\cite{mooij2005properties}. The extension to models with $\field{i}\neq0$, however, is not straightforward and the complete set of all stationary points is not known in general.}

\newsubsection{Number of Fixed Points on Regular Ising Graphs}{solutionsBP:background}
Remember that the Ising model exhibits critical regions in the parameter space (so-called phase transitions) where the behavior of the model changes abruptly, unless it is specified on a path graph~\cite[Chapter~12]{georgii}.
In the physics literature, one computes these phase transitions for infinite-size, or at least very large, models.
Smaller graphs render the computation of the phase transitions much more intricate.\footnote{
	Note that true phase transitions, for which the partial derivatives of $\fb$ vanish can only exist for infinite-size graphs~\cite{binder1987}; with slight abuse of notation we refer to the finite-size manifestations of phase transitions as phase transitions as well.}
If, however, all nodes have equal degree (and unitary potentials $\coupling{i}{j} = \coupling{}{}$ and $\field{i} = \field{}$), the phase transitions can be computed for small graphs as well.
Therefore, one replaces the finite-size graph by an infinite-size graph with identical properties, a so-called Cayley tree.\footnote{
	A Cayley tree is an infinite tree without loops that captures the interactions of a cyclic finite-size graph.}
Then, the Cayley tree provides a way to compute the phase transitions of the underlying graph in an analytical fashion~\cite{taga2004convergence}.


Let us partition the parameter space $\parameter$ into three distinct regions $\FRegion$, $\PRegion$, and $\AFRegion$. 
This terminology complies with the naming convention in statistical physics where the regions are referred to as ferromagnetic $\FRegion$, paramagnetic $\PRegion$, and antiferromagnetic $\AFRegion$.
We can then, in accordance with~\cite[Section 12.2]{georgii}, define the phase transitions partitioning the parameter space.
Therefore, let us first introduce the shorthand notation $w = \tanh{|J|}$; we further introduce the following function
\begin{align}
	p(J,d) = 
	\begin{cases}
		d \arctanh{\sqrt{\frac{d\cdot w -1}{d/w -1}}} - \arctanh{\sqrt{\frac{d - 1/w}{d-w}}} \quad &\text{if} \; J  > \arcoth(d)\\
		d \arctanh{\sqrt{\frac{d\cdot w -1}{d/w -1}}} + \arctanh{\sqrt{\frac{d - 1/w}{d-w}}} \quad &\text{if} \; J  < \arcoth(d)\\
		0 \quad &\text{else}.
	\end{cases}
	\label{eq:phaseTransitions}
\end{align}
Finally, the regions are specified according to
\begin{align}
	(J,\theta) &\in \FRegion \, \hspace*{1cm} \text{if} \ J  > 0, \ J > \hspace*{0.35cm}\arcoth(d)   \quad \text{and} \,|\theta| \leq p(J,d), \label{eq:phase1} \\
	(J,\theta) &\in \AFRegion\, \hspace*{0.7cm} \text{if}  \ J < 0, \ J < -\arcoth(d)  \quad \text{and} \,|\theta| < p(J,d) , \label{eq:phase2}\\
	(J,\theta) &\in \PRegion\,  \hspace*{1cm} \text{if} \ (J,\theta) \notin \FRegion \quad\text{and} \quad  (J,\theta) \notin \AFRegion. \label{eq:phase3}
\end{align}

For attractive models with positive couplings ($J>0$) BP converges to a unique fixed point inside $\PRegion$. 
This fixed point becomes unstable and two additional fixed points emerge inside $\FRegion$~\cite{yedidia2005, mezard2009}.
For repulsive models with negative couplings ($J<0$) BP only converges inside $\PRegion$ and not inside $\AFRegion$~\cite{mooij2005properties}.

The phase transitions of the complete graph will be computed according to~\eqref{eq:phase1}-\eqref{eq:phase3}.
The phase transitions of the grid graph can only be numerically estimated and are defined by sudden changes in the number of fixed points.
We will later discuss the specific influence of the graph-size on the phase transitions in Section~\ref{sec:stabilityBP:theoretical_analysis}.\\


%

\newsection{Fixed Point Equations of Belief Propagation}{solutionsBP:solving}
Now, let us consider the update equations~\eqref{eq:solutions:update} as a nonlinear discrete-time map (cf. Section~\ref{sec:bp:map}).
We will adhere to the notation of Chapter~\ref{chp:solving} and Definition~\ref{defn:map} in particular and consequently formulate the fixed point equations for the update rule of BP.
Then, we aim to solve the fixed point equations directly; 
this will yield the set of all BP fixed points.

We did not consider a particular form of message normalization so far, i.e., how to choose $\msgNorm[]{i}{j}$.
The normalization, however, affects the update rule and must therefore be defined explicitly before attempting to solve the fixed point equations;
we define the normalization terms $\msgNorm{i}{j}$ so that the messages along every edge sum up to one, i.e., 
$\msgNorm[n]{i}{j} = \sum_{\RVval{j}\in \sampleSpace{X}}  \msg[n]{i}{j}{} - 1$.

Let us recall the definition of fixed point messages from~\eqref{eq:defn:fp}.
Then, the update equations define the following set of polynomial equations $\setOfBPEq = \fpSetOfMessages-\BP(\fpSetOfMessages) = 0$ so that
\begin{align}
	\setOfBPEq = 
	\begin{cases}
		\msg[n]{i}{j}{+1} - \msgNorm[n]{i}{j}  \sum\limits_{\RVval{i} \in \sampleSpace{X}} \potential{\RV{i},\RV{j}}(\RVval{i},+1) \local{}{i}  \prod\limits_{\RV{k} \in \neighborsWO{i}{j}} \msg[n]{k}{i}{}\\ 
		\msg[n]{i}{j}{-1} - \msgNorm[n]{i}{j}  \sum\limits_{\RVval{i} \in \sampleSpace{X}} \potential{\RV{i},\RV{j}}(\RVval{i},-1) \local{}{i}  \prod\limits_{\RV{k} \in \neighborsWO{i}{j}} \msg[n]{k}{i}{}\\
		\msg[n]{i}{j}{+1} + \msg[n]{i}{j}{-1} -1.
	\end{cases}
	\label{eq:SetOfEq}
\end{align}
This system of polynomial equations consist of $S$ equations $(\BPEq[1], \ldots , \BPEq[s])$, where
\begin{align}
	S = 2|\setOfEdges| \cdot (|\sampleSpace{X}| + 1).
	\label{eq:NrEq}
\end{align}

It is advantageous to consider this polynomial system defined over the complex numbers instead of the real numbers.
Else, we would restrict ourselves and rule out several methods for solving systems of polynomial equations that rely on the definition over the complex numbers (cf. Section~\ref{sec:solving:nphc}).
Let us now define the set of solutions over the complex numbers, without accounting for multiplicity, according to
\begin{align}
	\variety(\setOfEq) \definition \{ (\setOfMessages, \setOfNorm) \in \COMPLEX : f_i(\setOfMessages,\setOfNorm) = 0 \text{ for all }f_i \in \setOfEq\}.
\end{align}
We are particularly interested in the set of solutions over strictly positive real numbers $\variety_{\REALPos}(\setOfEq) \subseteq \variety(\setOfEq)$.
Note that $\variety_{\REALPos}$ is of specific relevance as it directly relates to the fixed points of BP.

\begin{thm}[Fixed Points of BP]\label{prop:EqSys}
	Let $(\setOfMessages, \setOfNorm)$ be some set of messages and normalization terms. 
	Then, $(\setOfMessages, \setOfNorm)$ is a fixed point of BP, if and only if $(\setOfMessages, \setOfNorm) \in
	\variety_{\REALPos}(\setOfEq)$.
\end{thm}
\begin{proof}
	First, we show that every $(\setOfMessages, \setOfNorm) \in \realSol$ characterizes a fixed point of BP.
	All messages are positive by definition and are normalized according to~\eqref{eq:SetOfEq} so that they represent probabilities (cf. Lemma~\ref{lm:msg}). 
	Furthermore, it follows from \eqref{eq:SetOfEq} that $\setOfMessages-\BP(\setOfMessages) =  0$, which constitutes a fixed point.
	
	Conversely, consider some fixed point messages with its corresponding normalization coefficients $(\fpSetOfMessages, \fpSetOfNorm)$, it then follows by definition that 
	$\BP(\fpSetOfMessages) = \fpSetOfMessages$ and consequently $\setOfBPEq= 0$.
\end{proof}
The set of solutions over the strictly positive real number thus directly corresponds to the set of all fixed points according to $\setOfAllSol = \{\pseudomarginals=\decisionP(\fpSetOfMessages), \partitionBethe=\decisionZ(\fpSetOfMessages): (\fpSetOfMessages, \fpSetOfNorm) \in \realSol\}$.

\begin{cor}\label{prop:nonEmpty}
	Consider a graph with strictly positive potentials $\pairwiseShort{i}{j}$ and $\localShort{i}$ as e.g., with Ising potentials. 
	Then, the solution set $\variety_{\REALPos}(\setOfEq)$ is nonempty. 
\end{cor}
\begin{proof}
	For non-negative potentials the average energy is bounded from below and $\fb$ has at least one minimum \cite[Theorem 4]{yedidia2005}.
	Minima of the constrained $\fb$ correspond to BP fixed point solutions, the existence of which implies non-emptiness of 
	$\realSol$ by Theorem \ref{prop:EqSys}.
\end{proof}

\newsubsection{Solving the Fixed Point Equations}{solving:methods}
Solving systems of nonlinear polynomial equations is a classical problem in computational mathematics and a great variety of methods have been developed such as iterative solvers (cf. Section~\ref{sec:solving:numerical}), symbolic methods (cf. Section~\ref{sec:solving:symbolic}), and homotopy methods (cf. Section~\ref{sec:solving:nphc}).
We now briefly review the applicability of these methods for the purpose of computing the set of all solutions.

The main disadvantage of iterative solvers is the need for an already known initial guess in the vicinity of a solution.
Moreover, it is difficult, to obtain the full set of solutions with these methods; this is of particular relevance as we are specifically interested in the entire set of the (positive) real solutions $\variety_{\REALPos}(\setOfEq)$.

Symbolic methods, on the other hand, are capable of obtaining the entire solution set. 
In particular, the Gr\"obner basis method is available in a number of software-packages and has seen substantial development in the past several decades.
As discussed in Section~\ref{sec:solving:symbolic:groebner}, the main drawbacks are:
the limitation to rational coefficients (cf. Example~\ref{ex:instability_groebner}), a worst-case complexity that is double exponential in the number of variables~\cite{mayr1982complexity, moral1984upper}, and the limited scalability in parallel computations. 
Altogether, this limits the application of symbolic methods to smaller systems; in fact the Gr\"obner basis method did not converge for any of our considered models.

The NPHC method obtains the entire solution set by deforming a \emph{start} system that is trivial to solve to the \emph{target} system that we intend to solve.
Although maybe not as established as the Gr\"obner basis method, multiple software packages are available for the NPHC method as well. 
One crucial ingredient of the NPHC method is that a number of independent solution paths is tracked under the deformation of the homotopy. 
This suggests to track all solution paths independently, which makes the approach \emph{pleasantly parallelizeable}; this is essential in dealing with large polynomial systems.
Even though only positive real solutions are of interest in this work, the NPHC method requires us to extend the domain to the field of the complex numbers in order to guarantee the emergence of smooth solution paths. 
Obviously a notable computational overhead is introduced in doing so; especially if only few of the solutions fall onto the real line.
Various different forms of homotopies are available that differ in the complexity of bounding the number of solutions.
As discussed in Section~\ref{sec:solving:nphc} the optimal choice depends strongly on the particular target system and is rarely known up-front.

At large, the NPHC method is the most promising approach for a couple of reasons:
it obtains \emph{all} isolated nonzero complex solutions\footnote{
	Here, ``nonzero complex solutions'' refer to complex solutions of a system of polynomial equations where each variable is nonzero. A solution is considered to be isolated if it has no degree of freedom, i.e., there is an open set containing it but no other solutions.} 
that must include \emph{all} BP fixed points $\realSol$; it is readily available in software packages; and it has a level of parallel scalability that can deal with models of relevant size 

\newsubsection{Polyhedral Homotopy Method}{solving:polyhedral}
Let us briefly recall one of the most basic forms of homotopies from Section~\ref{sec:solving:nphc}:
\begin{align}
	\homotopy(\stateVec{x},t) = (1-t)\startSys(\stateVec{x}) + \gamma t \setOfEq(\stateVec{x}) = 0.
\end{align}
Alternatively, one can construct more advanced "nonlinear" homotopies where the parameter $t$ appears in nonlinear form, in order to reduce the computational costs. 
Among a great variety of polynomial homotopy constructions, the \emph{polyhedral homotopy method}, developed by B.~Huber and B.~Sturmfels~\cite{huber1995polyhedral}, 
is particularly suited for the systems of polynomial equations considered.

In applying the NPHC method to solve \eqref{eq:SetOfEq},
the choice of $\setOfStartEq$ (the trivial system of equations that 
the target system is deformed into) plays an important role in the overall
efficiency of the approach since different choices of $\setOfStartEq$ may induce a vastly different number of solution paths one has to track. 
The crucial part is to come up with a good upper bound on the number of solutions and to create an appropriate start system. 
Once this is solved, the desired solutions are simply obtained by tracking all independent solution paths.

Note that in each equation of~\eqref{eq:SetOfEq} only few of the monomials are present, i.e., the update equations of BP imply a sparse system of equations~\cite{huber1995polyhedral}. 
In our experiments, we observed that despite the rather high \emph{total degree}\footnote{%
	The total degree of a system of polynomial equations is the product of the degrees of each equation.
	It is a basic fact in algebraic geometry that the total number of isolated complex solutions a polynomial system has is bounded by its total degree (i.e., Bezout bound).
	Therefore the total degree serves as a crude measure of the complexity of the polynomial system.} 
$d_t$~\cite[pp.118] {sommese2005numerical}, each equation in~\eqref{eq:SetOfEq} contains only relatively few of the monomials.
Such sparse systems usually benefit from resorting  to more involved method homotopy methods that take the structure of the system into account and consequently provide tighter bounds on the number of solutions.
The number of solution paths one has to track when using the polyhedral homotopy method for solving a system of polynomial equations is given by the so-called \emph{Bernstein-Kushnirenko-Khovanskii (BKK) bound}: 
fixing the list of monomials that appear in the polynomial system, it is an important yet surprising fact in algebraic geometry 
that for almost all choices of the coefficients (in the probabilistic sense), the number of isolated
nonzero complex solutions is a fixed number which only depends on the list of monomials. 
This number is known as the BKK bound \cite{Bernstein75,Kushnirenko76,Khovanski78}. Intermediate steps in the determination of the BKK bound  are reused to create an appropriate start system. Using the fully parallel implementation \textsf{Hom4PS-3}~\cite{chen2014hom4ps} of the polyhedral homotopy method, we compute the BKK bound, that is tight in all our experiments, and obtain \emph{all} isolated positive solutions.

The polyhedral homotopy method exploits the structure of~\eqref{eq:SetOfEq}, but also requires some subtle steps. Rather than presenting all technical details we present an illustrative example to explain the underlying principles. For more details we refer the reader to the excellent overview papers \cite{li1997numerical,li2003solving, chen_homotopy_2015} or to \cite[Section 8.5.4]{sommese2005numerical} and the references therein.

\begin{example}[Polyhedral Homotopy]$ $\newline
	The essential steps in solving polynomial system with the polyhedral homotopy method are: 
	first, to compute a root count based on mixed volume computations~\cite[Section 3]{li2003solving}; 
	second, to come up with an easy to solve start system~\cite[Section 4]{li2003solving}; 
	and finally, to solve the start system and track the solution paths to the target system~\cite[Section 1]{li2003solving}.

	\emph{(i) Root Count:}  Consider the example system (taken from ~\cite[p.142]{sommese2005numerical}) with two unknown variables $\vm{x} = \{x_1,x_2\}$ and with 4 solutions.
	\begin{align}
		\mathbf{F}(\vm{x}) = 
		&\begin{cases}
			1+ax_1+bx_1^{2}x_2^{2}\\
			1+cx_1+dx_2+ex_1x_2^{2}.
		\end{cases} \label{eq:ExampleSystem}
	\end{align}
	The total degree of this system of equations is $d_t = 4\cdot 3  = 12$, which serves as an upper bound on the actual number of solutions. 
	If the system of equations is sparse, the BKK bound serves
	as a much tighter bound.
	
	Every equation $f_i \in \mathbf{F}(\vm{x})$ has an associated polytope $\polytope{i}$ which is the convex hull of the exponent vectors for all monomials of $f_i$. For $f_1 $ the polytope is
	\begin{align}
		\polytope{1} = \{(0,0)(1,0)(2,2)\},
	\end{align}
	which has a graphical representation in Figure~\ref{fig:q1_q2} (a). Similar for $f_2$ the polytope, shown in Figure~\ref{fig:q1_q2} (b), is
	\begin{align}
		\polytope{2}=\{(0,0)(1,0)(0,1)(1,2)\}.
	\end{align}
	
	Some important operation on polytopes are the computation of the Minkowski sum $\polytope{1}+\polytope{2} = \{s_1+s_2 : s_1 \in \polytope{1}, s_2 \in \polytope{2}\}$ and the computation of volumes, denoted by $V(\polytope{i})$. 
	Note that computing the BKK bound is a viable thing to do in any dimension; therefore we refer to $V(\polytope{i})$ as volume although the polytopes of this example only lie in the two-dimensional space.
	The computation of the mixed volume $M(\polytope{1},\polytope{2})$ is a combinatorial problem that is especially comprehensible in the case of two equations where
	\begin{align}
		M(\polytope{1},\polytope{2}) = V(\polytope{1}+\polytope{2}) - V(\polytope{1})-V(\polytope{2}).
	\end{align}
	For a generalization to higher dimensions see~\cite[p.140]{sommese2005numerical}.
	The polytope of $\polytope{1}+\polytope{2}$ is illustrated in Figure~\ref{fig:q1_q2} (c); the mixed volume is accordingly obtained by subtracting $V(\polytope{1})$ and $V(\polytope{2})$ from $V(\polytope{1}+\polytope{2})$,
	which equals the sum of all gray areas (known as mixed cells).  
	
	It is straightforward to see that each parallelogram has volume equal to 2;
	the BKK bound therefore equals $M(\polytope{1},\polytope{2}) = 4$ and thus provides a tight bound on the number of solutions.
	\begin{center}
		\includegraphics[width=0.75\linewidth]{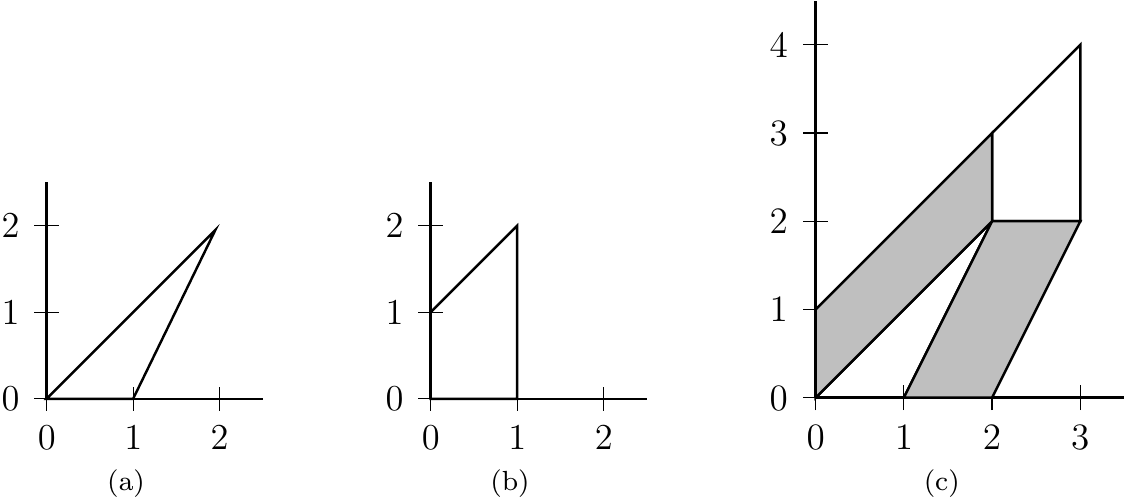}
		\captionof{figure}{(a) polytope $\polytope{1}$, (b) polytope $\polytope{2}$, and (c) Minkowski sum $\polytope{1}+\polytope{2}$.}
		\label{fig:q1_q2}
	\end{center}

	\emph{(ii) Start System:} The BKK bound does provide a tight bound on the number of solutions but does not immediately reveal the initial solutions of an appropriate start system $\mathbf{Q}(\stateVec{x})$ with $q_i(\stateVec{x}) = \sum_{a\in \polytope{i}} c_{i,a}\stateVec{x}^a$, where  $\stateVec{x}^a = x_1^{a_1}\cdot x_2^{a_2}$ and $c_{i,a}$ are random coefficients.
	
	However, the mixed volume computation can also be accomplished by introducing a lifting 
	\begin{align}
		\omega_i = \{\omega_i(a):a \in \polytope{i}\} \nonumber
	\end{align}
	for each $f_i$.  
	Thereby, we increase the dimension of the polytope $\polytope{i}$ to $\polytopeLifted{i}$ by adding one component to each exponent-vector $a$. This component is obtained by the lifting function $\omega_i(a)$. In our example we choose the lifting values $\omega_1 = \{0,0,0\}$ and $\omega_2 = \{0,1,1,3\}$; these values are obtained by the inner products $\omega_1(a) = (0,0)\circ(a_1,a_2)$ and $\omega_2(a) = (1,1)\circ(a_1,a_2)$.
	The polytopes are lifted accordingly so that 
	\begin{align} 
		\polytopeLifted{1} &= \{(0,0,0)(0,1,0)(2,2,0)\},\nonumber\\
		\polytopeLifted{2} &=\{(0,0,0)(1,0,1)(0,1,1)(1,2,3)\},\nonumber\\
		\polytopeLifted{1}+\polytopeLifted{2}&=\{(0,0,0)(0,1,0)(2,2,0)(0,2,1)(0,1,1)(3,2,1)(3,4,3)\}.\nonumber
	\end{align}
	Then the faces in the lower hull of $\polytopeLifted{1}+\polytopeLifted{2}$ correspond to cells shown in Figure~\ref{fig:q1_q2}, which is known as a fine mixed subdivision.

	These liftings, together with the random  coefficients $c_{i,a}$, now form the homotopy $\hat{\mathbf{Q}}(\stateVec{x},t)$ with $\hat{q}_i =\sum_{a\in \polytope{i}} c_{i,a}\stateVec{x}^a t^{\omega_i(a)}$ such that 
	\begin{align}
		\hat{\mathbf{Q}}(\stateVec{x},t) = 
		&\begin{cases}
			1+c_{1,1}x_1+c_{1,2}2x_1^{2}x_2^{2}\\
			1+c_{2,1}x_1t+c_{2,2}x_2t+c_{2,3}x_1x_2^{2}t^3.
		\end{cases}
	\end{align}
	By closer inspection, however, it is still not possible to identify the starting points because $\hat{q}_2(\stateVec{x},t=0) = 1$. This problem can be resolved according to~\cite[Lemma 3.1]{huber1995polyhedral}: i.e., initial values are obtained by solving a binomial system for every cell that contributes to the mixed volume computation (i.e., for every gray cell in Figure~\ref{fig:q1_q2}).
	One can then increase $t$ and obtain the solutions of the start system by tracking the solution paths to $\hat{\mathbf{Q}}(\stateVec{x},t=1) = \mathbf{Q}(\stateVec{x})$.
	
	\emph{(iii) Target System:} Finally we have all 4 solutions to $\mathbf{Q}(\stateVec{x})$. Now what remains is to construct a \emph{linear} homotopy according to~\eqref{eq:homotopy} and to increase $t$, starting at $t=0$. At $t=1$ the homotopy reduces to $\mathbf{F}(\stateVec{x})$ and provides the desired solutions of the target system.
\end{example}

\newsection{Stability of Fixed Points}{stabilityBP:stability}
We proceed according to the usual procedure in dynamical systems (cf. Chapter~\ref{chp:solving}) and, after computing the set of all fixed points, asses their (local) stability.
%
A fixed point is locally stable if a neighborhood exists such that messages inside this neighborhood (i.e., messages sufficiently close to the fixed point), converge to the fixed point under the considered map~\cite[pp.170]{Teschl2003}.

Note that the  Bethe free energy does not fully characterize the stability of a given fixed point (cf. Section~\ref{sec:bp:variational:correspondence}): 
although stable fixed points are local minima of $\FB$, local minima must not be stable~\cite{heskes2003stable}.
A general way of investigating the stability of fixed points is available by the method of Lyapunov~\cite[pp.93]{Scheinerman1996}, though it is sufficient for all graphs considered in this work to restrict our analysis to linearization (the indirect method of Lyapunov) as introduced in Section~\ref{sec:solving:stability}.

This approach is well-established in the dynamical systems literature and it may consequently seem rather surprising that such a stability analysis has not been considered so far for BP.
The main difficulty, however, is not the stability analysis as such, but the prerequisite of estimating all fixed points.
Indeed, a local stability analysis has been performed for Ising models with vanishing local potentials~\cite{mooij2005properties}, for which all fixed points are known.
We can, however, rely on the NPHC method, as discussed in the preceding section, to solve the fixed point equations and to obtain the set of \emph{all} fixed points.
Subsequently, it is evident how to analyze the stability by computing the Jacobian and thus linearizing BP in every fixed point.

\newsubsection{Reformulation of Belief Propagation for Binary Variables}{stabilityBP:reformulation}
Before we investigate the stability of all BP fixed points, we introduce an alternative parameterization of the update equations for the particular case of binary variables.
This parameterization reduces the number of variables and eases some calculations without changing the properties of BP.
Therefore, we express both messages along the same edge as a single message (cf.~\cite{mooij2005properties,knoll_stability}) defined by
\begin{align}
	\msgReparam[n]{i}{j} = \arctanh\big(\msg[n]{i}{j}{1} - \msg[n]{i}{j}{-1} \big).
\end{align}
Then, the update rule in~\eqref{eq:update} can be rewritten according to
\begin{align}
	\tanh(\msgReparam[n+1]{i}{j}) = \tanh(J_{ij})\tanh(\cavityField[n]{i}{j}), \label{eq:reparam}
\end{align}
where the cavity field $\cavityField[n]{i}{j}$ acts on $\RV{i}$, while neglecting the incoming message from $\RV{j}$, according to
\begin{align}
	\cavityField[n]{i}{j} \definition \theta_i + \sum \limits_{\RV{k}\in \neighborsWO{i}{j}} \msgReparam[n]{k}{i}. \label{eq:cavity}
\end{align}
Considering \eqref{eq:reparam} and~\eqref{eq:cavity} (cf.~\cite{opper2001tractable}) it becomes evident that, for binary pairwise models, BP corresponds to the so-called cavity method~\cite{mezard1987spin}.
Its name stems from the fact that we essentially dig a cavity into the model by removing $\RV{j}$. We then express the remaining field that acts $\RV{i}$ by the cavity field $\cavityField{i}{j}$.

The marginals, or the mean, can then -- similar as for BP -- be computed by all incoming messages according to 
\begin{align}
	\mean{i} &= \tanh \big( \cavityField{i}{j} + \msgReparam{j}{i} \big) \nonumber \\
	&= \field{i} + \sum \limits_{\RV{k}\in \neighbors{i}} \msgReparam{k}{i}.
\end{align}

In accordance with the consideration of BP as a dynamical system, we will denote the set of all re-parameterized messages by $\setOfMsgReparam$ and denote the mapping induced by BP as $\setOfMsgReparam[n+1] = \BP(\setOfMsgReparam[n])$.
Keeping the notation consistent, we say that BP converged to a fixed point
\begin{align}
	\fpSetOfMessagesReparam = \BP(\fpSetOfMessagesReparam),
	\label{eq:fp_reparam}
\end{align}
if successive messages remain unchanged under BP.
If, however, BP fails to converge, one can  try to achieve convergence by one of the many modifications discussed in Section~\ref{sec:bp:map:improving}.
One of these modifications that lends itself handsomely for our stability analysis is BP with damping.

Let us briefly state the update equations of BP with damping in terms of the reformulated messages again: $\BPVariant{D}(\setOfMsgReparam[n]) = (1-\epsilon)\BP(\setOfMsgReparam[n])+ \epsilon \setOfMsgReparam[n]$, where $\epsilon \in [0,1)$ is the damping factor.
Remember that, although the fixed points of BP with damping are fixed points of BP without damping as well (cf.~\eqref{eq:equal_stability}), the stability of the fixed points may change nonetheless.

\newsubsection{Linearization}{stabilityBP:linearization}
In order to assess the stability of the fixed points $\fpSetOfMessagesReparam$ we approximate $\mapBP{\cdot}$ by a linear function in its fixed point(s) and analyze the behavior of the linearized system. 
This is done by taking the 
partial derivatives of all messages, i.e., by analyzing the Jacobian matrix $\Jacobian$ with its elements defined as
\begin{align}
	\Jacobian_{mn} = \frac{\partial \msgReparam[\stationaryPoint]{i}{j}}{\partial \msgReparam[\stationaryPoint]{k}{l}}, \label{eq:jacobian}
\end{align}
where -- given some ordering -- $\edge{i}{j}$ and $\edge{k}{l}$ are the $m\textsuperscript{th}$ and the $n\textsuperscript{th}$ edge.

For binary pairwise models parameterized as in Section~\ref{sec:stabilityBP:reformulation}, the Jacobian is given as follows:
without loss of
generality\footnote{If all combinations of $m$ and $n$ are considered we would effectively consider all $\frac{n(n-1)}{2}$ possible edges. 
	All rows and columns that correspond to a message $\msgReparam{v}{w}$, where $\edge{v}{w} \notin \setOfEdges$ include only zero-values, however, and can therefore be neglected without changing  the eigenvalues.} 
we consider only messages where $\edge{i}{j} \in \setOfEdges$ and $\edge{k}{l} \in \setOfEdges$ such that
\begin{align}
	\Jacobian_{mn} = 
	\begin{cases}
		\frac{ \tanh(J_{ij}) \left(1-\tanh^2(\cavityField{i}{j})\right) }{1-\tanh^2(J_{ij}) \tanh^2(\cavityField{i}{j})} & \!\!\!   \text{if} \, i\!=\!l \, \text{and} \; k \! \in\! \neighborsWO{i}{j} \\
		0 \; & \!\!\! \text{else.}
	\end{cases} \label{eq:jacobian-long}
\end{align}

Let us briefly recap the properties of the eigenvalues with respect to the stability of the BP fixed points from Section~\ref{sec:solving:stability}.
A fixed point $\fpSetOfMessagesReparam$ is locally stable if all eigenvalues have absolute value strictly smaller than one, i.e., if $\Radius{\Jacobian} < 1$ 
and BP converges if initialized sufficiently close enough.
A fixed point is unstable with respect to BP if at least one eigenvalue exists outside the unit circle such that $\Radius{\Jacobian} > 1$.
For $\Radius{\Jacobian} = 1$ stability of the nonlinear system cannot be inferred by just looking at the linear system.\\

In addition to damping, there are many other variants of BP available that have the same set of solutions $V_{\mathbb{R}_+}^{*}$~\cite{wainwright2003tree-scheduling};
in particular, this includes different scheduling methods. 
Analyzing the stability of the fixed points under these variants, however, becomes problematic -- 
mainly because the update function changes with time.
Moreover, the fact that many messages are not updated, introduces eigenvalues with $\lambda_i = 1$, which renders the stability analysis by linearization impossible.


Therefore, of all the modified versions of BP, we will restrict our attention to damping.
The application of damping modifies the eigenvalue spectrum according to
\begin{align}
	\Spectrum{\Jacobian[D]} = \Spectrum{\Jacobian} \cdot (1-\epsilon) + \epsilon.
\end{align}
Note that all eigenvalues are reduced by a factor $(1-\epsilon)$ and experience a shift by $\epsilon$ along the real axis, i.e., $\Re{\Spectrum{\Jacobian[D]}} =  \epsilon + \Re{(1-\epsilon)\Spectrum{\Jacobian}}$. 
A fixed point $\fpSetOfMessagesReparam$ is thus locally stable under $\mapBP[D]{\cdot}$ if 
\begin{align}
	\Re{\Spectrum{\Jacobian}} < 1.
\end{align}

The correspondence between stability and the eigenvalue spectrum is summarized and visualized in Figure~\ref{fig:overview_stability_eigenvalues}.
A fixed point is stable under BP if all eigenvalues lie inside the unit circle (depicted by the blue area), under BP with damping if all eigenvalues have a real part strictly smaller than one (depicted by the gray area), and are unstable else (depicted by the green area).

\begin{figure}[t]
	\centering
	\includegraphics[width =0.3\columnwidth]{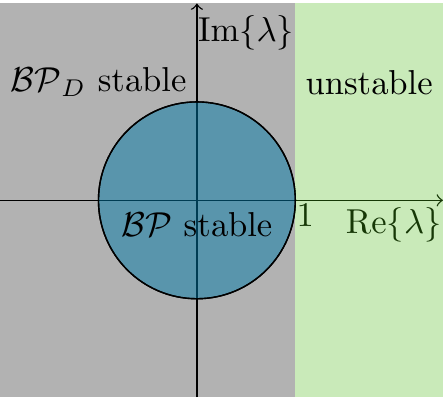}
	\caption{Eigenvalue spectrum of the Jacobian $\Spectrum{\Jacobian}$. BP is stable if all eigenvalues lie inside the blue region; BP with damping is stable if all eigenvalues lie inside the gray region; BP with and without damping is unstable if eigenvalues lie inside the green region.}
	\label{fig:overview_stability_eigenvalues}
\end{figure}

Note that the properties of the Bethe Hessian (and thus the stability of the fixed points) further relate to some concepts from graph-theory; in particular to the Ihara zeta function~\cite{watanabe} and  -- for models with vanishing local fields, i.e., where $\field{}=0$ -- to the non-backtracking matrix (known as the Hashimoto-matrix)~\cite{saade2014spectral,saade2017spectral}. 
Note that the latter connection is insightful by any means but is only valid in a well-behaved region where all eigenvalues of the non-backtracking matrix are inside the unit circle (cf.~\cite[Theorem 2.1]{saade2017spectral}).  

\newsection{Selected Models with Ising Potentials}{solutionsBP:fps_selected_models}

Finally, we will now consider a range of probabilistic graphical models, apply the NPHC method to the fixed point equations~\eqref{eq:SetOfEq}, obtain all fixed points by first finding all isolated non-zero complex solutions, and analyze the stability of all fixed points.

We first evaluate and compare the accuracy of all fixed points obtained by NPHC;
details for our evaluation criteria are presented in Section~\ref{sec:solutionsBP:experiments:evaluation}.
Note that we will already anticipate some results from the later stability analysis and separately evaluate the accuracy of stable and unstable fixed points.
However, here we rather ask \emph{if} a given fixed point is stable, instead of \emph{why} it is stable.
Furthermore, we present the evolution of the fixed points over the parameter space in Section~\ref{sec:solutionsBP:experiments:fp_evolution} and present some implications on the accuracy to better understand for which parameters  BP can be expected to provide good results.

The capability of NPHC to obtain all fixed points, subsequently allows for a thorough stability analysis in Section~\ref{sec:stabilityBP:empirical_analysis}. 
Our empirical observations further inspire some theoretical investigations on Ising models with unitary parameters in Section~\ref{sec:stabilityBP:theoretical_analysis}. 
There we show why convergence properties degrade with growing graph size and why strong local potentials help to achieve convergence. 
Finally, the performance of BP-decoding for error-correcting codes is analyzed in terms of the solution space in  Section~\ref{sec:solutionsBP:coding}.

\newsubsubsection{Considered Graphs}{stabilityBP:empirical_analysis:graphs}
We consider different realizations of the Ising model with attractive, repulsive, and mixed interactions on a range of graphs.
These graphs include complete graphs, grid-graphs, and grid-graphs with periodic boundary conditions (see Figure~\ref{fig:graphs}).

For the complete graph each pair of nodes is connected by an edge; it follows by definition that this is a regular graph, i.e., all variables $\RV{i}\in\setOfNodes$ have equal degree $\nodeDegree{i} = N-1$. 
Because of this, we can construct a Cayley tree to determine the phase-transitions.

The grid-graph has all edges aligned along the two-dimensional square lattice and, for finite-size graphs, contains variables of varying degrees. 
We additionally consider grid-graphs with periodic boundary conditions, where nodes on the boundary are joined by edges so that all variables have equal degree again.
\begin{figure}[t]
	\centering
	\includegraphics[width =0.7\columnwidth]{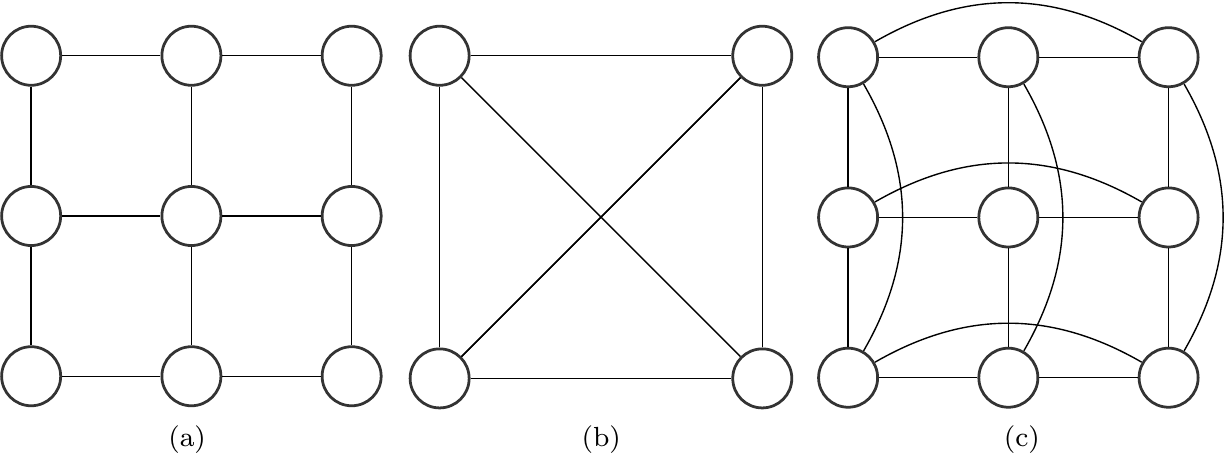}
	\caption{Considered Ising graphs: (a) grid-graph;  (b) complete graph; (c) grid-graph with periodic boundary conditions.}\label{fig:graphs}
\end{figure}

\newsection{Number of Fixed Points and Marginal Accuracy}{solutionsBP:accuracy_selected_models}
This section provides an exhaustive solution space analysis of BP in terms of analyzing the properties for all fixed points.
In order to obtain the set of all fixed points we have to solve the system of fixed point equations~\eqref{eq:SetOfEq} first.
The systems considered in this work are simply too large to be solved with symbolic methods or even with the NPHC method based on a linear homotopy. The BKK bound, however, takes into account the sparsity of the system $\eqSys$, induced by the graph structure, and reduces the number of solution paths to be tracked so that the problem can be solved in practice. We present a detailed runtime analysis in Section~\ref{sec:solutionsBP:experiments:runtime}. 
Note that both the structure of~\eqref{eq:SetOfEq} and the number of complex solutions in $V(\mathbf{F})$ 
remain the same if the graph structure is kept constant~\cite{chen2015network,chen2016network}; depending on the potentials, however, the number of solutions in $V_{\mathbb{R}_+}^{*}(\mathbf{F})$ may change. 

\newsubsection{Evalutation Criteria}{solutionsBP:experiments:evaluation}
We evaluate all fixed points in terms of their marginal accuracy and compare them with marginals obtained by an implementation of BP without damping.
Evaluation of the marginal accuracy requires the availability of the exact marginals; these are obtained by the junction tree algorithm that we can still resort to because of the limited size of the considered models.

In particular, we evaluate the correctness of the approximated marginals in terms of their accuracy (cf.~\eqref{eq:error_marginal}).
The expected mean (cf.~\eqref{eq:mean_expected}) provides an alternative way of evaluating the marginals and is particularly well suited for illustrative purposes.
We will therefore not only compare different fixed points in terms of the MSE but also visualize the averaged mean of the exact solution $\meanAvg$ as well as of all BP fixed points $\meanAvgApprox$.
Remember that $\meanAvg - \meanAvgApprox$ equals the average of all marginal errors (cf.~\eqref{eq:mean_expected}) and is thus well-suited for visually comparing the marginal accuracy of different fixed points.\\

We will not only compare different fixed points but also evaluate the effect of altering the evaluation function.
A combination of properly weighted marginals may improve the marginal accuracy considerable but is often problematic in practice for the lack of methods that obtain more than a single fixed point.
The NPHC method, however, computes the set of all solutions $\variety_{\REALPos}(\setOfEq)$ and -- by Theorem~\ref{prop:EqSys} -- yields the set of all fixed points $\setOfAllSol$.
This is the foundation for evaluating the marginal accuracy of different combinations.

To combine the marginals, we first need to compute the pseudomarginals $\pseudomarginals$ and the Bethe partition function $\partitionBethe$ for all fixed points.
Then, we utilize the obtained values of the Bethe partition function $\partitionBethe$ together with the pseudomarginals $\pseudomarginals$ and modify the evaluation function $\decisionP$ in three different ways.

First, all fixed points are weighted by their partition function and are combined so that
\begin{align}
	\pseudomarginals^{T} =  \frac{1}{\sum\limits_{\partitionBethe \in \setOfAllSol} \partitionBethe}
	\sum_{(\pseudomarginals,\partitionBethe)\in\setOfAllSol}\pseudomarginals\cdot\partitionBethe.
\end{align}

Second, we combine all fixed points that correspond to local minima of the Bethe partition function in the same way so that
\begin{align}
	\pseudomarginals^{M} =  \frac{1}{\sum\limits_{\partitionBethe \in \setOfMinimaSol} \partitionBethe}
	\sum_{(\pseudomarginals,\partitionBethe)\in\setOfMinimaSol}\pseudomarginals\cdot\partitionBethe.
\end{align}
Note that these combinations are not just some heuristics.
In particular the latter one reminds us of how the exact solution is expected to decompose according to the RSB assumption.

Finally,
we aim to select one specific fixed point and hope that accurate marginals are to be found at the maximum of the Bethe partition function (i.e., the global minimum of the Bethe free energy).
The fixed point maximizing $\partitionBethe$ is consequently selected and evaluated as well; more formally it is given by
\begin{align} 
	\pseudomarginals^{MAX}= \argmax_{\pseudomarginals\in \setOfAllSol} \partitionBetheWithP.
\end{align}

\newsubsection{Grid Graphs with Random Factors (Spin Glasses)}{solutionsBP:experiments:gridGraphRandom}
Consider a grid graph of size $N = 3 \times 3$ with randomly distributed parameters. All pairwise and local potentials are sampled uniformly; i.e., $(\coupling{i}{j},\field{i}) \sim \mathcal{U}(-K,K)$. The larger the support of the uniform distribution is, the more difficult the task of inference becomes; we choose $K=3$, which is large enough to make BP fail to converge sometimes (cf.~\cite{sutton2012improved}). 

According to~\eqref{eq:NrEq} the system of equations consists of $M = 72$ equations in $72$ unknowns. More specifically, \eqref{eq:SetOfEq} consists of 24 linear (i.e., normalization constraints), 40 quadratic, and 8 cubic equations; the total degree bounds the number of solutions by $d_t = 1^{24} \cdot 2^{40} \cdot 3^8 = 7.2 \cdot10^{15}$. 
Tracking such an amount of solution paths is not feasible in practice, even with a parallel implementation of the NPHC method. 

The system of equations in~\eqref{eq:SetOfEq}, however, is sparse. We can exploit this sparsity that is induced by the graph structure if we consider the BKK bound and reduce the computational complexity. 
The number of complex solutions for this graph is bounded by $\text{BKK}=608$. After creating a suitable start system the problem is straightforward to solve with the NPHC method. It actually turns out that the BKK bound is tight for all graphs considered.

In particular we evaluate 100 grid graphs with random factors: on 99 graphs BP converged after at most $10^4$ iterations. 
Although the grid graph has multiple loops and the constrained $\fb$ is not necessarily convex~\cite[Corr.2]{heskes2004uniqueness}, we observe that for all $100$ graphs NPHC obtains a unique positive real solution that corresponds to a unique BP fixed point.

\begin{figure}[t]
	\centering
	\includegraphics[width = 0.66\linewidth]{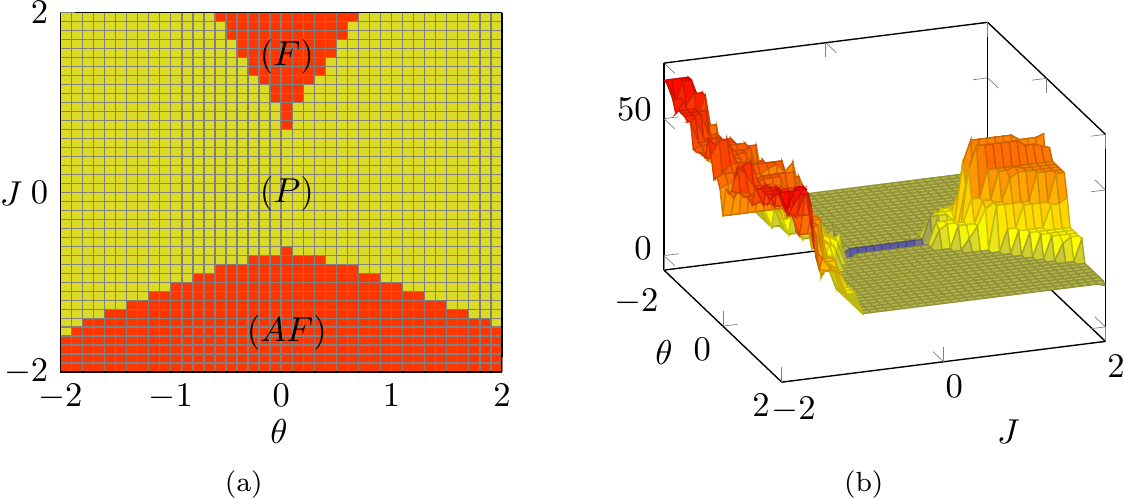}
	\caption[]{ Number of fixed points on the grid graph of size $N=3 \times 3$: (a)
		number of positive real fixed points (yellow: unique fixed point, red: three fixed points); (b) number of real fixed points. The increase in the number of both real solutions and positive real solutions, indicates a phase transition.}
	\label{fig:3x3-solutions}
\end{figure}
\newsubsection{Grid Graphs with Uniform Factors}{solutionsBP:experiments:gridGraphUniform}

We further analyze BP on grid graphs of size $N = 3 \times 3$  with constant potentials among all nodes and edges; i.e.,
we specify the couplings for all edges by $\coupling{i}{j} = J$ and specify the fields for all nodes by $\theta_i = \theta$. 
We apply BP and NPHC for 1681 graphs in the parameter region $(J,\theta) \in \{-2,-1.9,\ldots,1.9,2\}$ and illustrate the size of the solution set in Figure~\ref{fig:3x3-solutions}.

\begin{table*}[t]
	
	\renewcommand{\arraystretch}{1.3}
	\caption{\textsc{Mean Squared Error, i.e., $\EMarginal{m}$ of Marginals and Combined Marginals Obtained by BP and NPHC for the Grid Graph with Unitary Factors.}}
	\label{tab:3x3}
	\centering
	\begin{tabular}{l l l l l l l }
		\toprule
		\multicolumn{2}{c}{Parameters}&  \multicolumn{2}{c}{Fixed Points} & \multicolumn{3}{c}{Combined} \\ 
		\cmidrule(lr){1-2} \cmidrule(lr){3-4}  \cmidrule(lr){5-7}
		Couplings & Local Field & BP & NPHC & MAX & ALL & STABLE \\ \midrule
		$ J \in [-2,2]$ & $\theta \in [-2,2]$ & 0.197 & 0.016 & 0.037 & 0.005 & 0.004 \\  
		$ J \in \FRegion$ & $\theta \in \FRegion$ & 0.010 & 0.010 & 0.076 & $9.2\cdot10^{-4}$ & $1.0\cdot10^{-9}$\\ 
		$ J \in \AFRegion$ & $\theta \in \AFRegion$  & 0.836 & 0.050 & 0.126& 0.003 & $1.5\cdot10^{-6}$\\
		$ J \in \PRegion$ & $\theta \in \PRegion$ & 0.006 & 0.006 &  0.006 & 0.006 & 0.006\\ 
		\bottomrule
	\end{tabular}
\end{table*}

%

The number of solutions in $V_{\mathbb{R}_+}^{*}(\mathbf{F})$  is presented in Figure~\ref{fig:3x3-solutions} (a). In the well-behaved region $\PRegion$ of the parameter space a unique fixed point exists, whereas three fixed points exist in $\FRegion$ and $\AFRegion$ -- this is in accordance with statistical mechanics~\cite[p.43]{mezard2009}.\footnote{ 
	Note that the graph under consideration is of finite size and thus $\nodeDegree{i}$ varies among the nodes. As a consequence the partitioning according to \eqref{eq:phase1} - \eqref{eq:phase3}, is only an approximation.}
Interestingly, we observe a close relation between the onset of phase transitions and the increase in the number of real solutions in Figure~\ref{fig:3x3-solutions} (b).
Most of these solutions, however, correspond to negative message values 
that violate Lemma~\ref{lm:msg} and are not feasible.

BP converges to some fixed point on all $1681$ graphs within at most $10^4$ iterations. 
This raises a couple of questions: 
What is the approximation error of BP if it converges to the most accurate fixed point? 
Or, speaking in terms of free energies, how large is the gap between the global minimum of the constrained $\fb$ and the minimum of the Gibbs free energy?
To answer this question we evaluate the correctness of the approximated marginals by computing the MSE between the exact and the approximated marginals according to~\eqref{eq:error_marginal}. The results are presented in Table~\ref{tab:3x3}. 
Averaged over all graphs we can see that BP does not necessarily converge to the most accurate fixed point. 
For the NPHC method we present the MSE for the fixed point with the lowest MSE; this highlights the existence of fixed points, which give \emph{more} accurate approximations than BP. 
Looking at all parameter regions separately we can see that BP does converge to the global optimum in $\PRegion$, as well as in $\FRegion$. In the antiferromagnetic region $\AFRegion$, BP converges to a fixed point that does not necessarily give the best possible approximation.

If we consider regions with multiple fixed point solutions (i.e., $\FRegion$ and $\AFRegion$) it becomes obvious that the fixed point maximizing the partition function (i.e., $\pseudomarginals^{MAX}$) is not necessarily the best one; 
this is especially surprising as BP obtains the best possible fixed point solution inside $\FRegion$. 

For $\field{}=0$ it turns out that initializing all messages to the same value $\msg[1]{i}{j}{} = \mu^1=\frac{1}{|\sampleSpace{x}|}$ will result in a fixed point
where $\singleApprox{i}(\RVval{i})=0.5$ (and  $\mean{i} = 0$) for all $\RV{i} \in \setOfNodes{}$.
Although this fixed point is identical to the exact one, it may be unstable (cf. Figure~\ref{fig:3x3-slice}).

Inspired by these observations, one should not only consider the fixed point maximizing $\zbApproximate$ (i.e., $\pseudomarginals^{MAX}$), but rather obtain multiple fixed points by NPHC and combine them. Indeed, especially inside region $\FRegion$ a combination of \emph{all} fixed point solutions according to~\eqref{eq:rsb}, i.e., $\pseudomarginals^{T}$ increases the accuracy of the approximation. 
If we combine the fixed points at local minima only (which are all stable in this specific case), i.e., $\pseudomarginals^{M}$, the accuracy increases even more and gives the \emph{most} accurate approximation over the entire parameter space.

\newsubsection{Complete Graphs with Uniform Factors}{solutionsBP:experiments:2x2}
We consider a complete graph with $N = 4$ binary random variables and illustrate the number of solutions in Figure~\ref{fig:2x2-solutions}. 
The system of equations~\eqref{eq:SetOfEq} consists of $36$ equations in $36$ unknowns and has its number of solutions bounded by the total degree $d_t=1^{12}\cdot2^{24}=16.8\cdot10^6$. Similar as for the grid graph, a much tighter bound of $\text{BKK}=120$ is provided by the BKK bound. 
Among all four nodes, we apply unitary factors, specified by $\coupling{i}{j} = J$  and $\field{i} = \field{}$. 
This type of graph is particularly interesting because one can derive  exact conditions where phase transitions occur (cf. Section~\ref{sec:solutionsBP:background}).
\begin{figure}[t]
	\centering
	\includegraphics[width = \linewidth]{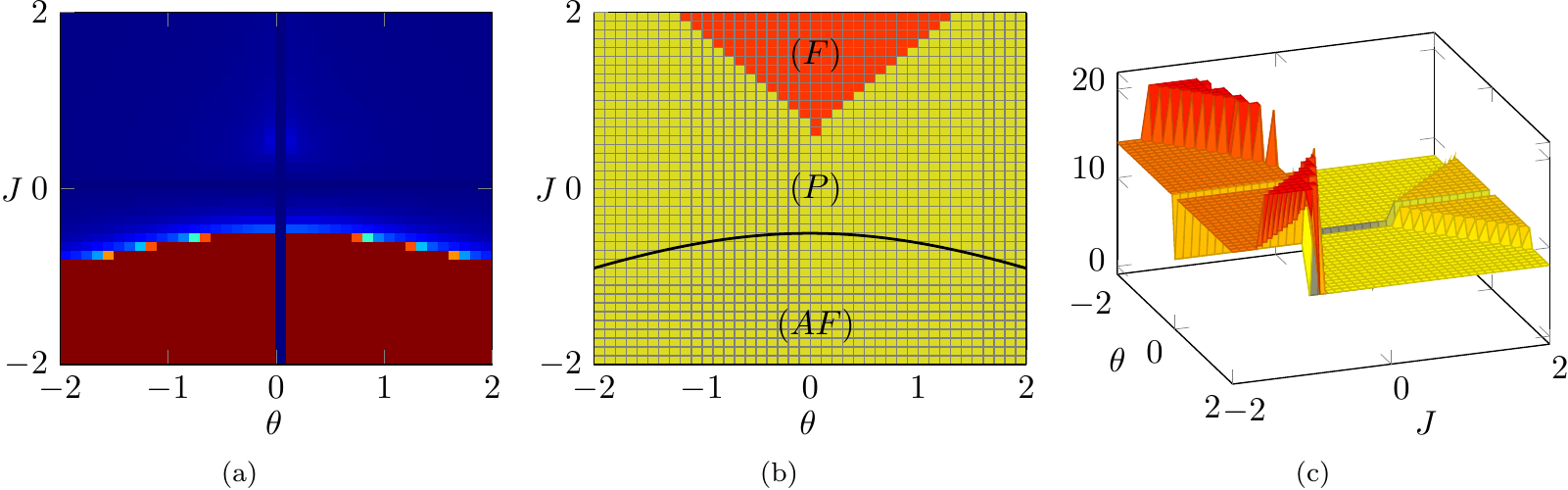}
	\caption[]{Fully connected graph with $N=2\times2$: (a) convergence of BP: for the blue region BP did convergence -- for the red region it did not converge after $4\cdot 10^5$ iterations; (b) number of fixed points (yellow: unique fixed point, red: three fixed points; (c) number of real solutions -- note the sudden increase at the onset of phase transitions.}
	\label{fig:2x2-solutions}
\end{figure}

\begin{table*}[!t]
	\renewcommand{\arraystretch}{1.3}
	\caption{\textsc{Mean Squared Error, i.e., $\EMarginal{m}$ of Marginals and Combined Marginals Obtained by BP and NPHC for the Fully Connected Graph With Unitary Factors.}}
	\label{tab:2x2}
	\centering
	\begin{tabular}{l l l l l l l}
		\toprule
		\multicolumn{2}{c}{Parameters}&  \multicolumn{2}{c}{Fixed Points} & \multicolumn{3}{c}{Combined} \\ 
		\cmidrule(lr){1-2} \cmidrule(lr){3-4}  \cmidrule(lr){5-7}
		Couplings & Local Field & BP & NPHC & MAX & ALL & STABLE \\ \midrule
		$ J \in [-2,2]$ & $\theta \in [-2,2]$ & 0.069 & 0.007 & 0.011 &0.003 & 0.0027\\ 
		$ J \in \FRegion$ & $\theta \in \FRegion$       & 0.034 & 0.033 & 0.070 &0.002 & $2.0\cdot10^{-8}$ \\ 
		$ J \in \AFRegion$ & $\theta \in \AFRegion$     & 0.304 & 0.004 & 0.004 &0.004 &0.004\\ 
		$ J \in \PRegion$ & $\theta \in \PRegion$  & 0.003 & 0.003 & 0.003 &0.003 &0.003 \\ \bottomrule
	\end{tabular}
\end{table*}

For $(J,\theta) \in \PRegion$, BP has a unique fixed point, which is a stable attractor in the whole message space \cite{mooij2005properties}. In $\FRegion$ three fixed points satisfy \eqref{eq:SetOfEq}, one of which is unstable and a local minimum of $\zbApproximate$. Both other fixed points are local maxima of $\zbApproximate$ and BP converges to one of them.

For repulsive models, BP only converges inside $\PRegion$ and not inside $\AFRegion$ as shown in Figure~\ref{fig:2x2-solutions} (a). We can see two interesting effects:
first, in Figure~\ref{fig:2x2-solutions} (c) the number of real solutions increases at the onset of phase transitions; second, even though the convergence of BP breaks down at the phase transition, a unique fixed  point exists inside $\AFRegion$ (Figure~\ref{fig:2x2-solutions} (b)) that gives an accurate approximation (cf. Table~\ref{tab:2x2}).

Similar as in Section~\ref{sec:solutionsBP:experiments:gridGraphUniform}, we asses the MSE of the marginals obtained by NPHC and BP; 
the results -- averaged over all graphs, and for each distinct region -- are presented in Table~\ref{tab:2x2}. 
Indeed, inside region $\AFRegion$ the marginals obtained by NPHC give a much better approximation than BP does.

Furthermore, note that the fixed point maximizing $\zbApproximate$ does not always give the best approximation. 
For $\field{}=0$, if all messages are initialized to the same value, BP obtains the exact marginals (cf. Section~\ref{sec:solutionsBP:experiments:gridGraphUniform}). 
If multiple fixed points exist, a weighted combination of all marginals according to~\eqref{eq:rsb} increases the accuracy -- only considering stable solutions, i.e., $\pseudomarginals^{M}$, gives the \emph{most} accurate approximations.


\newsubsection{Fixed Point Evolution}{solutionsBP:experiments:fp_evolution}
To better understand the influence of the parameters $(J,\field{})$ we specifically investigate how they affect the accuracy of the fixed points.
Therefore, we fix the values of $\field{} \in \{0,0.1,0.5\}$, vary $J\in[-2,2]$, and compare the fixed point solutions obtained by NPHC to the exact solution.
For illustrative purposes we consider the mean, averaged over all variables, of the exact solution $\meanAvg$ and of the approximate solution $\meanAvgApprox$ for all fixed points obtained by NPHC; 
the results are illustrated for both the grid graph and the complete graph in Figure~\ref{fig:3x3-slice} and Figure~\ref{fig:2x2-slice}.

The exact solution (red) is obtained by the junction tree algorithm~\cite{lauritzen-junction-tree}; solutions to~\eqref{eq:SetOfEq} are obtained by NPHC and are depicted by blue dots (stable) and by green or black dots (unstable). We further emphasize the fixed point that maximizes $\zbApproximate$, i.e., $\pseudomarginals^{MAX}$, in orange.

We will already discuss some implications of the local stability analysis here, but focus mainly on the accuracy of the fixed points and defer a thorough discussion of the stability analysis to Section~\ref{sec:stabilityBP:empirical_analysis}.\\

\begin{figure*}[!t] 
	\centering
	\includegraphics[width = \linewidth]{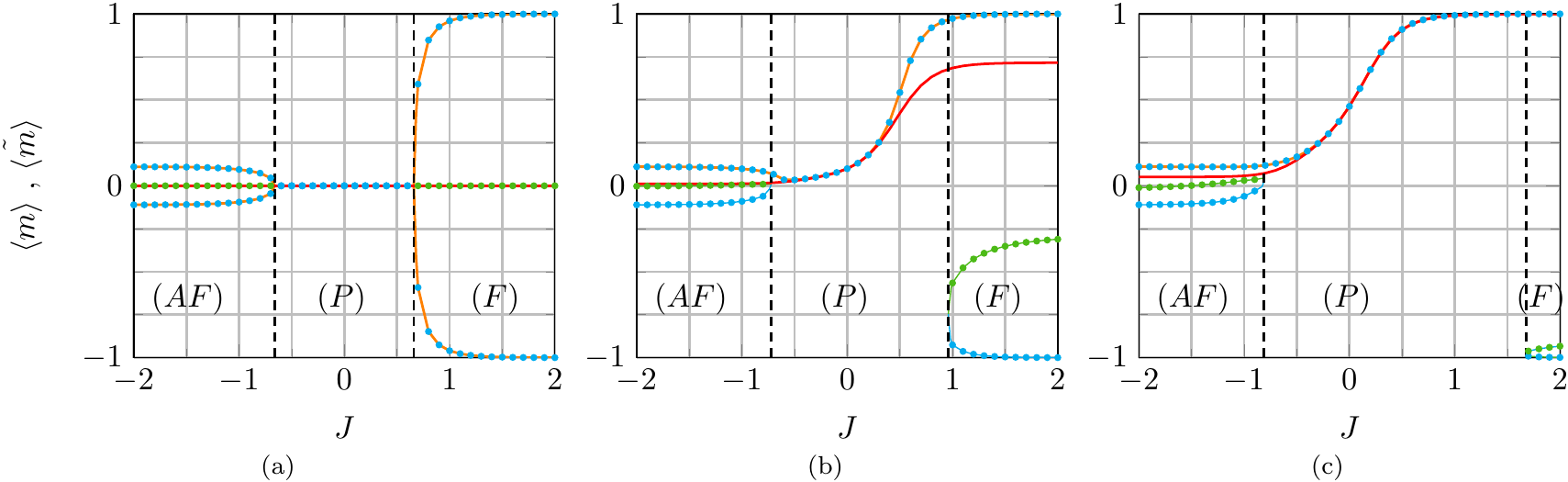}
	\caption[]{Results for the grid graph of size $N=3\times3$; average mean $\meanAvg$ and $\tilde{\meanAvg}$ for $J \in [-2,2]$  and for: (a) $\field{}=0$,  (b) $\field{}=0.1$, and (c) $\field{}=0.5$. The exact solution is illustrated in red. All fixed points obtained by NPHC are depicted by blue dots (stable) and by green dots (unstable). The fixed point maximizing the partition function is illustrated in orange.}
	\label{fig:3x3-slice}
\end{figure*}
\begin{figure*}
	\includegraphics[width = \linewidth]{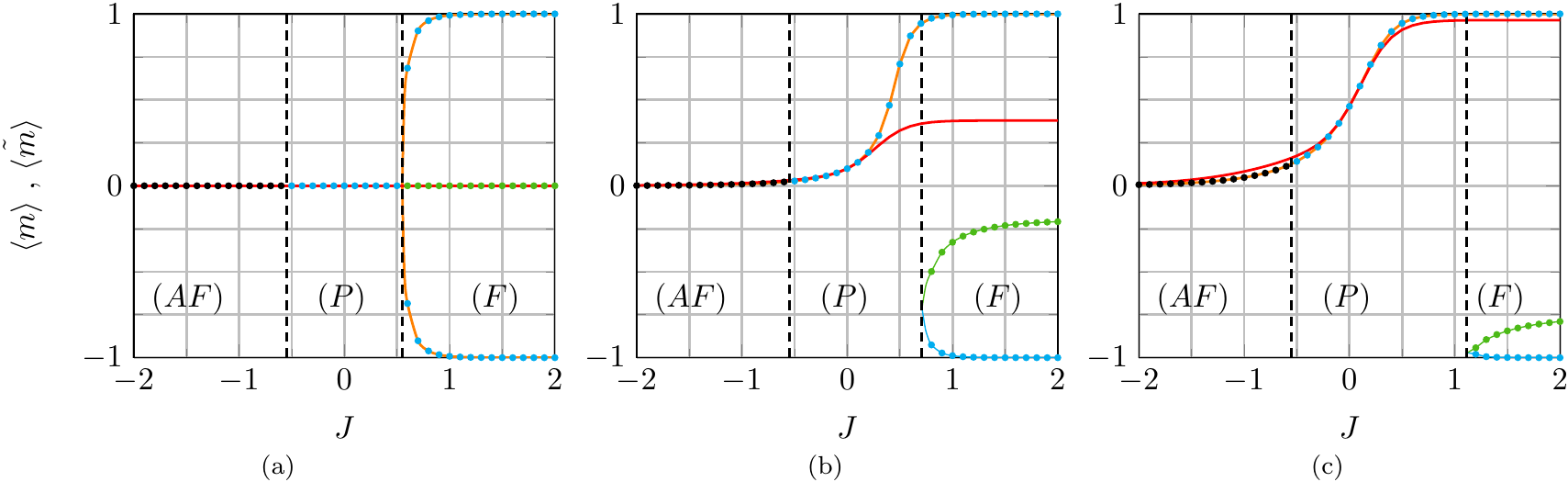}
	\caption[]{
		Results for the complete graph of size $N=2\times2$; average mean $\meanAvg$ and $\tilde{\meanAvg}$ for $J \in [-2,2]$ and for:  (a) $\field{}=0$,  (b) $\field{}=0.1$, and  (c) $\field{}=0.5$. The exact solution is illustrated in red. All fixed points obtained by NPHC are depicted by blue dots (stable) and by green/black dots (unstable). The fixed point maximizing the partition function is illustrated in orange.}
	\label{fig:2x2-slice}
\end{figure*}

All fixed points for a grid graph with $N=3\times3$ binary random variables are shown in Figure~\ref{fig:3x3-slice}.
The worst performance in terms of accuracy is observed for $\field{} = 0$.
Despite the existence of one fixed point that corresponds to the exact solution with $\meanAvg=\meanAvgApprox=0$, BP fails to provide accurate marginals as this particular fixed point is only stable inside $\PRegion$.
As the coupling strength $|\coupling{}{}|$ increases to the onset of phase transitions, two additional fixed points emerge.
These additional solutions are symmetric, stable, and guarantee the convergence of BP on this graph (see Figure~\ref{fig:3x3-slice}~(a)).

For $\field{} \neq 0$ a unique stable fixed point exists inside $\PRegion$. 
If we gradually increase $J$ until $(J,\field{}) \in \FRegion$ two additional fixed points emerge, one of which is unstable (see Figure~\ref{fig:3x3-slice}~(b)). Note that the fixed point maximizing $\zbApproximate$ (orange) remains stable for all values of $J \in [-2,2]$.
An increase in $\field{}$ (see Figure~\ref{fig:3x3-slice}~(c)) does enlarge the region where a unique fixed point exists and further increases the accuracy of the fixed point maximizing $\zbApproximate$. 
For $(J,\field{})\in\AFRegion$ a similar behavior is observed; i.e., for small values of $\field{}$ the unstable fixed point has the highest accuracy, but as $\field{}$ increases, the accuracy of the fixed point maximizing $\zbApproximate$ increases as well.\\

All fixed points for the complete graph with $N = 4$ binary random variables are shown in Figure~\ref{fig:2x2-slice}. 
For $\field{}=0$ and large values of $J$ the fixed point with $\meanAvg = \meanAvgApprox=0$ is unstable and is accompanied by two symmetric, stable fixed points (Figure~\ref{fig:2x2-slice}~(a)). 
In contrast to the grid graph a unique and accurate fixed point exists for $(J,\field{}) \in \AFRegion$; 
yet, this of no avail as the fixed point is unstable (see Figure~\ref{fig:2x2-solutions}~(a)).
This highlights that the existence of a unique, accurate, fixed point does not necessarily imply convergence of BP.

For $\field{}\neq 0$ the non-convergent region $\AFRegion$ is slightly reduced, but the problem of a unique unstable fixed point persists. 
In contrast to the grid graph, the complete graph contains odd-length cycles, thus allowing for frustrations if all edges are repulsive.  
This points at a close connection between the existence of frustrations and the existence of a unique unstable solution. 
The complete graph behaves similar to the grid graph for $(J,\field{}) \in \FRegion$, i.e, the accuracy of the fixed point maximizing $\zbApproximate$ increases as $\field{}$ increases (see Figure~\ref{fig:2x2-slice}~(b) and Figure~\ref{fig:2x2-slice}~(c)).



Our main findings are:
First, increasing the field $\theta$ increases the accuracy of the fixed point maximizing $\zbApproximate$.
Second, for $\field{} \neq 0$ the fixed point maximizing $\zbApproximate$ is unique and varies continuously under a change of $J$. 
Finally, the two stable fixed points are close to being symmetric, i.e., 
$\singleApprox{i}^{(1)}(\RVval{i}) \cong 1- \singleApprox{i}^{(2)}(\RVval{i})$. 
Consequently combining both stable fixed points will not lead to good approximations unless a proper weighting by $\zbApproximate$ is applied. 
Applying a proper weighting, however, leads to accurate approximations (cf. Table~\ref{tab:3x3}-~\ref{tab:2x2}).

\newsubsection{Runtime Analysis}{solutionsBP:experiments:runtime}
The time required for solving~\eqref{eq:SetOfEq} is presented in Table~\ref{tab:timing} for grid graphs with random factors (Section~\ref{sec:solutionsBP:experiments:gridGraphRandom}), grid graphs with unitary factors (Section~\ref{sec:solutionsBP:experiments:gridGraphUniform}), and fully connected graph with unitary factors (Section~\ref{sec:solutionsBP:experiments:2x2}). 
Comparing the overall computation time of the NPHC method to BP it becomes obvious  that NPHC is no alternative in terms of computational efficiency. 
It is, however, the only method that is guaranteed to obtain \emph{all} fixed points -- we were not able to apply the Gr\"obner basis method beyond a single-cycle graph with $N=4$. 
For our computations we utilized a cluster-system with 160 CPUs.

If we compare the overall computation time to the actual computation time utilizing the parallel implementation it becomes obvious that NPHC benefits tremendously from the high degree of parallelization. 
The runtime of BP depends mainly on the number of iterations and less on the size of the graph. Consequently, the stability of fixed points directly affects the performance of BP (cf. non-convergent region in Figure~\ref{fig:2x2-solutions} (a)).  The NPHC method is much less sensitive to the stability of fixed point solutions; the mixed volume computation, which has the largest influence on the overall runtime, rather depends on the number of variables in~\eqref{eq:SetOfEq}. Note the mixed volume computation does not depend on the parameters. If one is interested in the fixed points for different parameter-sets on the same graph it would suffice to compute the mixed volume and the start system only once; we did not do this to allow for a fair comparison.

\begin{table}[t!]
	\renewcommand{\arraystretch}{1}
	\setlength{\tabcolsep}{4pt}
	\caption{\textsc{Runtime Comparison between BP and NPHC (in Seconds).}}
	\label{tab:timing}
	\centering
	\begin{tabular}{@{}l l l l l l l@{}}
		\toprule
		& \multicolumn{2}{c}{Grid Graph:}&  \multicolumn{2}{c}{Grid Graph:} &\multicolumn{2}{c}{Fully Conn. Graph:}\\
		& \multicolumn{2}{c}{Random $(J,\field{})$}&  \multicolumn{2}{c}{Unitary $(J,\field{})$} &\multicolumn{2}{c}{Unitary $(J,\field{})$} \\
		\cmidrule(lr){2-3}                  \cmidrule(lr){4-5}                \cmidrule(lr){6-7}
		& total         & parallel   & total        & parallel     & total & parallel                \\ \midrule
		Mixed Vol.    & 1364.5        & 11.2 	&      1311.1  & 12.0         & 0.27  & -- \\ 
		Path Track.   & 69.0          & 0.97 	&    70.5      & 1.0          & 3.77  & --\\ 
		Post Proc. & 29.9  	   & 1.40 	&    43.6      & 4.1          & 2.3   & -- \\ \midrule
		NPHC            & \multicolumn{2}{l}{13.57}  	   &   \multicolumn{2}{l}{17.1}  & \multicolumn{2}{l}{6.34}\\
		BP              & \multicolumn{2}{l}{$3.6\cdot10^{-3}$} &  \multicolumn{2}{l}{$0.7\cdot10^{-3}$}  & \multicolumn{2}{l}{0.03}\\
		\bottomrule
	\end{tabular}
\end{table}
\setlength{\tabcolsep}{6pt}

\newsection{Empirical Stability Analysis of Belief Propagation on Ising Models}{stabilityBP:empirical_analysis}

So far we have been rather sloppy with the notion of stability and assumed knowledge about the stability of the fixed points without discussing how to gain this knowledge.
The fact that we have obtained all fixed points, irrespective of their properties under a specific map, admits a thorough stability-analysis for all of them.

In this section, we analyze the local stability of all fixed points for the models discussed so far.
Then, we compare the results to known results of both infinite and finite-size graphs with vanishing local fields $\field{}=0$. 
We restrict our analysis to attractive and repulsive models (cf. Section~\ref{sec:background:models:terminology}), because this allows us to change the behavior of BP with just a single parameter. 

We briefly present the results of~\cite{mooij2005properties} for vanishing fields in Section~\ref{sec:stabilityBP:empirical_analysis:vanishing}. 
Then we extend the analysis to graphs with non-vanishing fields, 
discuss some empirical observations, and interpret the implications. 
Note that the results of our stability analysis are also illustrated in terms of the average mean $\meanAvgApprox$ over the couplings-strength in Figure~\ref{fig:3x3-slice} (for the $3 \times 3$ grid graph) and~\ref{fig:2x2-slice} (for the complete graph of size $N=4$).
All fixed points of BP are colored according to the discussion of the eigenvalue spectrum (cf. Figure~\ref{fig:overview_stability_eigenvalues}):
fixed points are depicted in blue if stable under $\mapBP{\cdot}$, in black if stable under $\mapBP[D]{\cdot}$, and in green if unstable.
For reference we also illustrate the exact solution in red.
A more formal analysis is presented in Section~\ref{sec:stabilityBP:theoretical_analysis}.

\newsubsection{Vanishing Local Field}{stabilityBP:empirical_analysis:vanishing}
It is generally assumed that the case of vanishing local fields is the worst-case scenario~\cite{mooij2005properties}; we will confirm this assumption empirically and analytically.
For  $\theta = 0$ a trivial exact solution exists -- namely marginals that are uniform over all  states, i.e., for all $\RV{i} \in \setOfNodes$ the singleton marginals are $\pmfApprox{\RV{i}}(\RVval{i})=0.5$ and the mean is $m_i = 0$. 

For attractive models with $J>0$ sufficiently small, BP converges to this trivial (paramagnetic) fixed point, which is unique and stable. 
As the coupling-strength increases, the eigenvalue  with the largest magnitude $\lambda_{max}$ increases as well (see Theorem~\ref{thm:ev-scaling}) and as $\lambda_{max}$ crosses the unit circle the paramagnetic fixed point
remains unchanged but becomes unstable.
At the same time, two additional fixed points appear --
these fixed points are symmetric and stable (cf. Figure~\ref{fig:3x3-slice}~(a) and Figure~\ref{fig:2x2-slice}~(a)).

For repulsive models with $J<0$ all entries of the Jacobian swap in sign, i.e.,  $\Jacobian[-] = -\Jacobian[+]$, where $-$ and $+$ indicate the 
Jacobian for repulsive and attractive models, respectively. 
It follows that $\Spectrum{\Jacobian[-]} = -\Spectrum{\Jacobian[+]}$. 
Consequently, the local stability of the fixed point is invariant under a sign-change of $J$ and instability occurs precisely for the same coupling strength as before (cf. Figure~\ref{fig:3x3-slice} (a) and Figure~\ref{fig:2x2-slice} (a)). There is one important difference though -- the dominant eigenvalue is negative now, such that damping helps to achieve convergence.

\newsubsection{Non-Vanishing Local Field}{stabilityBP:empirical_analysis:non_vanishing}
For reasons of simplicity we restrict our analysis to 
models with a single value $\field{i} = \field{} \neq 0$ for all variables. 
Because the Ising model is symmetric with respect to the local fields $\field{i}$, it is sufficient to consider only non-negative local fields. 
Moreover, the same qualitative results hold if we allow for $\theta_i \geq 0$ in general.
\begin{figure}
	\centering
	\includegraphics[width = \linewidth]{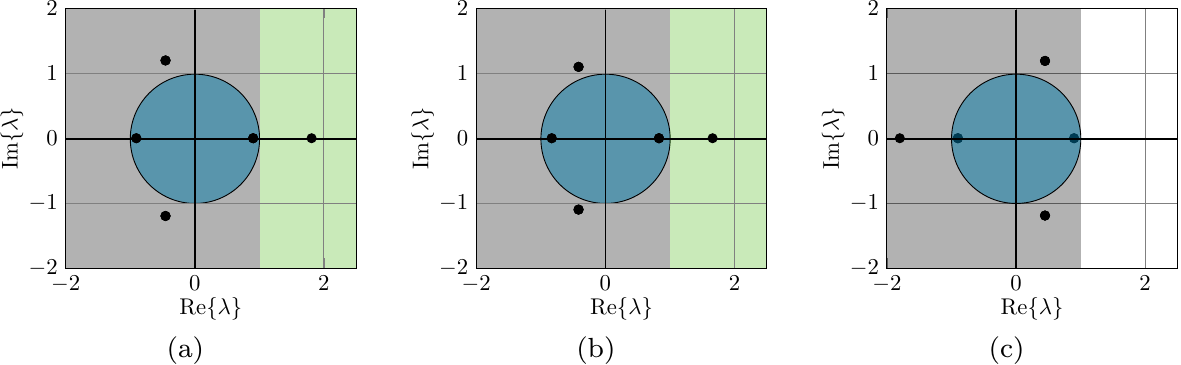}
	\caption{Eigenvalue spectra of the unstable fixed point of the complete graph with $N=4$. If all eigenvalues are inside the unit-circle BP converges without damping. If no eigenvalues lie on the right-hand side of the vertical line BP converges with damping. (a) $J=1.5$ and $\theta=0$:  notice that $\Re{\lambda_{max}} > 1$, i.e., damping does not help; (b) $J=1.5$ and $\theta=0.5$: notice how the external field reduces the magnitude of the eigenvalues (cf.  Theorem~\ref{thm:external-field}); (c) $J=-1.5$ and $\theta=0.5$: the fixed point is unstable but can be stabilized with damping because $\Re{\lambda_i} < 1$.} \label{fig:spectrum:2x2}.
\end{figure}

\begin{figure}
	\centering
	\includegraphics[width=\linewidth]{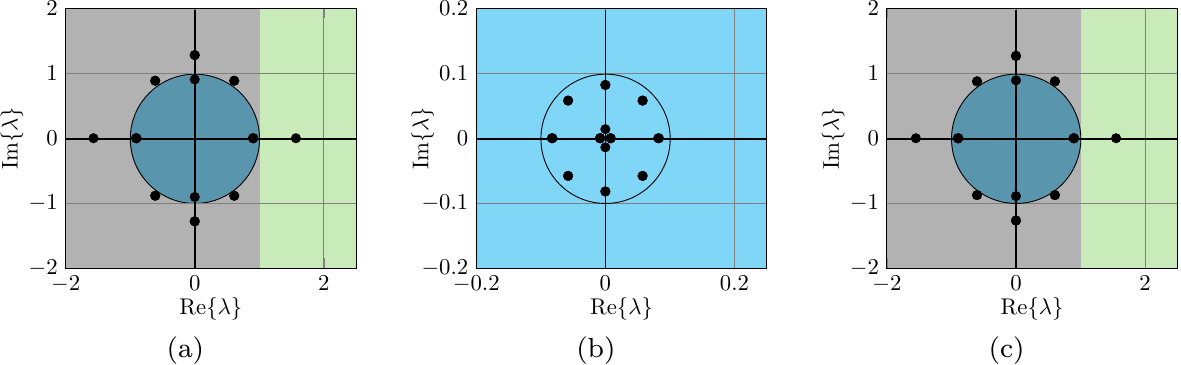}
	\caption{Eigenvalue spectra of the grid-graph with $N=9$; all eigenvalues are symmetric because the graph is bipartite (cf. Theorem~\ref{thm:bipartite}). (a) $J=1.5$, and $\theta=0$: the fixed point is unstable and damping does not help;  (b) $J=1.5$ and $\theta=0.5$: only a unique stable fixed point exists (cf. Figure~\ref{fig:3x3-slice}~(c)); note the qualitative difference in the spectrum;  (c) $J=-1.5$ and $\theta=0.5$: because of the symmetric spectrum, damping does not help in for repulsive models.}\label{fig:spectrum:3x3}
\end{figure}

\begin{figure}
	\centering
	\includegraphics[width=0.66\linewidth]{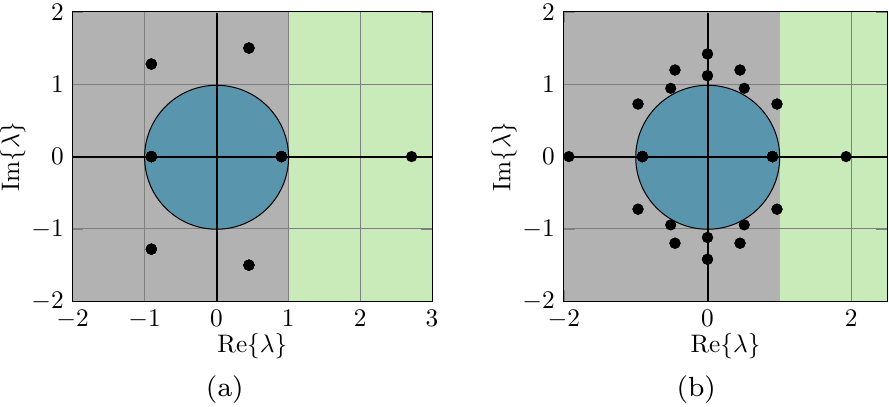}
	\caption{Eigenvalue spectra for $J=1.5$ and $\theta=0.5$ of the (a) grid-graph with $N=9$ and periodic boundary conditions; (b) grid-graph with $N = 16$. } \label{fig:spectrum:toric}
\end{figure}

\newsubsubsection{Attractive Models}{stabilityBP:empirical_analysis:ferromagnetic}

For small values of $J>0$ a unique fixed point exists to which BP converges and for which all eigenvalues lie inside the unit circle. 
If we gradually increase the coupling-strength to the point where instability occurred for $\field{}=0$, we observe that all eigenvalues are still inside the unit circle, i.e., the spectral radius $\Radius{\Jacobian} < 1$;
the fixed point is consequently stable (cf. Theorem~\ref{thm:external-field} for more details).  
If we further increase $\theta$, the spectral radius decreases and
a unique stable fixed point persists for even larger values of $J$ (compare Figure~\ref{fig:3x3-slice}~(b) and Figure~\ref{fig:2x2-slice}~(b) with Figure~\ref{fig:3x3-slice}~(c) and Figure~\ref{fig:2x2-slice}~(c)). 


Now, if we increase $J$ 
-- although the fixed point remains stable -- two additional fixed points emerge 
beyond some critical point $J_{C}(\graph,\theta)$ (see Figure~\ref{fig:3x3-slice} (b)~and~\ref{fig:2x2-slice} (b)).
As opposed to vanishing local fields, the unique fixed point is continuously deformed and remains stable though.
The second stable fixed point corresponds to some self-preserving state of magnetization\footnote{ 
	Self-preserving states are stable fixed points where $\tilde{\meanAvg}$ points into the opposite direction as $\field{}$.} and is accompanied by another unstable fixed point (Figure~\ref{fig:spectrum:2x2} (b)). 

The graph structure is supposed to influence the spectral radius as well. This becomes obvious if we enlarge the size of a graph while keeping its local structure unchanged, e.g., by increasing the grid-graph from $N_1=9$ to $N_2=16$. A comparison of Figure~\ref{fig:spectrum:3x3} with Figure~\ref{fig:spectrum:toric} reveals an increased spectral radius, i.e., $\Radius{\Jacobian[N_1]} < \Radius{\Jacobian[N_2]}$ (cf. Theorem~\ref{thm:ev-scaling}).
It is interesting that the largest eigenvalue -- and its symmetric counterpart for bipartite graphs -- are the only eigenvalues that experience a relevant increase in their real part.  The real part of all other eigenvalues 
${\lambda}_i \in \Spectrum{\Jacobian} \backslash \big\{ \lambda_{\max} : |\Re{ \lambda_{\max}}| = \Radius{\Jacobian}\big\} $
is bounded by $|\Re{{\lambda}_i}| < 1$.

In terms of accuracy, BP performs better for models with $\field{}\neq0$; 
compared to models with $\field{}=0$, a stable fixed point exists that lies closer to the exact solution.
Loosely speaking, increasing the local fields effectively reduces the influence of the couplings. This does not only lead to better convergence properties of BP (cf. Corollary~\ref{cor:finite-graphs}), but reduces the approximation-error as well.

\newsubsubsection{Repulsive Models}{stabilityBP:empirical_analysis:antiferromagnetic}
For small values of $J < 0$ a unique fixed point exists for which all eigenvalues lie inside the unit circle. If we further decrease $J$, a change of behavior can be observed beyond some critical value $J_{C}(\graph,-\theta)$. Interestingly, $|J_{C}(\graph,-\theta)| \leq |J_{C}(\graph,\theta)|$ with equality if and only if $\theta = 0$. In Figure~\ref{fig:spectrum:2x2} (b)~and~\ref{fig:spectrum:2x2} (c) we can also see that $\Radius{\Jacobian[-]} > \Radius{\Jacobian[+]}$, i.e., the invariance of local stability under sign-change of $J$ does not hold in general.

Let $J < J_{C}(\graph,-\theta)$, then it depends on the graph structure whether a unique fixed point exists and if damping is useful. It turns out that these two properties are closely connected:

First, consider the complete graph with $N = 4$ or the grid-graph with periodic boundary conditions; 
these models have a unique fixed point but frustrations exist because of cycles with odd-length. 
Consequently $\Spectrum{\Jacobian}$ is non-symmetric (cf. Theorem~\ref{thm:bipartite}). Besides the sign-change, the spectral radius increases as well so that $\Radius{\Jacobian[-]} > \Radius{\Jacobian[+]}$; the dominant eigenvalue, however, has a negative sign and $\REAL\left(\Spectrum{\Jacobian}\right) < 1$ so that an appropriate damping term exists that enforces BP to converge.

Second, consider both grid-graphs with $N_1=9$ and $N_2=16$; these graphs are bipartite and have a symmetric spectrum $\Spectrum{\Jacobian}$ (see Figure~\ref{fig:spectrum:3x3}, Figure~\ref{fig:spectrum:toric}~(b), and Theorem~\ref{thm:bipartite}). 
Because of the symmetric spectrum BP behaves similarly as in the attractive case: i.e., no damping term exists that would stabilize the unstable fixed point; 
two additional fixed points exist, however, that are stable.

Multiple fixed points for repulsive models exist if and only if the underlying graph is bipartite. Such graphs can be decomposed into two disjoint subsets $\setOfNodes = \mathbf{Y} \cup \mathbf{Z}$ such that the means follow a ''checkerboard-distribution'' where for all $\RV{i} \in \mathbf{Y}$ the mean is $m_i > 0$ and for all $\RV{j} \in \mathbf{Z}$ the mean is $m_j < 0$, or exactly the other way round. This explains the existence of two stable fixed points.
Note that the according average mean $\meanAvgApprox$ behaves as follows:
(i) if $N$ is even both stable fixed points are symmetric and have the same average mean;
(ii) if $N$ is odd both stable fixed points have different average means, because the subsets differ in size, i.e.,  $|\mathbf{Y}| \neq |\mathbf{Z}|$.
The difference in  $\meanAvgApprox$ reduces as the number of variables increases. In the limit of $N = \infty$, the average means of all fixed points collapse onto the same value.

To conclude our observations: either a unique fixed point exists, which may be unstable but can be stabilized by damping (black); or a fixed point exists which is unstable under any form of damping (green) but is accompanied by two stable fixed points (blue) (cf. Figure~\ref{fig:3x3-slice} and Figure~\ref{fig:2x2-slice}). 
Therefore, non-vanishing fields increase the accuracy of BP (Figure~\ref{fig:3x3-slice} and Figure~\ref{fig:2x2-slice}) and lead to better convergence properties (Theorem~\ref{thm:external-field}) in all experiments.

\newsection{Theoretical Stability Analysis}{stabilityBP:theoretical_analysis}

Here we present some more formal arguments and explain the observations made in Section~\ref{sec:stabilityBP:empirical_analysis}. We start with the Perron-Frobenius Theorem and specify some implications for the particular form of $\Jacobian$. Then we provide properties of the eigenvalue spectrum $\Spectrum{\Jacobian}$ and 
explain the influence of finite-size graphs in non-vanishing local fields.
We consider only connected graphs (cf. Section~\ref{sec:background:graphs}), where all $\RV{i}$ have minimum degree $\min \nodeDegree{i} \geq 2$. 
Note that it is straightforward to absorb any $\RV{i}$ with $d_i =1$ into its neighbor.
\begin{lm}\label{lm:irreducible}
	The Jacobian matrix $\Jacobian$  of a connected graph $G$ is irreducible if $\min \nodeDegree{i} \geq 2$.
\end{lm}
\begin{proof}
	The adjacency matrix of a connected undirected graph is irreducible. Let us construct the Jacobian $\Jacobian$ as in~\eqref{eq:jacobian}; note that we can partition $\Jacobian$ into $N \times N$ block matrices $[\Jacobian]_{ik}$ of size $d_i \times d_k$ each. These blocks contain non-zero values if and only if $\edge{i}{j}, \edge{k}{l} \in \setOfEdges$ for $X_j, X_k \in \setOfNodes$; i.e., if $a_{ik} = 1$. 
	It follows that $\Jacobian$ is irreducible as well.
\end{proof}

\begin{thm}[Perron Frobenius Theorem]
	A real non-negative square matrix $\boldsymbol{B}$ has its spectral radius bounded as follows:
	\begin{align}
		\min \limits_{i} \sum_j b_{ij} \leq \Radius{\boldsymbol{B}} \leq \max \limits_{i} \sum_j b_{ij}. \label{eq:perron}
	\end{align}
	If $\boldsymbol{B}$ is irreducible the largest eigenvalue is a positive real number $\lambda_{max} = \Radius{\boldsymbol{B}}$.
\end{thm}

\begin{cor}[Implications for regular graphs]\label{cor:constant-degree}
	A regular graph has $\nodeDegree{i} = d$ for all $\RV{i} \in \setOfNodes$. If the couplings and fields are constant,  $\sum \limits_{n} \Jacobian_{mn} = c$ is constant for all rows. By the Perron Frobenius Theorem and by Lemma~\ref{lm:irreducible} it follows that $\lambda_{max} = c$.
\end{cor}

Consider an infinite-size grid graph with unitary couplings  $J > 0$. 
Under these assumptions it is fairly straightforward to provide theoretical insights that explain why the existence of non-vanishing fields enhances the convergence properties. We generalize the results subsequently, as the same argument can be applied to finite-size graphs with different $J_{ij}$.
\begin{thm}\label{thm:external-field}
	Let $G_{\infty}$ be an infinite-size graph with $\nodeDegree{i} = d$ and purely attractive interactions $J > 0$. Then, the existence of a non-vanishing external field $\theta \neq 0$ stabilizes BP.
\end{thm}
\begin{proof}

	First assume a vanishing external field, i.e., $\theta =\theta_0= 0$, in which case the trivial fixed point has identical messages $\msg[\stationaryPoint]{i}{j}{} = 0.5$ so that
	$\msgReparam[\stationaryPoint]{i}{j} = 0$. Consequently, \eqref{eq:jacobian-long} reduces to
	\begin{align}
		\Jacobian[\theta_0]_{mn} \!=\!\begin{cases}\tanh(J) & \!\! \!\!   \text{if} \, i=l \; \text{and} \, k  \in \{\neighbors{i}\backslash \RV{j}\} \\
			0 \; & \!\!\!\! \text{else.}
		\end{cases}
	\end{align}
	Because all variables $X_i \in \setOfNodes(G_{\infty})$ have equal degree $\sum \limits_{n} \Jacobian_{mn} = c$ for all $m$. It follows from~\eqref{eq:jacobian-long} and the Perron-Frobenius Theorem that
	$\Radius{\Jacobian[\theta_0]} =  \tanh(J) \cdot (\nodeDegree{i}-1)$.
	Without loss of generality we exploit symmetry properties and assume an external field $\theta_{\delta} > 0$. By~\eqref{eq:update} and~\eqref{eq:reparam} it becomes obvious that every fixed point message $\msgReparam[\stationaryPoint]{i}{j} \neq 0$ and $\cavityField{i}{j} \neq 0$. 
	Note that qualitatively, it does not matter whether $\cavityField{i}{j}$ is positive or negative as we only consider $\tanh^2(\cavityField{i}{j})$.
	
	Because $0\leq \tanh(J) < 1$ 
	it follows that $\tanh(J) > \tanh^2(J)$ and consequently, as
	\begin{align}
		&\Jacobian[\theta_{\delta}]_{mn} = \nonumber \\
		&\begin{cases} \frac{ \tanh(J) \left(1-\tanh^2(\cavityField{i}{j})\right) }{1-\tanh^2(J) \tanh^2(\cavityField{i}{j})} & \!\!    \text{if} \, i=l \; \text{and} \, k \! \in\! \{\neighbors{i}\backslash \RV{j}\} \\
			0 \; & \!\! \text{else,}
		\end{cases}
		\label{eq:jacobian-bound}
	\end{align}
	all non-zero entries of $\Jacobian[\theta_0]_{mn}$ are element-wise larger than $\Jacobian[\theta_{\delta}]_{mn}$.
	Let us define a field-dependent scaling term $\kappa_{\theta} \in (0,1)$ of the Jacobian matrix , then $\Jacobian[\theta_{\delta}] = \kappa_{\theta} \Jacobian[\theta_0]$.
	
	Loosely speaking the existence of some non-vanishing field reduces all entries of the Jacobian matrix and consequently reduces the spectral radius.
\end{proof}

Now choose a critical value $J_C$ at the onset of instability\footnote{If $\lambda_{max} = 1$ we cannot infer from the linearized to the nonlinear map~\cite{Teschl2003}. We introduce an $\epsilon$-small term to avoid this subtlety.} so that $\Radius{\Jacobian[\theta_0]} = \lim \limits_{\epsilon \rightarrow 0} (1+\epsilon)$.
If the external field dampens the spectral radius, so that $\Radius{\Jacobian} < 1$ the unstable fixed point vanishes, and only a unique stable fixed point remains. Interestingly, this observation holds in all experiments, i.e.,  for attractive models an unstable fixed point exists if and only if multiple fixed points are present.

Suppose we either have purely attractive or purely repulsive interactions, then all entries of the Jacobian matrix have the same sign as the couplings and are bounded.
\begin{lm}\label{lm:jacobian-entries}
	All entries of the Jacobian are bounded by $\Jacobian_{mn} \in [0,1)$ if $J_{ij} > 0$ and by $\Jacobian_{mn} \in (-1,0]$ if $J_{ij} < 0$.
\end{lm}
\begin{proof}
	If all couplings have the same sign, the qualitative result of~\eqref{eq:jacobian-bound} holds as well. For purely repulsive interactions we just have to swap signs.
\end{proof}

\begin{thm}\label{thm:ev-scaling}
	The finite-size manifestations of phase transitions occur beyond the theoretical phase transitions.
	Let us define a parameter-set $(J_C , \theta_C)$ for which $\lambda_{max}$ crosses the unit-circle. 
	Then, for a graph $G_N$ with $N$ nodes, the values of the critical parameters $(J_C , \theta_C)$ decrease as $N$ increases.
	
	Consider two finite-size graphs with identical structure\footnote{For example two-dimensional grid-graphs with $N = 3\times3$ or $N=4\times4$ variables.}
	with different size $N_1< N_2$, then
	$\Radius{\Jacobian[N_1]} \leq \Radius{\Jacobian[N_2]} \leq \Radius{\Jacobian[\infty]}$.
\end{thm}
\begin{proof}
	Assume $J>0$ and w.l.g. $\theta = 0$ (for the influence of non-vanishing external field see Theorem~\ref{thm:external-field}). 
	Then on $G_{\infty}$ -- or for any other grid-graph with periodic boundary conditions -- the degree $\nodeDegree{i}=d$ is constant for all $X_i \in \setOfNodes$. It follows from Corollary~\ref{cor:constant-degree} that the largest eigenvalue is given by $\lambda_{max} = (\nodeDegree{i}-1) \cdot \tanh (J)$.\footnote{For vanishing fields, i.e., if $\field{i}=0$,  the Jacobian matrix is fully defined by $\coupling{i}{j}$.}
	Suppose we increase the couplings to $J^{new} > J$ and denote the change of its parameters as $\kappa = \frac{\tanh(J^{new})}{\tanh(J)}$. It is obvious  that $\lambda_{max}^{new} = \kappa \cdot \lambda_{max}$.
	
	Suppose we have some finite-size graph where $\nodeDegree{i}$ is not constant but depends on $\RV{i}$. If the couplings increase as before each row of $\Jacobian$ experiences a different amount of scaling. In all generality this is described by
	\begin{align}
		\begin{bmatrix}
			c_1 & & \vm{0}\\
			&\ddots & \\
			\vm{0}    &       &c_K
		\end{bmatrix} 
		\cdot \Jacobian \label{eq:scaling}
	\end{align}
	where $c_k$ depends on $\nodeDegree{i}$ and $K= \sum \limits_{i=1}^{N} \nodeDegree{i}$. Let us reformulate~\eqref{eq:scaling} in all detail to
	\begin{align}
		\ScaleMatrix = \sum \limits_{m=2}^{\max (\nodeDegree{i})} \boldsymbol{C}_m \Jacobian, \label{eq:jacobian-scale-matrix}
	\end{align}
	where 
	$\boldsymbol{C}_m$ is a diagonal matrix with values $\boldsymbol{C}_{m;kk} = \kappa_m$ if the $k\textsuperscript{th}$ line of $\Jacobian$ corresponds to a variable with $\nodeDegree{k} \geq m$ and $0$ otherwise. Since the largest eigenvalue is a positive real number it follows that the largest eigenvalue is the sum of the individual eigenvalues, so that
	\begin{align}\label{eq:ev-sum}
		\Radius{\boldsymbol{W}} = \sum \limits_{m=2}^{\max (\nodeDegree{i})} \Radius{\boldsymbol{C}_m \Jacobian}.
	\end{align}

	We still have to show that 
	\begin{align}
		\Radius{\Jacobian[N_1]}  \stackrel{(a)}{\leq} \Radius{\Jacobian[N_2]} \stackrel{(b)}{\leq} \Radius{\Jacobian[\infty]},
	\end{align}
	where $(b)$ follows from the existence of an associated eigenvector $x$ such that $\Radius{\ScaleMatrix} x \leq (\nodeDegree{i}-1)\cdot \tanh(J)$ with equality if and only if all $\boldsymbol{C}_m$ have only non-zero values on the main diagonal, i.e., all variables have equal degree. 
	As $G_{N_2}$ has a larger portion of nodes with $\nodeDegree{i} = \max \nodeDegree{i}$ than $G_{N_1}$, the average degree is also higher; 
	i.e., $\averageDegree[\graph_{N_1}] < \averageDegree[\graph_{N_2}]$. Then $(a)$ follows from~\eqref{eq:ev-sum}.
	
\end{proof}

By combination of Theorem~\ref{thm:external-field} and Theorem~\ref{thm:ev-scaling} we get:
\begin{cor}\label{cor:finite-graphs}
	The existence of a non-vanishing external field stabilizes BP on finite-size graphs.
\end{cor}

Next, we extend the above observations to varying couplings and fields.
We assume that all couplings $J_{ij}$ have the same sign, and all fields $\theta_i$ have the same sign, not necessarily the same as the couplings. Then the scaling coefficients $c_k$ in~\eqref{eq:scaling}  depend not only on $d_i$, but on $J_{ij}$ and $\theta_i$ as well. Still, $\Jacobian$ can only contain either positive entries or negative entries; it follows that:
\begin{cor}\label{cor:local-fields}
	For infinite-size grid-graphs with either purely attractive interactions $J_{ij} > 0$ or purely repulsive interactions $J_{ij} <0$, the existence of some non-vanishing fields  $\theta_i \neq 0$ stabilizes BP.
\end{cor}

\begin{thm}\label{thm:bipartite}
	The eigenvalue-spectrum of the Jacobian $\Jacobian$ is symmetric if and only if the underlying graph is bipartite.
\end{thm}
\begin{proof}
	The adjacency matrix of any bipartite graph can be rearranged and written in block form 
	\begin{align}
		\boldsymbol{A}=\begin{bmatrix} 
			\vm{0} & \boldsymbol{M} \\
			\boldsymbol{M}^T & \vm{0}
		\end{bmatrix},
	\end{align}           
	so that the eigenvalue-spectrum is symmetric~\cite{brouwer2011spectra}. 
	By the same arguments as in Lemma~\ref{lm:irreducible} it follows that $\Jacobian$ has the same structure as  $\vm{A}$ and thus a symmetric spectrum as well. 
\end{proof}
\begin{cor} \label{cor:frustrations}
	Assume we only have repulsive interactions. Then a graph is bipartite if and only if no frustrations occur.
\end{cor}
\begin{proof}
	A cycle is frustrated if and only if the product of all $J_{ij}$ along the corresponding edges is negative~\cite[p.45]{mezard2009}. Frustrations can therefore only occur in graphs with cycles of odd length, which implies that the graph cannot be bipartite~\cite[Prop.~2.27]{korte}.
\end{proof} 

\newsection{Application: Error-Correcting Codes}{solutionsBP:coding}
One of the most prominent applications where BP is successfully applied to loopy graphs  is iterative decoding.
We keep this section as self-contained as possible. For a thorough introduction we refer the interested reader to the textbooks~\cite{mackay2003, wymeersch2007}; the connection between BP and decoding is further explained in great detail in~\cite{kschischang1998iterative,aji2000generalized,kschischang2001factor}.

We consider a binary symmetric channel (BSC) with a binary input $X_i \in \sampleSpace{x}=\{0,1\} = \{x_i,\bar{x}_i\}$ and a binary output $Y_i \in \sampleSpace{Y}=\{0,1\}=\{y_i,\bar{y}_i\}$. 
The channel is specified by the error-probability $\epsilon$, where transmitted bits are flipped with probability $\epsilon$. 
That is $P_{X_i|Y_i}(x_i|y_i) = P_{X_i|Y_i}(\bar{x_i}|\bar{y_i}) = 1 - \epsilon$  and  $P_{X_i|Y_i}({x_i}|\bar{y_i}) = P_{X_i|Y_i}({x_i}|y_i) =\epsilon$ (cf. Figure~\ref{fig:bsc-fg} (a)). 
Additional, redundant bits help to detect and correct transmission errors.
The aim of error-correcting codes is to reach the desired error-correcting performance while introducing as little redundancy as necessary, i.e., to operate as close as possible to the theoretical limit.
Suppose we transmit a codeword with block length $N=7$ consisting of 4 source bits $X_1,\ldots, X_4$ and three parity-check bits $X_5,X_6,X_7$ that satisfy 
\begin{align}
	X_1\oplus X_2\oplus X_3 \oplus X_5 = 0, \nonumber \\
	X_2\oplus X_3\oplus X_4 \oplus X_6 = 0, \nonumber \\
	X_1\oplus X_3\oplus X_4 \oplus X_7 = 0, \nonumber
\end{align}
where $\oplus$ is an XOR, i.e., the sum in modulo-2 arithmetic. This linear irregular code is the (7,4) Hamming code~\cite[Chapter~1]{mackay2003}.

In this example we assume that the sent message is $\mathbf{x} = (0,0,0,0,0,0,0)$\footnote{Note that the properties of the BSC are independent of the transmitted codeword $\mathbf{x}$.} and that exactly one bit suffers from a bit flip.
For irregular codes the degree of the variables $\nodeDegree{i}$  varies; therefore, we consider two scenarios: either $\mathbf{y} = (1,0,0,0,0,0,0)$ or $\mathbf{y} = (0,0,0,0,0,1,0)$\footnote{Normally the performance of a code is studied over an ensemble of sent codewords where each bit flips with probability $\epsilon$.}.

\begin{figure}[t]
	\centering
	\includegraphics[width=0.7\textwidth]{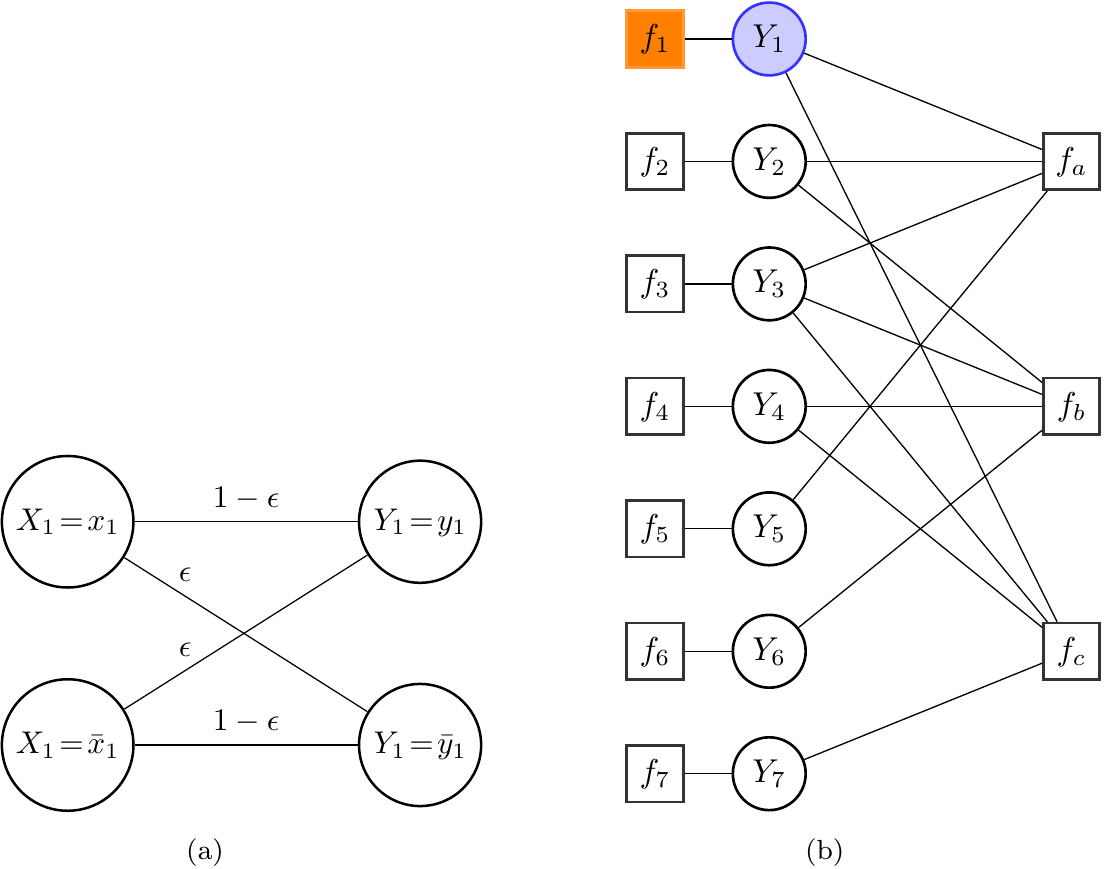}
	\caption{(a) Binary symmetric channel with error probability $\epsilon$; (b) factor graph for the (7,4) Hamming code that corresponds to~\eqref{eq:ldpc} where $Y_1$ is flipped.}
	\label{fig:bsc-fg}
\end{figure}

\begin{figure*}[t!]
	\centering
	\includegraphics[width=\textwidth]{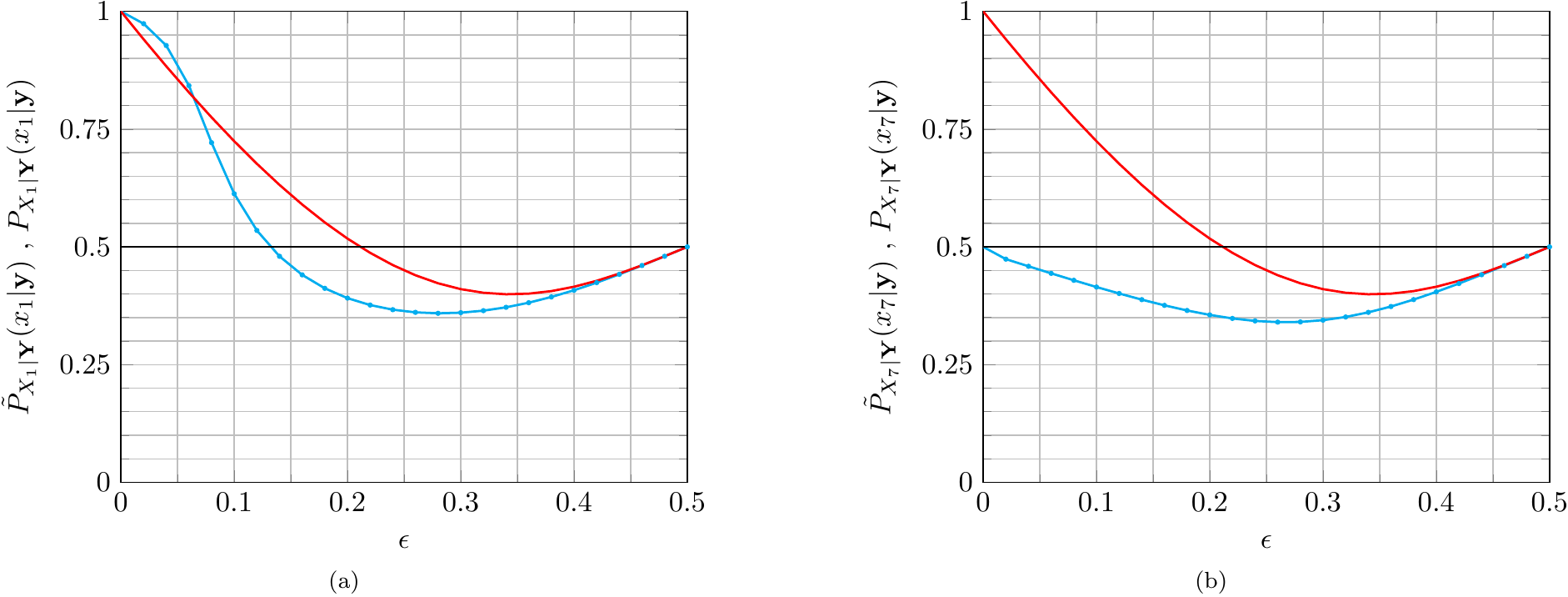}
	\caption[]{Results for the (7,4) Hamming code. We compare the exact solution $P_{X_i|\mathbf{Y}}(x_i|\mathbf{y})$ (red) to the approximate solution 
		$\tilde{P}_{X_i|\mathbf{Y}}(x_i|\mathbf{y})$ of NPHC (blue) as $\epsilon$ increases for: (a) $Y_1$ is flipped and (b) $Y_6$ is flipped.}
	\label{fig:ldpc}
\end{figure*}

It is often convenient to express a code in factorized  form and represent it explicitly with a factor graph. A factor graph consists of variable nodes $Y_i$ and factor nodes $f_A$ where each factor $f_A$ acts as a function on all variables connected $\mathbf{Y_i}=\{Y_i \in \neighbors{A} \}$.\footnote{Similar to Section~\ref{sec:background:graphs}, we use $\neighbors{\cdot}{}$ to specify the neighbors of nodes and variables.} 
On a factor graph, BP operates similar as introduced in Section~\ref{sec:bp:preliminaries}. 
Now, two different types of messages are sent along every edge: factor-to-variable messages $r_{Ai}$ and  variable-to-factor messages $q_{iA}$. All messages are iteratively updated according to
\begin{align}
	r_{Ai}^{n+1}(y_i) =  \sum_{\mathbf{y_k} : Y_k \in \{\neighbors{A} \backslash Y_i \}} f_A(\mathbf{Y_k} = \mathbf{y_k}, Y_i = y_i) \prod_{Y_k \in \{\neighbors{A} \backslash Y_i\}}  q_{kA}^{n}(y_k),
	\label{eq:factor-variable}
\end{align}
\begin{align}
	q_{iA}^{n+1}(y_i) = \alpha_{iA}^{n} \prod_{f_B \in \{\neighbors{i} \backslash f_A\}} r_{Bi}^{n}(y_i),
	\label{eq:variable-factor}
\end{align}
where $\alpha_{iA}$ is chosen such that $q_{iA}^{n+1}(y_i)+q_{iA}^{n+1}(\bar{y}_i) = 1$.
After all messages converged to a fixed point, the marginals of the variable nodes are approximated by the product of all incoming messages
\begin{align}
	\pmfApprox{Y_i}(y_i) = \frac{1}{Z} \prod_{f_B \in \neighbors{i}} r_{Bi}^{n}(y_i).
\end{align}
Details of the factor graph representation can be found in~\cite[Chapter~9]{mezard2009} and~\cite{loeliger2004introduction}.

Let us consider two types of factors: $f_i(Y_i) = P_{X_i|Y_i}(x_i|y_i)$ to model the BSC, and 
$f_A(\mathbf{Y_i})$ to verify if all parity-checks are satisfied. 
Then $f_A(\mathbf{Y_i}) = 1$ if the sum of all arguments $\sum_{\mathbf{Y_i}} y_i$ is even and $f_A(\mathbf{Y_i}) = 0 $ if the sum is odd.
The conditional probability for $\mathbf{X} = \mathbf{x}$ to be the codeword, given the received codeword $\mathbf{Y} = \mathbf{y}$ then is 
\begin{align}
	\begin{split}
		P_{\mathbf{X}|\mathbf{Y}}(\mathbf{x}|\mathbf{y}) =& \frac{1}{Z} \prod_{i=1}^{7}f_i(Y_i) \cdot f_a(Y_1,Y_2,Y_3,Y_5) \cdot 
		f_b(Y_2,Y_3,Y_4,Y_6) \cdot f_c(Y_1,Y_3,Y_4,Y_7).
	\end{split}
	\label{eq:ldpc}
\end{align}
The corresponding factor graph representation is shown in Fig~\ref{fig:bsc-fg} (b).

Now we create a system of equations similar as in Section~\ref{sec:solutionsBP:solving} and obtain all fixed points with NPHC. 
A unique fixed point exists for all settings -- and this fixed point is stable, which justifies the application of BP on error-correcting codes.
We further estimate the accuracy of the approximation; this relates to the following question.
If we communicate over a BSC, how vulnerable is BP decoding to an increased error probability?
To answer this question we obtain the fixed points for $\epsilon \in [0,0.5]$ and compare
the exact solution obtained by the junction tree algorithm 
$P_{X_i|\mathbf{Y}}(x_i|\mathbf{y})$ (red), to the approximate solution $\tilde{P}_{X_i|\mathbf{Y}}(x_i|\mathbf{y})$ obtained by NPHC (blue) in 
Figure~\ref{fig:ldpc}.\footnote{Note that for $\epsilon=0.5$ the transmission is random.}

An error can be corrected by BP decoding if $\tilde{P}_{Y_i|\mathbf{X}}(y_i|\mathbf{x}) > 0.5$. With exact decoding a single bit-flip can be corrected for $\epsilon < 0.21$. 
The fixed points obtained by NPHC, however, reveal that BP  does not utilize the full potential of the code (Figure~\ref{fig:ldpc}).
If $Y_1$ was corrupted the error can be corrected for $\epsilon < 0.13$ (Figure~\ref{fig:ldpc} (a)). BP fails to correct the error if $Y_6$ was flipped for all values of $\epsilon$ (Figure~\ref{fig:ldpc} (b)).
The reason therefore is a systematic error: the check-bit only has a single connection to the parity-check function $f_c$.
According to~\eqref{eq:variable-factor} $q_{6,c}^{n+1}(y_i) = \alpha_{7,7} \cdot r_{7,7}(y_i)$; therefore $Y_6$ does not incorporate any information from the remaining graph.
To conclude, the higher the connectivity of a node, the more information of other bits is taken into account and the better the error-correction capability of BP and NPHC. 

\newsection{Conclusion}{solutionsBP:conclusion}
This chapter was an attempt to get a deeper understanding of BP's behavior, with potential implications for 
deriving convergence guarantees and finding stronger conditions for uniqueness of BP fixed points. 

The specific focus was to apply the tools from Chapter~\ref{chp:solving}, namely the NPHC method and the concept of linearization to answer some of the open questions associated with BP (cf. Section~\ref{sec:intro:bp:questions}).
In particular we utilized the NPHC method to characterize the solution space and to obtain all BP fixed point solutions
before performing a local stability analysis to understand the convergence properties.\\

Although this approach provided many interesting insights and answered some long-standing questions, some additional questions emerged in the process.

The focus on relatively small, well-structured models provides an obvious starting point and was essential in developing the results presented.
Although many insights are expected to carry over to more general models, the generalization to larger models with fewer constraints on their potentials is primarily left open.
Moreover, the general models considered here are seemingly too small as to exhibit an actual disordered behavior (characterized by the existence of multiple fixed points).
The systems of polynomial equations considered so far, however, are already larger than most solvable problems; the generalization to even larger systems therefore remains problematic.

One key feature of the presented framework was the estimation of a favorable upper bound on the number of solutions.
In particular, the BKK bound utilized the sparsity of the polynomial system induced by the graph structure and was tight in all our experiments. 
Computing the BKK bound and creating the start system, however, took up the major part of the overall runtime in our experiments.
One will agree that the structure of the graph -- together with the update equations -- literally specifies the structure of the polynomial system. 
Yet, we only relied on generic, albeit well-established, methods to compute the BKK bound and to create the start-system.
This suggests that one should take the graph structure into account, potentially yielding a much more efficient algorithm for computing the BKK bound.\\

Regarding the approximation quality of BP, we exemplified an accuracy-gap between fixed points obtained by BP and the best possible fixed points (obtained by NPHC).
In practice this justifies the exploration of multiple fixed points:
one can then either consider a combination of weighted marginals, providing strikingly accurate results in the models considered;
or one can select the fixed point that maximizes $\partitionBethe$, often leading to the best approximation of the marginals.
We will further investigate this topic in Chapter~\ref{chp:accuracyBP} where we discuss the advantages and pitfalls of both approaches.
Moreover, we analyzed the fixed point maximizing the partition function in detail: we revealed how it continuously deforms under varying parameters and how the stability depends on the graph structure. 
The observation of a continuous deformation will serve as the main inspiration for Chapter~\ref{chp:selfguided} where we build upon this observation to propose an enhanced version of BP.\\

Regarding the convergence properties of BP, we answered several questions on the basis of our empirical observations and our theoretical analysis.
Graphs with vanishing local potentials are indeed a worst-case scenario;
strong local potentials reduce the magnitude of all eigenvalues, thus helping to achieve convergence, and additionally increase the accuracy of the most accurate fixed point.

The graph structure affects the convergence properties as the spectral radius increases with the model size, which consequently degrades the performance of BP.
Moreover, bipartite graphs exhibit a symmetric spectrum so that damping cannot help to achieve convergence;
note that the ineffectiveness of damping for bipartite models was recently also observed for models with more general potentials~\cite{ping2017belief}.
Based on our observations it seems reasonable to conjecture that damping can only be used to stabilize a fixed point if it is unique.
Whether this generalizes to spin glasses is by no means obvious.
It would, for the purpose of understanding the convergence properties of spin glass models, be of interest to specify a distribution of the eigenvalues based on the distribution of the potentials.\footnote{ 
	One could hope to obtain similar results as in the study of spectral densities on random matrices, where it is for example well-established that the eigenvalues of sparse symmetric matrices are distributed according to the Wigner semicircle law~\cite{wigner,erdos}.}

  \emptydoublepage
\newchapter{Enhancing Belief Propagation: Self-Guided~Belief~Propagation}{selfguided}
\openingquote{The most significant dimension of freedom\\ is the freedom from one’s own ego - in other words, from the feeling that I am the center of everything.}{Voytek Kurtyka}
\renewcommand{\pwd}{selfguidedBP}

This chapter builds upon the insights from the preceding chapter and presents one possible way to make the performance of BP more robust.
All models analyzed so far show that the minimum of the Bethe free energy is continuously deformed as the coupling strength increases.
As a consequence of this observation, we progressively incorporate the pairwise potentials and keep track of the evolving fixed point.
Although the idea of this algorithm was written down by the present author, such an approach may seem rather obvious;
in fact Alexander Ihler revealed that he also had a similar idea in his mind already for quite some time~\cite{ihler_personal}.

We begin this chapter with the introduction of our proposed algorithm \emph{self-guided belief propagation} (SBP) in Section~\ref{sec:selfguided:SBP}.
In Section~\ref{sec:selfguided:experiments} we apply SBP to a wide range of models and discuss our empirical observations before we provide a formal analysis for the special case of attractive models with unidirectional local potentials in Section~\ref{sec:selfguided:theory}.

All theoretical results have been conducted by the present author.
The implementation of the algorithm and the empirical evaluation in Section~\ref{sec:selfguided:experiments} were performed by the author's student Florian Kulmer. 
Both the theoretical and empirical results have been prepared for publication in~\cite{knoll_sbp}.

\newsection{Motivation}{stabilityBP:Introduction}
The analysis of BP's solution space in Chapter~\ref{chp:solutionsBP} revealed that accurate fixed points may exist even though BP fails to obtain them.
This is in accordance with the observation in~\cite{weller2013approximating} that the Bethe approximation is often accurate despite the failure of BP to converge.
The main aim of this chapter lies in enhancing BP so that it converges to accurate fixed points while preserving its simple nature (of just considering local interactions).

Besides empirical analyses and some restricted theoretical studies, it remains an open problem to obtain a rigorous understanding of the limitations of BP for general graphs.
Implementation details, as for example the initialization, play an important role in the case of multiple fixed points.
In this case, it may depend on the initialization whether BP provides accurate marginals or not.
The dependence of BP on such implementation details obviously poses a serious issue for providing performance guarantees.

One way to get convergence guarantees is to consider the equivalent optimization problem and to minimize the Bethe free energy $\FB$.
This, however, comes at the cost of an increased runtime complexity; 
polynomial-time algorithms only exist for restricted classes of problems and even approximating the global minimum might be problematic for graphical models with arbitrary potentials~\cite{chandrasekaran2011counting, shin2012complexity, weller2013approximating}. Hence, the pursuit for methods that approximate the marginals with both runtime- and convergence-guarantees is still ongoing.
These limitations motivate 
the search for modifications of BP that overcome these issues in order to increase the accuracy and enhance the convergence properties.
In this chapter, we introduce self-guided belief propagation (SBP) that aims to fill this gap.

The evolution of the fixed points revealed a close relationship between the coupling strength  and the performance of BP.
Especially strong couplings reduce the accuracy and deteriorate the convergence properties.
This and the observation that tuning the coupling strength continuously deforms the fixed point minimizing the Bethe free energy inspired us to construct a homotopy.

More precisely, we first consider only local potentials (where BP is exact and has a unique fixed point) and subsequently modify the model by increasing the pairwise potentials to the desired values.
SBP thus solves a deterministic sequence of models that iteratively refines the Bethe approximation towards an accurate solution that is uniquely defined by the initial model.

We evaluate SBP for grid-graphs, complete graphs, and random graphs with Ising potentials and, compared to BP, we observe superior performance in terms of accuracy; in fact SBP achieves more accurate results than Gibbs sampling in a fraction of runtime.
We theoretically demonstrate optimality of the selected fixed point for \emph{attractive} models with unidirectional local potentials.
Additionally SBP enhances the convergence properties and excels for \emph{general} models where SBP provides accurate results despite the non-convergence of BP.
We further expect that the ease of use lowers the hurdle for practical applications.

\newsection{Self-Guided Belief Propagation (SBP)}{selfguided:SBP}
In this section we present an intuitive justification of the proposed method and subsequently introduce SBP in detail. We further present practical considerations and pseudocode of SBP.  
A formal treatment of SBP is deferred to Section~\ref{sec:selfguided:theory}.

The current understanding of BP is that strong (pairwise) potentials negatively influence BP and that in incorporating the potentials slowly~\cite{braunstein2007encoding} may reduce the overall number of iterations.
Inspired by our observations in Chapter~\ref{chp:solutionsBP} that strong local potentials increase accuracy and lead to better convergence properties, we aim to reduce the influence of the pairwise potentials that negatively influence BP. 

SBP starts from a simple model with independent random variables and slowly incorporates the potential's strength, i.e., it solves the simple problem first and -- by repetitive application of BP, keeps track of the fixed point as the interaction strength is increased by a scaling term.
Again we resort to the homotopy continuation method for this purpose.

More formally, SBP considers an increasing length-$K$ sequence $\{\scaling_k\}$ where $k=1,\ldots,K$ such that $\scaling_{k} < \scaling_{k+1}$ and $\scaling_k \in [0,1]$ with $\scaling_1 = 0$ and $\scaling_K = 1$. This further indexes a sequence of probabilistic graphical models $\{\ugm[k]\}$ that converges to the model of interest $\ugm[K] = \ugm$. 
Every probabilistic graphical model has a set of potentials $\setOfPotentialsSBP{k} = \{\pairwiseSBP{i}{j}{k},\localSBP{i}{k}\}$ associated, where $\localSBP{i}{k} = \local{x}{i}$ and the pairwise potentials at index $k$ are exponentially scaled by
\begin{align}
	\pairwiseSBP{i}{j}{k} &=   \exp(\coupling{i}{j}\scaling_k x_i x_j)\nonumber \\
	&= \pairwise{x}{i}{j}^{\scaling_k}.
	\label{eq:scaling:potentials}
\end{align}
We further denote the fixed points of BP for $\ugm[k]$ by $\setOfMessages[\circ]_\iteration{k}$.
The initialization determines the performance of BP if multiple fixed points exist; SBP always provides a favorable initialization for the model $\ugm[k]$ by the preceding fixed point $\setOfMessages[\circ]_\iteration{k-1}$ and performs the composite function
\begin{align}
	\setOfMessages[\circ]_\iteration{K} = \BP^\circ_\iteration{K} \left( \BP^\circ_\iteration{K-1}\left( \cdots \BP^\circ_\iteration{1}\left( \setOfMessages[1]_{\iteration{1}} \right)   \right)\right).
	\label{eq:BP_composition}
\end{align}
This may lead to problems if the fixed point becomes unstable for some value $k<K$ in which case we cannot rely on BP to keep track of the fixed point anymore.
Instead, SBP  provides the last stable fixed point in that case, i.e., $\setOfMessages[\circ]_\iteration{k-1}$ as the final estimate.

In other words, SBP relaxes the problem of minimizing $\FB$ by making all variables independent (and the Bethe approximation exact).
Then, the problem is deformed into the original one by increasing $\scaling$ from  zero to one. 
Thereby, a stationary point $\FBStationary$ emerges as a well-behaved path (cf. Proposition~\ref{prop:properties} in Section~\ref{sec:selfguided:theory}) and SBP keeps track of it with BP constantly correcting the stationary point.

We illustrate how SBP 
approximates the marginals for a problem where BP fails to converge in Figure~\ref{fig:example}.
Initially, SBP obtains the pseudomarginals for $\scaling=0$ and then estimates the marginals of the desired problem by successively increasing $\scaling$ and running BP to keep track of the emerging solution path. 
Note that the approximated marginals are already close to the exact ones in this example; experiments show that this is often the case (cf. Section~\ref{sec:selfguided:experiments}).

\begin{figure}[t]
	\centering
	\includegraphics[width = 0.5\linewidth]{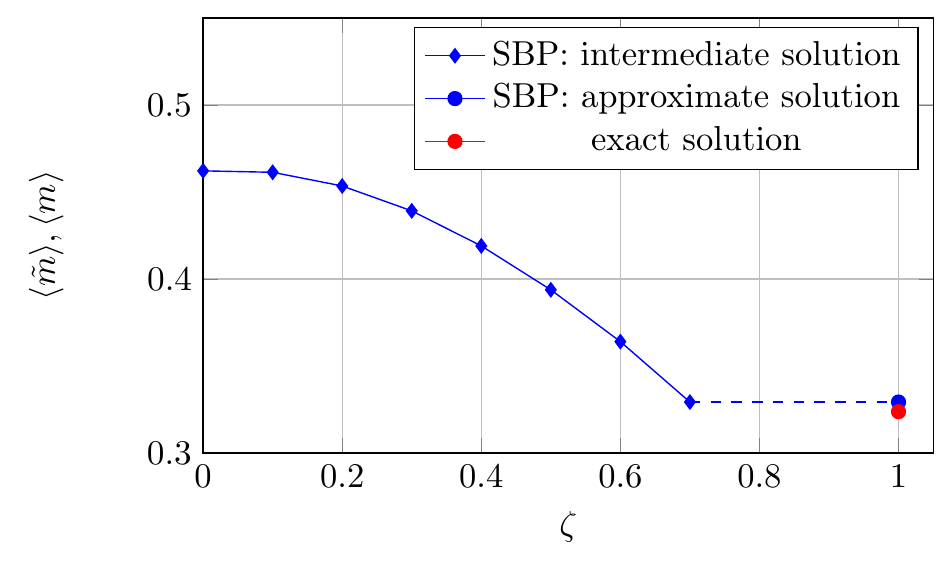} 
	\caption{ Illustrative example: SBP proceeds along the smooth solution path and obtains accurate marginals despite instability of the terminal fixed point.
		Note that the fixed point becomes unstable for $\scaling > 0.7$; SBP stops and provides the last stable solution (that is already close to the exact one) as an approximation. }
	\label{fig:example}
\end{figure}

\newsubsection{Practical Considerations}{selfguided:practical}
The procedure of SBP is essentially a path-tracking problem.
As already discussed for the NPHC method in Section~\ref{sec:solving:nphc}, such a problem consists of three key steps: 
prediction, correction, and step-size adaption.
While the implementation details of these steps may influence the performance, we have found in our experiments that the efficiency of SBP remains largely unaffected.
We will now describe the specific details used for all three steps throughout our experiments.\\

In practice the runtime of SBP is influenced by the difference between two successive fixed points $\setOfMessages[\circ]_{\iteration{k}}$ and $\setOfMessages[\circ]_{\iteration{k-1}}$ -- 
the difference is primarily determined by the number of steps $K$.
Ideally, to reduce the difference, $K$ should be as large as possible. 
This, however, increases the runtime as well (cf. Theorem~\ref{thm:complexity}); in practice we would choose $K$ as small as possible but as large as necessary.
Moreover, one can adaptively increase the step size if two successive fixed points are close, i.e., if $\setOfMessages[\circ]_{\iteration{k}}\simeq \setOfMessages[\circ]_{\iteration{k-1}}$ (cf.~\cite[pp.23]{sommese2005numerical},~\cite{allgower2003numerical}). Our experiments show that it is sufficient to use rather coarse steps (we used $K\leq 10$ for all reported experiments).

Additionally, instead of initializing $\BP_{\iteration{k}}$ with its preceding fixed point messages, i.e., $\setOfMessages[1]_{\iteration{k}} = \setOfMessages[\circ]_{\iteration{k-1}}$ one can (e.g., by spline  extrapolation)  estimate $\setOfMessages[1]_{\iteration{k}} = f(\setOfMessages[\circ]_{\iteration{k-1}},\setOfMessages[\circ]_{\iteration{k-2}},\ldots,\setOfMessages[\circ]_{\iteration{k-l}})$ so that $\setOfMessages[1]_{\iteration{k}} \cong \setOfMessages[\circ]_{\iteration{k}}$ to reduce the overall number of iterations. We empirically observed that the benefit diminishes for $l > 3$.

\newsubsection{Pseudocode}{selfguided:pseudocode}
Pseudocode of SBP is presented in Algorithm~\ref{alg:sbp}. 
The maximum number of iterations for BP is given by 
$N_{BP} = 10^3$. 
We randomly initialize $\setOfMessages[1]_{1}$ and either use fixed step size or adaptive step size ($ adaptive \ stepsize = 1 $).
The sequences of messages is contained in $\{\setOfMessages[\circ]_\iteration{k}\} = \{\setOfMessages[\circ]_\iteration{1},\ldots,\setOfMessages[\circ]_\iteration{k}\}$.

Cubic spline extrapolation is applied in \textsf{ExtrapolateMsg} to estimate the initial messages of the subsequent model.

We further present the pseudocode for the adaptive step size controller in 
Algorithm~\ref{alg:adaptivestepsize}.

\begin{algorithm}
	\caption{Self-Guided Belief Propagation (SBP)}\label{alg:sbp}
	\LinesNumbered
	\DontPrintSemicolon
	\SetKwFunction{BP}{BP}\SetKwFunction{Add}{add}\SetKwFunction{Extrapolate}{ExtrapolateMsg}\SetKwFunction{ScalePotentials}{ScalePotentials}\SetKwFunction{AdaptiveStepsize}{AdaptiveStepSize}
	\SetKwInOut{Input}{input}\SetKwInOut{Output}{output}
	\Indm
	\Input{Graph $\mathcal{G} = (\mathbf{X},\mathbf{E})$, Potentials $\setOfPotentials$}
	\Output{Fixed point messages $  \setOfMessages[\circ] $}
	\BlankLine
	\Indp
	initialization $ \setOfMessages[1]_\iteration{1} \leftarrow \setOfMessages[1] $\;
	\BlankLine
	$k \leftarrow 1$\;
	$step_{init}\leftarrow \frac{1}{10}$\;
	$\scaling_1\leftarrow 0$\;
	\While{$\scaling \le 1$}
	{
		$\setOfPotentials(\scaling_k)$ $\leftarrow$ \ScalePotentials{$\setOfPotentials,\scaling_k$}\;
		$(\setOfMessages,n)$ $\leftarrow$ \BP{$\setOfMessages[1]_\iteration{k},\setOfPotentials(\scaling_k),N_{BP}$}\;
		\eIf{ $n < N_{BP}$}
		{   $\setOfMessages[\circ]_\iteration{k} \leftarrow \setOfMessages$\; 
		}
		{ break\;	    
		}
		
		\eIf{ $ adaptive \ stepsize $}
		{   
			$\scaling_{k+1} \leftarrow \scaling_k +$ \AdaptiveStepsize{$ \{\setOfMessages[\circ]_\iteration{k}\} $, $ step_{init} $, $ k $}\;
		}
		{     $\scaling_{k+1} \leftarrow \scaling_k + step_{init}$\; 
		}
		
		$\setOfMessages[1]_\iteration{k+1}$ $\leftarrow$ \Extrapolate{$\{\setOfMessages[\circ]_\iteration{k}\}$,$\{\scaling_k\}$}\;
		$k \leftarrow k+1$\;
	}
	$ \setOfMessages[\circ] \leftarrow \setOfMessages[\circ]_\iteration{k-1}  $
\end{algorithm}
%

\begin{algorithm}
	\caption{Adaptive Step Size Controller}\label{alg:adaptivestepsize}
	\LinesNumbered
	\DontPrintSemicolon
	\SetKwFunction{Add}{add}\SetKwFunction{ComputeMarginals}{ComputeMarginals}\SetKwFunction{UnnormalizedProbability}{UnnormalizedMarginals}\SetKwFunction{MSE}{MSE}
	\SetKwInOut{Input}{input}\SetKwInOut{Output}{output}
	\Indm
	\Input{Sequence of messages $ \{\setOfMessages[\circ]_\iteration{k}\} $, $ step_{init} $, $ k $}
	\Output{$ step $}
	\BlankLine
	\Indp
	$step \leftarrow step_{init}$\;
	\BlankLine
	$ threshold $ $ \leftarrow $ $ 1 \cdot 10^{-3} $\;
	$l \leftarrow 1$\;
	\While{ $ \big( $ \MSE{$ \setOfMessages[\circ]_\iteration{k} $} $ - $ \MSE{$ \setOfMessages[\circ]_\iteration{k-l} $} $ \big) $ $ < threshold $}
	{
		$l \leftarrow l + 1$\;
		$step \leftarrow step + step_{init} \cdot l$\;
	}
\end{algorithm}
\renewcommand{\BP}{\mathcal{BP}}

\newsection{Experiments}{selfguided:experiments}
We apply SBP to attractive (Section~\ref{sec:selfguided:experiments:attractive}) and general  (Section~\ref{sec:selfguided:experiments:general}) models on $n \times n$ grid graphs of different size, and to complete and random graphs (average degree of $\averageDegree[\graph] = 3$) with $N=10$ random variables.
Experiments were performed for these graphs in order to render the computation of the exact marginals feasible and to make the results comparable to previous work~\cite{weller2014understanding,sontag2008new,meshi2009convexifying,srinivasa2016survey}.

\newsubsection{Experimental Settings}{selfguided:experiments:evaluation}

SBP is evaluated and compared to BP, $\BPD$ (BP with damping), and Gibbs sampling.
The exact marginals are obtained by the junction tree algorithm and the accuracy of the marginals is evaluated by the mean squared error $\EMarginal{m}$ between the approximate marginals at the $m$\textsuperscript{th} fixed point
$\pmfApprox{\RV{i}}^m(\RVval{i})$  and the exact marginals $\marginals{\RV{i}}(\RVval{i})$.
Additionally, we approximate the global minimum $\FBGlobalMin$ by~\cite{weller2013approximating} and evaluate the mean squared error ($\mseb$) between the approximate marginals and the marginals obtained at the global minimum of the Bethe free energy $\pmfApprox{\RV{i}}^{*}(\RVval{i})$.
We further compare the runtime of all methods by counting the overall number of BP iterations and the number of iterations for Gibbs sampling.\footnote{
	Computing the acceptance-probability requires  similar runtime as one BP message update.}

We consider $L=100$ models with random potentials for every experiment. The initial messages are randomly initialized 100 times for each of these $L$ models, before applying BP with and without damping. 
We consider BP (and $\BPD$) as converged if at least a single message initialization (out of 100) exists for which BP converges. We report the convergence ratio, i.e., the number of experiments (or probabilistic graphical models) for which BP converged at least once divided by the overall number of models $L$.

The reported mean-squared error $\EMarginal{m}$ and the number of iterations are averaged over all convergent runs of BP and $\BPD$ (i.e., BP\textsuperscript{$\circ$} and $\BPD^\circ$) while all runs that did not converge are discarded.
SBP, on the other hand, allows obtaining an approximation of the terminal fixed point in case that this fixed point is unstable, which prevents BP and $\BPD$ from converging.
Therefore, we average the error and the number of iterations over all $L$ models for SBP ($\SBP_{all}$), Gibbs sampling ($\text{Gibbs}_{all}$), and for minimization of the Bethe approximation ($\FBGlobalMin_{all}$).

For BP and SBP we set the maximum number of iterations to $N_{BP} = 10^3$ and use random message scheduling. 
For $\BPD$ we choose a large damping factor $\epsilon=0.9$ to account for strong couplings.
Such a large damping factor helps to prioritize convergence over runtime -- this admits comparison of marginal accuracy for a wide range of models. 
The maximum number of iterations, however, has to be increased to $N_{BP}=10^4$ in order to account for the slower convergence.
Carefully selecting a damping factor that depends on a given model may reduce the number of iterations until convergence but can not increase the accuracy; moreover, if chosen too small $\BPD$ may fail to converge.
The accuracy of SBP is only marginally affected by its parameters and we use the following parameters for all experiments: $ K\leq 10 $ , adaptive step size, and  cubic spline extrapolation. Gibbs sampling is run for $10^5$ iterations.

\begin{figure*}[t]
	\centering
	\includegraphics[width =\linewidth]{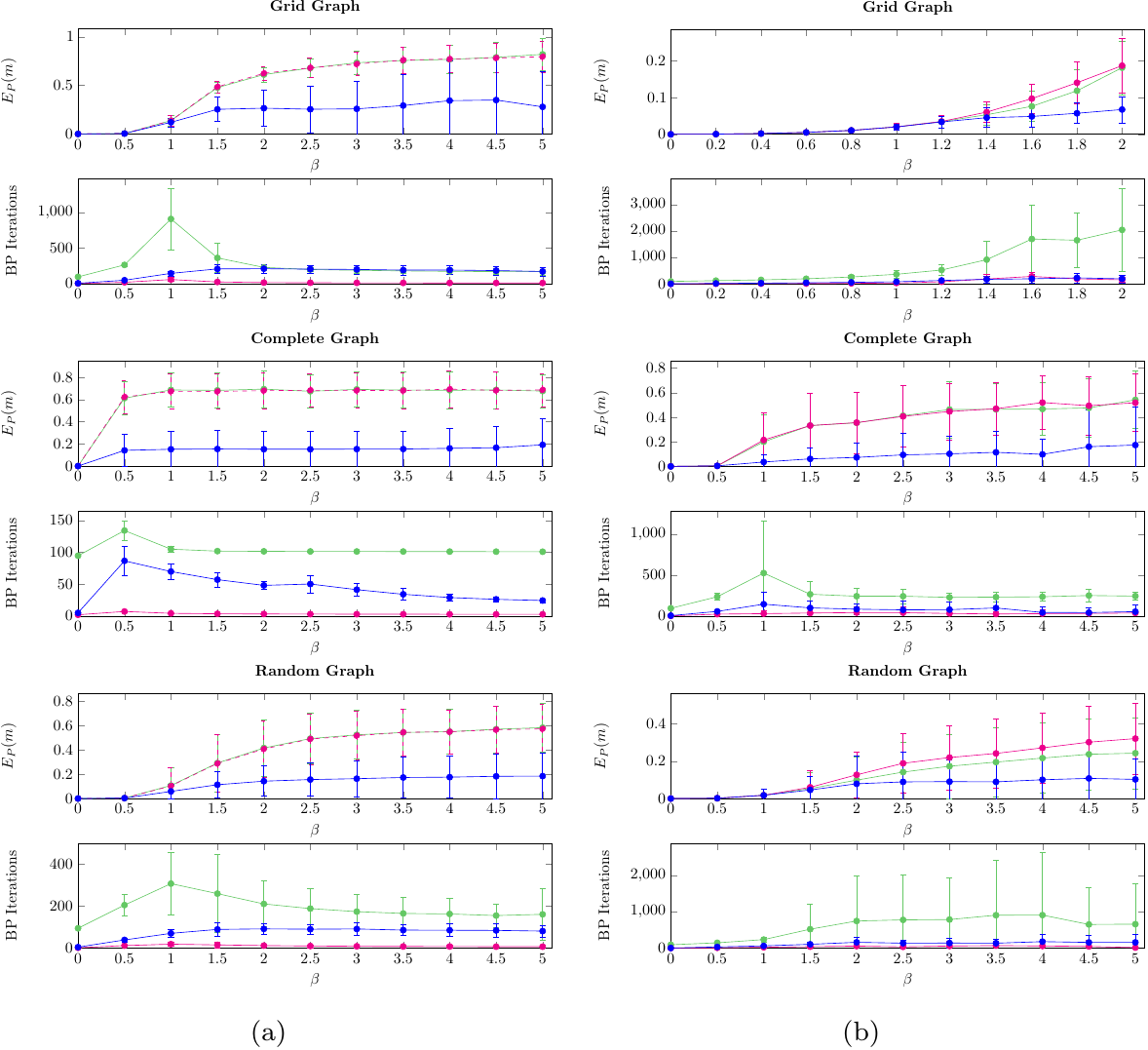}
	\caption{MSE, i.e., $\EMarginal{m}$, and number of iterations for: $ \SBP_{all} $ (blue), $ \text{BP}^\circ $ (magenta), and $ \BPD^\circ $ (green);	
		the local potentials are specified by $\field{i} \sim \mathcal{U}(-0.5, 0.5)$ and the pairwise potentials are specified by (a) $\coupling{i}{j} \sim \mathcal{U}(0,\beta)$ (attractive model); (b) $\coupling{i}{j} \sim \mathcal{U}(-\beta,\beta)$ (general model).}
	\label{fig:sbp:results}
\end{figure*}
\newsubsection{Attractive Models}{selfguided:experiments:attractive}

We consider grid graphs with $N=100$ random variables ($10 \times 10$), random graphs with $N=10$ random variables, and complete graphs with $N=10$ random variables. 
For each graph, we generate $L=100$ models for every value of $\beta \in \{0,0.5,\ldots,5\}$ and sample the potentials according to $\field{i} \sim \mathcal{U} (-0.5, 0.5)$ and $\coupling{i}{j} \sim \mathcal{U} (0,\beta)$; i.e., we consider 1100 different parametrizations.
Note that BP is randomly initialized 100 times for every considered parametrization.
We compute $\EMarginal{m}$ for every value of $\beta$ and visualize  the mean and the standard deviation of the $\mse$\footnote{Note that the mean-squared error is not Gaussian distributed but we report the standard deviation for simplicity.} as well as the number of iterations in Figure~\ref{fig:sbp:results} (a).

Note that BP (magenta) converges rapidly for all graphs considered; hence, there is no additional benefit for $\BPD$ (green) that only increases the number of iterations.
SBP (blue) only slightly increases the number of iterations as compared to BP and converges in fewer iterations than $\BPD$. 
Note that SBP is guaranteed to capture the global optimum if all local potentials are unidirectional (cf. Theorem~\ref{thm:attractive}-\ref{thm:attractive_vanishing}). 
But even if we do allow for random local potentials, we empirically observe that SBP consistently outperforms BP with respect to accuracy.
This becomes especially evident for models with strong couplings: 
these models exhibit multiple stable fixed points~\cite{knoll_fixedpoints} such that, depending on the initialization, BP often converges to inaccurate fixed points.

\begin{table*}[t]
	\caption{\textsc{Results for general models with $\coupling{i}{j} \in \{-1,1\}$ on grid graphs ($ N \!=\! 25$ and $ N \!=\! 100$), complete graphs ($ N \!=\! 10$), and random graphs ($ N\!=\!10$). 
			We report the MSE to the exact marginals, i.e.,  $\EMarginal{m}$, and the MSE to the Bethe approximation, i.e.,  $\mseb$, convergence ratio, and the overall number of BP iterations.
			Only converged runs are considered for $ \text{BP}^\circ $ and $ \BPD^\circ$ but all runs are considered for $\text{SBP}_{all}$, $\text{Gibbs}_{all}$, and $\FB^*_{all}$.}}
	\label{table:general}
	\centering
	\tiny
	\begin{center}
		\begin{tabular}{llcccccccccccc} 
			\toprule
			\multicolumn{2}{c}{}& \multicolumn{3}{c}{Grid Graph ($5\times5$)} & \multicolumn{3}{c}{Grid Graph ($10\times10$)} & \multicolumn{3}{c}{Complete Graph} & \multicolumn{3}{c}{Random Graph} \\
			\cmidrule(lr){3-5} \cmidrule(lr){6-8} \cmidrule(lr){9-11} \cmidrule(lr){12-14}
			&$\mathbf{\theta}$	&  ${0}$ & ${0.1} $  & ${0.4} $ & ${0} $ & ${0.1} $  & ${0.4} $& ${0} $ & ${0.1} $  & ${0.4} $ & ${0} $ & ${0.1} $  & ${0.4} $  \\ 
			\midrule
			
			\multirow{5}{*}{$\EMarginal{m}$ }		& $\text{BP}^\circ$  & 0.338 & 0.251 & 0.102		& -     & - & 0.184 &   0.463  & 0.466 & 0.356 & 0.252     & 0.202 & 0.101 \\
			& $\BPD^\circ$  & 0.226 & 0.198 & 0.066 	& 0.186 & 0.240 & 0.154 &  0.463 & 0.473 &  0.422 & 0.128 & 0.116 & 0.083  \\
			& $\SBP_{all}$  & \textbf{0.000}& 0.029 & \textbf{0.047}& \textbf{0.000}& \textbf{0.026} & \textbf{0.077}& \textbf{0.000}& \textbf{0.055}  & \textbf{0.074}  & \textbf{0.000}& 0.048 & 0.049  \\
			& $\FB^*_{all}$ &  0.036    &  0.042    &  0.069& - & - & - & -& - & -& - & - &- \\
			& $\text{Gibbs}_{all}$ & 0.001 & \textbf{0.016}    &   0.064& 0.001 & 0.037 &   0.120 &  0.096 &  0.096  & 0.077  & 0.001 & \textbf{0.011} & \textbf{0.048}   \\  \midrule
			
			\multirow{2}{1.5cm}{Convergence ratio} & $\text{BP}^\circ$  & 0.05    & 0.11    & 0.26  & 0.00    & 0.00    &  0.02   &   0.41    &   0.42   &   0.50   & 0.30     & 0.33    &   0.49 \\
			& $\BPD^\circ$ & 0.11     & 0.16    & 0.69 & 0.01     & 0.02    &  0.12  &  0.41    &   0.41  &   0.50   & 0.62     & 0.64    &   0.80    \\ \midrule
			
			\multirow{4}{1.5cm}{Number of iterations}  & $\text{BP}^\circ$   &    40  &  52  &  84 	&    -  &  -  &  102  &  17     &  17   &    18  	&    42  &  53  &   50   \\
			& $\BPD^\circ$ &   1370   &  1449  &  1735 &   2711  &  2313  &  2599 &   211    &  207   &   234 	&   1077   &  1057  &  873   \\
			&$\SBP_{all}$&   5   &  182  &   146&   5   &  149  &   209  &   5   &   51  &    110 	&   5   &  149  &  131  \\
			&$\text{Gibbs}_{all}$& $10^5$ &  $10^5$  &   $10^5$& $10^5$ &  $10^5$  &   $10^5$  &  $10^5$   & $10^5$  & $10^5$	&$10^5$ & $10^5$ & $10^5$\\ \midrule
			$\mseb$ &    $\SBP_{all}$  & 0.036 & 0.037 & 0.022 & - & - & - & -& - & -& - & - &-\\ 
			Global Min& $\SBP $& 100 & 10 & 23& - & - & - & -& - & -& - & - &- \\
			\bottomrule
		\end{tabular}
	\end{center}
\end{table*}

\newsubsection{General Models}{selfguided:experiments:general}

General models admit frustrated cycles and traditionally pose problems for BP and other methods that aim to minimize the Bethe approximation.

First, in order to evaluate the performance of SBP we consider $\field{i} = \field{} \in \{0,0.1,0.4\}$ and draw the couplings with equal probability from $\coupling{i}{j} \in \{-1,1\}$; the results are summarized in Table~\ref{table:general}.
Although BP and $\BPD$ fail to converge for most models we observe that SBP stops after only a few iterations and significantly outperforms BP in terms of accuracy. In fact, SBP achieves accuracy competitive with Gibbs sampling but requires three orders of magnitude fewer iterations.

Second, we further apply SBP to general graphs and evaluate
whether SBP provides a good approximation of the pseudomarginals that correspond to the global minimum of the Bethe free energy according to $\pseudomarginalsMinGlobal = \argmin_{\LPolytope} \FB(\pseudomarginals)$ (cf.~\eqref{eq:pseudomarginals_global}).
Therefore we consider grid graphs (of size $5\times 5$), which still allows us to approximate $\FBGlobalMin$ -- and the related pseudomarginals -- reasonable well by~\cite{weller2013approximating}. The results are summarized in Table~\ref{table:general} and show that SBP approximates $\pseudomarginalsMinGlobal$ within the accuracy of our reference method ($\mseb$).
We further report the number of times where SBP obtains the terminal fixed point, i.e., for $\ugm[K]$, in Table~\ref{table:general} (Global Min).
It becomes obvious that SBP  approximates the terminal fixed point reasonably well, despite frequently stopping for  $\scaling_k < 1$.
Moreover, closer inspection of the accuracy reveals that SBP does not only approximate the ``correct'' pseudomarginals well ($\mseb$), but concurrently provides an accurate approximation of the exact marginals ($\mse$).

\begin{figure}[!t]
	\centering
	\includegraphics[width=0.5\linewidth]{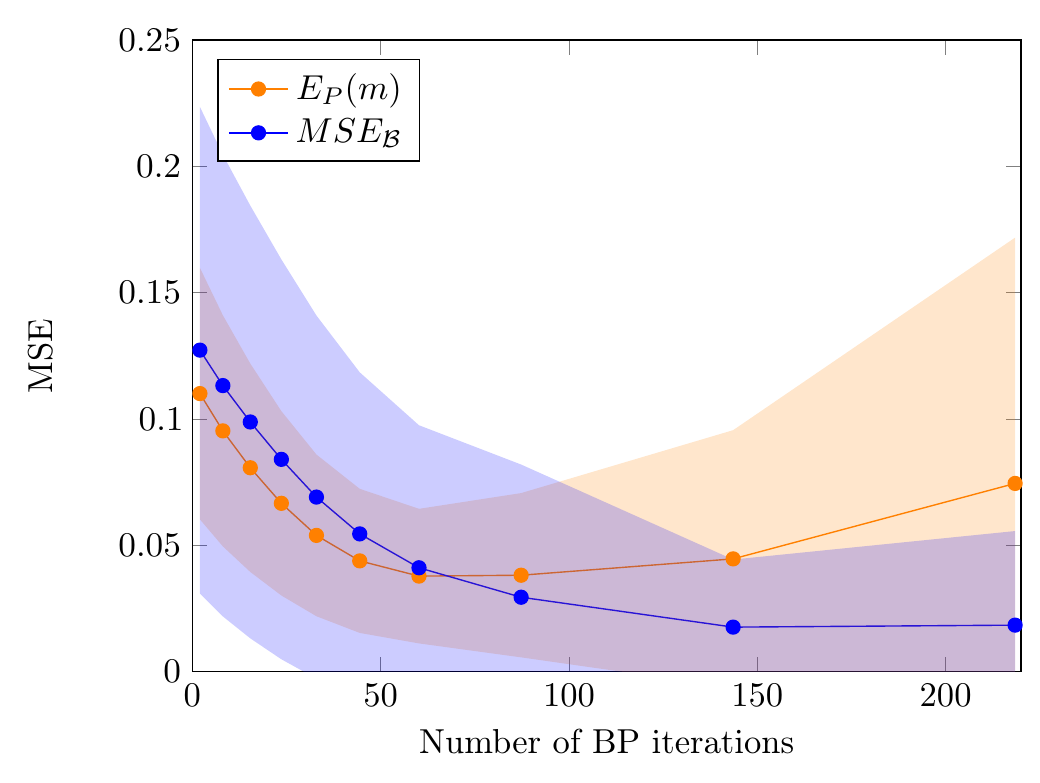}
	\caption{MSE, i.e., $\EMarginal{m}$, (orange) and $ \mseb $ (blue) over the cumulative number of iterations. Results are averaged over 100 grid graphs ($5 \times 5$) with $\field{i} = 0.4$ and $\coupling{i}{j} \in \{-1,1\}$.}
	\label{fig:mse_iterations}
\end{figure}

Third, we investigate how the approximation quality depends on the scaling parameter $\scaling_k$. Therefore, we depict the evolution of the MSE (to the exact solution) and $\mseb$ (to the approximate solution) in Figure~\ref{fig:mse_iterations}.
We observe that $\mseb$ (blue) decreases monotonically with every iteration, which empirically verifies that SBP proceeds along a well-behaved solution path (cf. Proposition~\ref{prop:properties}).
Note that $\mseb$ decreases rapidly in the first iterations and SBP spends a major part of the overall runtime for slight improvements. 
The MSE to the exact solution, on the other hand, decreases first until it increases again as SBP incorporates stronger couplings. 
Stronger couplings tend to degrade the quality of the Bethe approximation in loopy graphs and lead to marginals that are increasingly biased towards one state~\cite{weiss2000correctness,knoll_fixedpoints}.
This explains why the MSE to the exact solution increases as SBP converges towards the terminal fixed point.
One could exploit this behavior and restrict the runtime by stopping SBP after consumption of a fixed iteration budget; this may even increase the accuracy with respect to the exact solution.

Finally, we aim to investigate the influence of the coupling strength: therefore we consider $\field{i} \sim \mathcal{U} (-0.5, 0.5)$ and $\coupling{i}{j} \sim \mathcal{U} (-\beta,\beta)$ . 
For every $\beta \in [0,5]$ we execute $L=100$ experiments and present the averaged results in Figure~\ref{fig:sbp:results} (b). 
Note that we restrict the results to $\beta \leq 2$ on the grid graph because BP did only converge sporadically for models with stronger couplings.
SBP requires only slightly more iterations than BP and fewer than $\BPD$, even though we compare only to  models where BP (or $\BPD$) converged.  
The benefits of SBP become increasingly evident as the coupling strength increases. Again SBP (blue) significantly  outperforms BP\textsuperscript{$\circ$} (magenta) and $\BPD^\circ$ (green) on all graphs with respect to accuracy. 

\newsection{Theoretical Analysis of Self-Guided Belief Propagation}{selfguided:theory}
Here we present some more formal arguments and discuss the properties of SBP to understand under which conditions the algorithm (presented in Section~\ref{sec:selfguided:SBP}) can be expected to perform well.
We only present the most important Theorems and their implications below and defer the longer proofs to Appendix~\ref{sec:selfguided:proofs}.

\newsubsection{Definitions}{selfguided:theory:def}

First, we fix our notation: 
we denote the pseudomarginals, corresponding to local minima of the Bethe free energy, of $\ugm[k]$ by $\pseudomarginalsMinLocal(\scaling_k)$, and, with slight abuse of notation, we refer to the corresponding stationary point of the Bethe free energy by $\FBLocalMin{m}(\scaling_k) = \FB(\pseudomarginalsMinLocal(\scaling_k))$.
Note that the superscript $m$ accounts for the fact that we are only looking at the local minima of $\FB$ (cf. Section~\ref{sec:bp:variational:correspondence}).

It is beneficial to study the behavior of SBP as $K$ tends towards infinity. Therefore we consider the unit interval $\scaling \in [0,1]$ to be the compact support of the functions $\FB(\scaling)$ and $\pseudomarginals(\scaling)$.
SBP is inspired by the idea to proceed along a so-called \emph{solution path} as $\scaling$ increases from zero to one in order to obtain the marginal distributions for the model of interest.
Therefore, we shall consider a continuous homotopy function $\homotopy(\setOfMessages,\scaling): \mathbb{R}^{|\setOfMessages|+1} \rightarrow \mathbb{R}^{|\setOfMessages|}$ that is defined by
\begin{align}
	\homotopy(\setOfMessages,\scaling) = \setOfMessages - \BP(\setOfMessages) \quad \text{where} \quad \setOfPotentials = \setOfPotentials(\scaling).\label{eq:homotopy_sbp}
\end{align}
Then, a \emph{solution path} 
\begin{align}
	c(\scaling): \homotopy(\setOfMessages,\scaling) =0
\end{align}
exists that (i) has a start point $c(\scaling=0) =\setOfMessages: \homotopy(\setOfMessages,\scaling=0) = 0$, (ii) an endpoint $c(\scaling=1) =\setOfMessages: \homotopy(\setOfMessages,\scaling=1) = 0$, and (iii) is continuous over $\scaling \in [0,1]$, i.e., it connects the start- with the endpoint.
SBP then proceeds along some solution path from a given start- to  its endpoint. 
Note that the solution path $c(\scaling)$ is defined along the solutions to the fixed point equations and thus it implicitly defines the pseudomarginals $\pseudomarginalsMinLocal(\scaling)$ by~\eqref{eq:marginals:single} and~\eqref{eq:marginals:pw}.
In particular, we refer to the start- and endpoint by $\pseudomarginalsMinLocal(\scaling=0)$ and $\pseudomarginalsMinLocal(\scaling=1)$ respectively.

The following example in Figure~\ref{fig:solution_path} illustrates the fixed point evolution for a grid graph with attractive couplings.
This example exhibits a unique solution path according to our definition; note, however,
that a second curve exists, which lacks a start point and is therefore of no 
relevance for any method that proceeds along a solution path defined by the homotopy in~\eqref{eq:homotopy_sbp}.

%
\begin{figure}[t]
	\centering
	\includegraphics[width=0.5\linewidth]{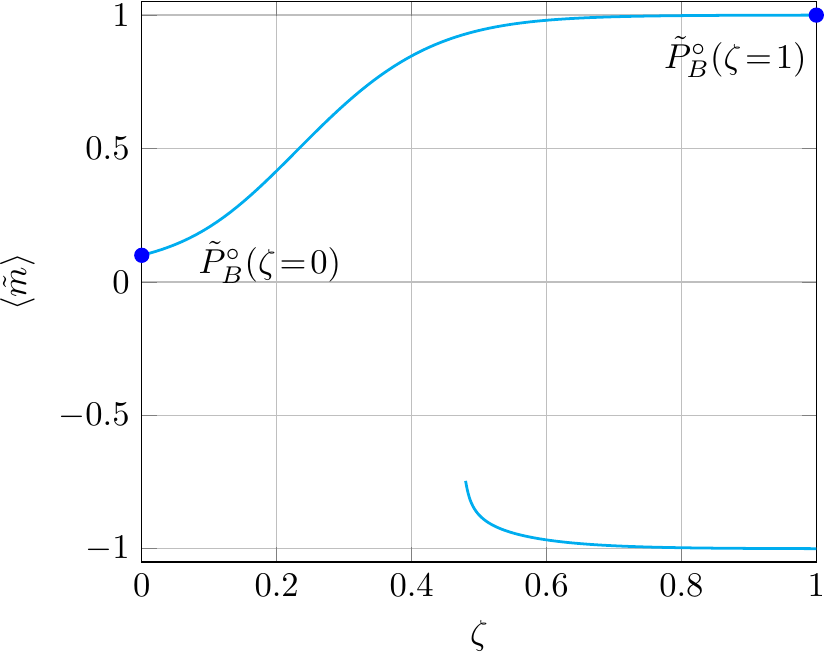}
	\caption{ Solution path (cf. Section~\ref{sec:selfguided:theory:def}) for a grid graph; start- and endpoint are depicted by blue points. Note how Propositions~\ref{prop:properties}.1~-~\ref{prop:properties}.2 are fulfilled.}
	\label{fig:solution_path}
\end{figure}

\newsubsection{Properties of Self-Guided Belief Propagation}{selfguided:properties}
The following proposition summarizes the main properties of the solution path that is specified and followed by SBP.
\begin{prop}[Properties for attractive and general models] $ $\par
	\begin{enumerate} 
		\item [(1)] BP has a unique fixed point $\setOfMessages[\circ]_{\iteration{1}}$ for $\scaling_{1}=0$, so that SBP has a unique start point $\startp$. (cf. Theorem~\ref{thm:init})
		\item [(2)] A smooth (i.e., continuous) solution path originates from the start point $\startp$. \\(cf. Theorem~\ref{thm:smooth})
		\item [(3)] SBP efficiently proceeds along this (unique) solution path. \\(cf. Theorem~\ref{thm:complexity})
	\end{enumerate}
	\label{prop:properties}
\end{prop}

\begin{thm}[Proposition~\ref{prop:properties}.1]\label{thm:init}
	A unique solution exists for $\scaling=0$, i.e., a single start point $\startp$ exists, and BP is guaranteed to converge. Moreover, this start point is exact, i.e.,  $\singleApprox{i}(\scaling=0) = \singleExact{i}(\scaling=0)$ for all $\RV{i} \in \setOfNodes$.
\end{thm}
\begin{proof}
	First, let us obtain the singleton marginals by summing over the pairwise marginals such that $\pmfApprox{\RV{i}}(x_i) = \sum_{x_j \in \sampleSpace{X}} \pairwiseApprox{i}{j}(x_i,x_j)$.
	For $\scaling_1 = 0$, that is for  $\pairwiseSBP{i}{j}{\iteration{1}} = 1$ it follows that marginalizing over the pairwise marginals (cf.~\eqref{eq:marginals:pw}) equates to 
	\begin{align}
		\pmfApprox{\RV{i}}(x_i)= \localShort{i}\!\!\! \prod_{\RV{k} \in \{\neighbors{i} \backslash \RV{j}\}} \msg[\circ]{k}{i}{} 
		\cdot \sum_{x_j \in \sampleSpace{X}} \localShort{j} \!\!\!\prod_{\RV{l} \in \{\neighbors{j} \backslash \RV{i}\}} \msg[\circ]{l}{j}{}. \label{eq:init} 
	\end{align}
	Note that according to~\eqref{eq:update} $\msg[\circ]{j}{i}{}$ is equivalent to the second part of~\eqref{eq:init} so that
	\begin{align}
		\pmfApprox{\RV{i}}(x_i) =  \localShort{i} \prod_{\RV{k} \in \neighbors{i}} \msg[\circ]{k}{i}{},
	\end{align}
	which equals~\eqref{eq:marginals:single}. It follows that $\pmfApprox{\RV{i}}(+1)  =  e^{\field{i}}/(e^{\field{i}}+e^{-\field{i}}) = \marginals{\RV{i}}(+1) $.
\end{proof}
Note that Theorem~\ref{thm:init} concurs with the sandwich-bound~\cite[Th.4]{weller2013bethe} that reduces to  
$\marginals{\RV{i}}(+1) = {\field{i}}/(e^{\field{i}}+e^{-\field{i}}) = \pmfApprox{\RV{i}}(+1)$ for $\coupling{i}{j}=0$. 
Theorem~\ref{thm:init} thus reduces the problem of initializing SBP to computing $\setOfMessages[\circ]_{1}$, which can be done in linear time.

\begin{thm}[Proposition~\ref{prop:properties}.2]\label{thm:smooth}
	Let $\pseudomarginalsMinLocal(\scaling)$ be the 
	pseudomarginals that are defined along the unique solution path that originates from $\startp$. Then, $\pseudomarginalsMinLocal(\scaling)$ and the associated stationary points $\FBLocalMin{m}(\scaling)$ are continuous on their compact support $\scaling \in [0,1]$.
\end{thm}
\begin{proof}
	First, we show that the Bethe free energy $\FB(\scaling)$ itself is an analytic function. Consider~\eqref{eq:f_bethe} with the pairwise potentials defined by~\eqref{eq:scaling:potentials}. Then, the derivative with respect to $\scaling$ is given by
	\begin{align}
		\frac{\partial \FB(\scaling)}{\partial \scaling} &=\!\! -\frac{\partial }{\partial \scaling}\!\!\sum_{\edge{i}{j} \in \setOfEdges} \sum_{x_i,x_j} \!\!\pmfApprox{\RV{i},\RV{j}}(x_i,x_j) \ln \pairwiseShort{i}{j}  \nonumber  \\
		&= \!\!-\frac{\partial}{\partial \scaling}\!\!\sum_{\edge{i}{j} \in \setOfEdges} \sum_{x_i,x_j} \!\! \pmfApprox{\RV{i},\RV{j}}(x_i,x_j) \cdot \scaling \cdot \coupling{i}{j} \cdot x_i x_j\nonumber\\
		& = - \sum_{\edge{i}{j} \in \setOfEdges} \coupling{i}{j} \cdot \correlation{i}{j}.  \label{eq:continuouslyDifferentiable}
	\end{align}
	As an immediate consequence, we observe that $\FB(\scaling)$ is continuously differentiable\footnote{Strictly speaking $\FB(\scaling)$ is an analytic function.} as \eqref{eq:continuouslyDifferentiable} is a finite sum over finite terms.\footnote{This is in accordance with the fact that true phase transitions (singularities in the derivative of the free energy) can occur only in the thermodynamic limit, where \eqref{eq:continuouslyDifferentiable} is an infinite sum that equates to infinity.}
	
	We specifically consider the minimum $\FBLocalMin{m}(\scaling)$ that emerges from the global minimum $\FBLocalMin{m}(\scaling=0) = \FBGlobalMin(\scaling=0)$; this start point is unique by Theorem~\ref{thm:init}. 
	It follows by~\eqref{eq:continuouslyDifferentiable} that 
	$\FBLocalMin{m}(\scaling)$ varies in a continuous fashion along the unique solution path for $\scaling  \in [0,1]$.
	Further note that stationary points are in a one-to-one correspondence with fixed points of BP which completes the proof.
	
	Further note that the set of stationary points is finite (cf. Section~\ref{sec:bp:variational:energy_landscape})\footnote{This is also required for the one-step replica symmetry breaking assumption~\cite[Section19]{mezard2009}.} and that pitchfork bifurcations may only  occur if $\field{i} = 0$~\cite{pitkow2011learning}, in which case SBP obtains the exact solution (cf. Theorem~\ref{thm:attractive}).
\end{proof}

Theorem~\ref{thm:smooth} substantiates the claim that a smooth solution path emerges from the simple problem.

\begin{thm}[Proposition~\ref{prop:properties}.3]\label{thm:complexity}
	There exists some scaling factor $\scaling_k \leq 1$ for which  SBP converges to $\pseudomarginalsMinLocal(\scaling_k)$ in $\bigO(K N_{BP})$.
\end{thm}
\begin{proof}[Proof of Theorem~\ref{thm:complexity}]
	SBP increases $\scaling_k$ as long as BP converges in fewer than $N_{BP}$ iterations, and stops otherwise. Consequently, BP 
	corrects the accuracy of the fixed point for 
	each value $\scaling_k$ within a bounded number of iterations. The runtime of SBP is further determined by the choice of $K$, i.e., the step-size (cf. Section~\ref{sec:selfguided:practical}).
	Assume that SBP converges for $\scaling_k$, then it does so in $\bigO(K \cdot N_{BP})$.
\end{proof}
SBP is consequently capable of efficiently tracking the fixed point that emerges as $\scaling$ increases and requires $KN_{BP}$ iterations at most.
SBP may, however, only converge to a surrogate model for $\scaling_k < 1$ and is not guaranteed to obtain the pseudomarginals of the desired problem.
One can characterize this error by computing a bound on  $|\FBLocalMin{m}(\scaling_k) - \FBLocalMin{m}(\scaling_K)|$ given the difference between $\setOfPotentials_{k}$ and $\setOfPotentials_{K}$ (cf.~\cite[Theorem 16]{ihler2005loopy}).
\begin{cor}
	SBP obtains the pseudomarginals of the desired problem if and only if the endpoint $\terminalp$ is stable, i.e., if $\terminalp \in \setOfStableSol$. 
	This is an immediate consequence of the fact that the convergence properties can only degrade along a given solution path (cf. Section~\ref{sec:stabilityBP:theoretical_analysis}).
\end{cor}

Theorem~\ref{thm:init}-\ref{thm:complexity} are of fundamental importance but do not relate to the accuracy of the obtained stationary point.  
Assessing the quality of the Bethe approximation and the accuracy of BP for general models is still an open research question that is beyond the scope of this thesis.
However, we further present Proposition~\ref{prop:attractive} to discuss the accuracy of the obtained solution for \emph{attractive} models.
\begin{prop}[Properties for attractive models with unidirectional fields] 
	The solution path $c(\scaling)$ leads towards an accurate solution with $\pseudomarginalsMinLocal(\scaling_k) = \pseudomarginalsMinGlobal(\scaling_k)$ and $\FB(\pseudomarginalsMinLocal(\scaling_k)) = \FB(\!\pseudomarginalsMinGlobal(\scaling_K))$. (cf. Theorem~\ref{thm:attractive} and Theorem~\ref{thm:attractive_vanishing})
	\label{prop:attractive}
\end{prop}

We start by generalizing Griffiths' inequality~\cite{griffiths1967correlations} to the fixed points of BP (Lemma~\ref{lm:griffith}) and subsequently provide Theorem~\ref{thm:attractive}-\ref{thm:attractive_vanishing} that discuss the accuracy of the fixed point obtained by SBP.

\begin{lm}\label{lm:griffith}
	Consider two attractive probabilistic graphical models $\ugm[0]$ and $\ugm[1]$ with equal $\graph$ and with the potentials specified by $\field{i} > 0$ and by $\coupling{i}{j}^{0}$ and $\coupling{i}{j}^{1}$, where $\coupling{i}{j}^{0} < \coupling{i}{j}^{1}$ for all $\edge{i}{j} \in \setOfEdges$.
	Let us consider a fixed point of BP with positive means $\meanMinLocal{i}^{0} \in (0,1]$.
	Then,  $\meanMinLocal{i}^{0} < \meanMinLocal{i}^{1}$ and $\chi_{ij}^{0} < \chi_{ij}^{1}$.
\end{lm}
The proof of Lemma~\ref{lm:griffith} contains some tedious, although not too complicated, algebraic manipulations and is thus deferred to Appendix~\ref{sec:selfguided:proofs:2}

\begin{thm}[Prop.~\ref{prop:attractive}]\label{thm:attractive}
	Consider an attractive model with $\field{i} > 0$.\footnote{Note that equal results can be obtained for $\field{i} < 0$ because of symmetry properties.}
	Then, $\meanMinLocal{i}(\scaling)$ increases monotonically along the solution path $c(\scaling)$; in particular SBP minimizes the Bethe approximation error and is optimal with respect to marginal accuracy, i.e.,
	\begin{align}
		\pseudomarginalsMinLocal(\scaling) &=\argmin\limits_{\pseudomarginals^{k}\in\setOfMinimaSol}\EMarginal{k} \nonumber \\
		&= \argmin\limits_{\pseudomarginals^{k}\in\setOfMinimaSol}\EPartition{k}. 
	\end{align}
\end{thm}
We will only sketch the proof of Theorem~\ref{thm:attractive} here and refer to Appendix~\ref{sec:selfguided:proofs:th9} for a more elaborate treatment.
Essentially, we have to bring the preceding Theorems together to show that SBP obtains the most accurate marginals for attractive models with unidirectional fields.

Theorem~\ref{thm:init} explains how the initial model can be solved exactly and that it has positive mean and correlation for $\field{i}>0$.
Consequently, Lemma~\ref{lm:griffith} and Theorem~\ref{thm:smooth} apply and guarantee a continuous solution path with monotonically increasing means $\meanMinLocal{i}(\scaling)$.
Optimality of the obtained fixed point then is an immediate consequence of the definition of $\FB$ (cf.~\eqref{eq:f_bethe}), the RSB assumption (which holds for these models, cf. Section~\ref{sec:solutionsBP:experiments:gridGraphUniform}), and the existence of at most two local minima, i.e., $\setOfMinimaSol \leq 2$, for the considered models (cf. Lemma~\ref{lm:optimality_2sol} and \cite{weller2014understanding}).

Note that an attractive model with unidirectional fields is a special case of a (larger) attractive model with vanishing fields (cf.~\cite{fisher1967critical} and~\cite[Section 1.2]{saade2017spectral}).

\begin{thm}[Proposition~\ref{prop:attractive}] \label{thm:attractive_vanishing}
	Consider an attractive model with $\field{i} = 0$. Then, SBP obtains the exact solution, i.e.,  $\pseudomarginalsMinLocal(\scaling=1) = \pseudomarginalsExact(\scaling=1)$.
\end{thm}
\begin{proof}
	Only considering attractive models makes it straightforward to calculate the exact solution if all $\field{i}=0$. 
	The exact solution always has zero mean for all random variables, i.e.,   $\mean{i}(\scaling)=0$ for all values of $\scaling$. 
	Moreover, the exact solution coincides with a stationary point  $\FBLocalMin{m}(\scaling_k)$~\cite{mooij2007sufficient}.
	Depending on the coupling strength, this stationary point is a stable minimum (for sufficiently small values of $\coupling{i}{j}<J_C$) and a local maximum (for $\coupling{i}{j}>J_C$) (cf. Section~\ref{sec:stabilityBP:empirical_analysis:vanishing} or to \cite[pp.385]{mezard2009}).
	
	SBP consistently obtains this fixed point irrespective of its stability. 
	Therefore, note that by Theorem~\ref{thm:init}  all $\meanMinLocal{i}(\scaling=0) = 0$ and  all messages are equal, i.e.,  $\setOfMessages[\circ]_{1}(\scaling=0) = 1/2$ for all $\edge{i}{j} \in \setOfEdges$.
	SBP remains exactly on this fixed point and obtains the exact marginals as these fixed point messages can be represented exactly, without any quantization errors in binary arithmetic. 
\end{proof}

To conclude, for attractive models with $\field{i} \leq 0$ or $\field{i} \geq 0$,  SBP either obtains the fixed point that corresponds to the global minimum $\FBGlobalMin$ (Theorem~\ref{thm:attractive}), or it obtains the fixed point that corresponds to the exact solution, i.e., to $\min \FGibbs{}$ (Theorem~\ref{thm:attractive_vanishing}).

For \emph{general models} $\mean{i}$, need not increase monotonically along the solution path and it is not obvious whether SBP obtains the most accurate fixed point. 
However, the experimental results in Section~\ref{sec:selfguided:experiments:general} at least corroborate that SBP converges to accurate fixed points in many cases.

\newsection{Conclusion}{selfguided:conclusion}
This chapter aimed at deriving a simple and robust method to efficiently perform approximate inference for models where BP fails to converge.
The observations from our solution space analysis in Chapter~\ref{chp:solutionsBP} suggested that an accurate fixed point lies at the end of a solution path that emerges from the origin (where $\coupling{i}{j} = 0$).
We managed to enforce convergence towards this fixed point by changing the updated function accordingly.
To do so, we resorted to the homotopy method once again and utilized its proven capability of keeping track of solution paths, where we relied on BP in the correction step.

Our theoretical analysis of SBP for attractive models with unidirectional local fields revealed the existence of a unique, well-behaved solution path that leads to the best possible fixed point of BP and that can be tracked successfully by SBP.
We further applied SBP to general models, where, despite the lack of theoretical guarantees, it exhibited a promising performance.
Overall SBP consistently obtained marginals of better quality than BP with and without damping and approximated the marginals well on models for which BP did fail to converge at all.\\

One surprising effect of SBP was that it sometimes provided marginals that were even more accurate than the ones at the global minimum of $\FB$.
This occurred frequently, specifically whenever SBP terminated early because of BP failing to converge in the correction step before the pairwise potentials were set to their desired value.

So far we can only provide some hand-waving arguments for this behavior:
Therefore, note that common wisdom of BP tells us that BP usually fails to converge because of strong pairwise potentials;
additionally, we have seen how the marginals become overconfident and tend to be strongly biased towards one state as the coupling-strength increases.
The latter property explains why the approximation quality of $\FB$ degrades as its minimum moves further away from the minimum of $\FGibbs{}$.
This implies that following the solution path to its very end only makes the accuracy of the estimated marginals worse.
It seems that the former property implicitly prevents this from happening and that the stability breaks down somewhat close to the exact solution.

First empirical evaluations suggest that this is indeed quite often the case.
This raises an intriguing theoretical question:
Is there really some fundamental connection between the exact marginals and the approximated marginals that are obtained at the onset of instability?-- 
and if so,  what is the nature of this connection?\\

In a more concrete manner, it also remains to generalize our theoretical analysis.
So far we have only discussed the case of attractive models with unidirectional local fields, which -- admittedly -- is a rather restrictive setting.
Although this restriction simplified our analysis significantly and provided some interesting insights into the nature of SBP, it still leaves some room for improvement.

Going back to our theoretical analysis once more, our statement of optimality relied on the fact that the global minimum of $\FB$ is optimal with respect to the accuracy of both the marginals and the free energy.
Yet, it is not obvious if such a relationship generalizes to more general models as well.
We will thus devote the next chapter to studying the underlying principles that relate  the approximation quality of the marginals and of the free energy to each other.

  \emptydoublepage
\newchapter{Understanding Belief Propagation: Accurate~Marginals~or~Partition~Function}{accuracyBP}
\openingquote{To the wise, life is a problem; \\to the fool, a solution.}{Marcus Aurelius}
\renewcommand{\pwd}{accuracyBP}

This chapter introduces and investigates a novel class of models; these are attractive models with varying local potentials that we term \emph{patch-potential models}.
The inspiration for this chapter is twofold:
First, our solution space analysis was restricted to relatively small models so far, which limits the insights; this fueled the search for models that exhibit more than three fixed points.
Second, the proof for optimality of SBP relied on the fact that the global minimum of the Bethe free energy provides the most accurate marginals.
The reliance on this crucial detail led to questioning the common assumption that such a correspondence extends to more general models as well.

We begin with the introduction of patch potential models and the specification of the considered models in Section~\ref{sec:accuracyBP:models}.
The main contribution of Section~\ref{sec:accuracyBP:evalulate_accuracy} is the discussion of the solution space properties by means of an example.
In particular, this includes a detailed evaluation of the error in the marginals and the partition function for a wide range of parameters.
We then provide theoretical arguments in Section~\ref{sec:accuracyBP:theory} that explain the differences between accurate marginals and an accurate partition function and show under which conditions a fixed point exists that approximates both quantities well.

The content of this chapter constitutes the major part of the work in~\cite{knoll_accuracy}.

\newsection{Motivation}{accuracyBP:intro}
We have restrained ourselves to relatively well-understood models so far:
both the exhaustive analysis of the solution space in Chapter~\ref{chp:solutionsBP} and the validation of the proposed algorithm in Chapter~\ref{chp:selfguided} are limited to models with identical or random potentials.
While models with identical potentials either have one or three fixed points -- and are thus too restrictive as to fully understand  the behavior of BP --
models with random potentials lend themselves to an analysis in terms of expected behavior.
The randomness of the potentials, however, introduces a solution space so complex that it cannot be analyzed exhaustively anymore.

We would like to have a model class, simple enough so that it can be well understood; yet complex enough so that it exhibits a rich solution space with more than just three fixed points.
For this particular purpose, we introduce \emph{patch potential models}, a rich class of attractive models with inherent structure.
These models exhibit many interesting phenomena and provide deep insights into the relationship between the approximation quality of the marginals and the partition function.\\

So far we have observed that the approximation quality may be severely affected by the existence of multiple fixed points with varying accuracy.
In the optimization literature, this issue is often eluded by obtaining and combining all fixed points (cf. RSB theory in Section~\ref{sec:intro:bp:improving:evaluation}).
Although a similar approach seems promising for the models studied so far (cf. Chapter~\ref{chp:solutionsBP}), 
computing all fixed points is problematic for more general models. 
It is often simply not possible to obtain the set of all fixed points in an efficient manner.
Remember that one particular aim of Chapter~\ref{chp:solutionsBP} was to specify model classes for which a unique fixed point exists;
with the focus on patch potential models, we adhere to this tradition
but aim to specify model classes for which a structured and well-behaved solution space is present.
Instead of conditions for uniqueness, we are interested in conditions that allow the set of all fixed points to be obtained efficiently.\\

On the other hand, if one is only interested in selecting a single fixed point, one should obviously strive for selecting the best possible fixed point available.
Unfortunately, however, fixed points cannot be compared with respect to the marginal accuracy unless the exact solution is available.

Considering the equivalent variational interpretation admits various provable convergent algorithms (cf. Section~\ref{sec:bp:variational:variants}). 
Note that the correspondence to the Bethe free energy allows one to make some more quantitative statements regarding the approximation quality.
At least for attractive models, it is well established that the Bethe free energy upper bounds the Gibbs free energy~\cite{ruozzi2012bethe} so that the global minimum provides the most accurate approximation of the partition function.
Similar properties are not known for the marginal accuracy, however, and, except for rather simple models (as analyzed in Chapter~\ref{chp:solutionsBP}), it remains an open question whether accurate marginals are to be obtained at the global minimum of the Bethe free energy.
It is a common conjecture that this is the case;
if true, this would provide a simple way of comparing the marginal accuracy of multiple solutions in terms of their corresponding partition function.
This chapter studies this relationship in detail and further provides sufficient conditions for the global minimum of $\FB$ to yield the most accurate marginals.

 \newsection{Model Specifications}{accuracyBP:models}
We focus on one specific model class that represents a special case of binary pairwise models.
We will consider only finite-size attractive models
i.e., where all couplings $\coupling{i}{j}> 0 $ are positive; 
specifically, we consider models with equal couplings $\coupling{i}{j} = \coupling{}{}$ for all edges $\edge{i}{j} \in \setOfEdges$. 
We shall further distinguish three different types of attractive models that show increasingly complex behavior:
(i) attractive models with vanishing local fields $\field{i} = 0$;
(ii) attractive models with unidirectional fields, i.e., either $\field{i} < 0$ or $\field{i} > 0$; and
(iii) finally, attractive models with arbitrary local fields. Such models are particularly interesting in terms of their phase transitions and are studied under the name of ferromagnetic random-field Ising models (RFIM) in physics where all $\field{i}$ are drawn according to some distribution.

\newsubsection{Definitions}{accuracyBP:models:definitions}
Attractive models with vanishing fields either have a unique or two \emph{symmetric} fixed points both for infinite-size models~\cite{mezard2009} as well as for finite-size models (cf. Chapter~\ref{chp:solutionsBP}). 
The marginals of two fixed points $m$ and $k$ are considered as \emph{symmetric} if 
\begin{align}
	\singleApprox{i}^{m}(+1) =  1-\singleApprox{i}^{k}(+1)
\end{align}
for all $\RV{i}$. 
An eminent consequence of how the Bethe free energy was defined as a function of the pseudomarginals (cf.~\eqref{eq:f_bethe}) is that symmetric fixed points must also have the same value of $\FB$.
Attractive models with unidirectional fields show a similar behavior and -- although not exactly symmetric -- have two fixed points that are almost symmetric.

Another important concept are \emph{flipped} random variables: a random variable is flipped if the marginals are not aligned with the local potential, i.e., if
\begin{align}
	\Big(\frac{\pmfApprox{\RV{i}} (+1)}{\pmfApprox{\RV{i}} (-1)} - 1\Big) \field{i} < 0.
\end{align}
We further say that a fixed point is \emph{state-preserving} if 
no random variable is flipped. If all marginals are in favor of the same state $\RVval{i}$, i.e., if $\pmfApprox{\RV{i}}(\RVval{i}) > 0.5$ for all $\RV{i} \in \setOfNodes$ we call the corresponding fixed point \emph{biased towards $\RVval{i}$}.

Attractive models with arbitrary local fields  exhibit many non-trivial properties, may have a complex solution space, and are  studied as one of the simplest forms of disordered systems~\cite{young1998spin}. 
Disordered systems are systems that potentially have many fixed points, whereas many random variables are flipped.

\newsubsection{Patch Potential Models}{accuracyBP:models:patch_potential_models}
The definition of patch potential models follows the definitions of the RFIM, with the main difference that the local potentials are not i.i.d but obey a correlation between neighboring random variables. 
Moreover, we will only consider models with identical values for all local fields, albeit possibly with different sign, i.e., $\field{i} \in \{-\field{},+\field{}\}$.
\begin{defn}
	Patch potential models are binary pairwise models in accordance with~\eqref{eq:exponential} that have attractive couplings $\coupling{i}{j} =  J>0$ 
	and that consist of multiple non-overlapping patches $\patch{i}$ with $\graph = \bigcup_{i} \patch{i}$. 
	A patch $\patch{i} = (\patchRV{i},\patchEdges{i})$
	is a connected subgraph that is induced by  a subset of nodes $\patchRV{i} \subset \setOfNodes$ with identical local potentials $\field{i} = \field{}$ or $\field{i} = -\field{}$, where
	$\patchEdges{i} = \{\edge{i}{j}\in\setOfEdges: \RV{i}, \RV{j} \in \patchRV{i}\}$.
	\label{def:model}
\end{defn}

Note that we will only consider models with sufficiently large patches, so that the exact marginals are state-preserving.
Let us first consider a minimal example that is rich enough to exhibit some non-trivial (i.e., non-symmetric) fixed points while being structured enough 
to admit only few fixed points.
This example serves as a model that allows us to get some intuition (cf. Section~\ref{sec:accuracyBP:evalulate_accuracy}) before we discuss the properties of patch potential models in a more general manner (cf. Section~\ref{sec:accuracyBP:theory}).
\begin{example}[Patch Potential Model]\label{ex:patch}$ $\newline
	Let $\graph = (\setOfNodes, \setOfEdges)$ be a regular two-dimensional grid graph of size $n\times n$ with two equal-sized patches.
	All variables in $\patch{1}$ experience a positive local field $\field{1}=\field{}$ whereas all variables in $\patch{2}$ experience the same negative local field $\field{2}=-\field{}$ (cf. Figure~\ref{fig:exact}).   
	
	\begin{center}
		\includegraphics[width=0.3\linewidth]{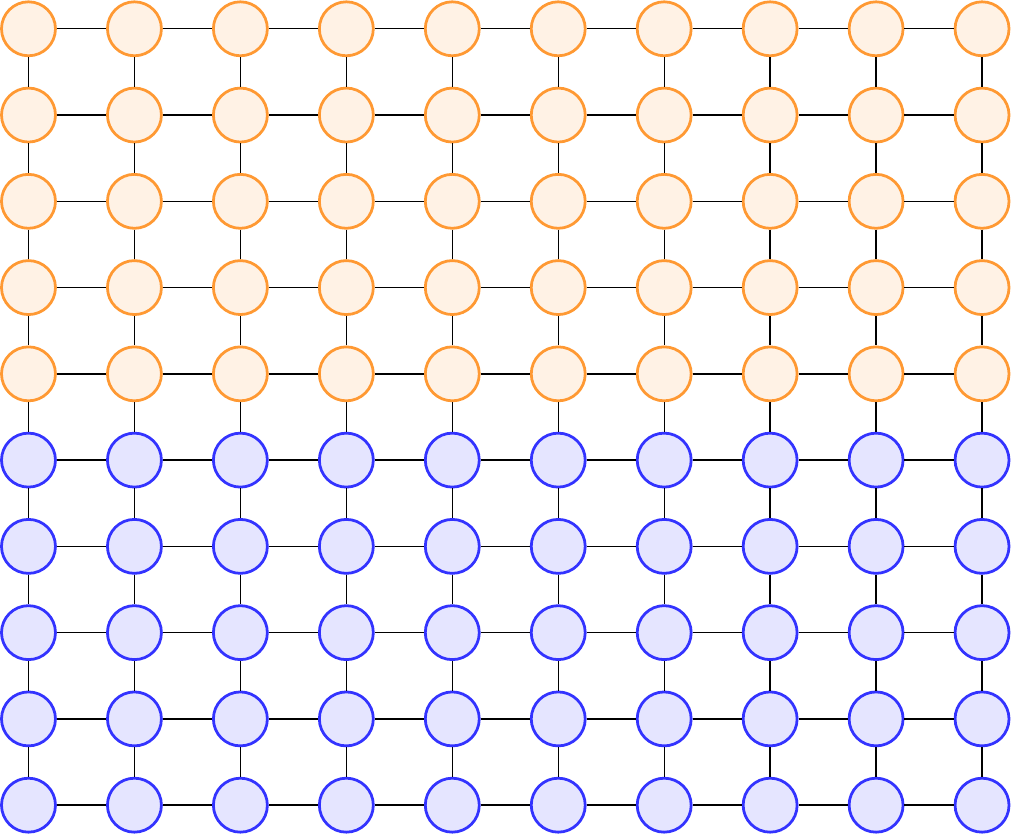}
		\captionof{figure}{Exact solution for \emph{Example~\ref{ex:patch}}. Nodes are depicted in orange if $\pmf{\RV{i}}(X_i=1) > 0.5$ and in blue otherwise; the opacity illustrates the value of the marginals.}
		\label{fig:exact}
	\end{center}
\end{example}

The patch potential model is especially appealing as the composition of relatively few patches admits a simplified treatment and comes with a couple of beneficial properties.
In particular, we can identify a region in the parameter space $\parameter$ that features a structured and well-behaved solution space (cf. Section~\ref{sec:accuracyBP:solution_space}).

\newsection{Fixed Point Behavior}{accuracyBP:evalulate_accuracy}
If BP converges, it often provides accurate results; if multiple fixed points exist, however, the performance may vary considerably between different fixed points. 
We briefly discuss the RSB (replica symmetry breaking) assumption that expresses the exact marginals as a combination of all fixed points and illustrate why its success is limited to optimization problems so far (Section~\ref{sec:accuracyBP:combination}). 

Then, we discuss the solution space of Example~\ref{ex:patch} over a range of parameters and specify different regions according to the structure of the solution space (Section~\ref{sec:accuracyBP:solution_space}). 

Assessing the approximation quality of a specific fixed point is required to state performance guarantees of BP.
We recap existing results (Section~\ref{sec:accuracyBP:accuracy}) and discuss how the error of the pseudomarginals and the Bethe partition function are related for patch potential models (Section~\ref{sec:accuracyBP:relationship}).

\newsubsection{Combination Of Fixed Points}{accuracyBP:combination}
(Non-) convexity of the Bethe free energy depends on the structure of the graph and the potentials. 
If the model has loops and sufficiently strong couplings multiple local fixed points will exist (cf. Chapter~\ref{chp:solutionsBP}).
The common narrative considers the existence of multiple fixed points as particularly problematic.
While this arguably creates a scenario where the performance of BP is more inconsistent, it also opens the door for methods that rely on the RSB assumption (cf. Section~\ref{sec:intro:bp:improving:evaluation}).
Remember that this allows us to form the exact solution according to
\begin{align}
	\pmf{\RV{i}}(\RVval{i}) = \frac{1}{\sum_m \partitionBethe^m }\sum_{m=1}^M \partitionBethe^m \singleApprox{i}^m(\RVval{i}), \nonumber
\end{align}
if all local minima of $\FB$ are known and $\setOfMinimaSol=\big\{ \fixedPointTuple{1},\ldots, \fixedPointTuple{M} \big\}$ is available.
Note that the RSB-assumption is well-established for infinite-size models~\cite{mezard1987spin};
the validity of it for finite-size models, however, remains controversial.
Yet, the doubt about finite-size models primarily stems from the lack of rigorous analysis and apart from that, the RSB assumption is empirically well-confirmed for finite-size models as well (an exhaustive overview on recent developments is presented in~\cite{lage2013rsb}).
Additionally, we computed the exact solution for all considered models using the junction tree algorithm and confirmed the RSB assumption for the patch-potential models as well.

One efficient way to evaluate~\eqref{eq:rsb} for constrained satisfaction problems is known as survey propagation~\cite{braunstein2005survey}.
The extension to more general models,however, still remains somewhat elusive.

\newsubsubsection{Approximate Survey Propagation}{accuracyBP:survey}
We have seen in Chapter~\ref{chp:solutionsBP} that a convex combination of the fixed points yields the exact solution in the presence of multiple fixed points.
Despite such promising results, the estimation of the marginals according to this combination is hindered by the need for all $M$ solutions.
Obtaining all fixed points that correspond to local minima of the Bethe free energy is a complex task only possible for small-scale models (as in Chapter~\ref{chp:solutionsBP}) and models with certain structure (e.g., random graphs \cite{coja2019bethe}),  or potential-type (e.g., for optimization problems~\cite{zdeborova2016statistical}).

An approximate version of survey propagation was recently applied to similar models as in this work~\cite{srinivasa2016survey}.
This was achieved by assuming that the fraction of randomly initialized BP runs $P^m_{\mu}$  that converges to the $m^{th}$ fixed point provides an approximation of the partition function $\partitionBethe^m$.
This assumption is valid for attractive models with vanishing local fields; yet it is unclear how this generalizes to models with non-vanishing local fields.

We aim to validate the assumption for regular grid graphs with $n \times n$ variables, $\field{} \neq 0$, and with couplings large enough to admit two fixed points.
Therefore, we compare both measures for both fixed points by relating the ratio between the partition functions
$\partitionBethe^1/\partitionBethe^2$ to the ratio $P^1_{\mu}/P^2_{\mu}$.
The log-ratio\footnote{The log-ratio is independent of the coupling strength as long as $\coupling{}{}$ is large enough to admit two fixed points.} between both measures is depicted in Figure~\ref{fig:message_distribution}.
One would expect a constant value close to zero if $P^m_{\mu}$ provides a good estimate of $\partitionBethe^m$;
this is obviously not the case as $\partitionBethe^1/\partitionBethe^2$ grows more rapidly.
We conclude that the fraction of convergent BP runs serves as a poor estimate of the partition function with the consequence that an approximate evaluation of~\eqref{eq:rsb} leads to inaccurate marginals.
This is particularly true as the local field and the model size increase.
\begin{figure}[t]
	\centering
	\includegraphics[width=0.6\linewidth]{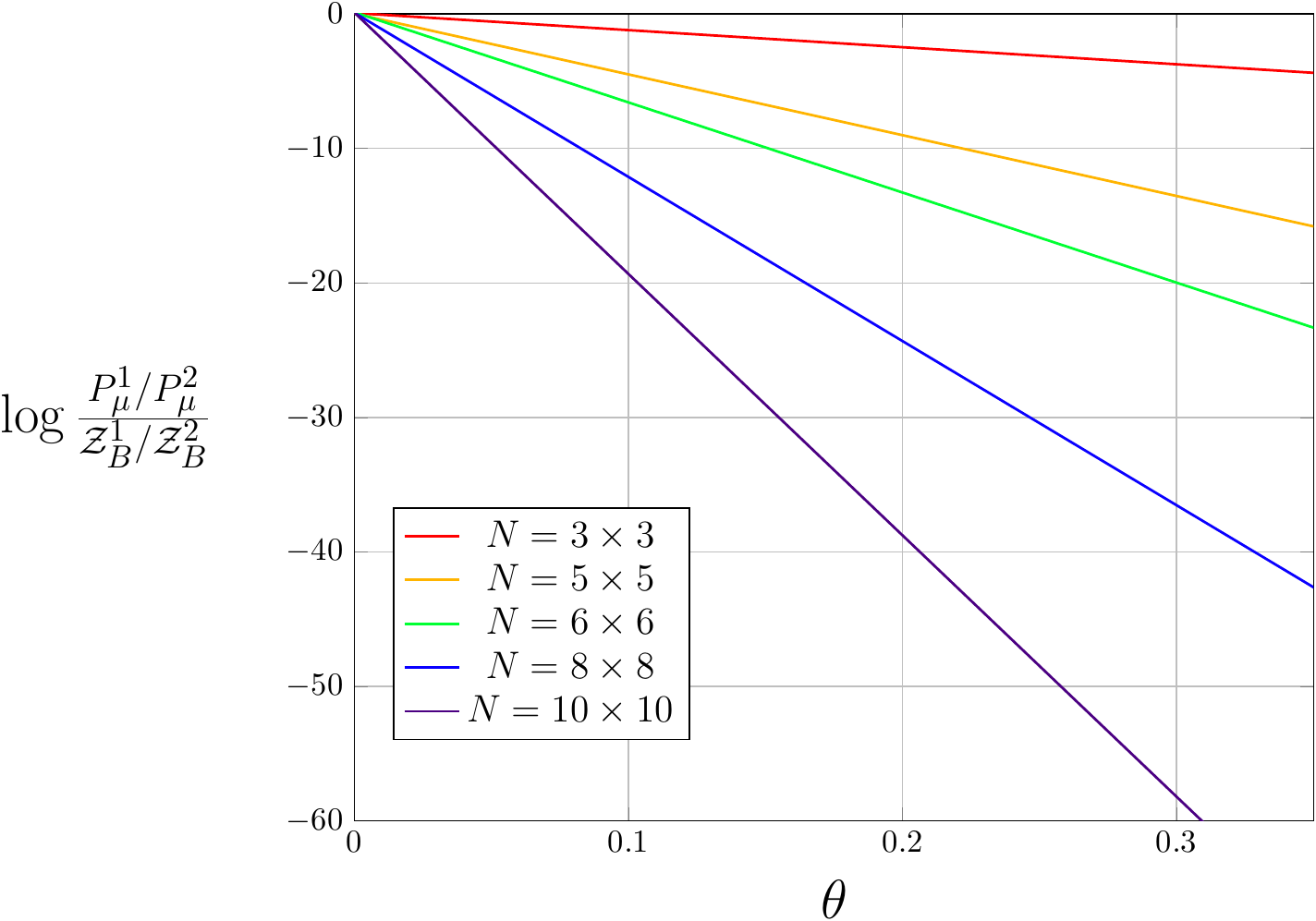}
	\caption{ $P^m_{\mu}$ is the fraction of BP runs that converge to fixed point $m$ with the corresponding Bethe partition function $\partitionBethe^m$. The mismatch increases with $N$ and $\field{}$.}
	\label{fig:message_distribution}
\end{figure}

This raises two immediate questions:
(i) Can we specify certain model-structures or parameter configurations that grant efficient methods to obtain all fixed points in order to evaluate~\eqref{eq:rsb}?
(ii) If we obtain a subset of all fixed points $\tilde{\setOfMinimaSol} \subset \setOfMinimaSol$, can we compare the available fixed points and  select the best one?

\newsubsection{Solution Space of Patch Potential Models}{accuracyBP:solution_space}
The solution space for a wide range of patch potential models is analyzed to answer whether parameter configurations exist for which all fixed points can be obtained efficiently. 
A more formal analysis that explains the subsequent observations is presented in Section~\ref{sec:accuracyBP:theory}.

Let $\graph$ be a $10 \times 10$ grid graph with two equal-sized patches (Example~\ref{ex:patch}). 
This model exhibits three different regions, separated by critical values $J_A(\field{})$ and $J_C(\field{})$; see
Figure~\ref{fig:sol_space_sketch} and Figure~\ref{fig:regions} for an illustration of 
the decomposition into multiple fixed points according to~\eqref{eq:rsb}.

\begin{figure}[!h]
	\centering
	\includegraphics[width=0.6\linewidth]{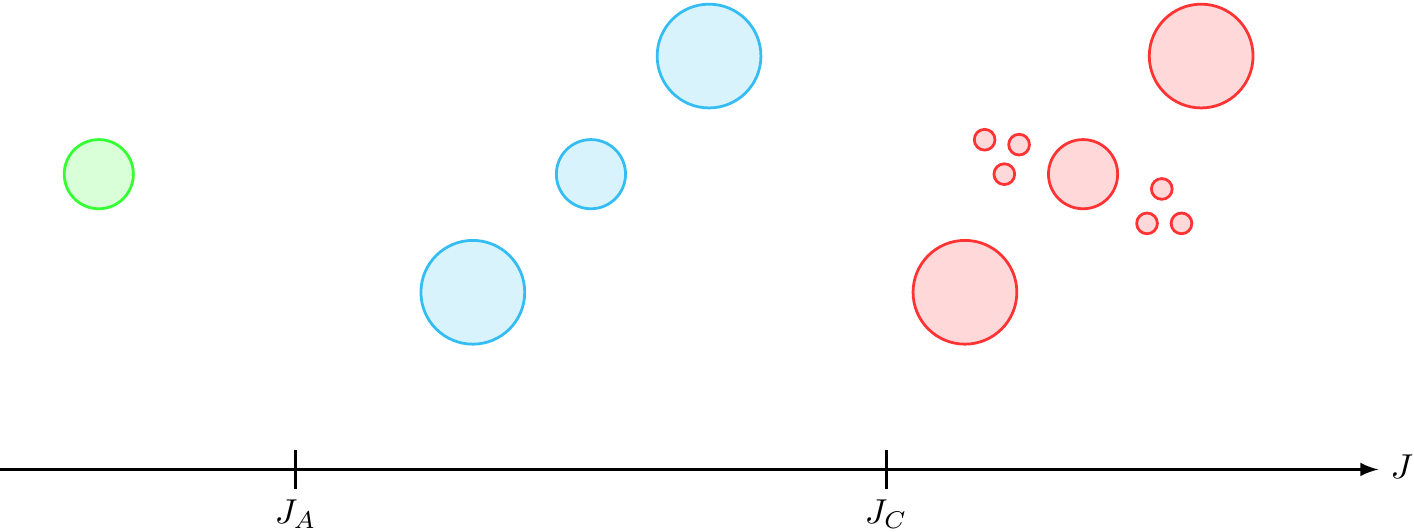}
	\caption{Illustration of the fixed points for all regions. The circle-width corresponds to the value of $\FB$.}
	\label{fig:sol_space_sketch}
\end{figure}

A unique fixed point exists for $J<J_A(\field{})$, i.e., inside region $(I)$, and BP converges; this fixed point is state-preserving but slightly overestimates the marginals (cf. Section~\ref{sec:accuracyBP:relationship}).
Additional fixed points emerge inside region $\Region$ as the coupling strength increases to $J_A(\field{}) < J < J_C(\field{})$. There are three fixed points (cf. Theorem~\ref{thm:existence}) and all three fixed points are stable (cf. Theorem~\ref{thm:stability}). These fixed points consist of two symmetric fixed points where all marginals favor one particular state and one state-preserving fixed point (cf. Section~\ref{sec:accuracyBP:relationship}).
As the coupling strength increases even further to $J>J_C(\field{})$, i.e., inside region $(III)$, all three fixed points remain but are suddenly accompanied by many more fixed points. It will therefore be increasingly hard to obtain all fixed points numerically, so that one can only hope to obtain a subset of all fixed points in practice.

The actual boundaries between the regions are numerically estimated and are depicted in Figure~\ref{fig:regions}. 
The fixed points are obtained by repeated application of BP (2000 times for each $\parameter$) with different random initial conditions.
Furthermore, we apply random scheduling to enhance the convergence properties as any predetermined schedule would favor a specific fixed point.

To answer question (i) from Section~\ref{sec:accuracyBP:survey}: one region exists in the parameter space (illustrated in blue) for which 
all fixed points can be obtained efficiently.
For region $(III)$ (illustrated in red), however, the number of fixed points suddenly increases and we cannot rely on BP to obtain all fixed points.

\newsubsection{Approximation Accuracy}{accuracyBP:accuracy}
\begin{figure}[t]
	\centering
	\includegraphics[width=0.5\linewidth]{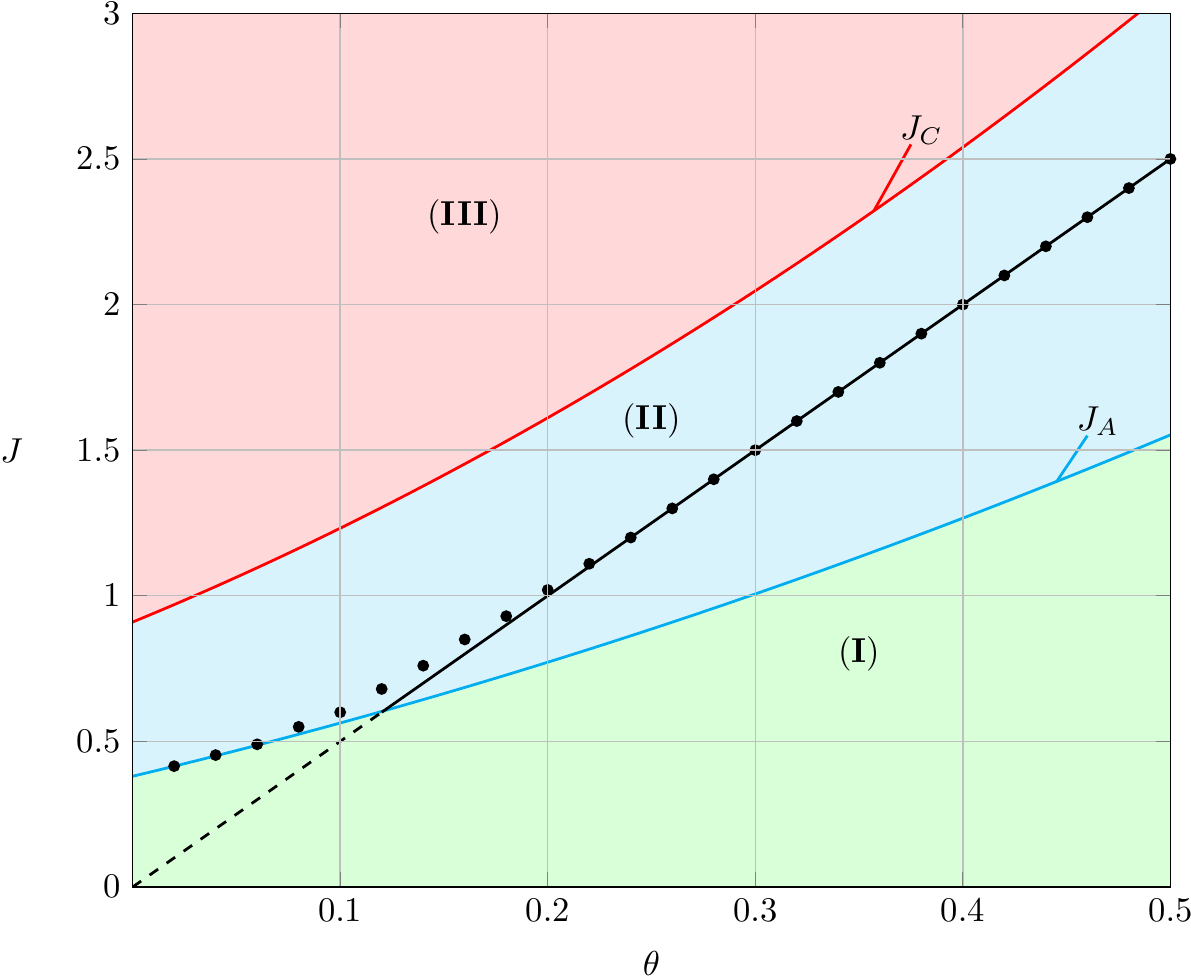}
	\caption{Illustration of all regions and boundaries for Example~\ref{ex:patch}: 
		the black dots depict the boundary below which~\eqref{eq:marginal_equal_partition} holds;
		the approximated boundary according to~\eqref{eq:threshold_approx} is depicted by the solid black line.}
	\label{fig:regions}
\end{figure}
Let $\coupling{}{} > \coupling{c}{}(\field{})$ and assume that a subset $\tilde{\setOfMinimaSol} \subset \setOfMinimaSol$ of all fixed points that constitute minima of $\FB$   is provided; then, how can we select the best one?
Unfortunately, there is no way to tell us how accurate a particular fixed point is (if we do not have access to the exact solution). It is therefore an important problem in its own to measure the accuracy, or at least provide a bound on the approximation error.
We will first discuss established results regarding the accuracy of both the Bethe partition function and the pseudomarginals.
Subsequently, we will delve into the particularities for patch potential models and show how the accuracy may differ between both objectives.\\

We measure the error of the partition function $\partitionBethe^m = \partitionBethe(\pseudomarginals^m)$ of the $m\textsuperscript{th}$ fixed point by $\EPartition{m}$ in terms of the relative error between the log-partition functions according to~\eqref{eq:error_partition}.
Note that we only consider attractive models, for which the Bethe partition function also bounds the partition function, i.e., $\partitionBethe < \partitionFunction$~\cite{ruozzi2012bethe};
obtaining the global minimum of $\FB$ is therefore optimal with respect to the error of the partition function as $\argmin_{\partitionBethe^m}(\EPartition{m}) = \exp(-\min_{\LPolytope} \FB(\pseudomarginals^m))$.

The error of the singleton marginals $\EMarginal{m}$ is measured by the mean squared error (MSE) according to~\eqref{eq:error_marginal}.
We are not aware of an explicit relationship that connects both worlds and relates the error of the marginals to the error of the partition function, except for
homogeneous\footnote{These are models that have a single value $J$ for all edges and a single value $\field{}$ for all variables}
attractive models (cf. Lemma~\ref{lm:optimality_2sol}). It is therefore often assumed that minimizing $\FB$ will be optimal in terms of marginal accuracy for more general models as well (cf. Chapter~\ref{chp:selfguided} and~\cite{weller2014understanding}).
More formally this assumption states that
\begin{align}
	\argmin_{m\in\setOfMinimaSol} \EPartition{m} = \argmin_{m\in\setOfMinimaSol} \EMarginal{m}.
	\label{eq:marginal_equal_partition}
\end{align}
This assumption is, however, not valid in general as we will now show in Section~\ref{sec:accuracyBP:relationship}.

\newsubsection{Comparison Of Marginal Accuracy and Partition Function Accuracy}{accuracyBP:relationship}
We aim to evaluate the relationship between the accuracy of the pseudomarginals and the accuracy of the partition function and whether ~\eqref{eq:marginal_equal_partition} holds in general.
First, we state that~\eqref{eq:marginal_equal_partition} does hold for homogeneous attractive models that have two fixed points at most~\cite{weller2014understanding}.
This is a direct consequence of~\eqref{eq:rsb} or by casting the model as a larger one with vanishing local fields (cf. discussion after Theorem~\ref{thm:attractive}).
\begin{lm}\label{lm:optimality_2sol}
	Attractive models with identical values $\field{i}=\field{}$ have two fixed points for $J> J_A(\field{})$. The fixed point $m$ that minimizes $\EPartition{m}$ further provides the global minimum $\min_{\LPolytope}(\FB)$ and minimizes $\EMarginal{m}$ as well. 
\end{lm}

Second, we empirically validate  whether minimizing $\FB$ will provide the most accurate marginals for Example~\ref{ex:patch}.
Figure~\ref{fig:ez_ep} illustrates the error in the marginals and the error in the partition function for all fixed points.
The fixed point that provides the global minimum to $\FB$, and thus minimizes $\EPartition{m}$ is emphasized in blue, the fixed point minimizing $\EMarginal{m}$ is emphasized in red, whereas the fixed point minimizing both quantities jointly is emphasized in green.

Let us take a closer look at region $\Region$ in particular:
three fixed points exist that can be combined to yield the exact solution (see Figure~\ref{fig:exact} for the exact solution).
Two of these fixed points, $r$ and $q$, are each biased towards one state and, because of the symmetric model, have identical values $\FB^r=\FB^q$.
The state preserving fixed point $p$ on the other hand provides the most accurate marginals inside $\Region$.
However, while $p$ also provides the global minimum of $\FB$ for small values of $\coupling{}{}$, 
Figure~\ref{fig:ez_ep} shows that $\FB^p$ turns into a local minimum for $J\geq0.65$ .
No principle relationship between the accuracy of the marginals and the partition function can therefore be observed inside $\Region$ and~\eqref{eq:marginal_equal_partition} does not necessarily hold (cf. Theorem~\ref{thm:marginal_not_equal_partition}).

For region $(III)$ many more fixed points $(u,v,\ldots)$ emerge that all have similar values $\EPartition{u}$ and $\EMarginal{u}$; we visualize some of them in Figure~\ref{fig:ez_ep}.
These fixed points provide slightly more accurate marginals than the state-preserving one, although it should be noted that all fixed points do not approximate the marginals well inside $(III)$.
On the contrary, considering $\EPartition{u}$, these additional fixed points provide the worst approximation to the partition function and have even higher values $\FB^u> \FB^p > \FB^q$.
The biased fixed points $p,q$ that approximate the marginals worst, on the other hand, approximate the partition function relatively well.

Why fixed points exist that minimize the marginal error but are only local minima of $\FB$ can, however, not be answered by the above observations.
Closer inspection of $\FB$ for different types of fixed points reveals a threshold (black dots in Figure~\ref{fig:regions}) below which~\eqref{eq:marginal_equal_partition} holds.
Some mild assumptions on the solution space lead to a lower bound on this threshold (cf. Theorem~\ref{thm:marginal_equal_partition}) according to
\begin{align}
	2\coupling{}{}\sqrt{N}-\field{}N = 0.
	\label{eq:threshold_approx}
\end{align}
This bound, illustrated by the solid black line in Figure~\ref{fig:regions}, becomes asymptotically exact. Note that the slope, defined by~\eqref{eq:threshold_approx} increases with the model size $N$
so that the global minimum of $\FB$ provides the most accurate marginals for a wider range of parameters.
\begin{figure*}[t]
	\centering
	\includegraphics[width=\textwidth]{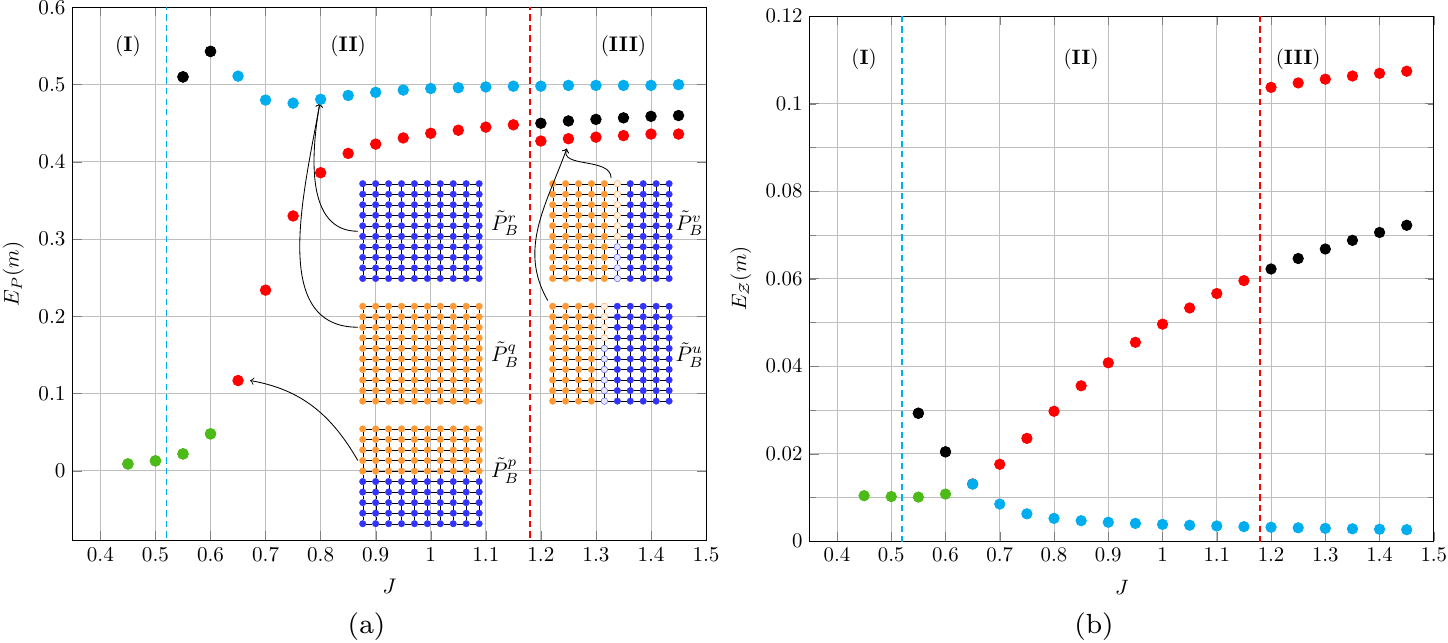}
	\caption{Accuracy of the marginals (a) and of the partition function (b) for Example~\ref{ex:patch} with $|\field{i}|=0.1$: we emphasize the fixed points minimizing $\EPartition{m}$ (blue), minimizing $\EMarginal{m}$ (red), and minimizing both quantities (green).}
	\label{fig:ez_ep}
\end{figure*}

\newsection{Theoretical Analysis}{accuracyBP:theory}
Here we properly define the boundaries $J_A(\field{})$ and $J_C(\field{})$ between different regions and provide formal arguments that explain the observations from Section~\ref{sec:accuracyBP:solution_space}. While some properties are directly attributable to~\eqref{eq:rsb}, several results are based on the fact that the patch potential model consists of multiple patches with a unidirectional local field.
First, we need to prepare an alternative update equation that makes the interactions between two patches more explicit. For that purpose, we will introduce an effective field that acts on the boundary of each patch and incorporates the influence form all other patches.

We only present the most insightful proofs below and defer some longer proofs to the Appendix~\ref{sec:accuracy:proofs}.
Additionally, we prepare some corollaries that simplify the results for models with two equal-sized patches as in Example~\ref{ex:patch}.

\newsubsection{Effective Field}{accuracyBP:theory:effective}
We introduce an effective field $\effectiveField{i}$ 
for all variables that lie on the patch-boundary to incorporate the interactions with the neighboring patches.
\begin{thm}[Effective Field] \label{thm:effective_field}
	Let $\RV{i}$ be a variable on the boundary of patch $\patchRV{i}$  that receives messages from inside, i.e., $\RV{k}\in\patchRV{i}$, and outside, i.e., $\RV{j} \in \{\setOfNodes \backslash \patchRV{i}\}$, the patch.
	The effective field $\effectiveField{i}$ acts on the boundary according to
	\begin{align}
		\effectiveField{i} = \field{i}+\sum_{\RV{j}\in \{\neighbors{i}\backslash \patchRV{i}\}}\arctanh (2\msg{j}{i}{1} -1).
		\label{eq:effective_field}
	\end{align}
	\vspace*{-0.4cm}
\end{thm}
The proof of Theorem~\ref{thm:effective_field} is presented in Appendix~\ref{sec:accuracy:proofs:effective_field}.

Messages from outside the patch are now subsumed by $\effectiveField{}$
and the additive terms in~\eqref{eq:effective_field} will be positive if $\msg{j}{i}{1} > \msg{j}{i}{0}$ and negative otherwise. 
This is particularly important in the definition of the region boundaries and 
admits ``independent'' treatment of every patch.

\newsubsection{Definition of Region \texorpdfstring{$\Region$}{}}{accuracyBP:theory:region}
The notion of an effective field (Theorem~\ref{thm:effective_field}) allows us to define the boundaries between the three distinct performance regions of patch potential models. We discuss the solution space in detail and what can be said about the performance of BP.
Let us denote the second region, i.e., the region where the global behavior can be inferred by treating the patches individually by $\Region = \{\field{},\coupling{}{}\}$.

\begin{defn}[Region]\label{def:region}
	A parameter set $(\field{},\coupling{}{}) \in \Region$ if and only if the following conditions are satisfied:\\
	(1.) Let $J_{A}(\patch{i},\field{})$ denote the critical value for the couplings beyond which multiple fixed points exist.\footnote{
		Note that an analytical solution only exists for graphs with vanishing fields of infinite size or periodic boundary conditions, but the threshold can be estimated numerically.} Then every patch $\patch{i} \in \graph$ must have its respective threshold below the actual coupling strength, i.e., $J_{A}(\patch{i},\field{}) < \coupling{}{}$\\
	(2.) Consider all pairs of patches $\patch{i}$ and $\patch{j}$; if one patch, e.g., $\patch{i}$ has its variables flipped, the imposed effective field on the boundary must stabilize the second patch $\patch{j}$ so that  $J < J_{A}(\patch{j}, \effectiveField{}) = J_{C}(\patch{j},\field{})$.
\end{defn}
These conditions implicitly define the ``well-behaved'' region $\Region$. Definition~\ref{def:region}.1 provides the lower boundary of region $\Region$ as only a unique fixed point would exist otherwise.
It may be less obvious how Definition~\ref{def:region}.2 provides the upper boundary of region $\Region$. 
Note that $J < J_{A}(\patch{j}, \effectiveField{})$ is a necessary condition
if $\patch{i}$ is flipped, as parts of $\patch{j}$ would flip otherwise and lead to disordered behavior (cf. Figure~\ref{fig:ez_ep}). The restriction to $\Region$ and the exclusion of disordered solutions further validates the RSB assumption~\cite[Chapter~19]{mezard2009}.

\newsubsection{Properties Of Region \texorpdfstring{$\Region$}{}}{accuracyBP:theory:properties}
In this chapter, we are particularly interested in understanding the properties of BP inside region $\Region$ that complies with the subsequent properties.
Note that the properties inside region $\Region$ are a direct consequence of Definition~\ref{def:region}.
Many arguments will rely on the fact that every attractive model with vanishing or unidirectional local fields either has a unique stable or two stable fixed points (cf. Section~\ref{sec:accuracyBP:models}).

We will first bound the number of possible BP fixed points.
\begin{thm} [Existence]\label{thm:existence}
	Let $\ugm$ be a patch potential model with $(\field{}, \coupling{}{}) \in \Region$. 
	The amount of fixed points $M$ grows with the number of patches (rather than the number of variables). Specifically, we have
	$M = \bigO( 2^{( |\patch{i}|)})$, where $|\patch{i}|$ denotes the number of patches.
\end{thm}
\begin{proof}
	First, assume that  a given patch $\patch{i}$ is flipped; then by the definition of the patch potential model (Definition~\ref{def:model}) and by Theorem~\ref{thm:effective_field} it follows that effective field $\effectiveField{}$ is aligned with the local field of the variables at any neighbor patch $\patch{j}$ so that
	\begin{align}
		\sgn (\effectiveField{j}) &= \sgn (\field{j}),\\
		|\effectiveField{j}| &> |\field{j}|.
	\end{align}
	
	Further, let us recall the definition of $\Region$ (Definition~\ref{def:region}.2 in particular). It follows that the effective field stabilizes its neighbor patch $\patch{j}$, i.e., $J_A(\patch{j},\effectiveField{}) > J_A(\patch{j},\field{})$ so that, according to the definition $J< J_A(\patch{j},\effectiveField{})$, and $\patch{j}$ admits only a unique solution.
	
	Second, assume that   $\patch{i}$ is not flipped; then it follows by Theorem~\ref{thm:effective_field} that $\effectiveField{i} < \field{i}$ for $\RV{i}$ on the boundary of the neighbor patch $\patch{j}$. This decrease in the local field reduces the threshold for the existence of two solutions to smaller values of $\coupling{}{}$ (cf.~\cite{knoll_stability}) so that
	\begin{align}
		J_A(\patch{j},\effectiveField{}) < J_A(\patch{},\field{}) < J.
	\end{align}
	By Definition~\ref{def:region}.1 it follows that the neighbor patch $\patch{j}$ has two fixed points now, and it depends on the initialization to which one BP will converge.
	
	Finally, we aim to show that the number of possible fixed points is bounded. Therefore we want to stress that the above arguments show how patches can either be aligned with the local potential or be flipped; it is crucial that every patch acts as one instance and that all variables belonging to one patch are aligned. 
	Else the fixed point would be disordered which we rule out precisely by Definition~\ref{def:region}.
	This and the fact that we are considering binary random variables limits the number of possible solutions to
	\begin{align}
		M \leq 2^{|\patch{i}|},
	\end{align}
	where $|\patch{i}|$ denotes the overall number of patches.
\end{proof}

\begin{cor}[Example~\ref{ex:patch}]
	Let $\ugm$ be a patch potential models with two equal-sized patches (cf. Example~\ref{ex:patch}). Then, for $(\field{},\coupling{}{}) \in \Region$ three fixed points exist; these are one state preserving fixed point and two fixed points that have all variables biased towards one of both states.
	Note that both patches can not be flipped simultaneously inside $\Region$ (cf. proof of Theorem~\ref{thm:existence}) as one patch would stabilize, i.e., prohibit from flipping, the second patch.
\end{cor}
Theorem~\ref{thm:existence} is of great practical relevance for the RSB assumption~\eqref{eq:rsb}, i.e., whether a combination of BP fixed points can form the exact solution.
The fact that there is a relatively small number of fixed points makes the task of obtaining them practically feasible.
Existence alone, however, is not sufficient as we have to rely on some numerical method that obtains all fixed points; if we aim to apply BP for that matter there is the additional requirement for all fixed points to be stable.
Fortunately, it turns out that all fixed points inside $\Region$ are stable indeed.
\begin{thm} [Stability]\label{thm:stability}
	Let $\ugm$ be a patch potential model with $(\field{}, \coupling{}{}) \in \Region$. Then, every fixed point $\pseudomarginals^m$ is a stable fixed point for BP.
\end{thm}

Finally, as an immediate consequence of the limited amount of fixed points (Theorem~\ref{thm:existence}), all of which are stable (Theorem~\ref{thm:stability}), it follows that the exact solution can be computed according to~\eqref{eq:rsb} in practice.
One can for example apply BP repeatedly, possibly in parallel, with random initialization to obtain and combine all fixed points.

\newsubsubsection{Marginal Accuracy}{accuracyBP:theory:acccuracy}
\begin{thm}[Marginal Accuracy]\label{thm:marginal_accuracy}
	The MSE of the singleton marginals  $\EMarginal{k}$ of the $k^{th}$ solution $\pseudomarginals^k$ relates to the ratio of the Bethe partition functions according to
	\begin{align}
		\EMarginal{k} = \frac{2}{N (\sum\limits_m \partitionBethe^m)^2} \sum_{\RV{i}\in\setOfNodes}\! \Big|\! \sum_{m\backslash k} \partitionBethe^m \left(\pmfApprox{\RV{i}}^m - \pmfApprox{\RV{i}}^k\right) \Big|^2 \nonumber.
	\end{align}
\end{thm}
\begin{proof}
	According to~\eqref{eq:rsb} we can express the exact solution by the convex combination of all fixed points. 
	Consequently, using symmetry properties of the binary random variables,  the error is given by
	\begin{align}
		\EMarginal{k} &= \frac{2}{N}  \sum_{\RV{i}\in\setOfNodes} \bigg|\frac{\sum\limits_m \partitionBethe^m \pmfApprox{\RV{i}}^m}{\sum\limits_m \partitionBethe^m} - \pmfApprox{\RV{i}}^k\bigg|^2 \nonumber \\
		&= \frac{2}{N} \sum_{\RV{i}\in\setOfNodes}  \bigg| \frac{\sum\limits_m \partitionBethe^m \pmfApprox{\RV{i}}^m - \sum\limits_m \partitionBethe^m \pmfApprox{\RV{i}}^k}{\sum\limits_m \partitionBethe^m}\bigg|^2 \nonumber \\
		& = \frac{2}{N \big( {\sum\limits_m \partitionBethe^m} \big)^2 }  \sum_{\RV{i}\in\setOfNodes}\Big| \sum_{m\backslash k} \partitionBethe^m \pmfApprox{\RV{i}}^m - \pmfApprox{\RV{i}}^k \Big|^2,
	\end{align}
	where we first, bring everything  on the same denominator so that the $k^{th}$ contribution cancels out subsequently.
\end{proof}

Representing the MSE according to Theorem~\ref{thm:marginal_accuracy} is particularly appealing as it omits the need for expressing the exact marginals.
This further provides a way to express the ratio of the marginal error between two fixed points.

\begin{cor}\label{cor:marginal_ratio}
	The MSE-ratio of two fixed points $k$ and $l$ is a ratio of weighted partition functions according to
	%
	\begin{align}
		\frac{\EMarginal{k}}{\EMarginal{l} } &= \frac{ \sum\limits_{\RV{i}\in\setOfNodes}|\sum\limits_{m\backslash k } \partitionBethe^m (\pmfApprox{\RV{i}}^m- \pmfApprox{\RV{i}}^k)|^2}
		{ \sum\limits_{\RV{i}\in\setOfNodes}|\sum\limits_{m\backslash l } \partitionBethe^m (\pmfApprox{\RV{i}}^m- \pmfApprox{\RV{i}}^l)|^2}.
		\label{eq:marginal_ratio}
	\end{align}
\end{cor}
Expressing the ratio of the marginal error according to~\eqref{eq:marginal_ratio} is advantageous in elaborating on the difference between accuracy of the approximated marginals and the approximated partition function.
We define the mismatch between $\pmfApprox{\RV{i}}^m$ at two fixed points $k$ and $l$ by
\begin{align}
	Q_i(k,l) = \pmfApprox{\RV{i}}^k(\RV{i}=1) - \pmfApprox{\RV{i}}^l(\RV{i}=1).
\end{align}

Now, let us denote the error of the state preserving fixed point by $\EMarginal{p}$
and of the fixed point that has all marginals biases towards $\RVval{i}=1$ by $\EMarginal{q}$.
Then -- maybe non-surprising as the exact solution is state preserving as well -- we show that the state-preserving fixed point has the most accurate marginals.
We first discuss the error-ratio in a general manner in Theorem~\ref{thm:marginal_not_equal_partition} before considering the special case of a model with two patches, i.e., for Example~\ref{ex:patch}.

\begin{thm}[Error Ratio]\label{thm:marginal_not_equal_partition}
	Let $\ugm$ be a patch potential model with $(\field{},\coupling{}{}) \in \Region$. The state preserving fixed point $p$ provides more accurate marginals than the fixed point $q$ that has all marginals biased to one state, i.e.,
	\begin{align}
		\frac{\EMarginal{p}}{\EMarginal{q}} < 1.
	\end{align}
\end{thm}
\begin{proof}
	To show that, irrespective of the value of $\partitionBethe$, $\frac{\EMarginal{p}}{\EMarginal{q}} < 1$ we assume that the state preserving fixed point does not minimize the Bethe free energy, i.e., $\partitionBethe^q > \partitionBethe^p$.
	
	Without loss of generality, we make some prior assumptions on the model:
	First, we assume that the overall number of variables with a positive local field equals the number of variables with a negative local field, i.e.,
	\begin{align}
		|\{\RV{i}\in\setOfNodes: \field{i}=+\field{}\}| = |\{\RV{j}\in\setOfNodes: \field{i}=-\field{}\}|.
		\label{eq:symmetry}
	\end{align}
	Second, we assume that all patches are of equal size.
	And finally, we group all possible fixed points according to their marginals and denote them as follows:
	the state-preserving fixed point is referred to as $p$;
	all fixed points that have more patches biased towards $X_i = +1$ are referred to as $s = 1,\ldots,S$, with $q$ being the fixed point that has all variables biased towards $X_i = +1$;
	all fixed points that have more patches biased towards $X_i = -1$ are referred to as $t = 1,\ldots,T$, with $r$ being the fixed point that has all variables biased towards $X_i = -1$.
	
	This has some implications that will ease the subsequent analysis significantly. Specifically, the number of fixed points favoring one state equals the number of fixed points favoring the other state, i.e., $S=T$ by~\eqref{eq:symmetry} and by Theorem~\ref{thm:existence}.
	
	Another important consequence of incorporating the interactions between patches into an effective field is that the number of variables that favor one state has an immediate influence on the value of the singleton marginals. It can be shown that the effective field, if stronger than the local field -- note that the effective field is stronger than the local field whenever two neighboring patches are biased towards the same state -- increases the bias of the variables. This intuitive statement is a consequence of the Griffiths-Hurst-Sherman inequality~\cite{griffiths1967correlations} that can be extended to specific fixed points by straightforward manipulations (cf.~\cite{knoll_sbp}).
	In essence this means that for all variables $\RV{i} \in \setOfNodes$ we have
	\begin{align}
		\pmfApprox{\RV{i}}^q(\RV{i}=1) \geq \pmfApprox{\RV{i}}^s(\RV{i}=1) \geq \pmfApprox{\RV{i}}^p(\RV{i}=1)
		\geq \pmfApprox{\RV{i}}^t(\RV{i}=1) \geq \pmfApprox{\RV{i}}^r(\RV{i}=1).
		\label{eq:marginal_inequality}
	\end{align}
	
	Moreover, as all patches have equal size and because $\coupling{i}{j} = \coupling{}{}$ as well as $\field{i} \in \{-\field{},\field{}\}$ every fixed point $s$ has a symmetric fixed point $t$ that has the same value for the approximate partition function (cf. Proof of Theorem~\ref{thm:marginal_equal_partition}). 
	That is, except for the state-preserving fixed point $p$ all fixed points come in couples that satisfy
	\begin{align}
		\partitionBethe^s = \partitionBethe^t.
		\label{eq:partiton_equality}
	\end{align}
	
	We will further utilize the properties of the mismatch $\mismatch{k}{l}$; in particular the symmetry property
	\begin{align}
		\mismatch{k}{l} = - \mismatch{l}{k},
		\label{eq:mismatch_symmetry}
	\end{align}
	and the expansion property
	\begin{align}
		\mismatch{k}{l} = \mismatch{k}{m} + \mismatch{m}{l}.
		\label{eq:mismatch_expansion}
	\end{align}
	
	Finally, we express the error ration between the state-preserving fixed point $p$ and the biased fixed point $q$ according to   Corollary~\ref{cor:marginal_ratio} so that 
	\begin{align}
		\frac{\EMarginal{p}}{\EMarginal{q} } 
		&= 
		\frac{ \sum\limits_{\RV{i}\in\setOfNodes}|\sum\limits_{m\backslash p } \partitionBethe^m \mismatch{m}{p}|^2}
		{ \sum\limits_{\RV{i}\in\setOfNodes}|\sum\limits_{m\backslash q } \partitionBethe^m  \mismatch{m}{q}|^2} \nonumber \\
		&\stackrel{(a)}{=} 
		\frac{ \sum\limits_{\RV{i}\in\setOfNodes}|\sum\limits_{s } \partitionBethe^s \mismatch{s}{p} 
			+ \sum\limits_{t } \partitionBethe^t \mismatch{t}{p}|^2}
		{ \sum\limits_{\RV{i}\in\setOfNodes}|\sum\limits_{s\backslash q } \partitionBethe^s \mismatch{s}{q}
			+ \sum\limits_{t } \partitionBethe^t \mismatch{t}{q}
			+ \partitionBethe^p \mismatch{p}{q}|^2} \nonumber \\
		&\stackrel{(b)}{=} 
		\frac{ \sum\limits_{\RV{i}\in\setOfNodes}|\sum\limits_{s\backslash q} \partitionBethe^s \big(\mismatch{s}{p} + \mismatch{t}{p}\big)
			+ \partitionBethe^q \big(\mismatch{q}{p}+\mismatch{r}{p}\big)|^2}
		{ \sum\limits_{\RV{i}\in\setOfNodes}|\sum\limits_{s\backslash q} \partitionBethe^s \big(\mismatch{s}{q} + \mismatch{t}{q}\big)
			+ \partitionBethe^q \mismatch{r}{q}
			+ \partitionBethe^p \mismatch{p}{q}|^2},
	\end{align}
	where (a) follows from splitting the sum into the fixed points $s$ that are more biased towards $\RV{i} = 1$ and into the fixed points $t$ that are more biased towards $\RV{i} = -1$.
	Note that the state-preserving fixed point $p$ does not belong to either set and is consequently expressed explicitly in the denominator.
	For (b) we make use of~\eqref{eq:partiton_equality} and arrange the terms by making the dependence on $q$ and $r$ explicit so that the sum goes over the same terms in the numerator and in the denominator.
	
	We can further express the error ratio and bound it using Jensen's inequality according to
	\begin{align}
		\frac{\EMarginal{p}}{\EMarginal{q} } 
		\stackrel{(a)}{=} & 
		\frac{ \sum\limits_{\RV{i}\in\setOfNodes}|\sum\limits_{s\backslash q} \partitionBethe^s \big(\mismatch{s}{p} + \mismatch{t}{p}\big)
			+ \partitionBethe^q \big(\mismatch{q}{p}+\mismatch{r}{p}\big)|^2}
		{ \sum\limits_{\RV{i}\in\setOfNodes}\Big(\sum\limits_{s\backslash q}  \partitionBethe^s \big|\big(\mismatch{s}{q} + \mismatch{t}{q}\big)\big|
			+ \partitionBethe^q \big|\mismatch{r}{q}\big|
			+ \partitionBethe^p \big|\mismatch{p}{q}\big|\Big)^2} \nonumber \\
		\leq &
		\frac{\sum\limits_{\RV{i}\in\setOfNodes}\Big(\sum\limits_{s\backslash q} \partitionBethe^s \big|\big(\mismatch{s}{p} + \mismatch{t}{p}\big)\big|
			+  \partitionBethe^q\big| \big(\mismatch{q}{p}+\mismatch{r}{p}\big)\big| \Big)^2}
		{\sum\limits_{\RV{i}\in\setOfNodes}\Big(\sum\limits_{s\backslash q}  \partitionBethe^s \big|\big(\mismatch{s}{q} + \mismatch{t}{q}\big)\big|
			+ \partitionBethe^q \big|\mismatch{r}{q}\big|
			+ \partitionBethe^p \big|\mismatch{p}{q}\big|\Big)^2},
	\end{align}
	where separating the norm does not change the result in (a) because of~\eqref{eq:marginal_inequality}, which implies $\mismatch{m}{q} < 0$ for all fixed points $m\neq q$.
	
	For completing the proof we make use of the symmetry property~\eqref{eq:mismatch_symmetry} and the expansion property~\eqref{eq:mismatch_expansion} in order to rearrange the terms; in particular note that
	$\mismatch{q}{s}+ \mismatch{q}{t} = \mismatch{q}{s}+\mismatch{q}{s}+\mismatch{s}{p}+\mismatch{p}{t}$, and that
	$\mismatch{q}{r} = \mismatch{q}{p} + \mismatch{p}{r} $
	so that 
	\begin{align}
		& \frac{\EMarginal{p}}{\EMarginal{q} }\leq \nonumber\\
		& \hspace*{-25pt}  \frac{\sum\limits_{\RV{i}\in\setOfNodes}\Big(\sum\limits_{s\backslash q} \partitionBethe^s \big|\big(\mismatch{s}{p} - \mismatch{p}{t}\big)\big|
			+  \partitionBethe^q\big| \big(\mismatch{q}{p}-\mismatch{p}{r}\big)\big| \Big)^2}
		{\sum\limits_{\RV{i}\in\setOfNodes}\hspace*{-4pt}\Big(\hspace*{-2pt} \sum\limits_{s\backslash q}\hspace*{-2pt}  \partitionBethe^s \big|\big(\mismatch{s}{p}\hspace*{-2pt}+\hspace*{-2pt}\mismatch{p}{t}\hspace*{-2pt}+\hspace*{-2pt}\mismatch{q}{s}\hspace*{-2pt}+\hspace*{-2pt}\mismatch{q}{s}\big)\big| \hspace*{-2pt} 
			+\hspace*{-2pt}  \partitionBethe^q \big|\mismatch{q}{p} + \mismatch{p}{r} \big| \hspace*{-2pt}
			+\hspace*{-2pt} \partitionBethe^p \big|\mismatch{q}{p}\big|\Big)^2}\nonumber.
	\end{align}
	Note that we have applied~\eqref{eq:mismatch_symmetry} so that every mismatch-term is strictly positive.
	It is thus straightforward to see, by comparing all terms, that the numerator is strictly smaller than the denominator for every variable $\RV{i} \in \setOfNodes$ so that
	\begin{align}
		\frac{\EMarginal{p}}{\EMarginal{q} }  < 1.
	\end{align}
\end{proof}

In particular for models with two equal-sized patches, we can simplify the error ratio~\eqref{eq:marginal_ratio} considerably.
\begin{cor}[Example~\ref{ex:patch}]\label{cor:error_ratio}
	Let $d = \mismatch{q}{r}>0$, then
	\begin{align}
		\frac{\EMarginal{p}}{\EMarginal{q} } < \frac{ \sum_{\RV{i}\in\setOfNodes}|\partitionBethe^q d|^2}
		{ \sum_{\RV{i}\in\setOfNodes}|\partitionBethe^q d + \partitionBethe^p \mismatch{p}{q}|^2} < 1.
	\end{align}
\end{cor}
Note that Corollary~\ref{cor:error_ratio} is an immediate consequence of Theorem~\ref{thm:marginal_not_equal_partition}. 
Additionally we present an alternative proof that admits some intuitive arguments in Appendix~\ref{sec:accuracy:proofs:error_ratio_ex1}.

It follows that the state preserving fixed point $p$ minimizes the marginal error inside $\Region$ irrespective of $\FB^p$.
This has drastic implications and forbids any relationship between the fixed point minimizing the marginal error and the one minimizing the partition function error.

\newsubsubsection{Free Energy Minimizing Fixed Point}{accuracyBP:theory:minimum}
However, despite Theorem~\ref{thm:marginal_not_equal_partition}, the question remains where the difference between $\EPartition{m}$ and $\EMarginal{m}$ stems from?

We will now answer this question and provide conditions for $\argmin\EPartition{m} = \argmin\EMarginal{m}$ to be valid. 
We further present an approximate condition for the 
state-preserving fixed point $p$ to simultaneously provide the most accurate marginals and minimize $\FB$.
Therefore, let us define the following variables (cf. Section~\ref{sec:accuracy:proofs:marginal_equal_partition} in the appendix for a formal introduction):
$\setOfEdges_P$ is the set of all boundary edges; 
$\setOfEdges_C $ is the set of edges between variables that favor different states; 
$N_f$ and  $N_c$ are the numbers of flipped and non-flipped variables;
and $\Delta \SB$ is the difference in the entropy between two fixed points.
\begin{thm}\label{thm:marginal_equal_partition}
	Let us consider the state-preserving fixed point $p$  with $\FB^p$ and some other fixed point with $\FB^m$. 
	Then,  $\FB^p < \FB^m$  is the global minimum if
	\begin{align}
		2J(|\setOfEdges_P|-|\setOfEdges_C |) < \field{}(N-N_c+N_f) + \Delta \SB.
		\label{eq:bound_approx}
	\end{align}
\end{thm}
The proof of Theorem~\ref{thm:marginal_equal_partition} is deferred to Appendix~\ref{sec:accuracy:proofs:marginal_equal_partition}.

For models with two equal-sized patches and couplings strong enough for the entropy-term to vanish
we can further simplify~\eqref{eq:bound_approx} significantly and 
state that:
\begin{cor}[Example~\ref{ex:patch}]\label{cor:marginal_equal_partition}
	The state-preserving fixed point provides the most accurate marginals and the global minimum $\FB^p$ if 
	$\parameter \in \Region$ and if
	\begin{align}
		2 \sqrt{N} \coupling{}{} < N\field{}.
		\label{eq:bound_approx2}
	\end{align}
\end{cor}
\begin{proof}
	For the specific case of a grid graph with two equal-sized patches (Example~\ref{ex:patch}) we can further simplify the condition from Theorem~\ref{thm:marginal_equal_partition}.
	Therefore, note that for $\FB^p < \FB^q$ to be satisfied, we have $N_c = N_f$ and $|\setOfEdges_P| = \sqrt{N}$ so that~\eqref{eq:bound_exact_extreme} reduces to
	\begin{align}
		2 J\sqrt{N} & \leq \field{}N+\Delta \SB. \label{eq:bound_exact_extreme_ex1}
	\end{align}
	The definition of $\Region$ requires strong interactions $\coupling{}{}$ so that the entropy terms in the free energy is small. We can consequently approximate~\eqref{eq:bound_exact_extreme_ex1} by neglecting the entropy terms.
	\begin{align}
		2 J\sqrt{N} & \leq \field{}N.\label{eq:bound_exact_extreme_ex1_approx}
	\end{align}
\end{proof}

These sufficient conditions for~\eqref{eq:marginal_equal_partition} provide a guideline when it would be safe to select the fixed point according to the partition function value.
This correspondence tends to hold for models with strong local potentials $\field{}$ and with increased model-size $N$ as shown in Corollary~\ref{cor:marginal_equal_partition}. 

\newsection{Conclusion}{accuracyBP:conclusion}
In Section~\ref{sec:selfguided:theory} we had relied on the assumption that  finding the global minimum of $\FB$ gives us the most accurate marginals and partition function.
This is a strong assumption that -- although never formally verified -- was generally believed to be true.


In this chapter we aimed to verify whether such a relationship holds in general.
Therefore, we studied the relationship between the accuracy of the marginals and the partition function at stationary points of the Bethe free energy, with far-reaching implications for approximate inference methods that operate on the Bethe free energy.
To do so, we introduced a new class of models first.
These so-called patch potential models exhibit lots of structure that enabled us to analyze the solution space thoroughly.
Yet, patch potential models proved to be flexible enough to produce some surprising insights.

First, we elaborated on the existence of a well-behaved region in parameter space for which the number of fixed points depends only on the number of patches (instead of the variables).

Second, this well-behaved region allowed us to assess the correspondence between the accuracy of the marginals and the partition function;
we showed that this common assumption is \emph{not} true in general and explained why -- instead of the global minimum -- the most accurate marginals may be found at a local minimum of $\FB$.

Finally, we inspected under which conditions this mismatch arises and further introduced guarantees for the global minimum of $\FB$ to remain optimal with respect to marginal accuracy.
Note that the mismatch is effectively attributable to finite-size effects which might explain why this behavior is not known in the physics literature, where the focus lies on infinite-size -- or at least very large -- models. 
Yet, models of relatively small size do play an important role in many applications of BP.\\

The analysis of patch potential models may serve as a foundation that enables the extension of survey propagation to problems beyond constrained satisfaction problems.
So far this has been limited by the potentially huge amount of fixed point.
Our definition of the well-behaved region and the study of patch potential models as a whole suggests that it may be very much possible to obtain the set of all fixed points efficiently, as long as the model is sufficiently structured.

Moreover, our newly developed understanding of the relationship between the accuracy of the marginals and the partition function has a notable impact on our theoretical analysis of SBP in Section~\ref{sec:selfguided:theory}.
There we have limited our attention to attractive homogeneous models where the accuracy of both quantities can be used in an interchangeable way.
Whether it is the minimum (i.e., the Bethe free energy) or the minimizer (i.e., the pseudomarginals) that is responsible for the optimality of SBP needs to be reassessed.
It is not clear if it is really the global minimum of $\FB$ that originates from the start point.
On the contrary, our minimal patch potential model of Example~\ref{ex:patch} shows that the state-preserving fixed point defines the solution path, so that SBP may very well be in favor of obtaining the fixed point with the most accurate marginals.
This could be an explanation for the surprisingly accurate results of SBP for general models and  would immensely strengthen the practical relevance of SBP as no method is known so far that deliberately aims to obtain fixed point with the most accurate marginals.

  \emptydoublepage
\newchapter{Discussion and Open Questions}{discussion}
\openingquote{There must be some way out of here, \\Said the joker to the thief.}{Bob Dylan}
In this thesis, we performed a systematic analysis of BP's solution space.
Finding all fixed points of BP is a relevant but hard problem, except for a few restricted models.
We took inspiration from a set of diverse scientific fields to tackle this problem: 
in particular, we framed BP as dynamical systems and formulated the fixed point equations.
Drawing from computational mathematics, we established a way to solve 
the set of fixed point equations and computed all fixed points for a range of problems.

The knowledge of the full solution space then served as a cornerstone for evaluating the accuracy and the convergence properties of BP.
Moreover, we related those performance properties to the parameterization of a given model (i.e., the model size and the specifications of the potentials) and explained how the performance of BP changes if we tune the model.

A detailed summary of our main contributions is presented in each chapter's conclusion.
Here we  summarize our most interesting findings and discuss how we advanced the understanding of BP.
Our findings satisfy a twofold interest:
they advance the theoretical understanding and suggest various possibilities of enhancing BP's performance.

\begin{itemize}
 \item From a theoretical perspective, we made the influence of the model parameters on the performance of BP explicit.
 We sum up our key-insights as follows:
 
 The influence of the model size is hard to pinpoint down.
 On the one hand, larger models tend to have a higher average degree, which affects the convergence properties slightly negatively.
 On the other hand, we established a positive effect of the model size on the marginal accuracy.
 
 The influence of the potentials is much more apparent.
 Strong pairwise potentials influence the performance of BP detrimentally,
 whereas strong local potentials increase the accuracy and the convergence properties.
 
 Additionally, we demonstrated a considerable accuracy gap between the best possible and the worst possible fixed point of BP.
 This poses a fundamental problem when deriving performance guarantees as is not obvious to which fixed point BP will converge.

 \item From a practical perspective, we justified the exploration of multiple fixed points, intending to select one with good approximation quality.
 We discussed the underlying mechanisms and limitations of established modifications, for example, damping or scheduling.

 We further proposed one particular modification that aims to obtain the global minimum of the Bethe free energy. This method, SBP, consistently obtains more accurate marginals than BP and provides accurate marginals for models where BP fails to converge.

 Even the minimization of the Bethe free energy, however, has its limitations if one is only interested in the marginals; we exemplified and explained why the most accurate marginals may not be found at the global minimum of the Bethe free energy.
\end{itemize}

The analysis of BP's solution space addressed some long-standing questions but we are still far away from understanding every aspect of BP.
Below we summarize some intriguing questions that emerged while working on this thesis and that were left unanswered.
We distinguish three different directions that appear to be worth pursuing in particular:

\begin{enumerate}
 \item \textbf{Structure of the Solution Space:}
 
 Although this thesis obtained the full solution space for some models and, in doing so, provided novel insights into the properties of BP, computing the solution space for more complex models remains elusive.
 As we have seen in Chapter~\ref{chp:accuracyBP}, however, it is precisely the consideration of more complex models that can lead to fascinating insights.

 The extension of the NPHC method to larger models was primarily hindered by the involved computation of the root count.
 So far, we only relied on some meta-heuristics to compute the root count.
 Most relevant models, be it grid graphs or models arising in the context of error-correcting codes, however, exhibit a special graph structure that directly determines the structure of the fixed point equations.
 A first step towards computing the root count for complex models could be to exploit the graph structure and thus reduce the overall complexity of this crucial step.
 Alternatively, it would be relevant (from a mathematical perspective) to reformulate the problem to an equivalent system of equations with fewer complex solutions and a lower BKK bound.
 Such problem-tailored approaches would extend the applicability of the NPHC method and, when applied to large models, would have the potential to provide even more insights into the behavior of BP.

 Traditionally, one is often interested in specifying regions with a unique fixed point.
 We have shown that this seems to be a rather limiting point of view as combinations of multiple fixed points often provide strikingly accurate results.
 Although obtaining the set of all fixed points is, in general, a hard problem,  we have seen in Chapter~\ref{chp:accuracyBP} that certain regions in the parameter space exist for which only a few fixed points are present.
 This suddenly opens the door for numerical methods that efficiently obtain all fixed points and enhance the accuracy by combining them.
 
 When defining well-behaved regions of BP we would consequently like to initiate a shift from conditions for uniqueness to conditions for structured, albeit possibly multiple, fixed points.
 Fairly substantial work has been conducted in this matter for the special case of optimization problems.
 Developing similar results for more general models is a necessary step to assess the relevance of methods that rely on considering multiple fixed points and to elaborate on their capabilities. 
 
 \item \textbf{Global Convergence Properties:}
 We have put a lot of emphasize on the  local stability analysis in Chapter~\ref{chp:solutionsBP}.
 Let us stress the notion of \emph{local} here again:
 we considered a fixed point as stable if the messages will converge to the fixed point inside its neighborhood.
 In general, one can, however, not count on bringing the messages close enough to the fixed point and it is not only the local stability that is important for the convergence properties of BP;
 even more so it is the \emph{region of attraction}.
 The region of attraction describes the range of message values for which BP will converge to a given fixed point.
 This is particularly relevant in practice as BP sometimes fail to converge despite the existence of  stable fixed points (unless initialized very close to the fixed point).
 
 Such a global perspective also highlights the influence of the initial message values on the convergence properties of BP.
 Besides being convenient, uniform initialization tends to favor good fixed points; it remains an open problem to answer what is so special about uniform initialization and under which conditions uniform initialization is optimal.
 We expect 
 that extending the local stability analysis and accounting for the regions of attraction as well will provide many answers to those questions.
 Besides, this could potentially lead to variants of BP that, even if initialized 
 far away from any fixed point, explore the parameter space in a way such that the messages enter a region of attraction at some point, ultimately driving the messages into convergence.

 \item \textbf{Enhancing BP with Respect to Marginal Accuracy:}
 The main contribution of Chapter~\ref{chp:accuracyBP} was to elaborate on the difference between the accuracy of the marginals and the partition function.
 Our analysis revealed that the fixed point providing the most accurate marginals may only be a local minimum.
 Note that most variants that aim to enhance BP are either heuristics that only aim to enforce convergence or are theoretically well-motivated and focus on the quality of the partition function approximation.
 
 For many applications, we are, however, ultimately interested in enhancing the quality of the marginals (cf. error-correcting codes in Section~\ref{sec:solutionsBP:coding}).
 It would thus be of great relevance to manipulate the update equation in a way that enforces convergence towards a fixed point with accurate marginals.
 Enforcing convergence towards one particular fixed point would further enable one to come up with realistic  performance guarantees, even in the presence of multiple fixed points.
 
 We have proposed SBP that is guaranteed to obtain the most accurate marginals for attractive models with unidirectional potentials.
 Looking at the evolution of the fixed points on patch-potential models, it seems as if SBP is capable of tracking the state-preserving fixed point as well.
 Thus, SBP is already one possible candidate for obtaining accurate marginals.
 It would be interesting to elaborate on this observation and substantiate this claim in a more formal way.
 Essentially, the advantages of SBP boil down to providing a favorable initialization for BP.
 This highlights the importance of the initial messages once more and raises the question if there is a general way of initializing the messages that will lead to accurate marginals.

 Besides the role of the initialization, scheduling plays an important role in determining the fixed point of BP.
 Again, one would hope that ``good'' scheduling methods enforce convergence towards an accurate fixed point.
 Existing scheduling methods are either adaptive or only take the structure of the model into account.
 The performance of BP, however, not only depends on the structure of the model but on the potentials as well.
 Therefore, if the aim is to obtain accurate marginals, the value of the potentials should not be neglected and it seems promising to come up with a fixed schedule that takes the values of the potentials into account as well, for example by scheduling the messages in a way that prioritizes regions with strong local potentials.
 
%
%
%
\end{enumerate}

  \emptydoublepage
\appendix
\newchapter{Appendix}{appendix}
\newsection{Proofs from Chapter~\ref{chp:selfguided}}{selfguided:proofs}
This Section contains all the detailed proofs deferred from Section~\ref{sec:selfguided:theory}.

\newsubsection{Proof of Lemma~\ref{lm:griffith}}{selfguided:proofs:2}
Essentially, we first show that $\msg{i}{j}{1}/ \msg{i}{j}{-1}$ increases monotonically with $\coupling{i}{j}$ and then express the pseudomarginals in terms of~\eqref{eq:marginals:single} and~\eqref{eq:marginals:pw}.

Let us denote the messages on $\ugm[0]$ and on $\ugm[1]$ by $\msg{i}{j}{}_0$ and  $\msg{i}{j}{}_1$ respectively. 
Further, let all local potentials be positive, i.e., $\field{i} > 0$, and let all pairwise potentials of $\ugm[1]$ be $\epsilon$-larger than those of $\ugm[0]$, i.e., $0 < \coupling{i}{j}^{0} = \coupling{i}{j}^{1}-\epsilon$. Note that by assumption $\mean{i} \in (0,1]$ so that 
\begin{align}
	\msg{i}{j}{+1} \geq \msg{i}{j}{-1}. \label{eq:appendix:ratio}
\end{align}


First, we show that for all $\edge{i}{j} \in \setOfEdges$
\begin{align}
	\frac{\msg[\circ]{i}{j}{+1}_0}{\msg[\circ]{i}{j}{-1}_0} < \frac{\msg[\circ]{i}{j}{+1}_1}{\msg[\circ]{i}{j}{-1}_1}.
	\label{eq:appendix:msgratio1}
\end{align}
Therefore, consider the update rule of~\eqref{eq:update} for both states 

\begin{align}
	\msg[n+1]{i}{j}{+1}_1 \propto   e^{\coupling{i}{j} + \field{i} +\epsilon} \!\!\!\!\!\!\!\!\!  \prod \limits_{\RV{k} \in \{\neighbors{i} \backslash \RV{j}\}} \!\!\!\!\!\!\!\!\! \msg[n]{k}{i}{+1}_1 
	+ e^{-\coupling{i}{j} - \field{i} -\epsilon} \!\!\!\!\!\!\!\!\! \prod \limits_{\RV{k} \in \{\neighbors{i} \backslash \RV{j}\}} \!\!\!\!\!\!\!\!\! \msg[n]{k}{i}{-1}_1, 
	\label{eq:appendix:update_pos}
\end{align}
and 
\begin{align}
	\msg[n+1]{i}{j}{-1}_1 \propto   e^{-\coupling{i}{j} + \field{i} -\epsilon} \!\!\!\!\!\!\!\!\! \prod \limits_{\RV{k} \in \{\neighbors{i} \backslash \RV{j}\}} \!\!\!\!\!\!\!\!\! \msg[n]{k}{i}{+1}_1
	+ e^{\coupling{i}{j} - \field{i} +\epsilon} \!\!\!\!\!\!\!\!\! \prod \limits_{\RV{k} \in \{\neighbors{i} \backslash \RV{j}\}} \!\!\!\!\!\!\!\!\! \msg[n]{k}{i}{-1}_1.
	\label{eq:appendix:update_neg}
\end{align}
In~\eqref{eq:appendix:update_pos} the larger product is multiplied by $e^{\epsilon}$ and the smaller product is divided by $e^{\epsilon}$. For~\eqref{eq:appendix:update_neg} it is exactly the other way round so that the ratio between the messages increases which proofs~\eqref{eq:appendix:msgratio1}. 
We shall denote the imposed difference $\delta \in \REALPos$ on the messages by 
\begin{align}
	\msg[\circ]{i}{j}{+1}_1 = \msg[\circ]{i}{j}{+1}_0+\delta,\label{eq:appendix:msgdiff1}\\
	\msg[\circ]{i}{j}{-1}_1 = \msg[\circ]{i}{j}{-1}_0-\delta.\label{eq:appendix:msgdiff2}
\end{align}

Second, we show that $m_{i}^{0} < m_{i}^{1}$ which is an immediate consequence of plugging~\eqref{eq:appendix:msgdiff1} and~\eqref{eq:appendix:msgdiff2} into~\eqref{eq:marginals:single}.

Finally, it remains to show that $0 \stackrel{(i)}{<} \chi_{ij}^{0} \stackrel{(ii)}{<} \chi_{ij}^{1}$. Without loss of generality we assume that all variables have equal degree $d+1$ and constant coupling strength $\coupling{i}{j}  = \coupling{}{}$. 
First we show that (i) holds, i.e., $\chi=\chi_{ij}^{0}$ is positive. 
Let us express the marginals by~\eqref{eq:marginals:pw} and denote the messages by $\mu = \msg{i}{j}{1}_0$. It follows that $\msg{i}{j}{-1}_0 = (1-\mu)$ and that
\begin{align}
	\chi\! &=\! e^{\coupling{}{}+2\field{}} \mu^{2d}\!\! + \!e^{\coupling{}{}-2\field{}} (1\!-\!\mu)^{2d}\!\! - 2 e^{-J} \mu^{d} (1\!-\!\mu)^{d}.
	\label{eq:correlation:appdx}
\end{align}
Let us further represent the messages by $\mu = 1/2+x$ with $x \in [0,1/2]$. It follows that
\begin{align}
	\chi \stackrel{(a)}{\geq} &\left(\left(1/2+x\right)^{2d}+\left(1/2-x\right)^{2d}\right) 
	- 2 \left(1/2+x\right)^d\left(1/2-x\right)^d 
	= &\left(\left(1/2-x\right)^d-\left(1/2+x\right)^d\right)^2 \label{eq:correlation:positive:3}\nonumber\\
	\stackrel{(b)}{\geq}  & 0,  
\end{align}
where $(a)$ follows from neglecting all exponential terms and thus upper bounding the positive term and lower bounding the negative term (with equality if and only if $\coupling{}{} = 0$ and $\field{} = 0$) and
$(b)$ is a direct consequence of the square in~\eqref{eq:correlation:positive:3}.
Now let us show that (ii) holds, i.e., 
$\chi$ increases monotonically, by taking the derivative of~\eqref{eq:correlation:appdx}, so that
\begin{align}
	\frac{\partial}{\partial \mu}\chi = 
	& 2d \left(e^{\coupling{}{}+2\field{}} \mu^{2d-1} - e^{\coupling{}{}-2\field{}} \left(1-\mu\right)^{2d-1}\right) 
	+\!\! 2 d e^{-J}\!\! \left( \mu^{d} (1-\mu)^{d-1} \!\!- \mu^{d-1} (1-\mu)^{d} \right) \label{eq:correlation:monotone:1}\\
	\stackrel{(a)}{\geq} &  2 d e^{-J} \left( \mu^{d} (1-\mu)^{d-1} - \mu^{d-1} (1-\mu)^{d} \right) \label{eq:correlation:monotone:2}\nonumber\\
	\stackrel{(b)}{\geq} & 0,
\end{align}
where $(a)$ follows from neglecting the, strictly positive, first term in~\eqref{eq:correlation:monotone:1}, and $(b)$ 
is a direct consequence from~\eqref{eq:appendix:ratio}.
\newsubsection{Proof of Theorem~\ref{thm:attractive}}{selfguided:proofs:th9}
A unique start point $\pseudomarginalsMinLocal(\scaling=0)$ exists by Theorem~\ref{thm:init} that equals the exact pseudomarginals $\pseudomarginalsExact(\scaling=0)$ and, for $\field{i}> 0$, has positive mean and correlation), i.e., $\meanMinLocal{i}(\scaling=0) > 0$.


Consequently, Lemma~\ref{lm:griffith} applies, which further implies
that $\meanMinLocal{i}(\scaling)$ and $\chi_{ij}(\scaling) $ are monotonically increasing; moreover, $\meanMinLocal{i}(\scaling)$ and $ \chi_{ij}(\scaling) $ are continuous by Theorem~\ref{thm:smooth}. 
This further implies that the Bethe free energy $\FBLocalMin{m}(\scaling)$ decreases. 
Let $\FBLocalMin{m}(\scaling=1)$ correspond to the endpoint of the solution path $c(\scaling)$ that 
emerges from the origin; then, it immediately follows that the error with respect to the endpoint $\FBLocalMin{m}(\scaling=1)$ decreases along the solution path: i.e., consider two arbitrary values $m,k \in [0,1]$ such that $k>m$, then   
$|\FBLocalMin{m}(\scaling_m)-\FBLocalMin{m}(\scaling=1)| \geq |\FBLocalMin{m}(\scaling_{k})-\FBLocalMin{m}(\scaling = 1)|$.\\

It remains to show that SBP obtains the fixed point $\FBLocalMin{m}(\scaling=1)$ that minimizes the error with respect to the exact free energy,
i.e., $m = \argmin_{m\in\setOfMinimaSol}\EPartition{m}$ .
Therefore consider the fact, that attractive models with $\field{i} > 0$ have a unique fixed point 
that satisfies  $\meanMinLocal{i}(\scaling) \in (0,1]$~\cite{yedidia2005}.
A second minimum with negative means, however, may emerge for sufficiently large values of $\coupling{i}{j}$. We denote this alternative stationary point by $\FBLocalMin{n}$.  
This minimum $\FBLocalMin{n}$, if it exists, is close to being symmetric, i.e., $\meanMinLocal{i}^m(\scaling) -\epsilon = - \meanMinLocal{i}^n(\scaling)$ and $\chi_{ij}^m(\scaling) = {\chi}^{n}_{ij}(\scaling) +\epsilon$.

Now let us express the Bethe free energy in~\eqref{eq:f_bethe} of both fixed points $\pseudomarginalsMinLocal$ and $\pseudomarginalsMinLocalNeg$ in terms of their energy and entropy according to $\FB = \EB - \SB$ (cf. Section~\ref{sec:bp:variational}). Then, as a consequence of symmetry of the entropy $\SB(\pseudomarginalsMinLocalNeg) \cong \SB(\pseudomarginalsMinLocal)$
and as a consequence of singleton marginals that are not aligned to the local potentials in 
\eqref{eq:energy} $\EB(\pseudomarginalsMinLocalNeg)>\EB(\pseudomarginalsMinLocal)$. 
It follows that $ \FB(\pseudomarginalsMinLocalNeg) \geq \FB(\pseudomarginalsMinLocal)$ (cf.~\cite{pitkow2011learning}).
Consequently, with $\FB(\pseudomarginalsMinLocal)$ being more negative it follows that the fixed point $m$ constitutes the global minimum of the Bethe free energy.\\

That is, SBP proceeds along a solution path that leads towards the global minimum of the Bethe approximation. 
In particular, by considering the fact that the exact free energy is upper bounded by the Bethe approximation for attractive models~\cite{ruozzi2013bethebound}, this implies that the fixed point obtained by SBP indeed minimizes the approximation error, i.e,  $m = \argmin_{m\in\setOfMinimaSol}\EPartition{m}$ holds.

This concurrently implies that $\pseudomarginalsMinLocal(\scaling)$  is optimal with respect to marginal accuracy,
i.e., no stable fixed point -- which corresponds to a local minimum of the Bethe free energy -- exists that provides more accurate marginals. 
Note that attractive models exhibit so-called replica symmetric solutions where the Bethe free energy 
has two minima at most. 
In particular, this allows one to express 
the exact marginals as a convex combination of all fixed points, i.e, $\pseudomarginalsExact =  \frac{1}{\sum\limits_{\partitionBethe \in \setOfMinimaSol} \partitionBethe}
\sum\limits_{(\pseudomarginals,\partitionBethe)\in\setOfMinimaSol}\pseudomarginals\cdot\partitionBethe$ (cf. Section~\ref{sec:intro:bp:improving:evaluation}).

It follows by the existence of at most two solutions that the fixed point that minimizes the Bethe free energy is also more accurate.

\emptydoublepage
\newsection{Proofs from Chapter~\ref{chp:accuracyBP}}{accuracy:proofs}
This Section contains all the detailed proofs for Chapter~\ref{chp:accuracyBP} 
\newsubsection{Proof of Theorem~\ref{thm:effective_field}}{accuracy:proofs:effective_field}
Here we show how all incoming messages from outside the patch can be subsumed into an effective local field.
\begin{proof}
	First let us revisit the update equation from $\RV{i}$ to $\RV{j}:\RV{j} \in \patchRV{i}$
	\begin{align}
		\msg[n+1]{i}{j}{} \propto \sum \limits_{x_i \in \sampleSpace{X}}  \pairwiseShort{i}{j}\localShort{i} \!\!\! \!\prod \limits_{\RV{k} \in \{\neighbors{i} \backslash \RV{j}\}} \!\!\!\! \msg[n]{k}{i}{}. \nonumber
	\end{align}
	Now we group the incoming messages into two groups, i.e., messages coming from outside the patch and messages coming from inside the patch so that
	\begin{align}
		\msg[n+1]{i}{j}{} \propto \sum \limits_{x_i \in \sampleSpace{X}}  \pairwiseShort{i}{j}\localShort{i} 
		\prod \limits_{\RV{k} \in \neighbors{i} \backslash \{\RV{j} \cap \patchRV{i}\}} \!\!\!\!\!\!\!\! \msg[n]{k}{i}{}
		\prod \limits_{\RV{k} \in \{ \patchRV{i} \cap \neighbors{i} \backslash \RV{j}\}} \!\!\!\!\!\!\!\! \msg[n]{k}{i}{}. \nonumber
	\end{align}
	Now we make use of the fact that we are only dealing with binary random variables for which $\msg{i}{j}{-1} = (1-\msg{i}{j}{1})$ and express the message explicitly by
	\begin{align}
		\msg[n+1]{i}{j}{} \propto& \exp(J\RVval{j}) \exp(\field{i})  
		\prod \limits_{\RV{k} \in \neighbors{i} \backslash \{\RV{j} \cap \patchRV{i}\}} \!\!\!\!\!\!\!\! \msg[n]{k}{i}{1}
		\prod \limits_{\RV{k} \in \{ \patchRV{i} \cap \neighbors{i} \backslash \RV{j}\}} \!\!\!\!\!\!\!\! \msg[n]{k}{i}{1} \nonumber \\
		& + \exp(-J \RVval{j}) \exp(-\field{i})  
		\prod \limits_{\RV{k} \in \neighbors{i} \backslash \{\RV{j} \cap \patchRV{i}\}} \!\!\!\!\!\!\!\! (1-\msg[n]{k}{i}{1}) 
		\prod \limits_{\RV{k} \in \{ \patchRV{i} \cap \neighbors{i} \backslash \RV{j}\}} \!\!\!\!\!\!\!\! (1-\msg[n]{k}{i}{1}).
	\end{align}
	
	We now want to get rid of the first product and absorb the influence of these messages into the local field.
	In particular we aim to express it according to
	\begin{align}
		\exp \big( \effectiveField{i}\RVval{i}) & =  \localShort{i}  \prod \limits_{\RV{k} \in \neighbors{i} \backslash \{\RV{j} \cap \patchRV{i}\}} \!\!\!\!\!\!\!\! \msg[n]{k}{i}{} \label{eq:ef1} \nonumber \\
		&= \exp\big((\field{i}+c)\RVval{i}\big)\cdot\exp(g).
	\end{align}
	
	In order to do so we take the logarithm of the product over all messages in~\eqref{eq:ef1} and put them into the exponent with the local field so that for $\RVval{i}=+1$
	\begin{align}
		\exp (\effectiveField{i})  &= \exp\bigg(\field{i} +\sum \limits_{\RV{k} \in \neighbors{i} \backslash \{\RV{j} \cap \patchRV{i}\}} \log\big(\msg{k}{i}{1}\big)\bigg) \nonumber\\
		& = \exp\big(\field{i} + \sum \limits_{\RV{k} \in \neighbors{i} \backslash \{\RV{j} \cap \patchRV{i}\}} c_i + \sum \limits_{\RV{k} \in \neighbors{i} \backslash \{\RV{j} \cap \patchRV{i}\}} g_i\big), 
		\label{eq:equate1}
	\end{align}
	and for $\RVval{i}=-1$
	\begin{align}
		\exp (-\effectiveField{i}) &= \exp\bigg(-\field{i} + \sum \limits_{\RV{k} \in \neighbors{i} \backslash \{\RV{j} \cap \patchRV{i}\}} \log\big(1-(\msg{k}{i}{1}\big)\bigg) \nonumber\\
		& = \exp\big(-\field{i} - \sum \limits_{\RV{k} \in \neighbors{i} \backslash \{\RV{j} \cap \patchRV{i}\}} c_i + \sum\limits_{\RV{k} \in \neighbors{i} \backslash \{\RV{j} \cap \patchRV{i}\}}  g_i\big).
		\label{eq:equate2}
	\end{align}
	Equating the coefficients for $c_i$ and $g_i$ in~\eqref{eq:equate1} and~\eqref{eq:equate2} gives us the final results:
	\begin{align}
		c_i &= \frac{1}{2}\bigg( \log \msg{k}{i}{1} -  \log \big(1-\msg{k}{i}{1}\big)\bigg), \nonumber \\
		c   &= \sum\limits_{\RV{k} \in \neighbors{i} \backslash \{\RV{j} \cap \patchRV{i}\}} \arctanh\big(2\msg{k}{i}{1} - 1\big),\\
		g   &= \frac{1}{2} \log \prod \limits_{\RV{k} \in \neighbors{i} \backslash \{\RV{j} \cap \patchRV{i}\}} \Big(\msg{k}{i}{1}-\msg{k}{i}{1}^2\Big).
	\end{align}

	Note that we can further express the message for the second state $\msg{i}{j}{-1}$ in a similar way, with the only difference that the values of the pairwise potentials change. Consequently we get exactly the same result for $g$ again;
	this allows us to neglect the influence of $g$ altogether, as it will be canceled out when normalizing the messages so that the sum up to one.
\end{proof}


%


\newsubsection{Proof of Corollary~\ref{cor:error_ratio}}{accuracy:proofs:error_ratio_ex1}
We want to compare error of the state-preserving fixed point $\EMarginal{p}$ to the error $\EMarginal{q}$ of a fixed point that has all marginals biased towards one state.
We already discussed the error-ratio in a general manner in the proof of Theorem~\ref{thm:marginal_not_equal_partition} in Section~\ref{sec:accuracyBP:theory:properties};
here we provide a more accessible proof for the special case of a model with two patches, i.e., for Example~\ref{ex:patch}.

\begin{proof}
	For a symmetric model with two equal-sized patches we evaluate the error ratio between the state-preserving fixed point $p$ and one of the fixed points that have all marginals biased towards one state, these are $q$ and $r$ and have \emph{symmetric} marginals.
	Further assume that the fixed points $q$ and $r$ minimize the Bethe free energy, i.e., 
	\begin{align}
		\FB^q = \FB^r &< \FB^p,\\
		\partitionBethe^q = \partitionBethe^r &> \partitionBethe^p.\label{eq:bethe_optimal}
	\end{align}
	
	Then we want to show that~\eqref{eq:bethe_optimal} does not imply that $\EMarginal{q} < \EMarginal{p}$, i.e., we want to show that 
	$\frac{\EMarginal{p}}{\EMarginal{q} } < 1$ despite~\eqref{eq:bethe_optimal}; therefore, we express the ratio of the marginal errors according to 
	\begin{align}
		\frac{\EMarginal{p}}{\EMarginal{q} } &= \frac{ \sum_{\RV{i}}|\sum_{m\backslash p} \partitionBethe^m \mismatch{m}{p}|^2}
		{ \sum_{\RV{i}}|\sum_{m\backslash q } \partitionBethe^m \mismatch{m}{q}|^2} \nonumber \\
		&\stackrel{}{=}\frac{ \sum_{\RV{i}}|\partitionBethe^q \mismatch{q}{p}+\partitionBethe^r \mismatch{r}{p}|^2}
		{ \sum_{\RV{i}}|\partitionBethe^q \mismatch{r}{q} + \partitionBethe^p \mismatch{p}{q}|^2}.
		\label{eq:marginal_ratio2}
	\end{align}
	Note that because of~\eqref{eq:bethe_optimal} we have
	\begin{align}
		\frac{\EMarginal{p}}{\EMarginal{q} } &\stackrel{}{=} \frac{ \sum_{\RV{i}}|\partitionBethe^q \big( \mismatch{q}{p}+\mismatch{r}{p}\big)|^2}
		{ \sum_{\RV{i}}|\partitionBethe^q \mismatch{r}{q} + \partitionBethe^p \mismatch{p}{q}|^2}\nonumber\\
		&\stackrel{(a)}{=}\frac{ \sum_{\RV{i}}|\partitionBethe^q \big( \mismatch{q}{p}+\mismatch{r}{p}\big)|^2}
		{ \sum_{\RV{i}}\big(\partitionBethe^q \mismatch{q}{r} + \partitionBethe^p \mismatch{q}{p}\big)^2}.
	\end{align}
	where (a) follows from the fact that $\mismatch{r}{q} < 0$,  $\mismatch{p}{q}<0$, and the symmetry property~\eqref{eq:mismatch_symmetry}.
	We further denote the constant difference between the biased fixed points by
	\begin{align}
		0<  \mismatch{q}{r} = d < 1\label{eq:mismatch_bound}.
	\end{align}
	We can use the expansion property~\eqref{eq:mismatch_expansion} to bound the numerator as $|\mismatch{q}{p}+\mismatch{r}{p}|^2 < \mismatch{q}{r}^2 = d^2$ so that
	\begin{align}
		\frac{\EMarginal{p}}{\EMarginal{q} } < \frac{ \sum_{\RV{i}}\partitionBethe^q d^2}
		{ \sum_{\RV{i}}\big(\partitionBethe^q d + \partitionBethe^p \mismatch{q}{p}\big)^2}.
	\end{align}
	Which completes the proof as $ \mismatch{q}{p} > 0$
\end{proof}

\newsubsection{Proof of Theorem~\ref{thm:marginal_equal_partition}}{accuracy:proofs:marginal_equal_partition}
Let us consider three different stationary points $\FB^p$ (state-preserving), 
$\FB^q$ (biased to one state), and $\FB^m$ that has some patches flipped.
Note that we consider $\FB^q$ as a limiting case for $\FB^m$.
We will denote the number of variables that are aligned with the local field $N_c$ and the number of flipped variables $N_f$.

The set of boundary edges that connects two patches is denoted by
\begin{align}
	\setOfEdges_P = \{\edge{i}{j}\in\setOfEdges : \RV{i} \in \patchRV{i}, \RV{j} \in \patchRV{j} \neq \patchRV{i}\},  
\end{align}
and the set of edges that connects two patches that have their variables not aligned is denoted by
\begin{align}
	\setOfEdges_C = \{\edge{i}{j} \in \setOfEdges: \RV{i} \in \patchRV{i}, \RV{j} \in \patchRV{j} \neq \patchRV{i}, \sgn \big( \pmfApprox{\RV{i}}(\RV{i}=1)-0.5 \big) \neq \sgn \big( \pmfApprox{\RV{j}}(\RV{j}=1) -0.5\big)\}.
\end{align}
Note that $\setOfEdges_P$ is constant for a specified model, whereas $\setOfEdges_C$ depends on the specific fixed point.
We consequently have $\setOfEdges_C  \leq \setOfEdges_P$ with equality for the state preserving fixed point $p$.

Our analysis is restricted to $\parameter \in \Region$ per definition: one crucial consequence is that $J>J_A(\field{})$ and that most marginals either have $\pmfApprox{\RV{i}}(\RVval{i}) \approx 1$ or $\pmfApprox{\RV{i}}(\RVval{i}) \approx 0$.
We will exploit this fact and express all marginals according to $\pmfApprox{\RV{i}}(\RVval{i}) \in \{0,1\}$ which allows us to simplify $\FB$, as defined in Section~\ref{sec:bp:variational}, according to 
\begin{align}
	\FB^p &= -N\field{}-\big(|\setOfEdges|-2|\setOfEdges_P| \big)\coupling{}{} - \SB^p\\
	\FB^m &= -\big(N_c - N_f \big)\field{}-\big(|\setOfEdges|-2|\setOfEdges_C| \big)\coupling{}{} - \SB^m.
\end{align}
Let $\Delta \SB = \SB^p-\SB^m$ be the difference in the entropy, then we can express the conditions for the state-preserving fixed point to have a lower value $\FB^p \leq \FB^m$ according to
\begin{align}
	-N\field{}-\big(|\setOfEdges|-2|\setOfEdges_P| \big)\coupling{}{} &\leq -\big(N_c - N_f \big)\field{}-\big(|\setOfEdges|-2|\setOfEdges_C| \big)\coupling{}{} +\Delta \SB \nonumber \\
	2 J(-|\setOfEdges_C|+|\setOfEdges_P|) & \leq \field{}(N-N_c+N_f)+\Delta \SB \label{eq:bound_exact}
\end{align}
Now let us express~\eqref{eq:bound_exact} for the fixed points that has all variables biased to one state, i.e., $|\setOfEdges_C|=0$. Then, the state-preserving fixed point has a lower value $\FB^p < \FB^q$ if
\begin{align}
	2 J|\setOfEdges_P| & \leq \field{}(N-N_c+N_f)+\Delta \SB \label{eq:bound_exact_extreme}
\end{align}



\newpage
\newsection{Pseudocode}{appendix:pseudocode}
We present the pseudocode for NIBP and WDBP. Removing the if then else clause in line 8 to 11 of NIBP and substituting it with $\mu_{m}^{n}$ $\leftarrow$ $\mu_{m}^{n+1}$ reduces Algorithm~\ref{alg:nibp} to RBP. The maximum number of iterations is denoted by $N_{BP} = 2.5 \cdot 10^5$ and $\epsilon = 10^{-3}$. $\text{NrOfMessages} = 2|\setOfEdges|$ denotes the overall number of messages in the graph.
\begin{algorithm}
	{
		\caption{Noise Injection Belief Propagation (NIBP)}\label{alg:nibp}}
	\LinesNumbered
	\DontPrintSemicolon
	\SetKwData{NrOfMessages}{NrOfMessages}\SetKwData{NrUpdates}{NrUpdates}\SetKwData{Aand}{\textbf{and}}
	\SetKwFunction{ComputeUpdate}{ComputeUpdate}\SetKwFunction{IndexMax}{IndexMax}\SetKwFunction{OscillationDetection}{OscillationDetection}
	\SetKwInOut{Input}{input}\SetKwInOut{Output}{output}
	\Indm  
	\Input{Graph $G = (\mathbf{X},\mathbf{E})$}
	\Output{Converged messages $\setOfMessages[n]$}
	\BlankLine
	\Indp
	initialization\;
	\BlankLine
	\For{$m \leftarrow 1$ \KwTo \NrOfMessages}
	{ 
		$\mu^{n+1}_{m}$ $\leftarrow$ \ComputeUpdate{$\setOfMessages[n]$}\; 
		$r_m \leftarrow$ $|\mu_m^{n} - \mu_m^{n+1}|$
	}
	$k\leftarrow1$\;
	\While{$n < N_{BP}$ \Aand $\max |\setOfMessages[n] - \setOfMessages[n+1]| > \epsilon$}
	{ 
		$m \leftarrow \argmax_m \vm{r}^{n}$\;
		\eIf{ \OscillationDetection{$\mu_m^{n}$,L}}
		{	$\mu_{m}^{n}$ $\leftarrow$ $\mu_{m}^{n+1} + \mathcal{N}(0,\sigma)$ \;
		}
		{ $\mu_{m}^{n}$ $\leftarrow$ $\mu_{m}^{n+1}$ \;
		}
		\For{$j \leftarrow 1$ \KwTo \NrOfMessages}
		{
			$\mu_{j}^{n+1} \leftarrow$ \ComputeUpdate{$\setOfMessages[n]$}\;
			$r_j \leftarrow |\mu_{j}^{n+1} - \mu_{j}^{n}|$
		}   
		$n = n + 1$
	}
	
\end{algorithm}
\begin{algorithm}{
		\LinesNumbered
		\DontPrintSemicolon
		\caption{Weight Decay Belief Propagation (WDBP)}\label{alg:wdbp}
		\SetKwData{NrOfMessages}{NrOfMessages}\SetKwData{NrUpdates}{NrUpdates}\SetKwData{Aand}{\textbf{and}}
		\SetKwFunction{ComputeUpdate}{ComputeUpdate}\SetKwFunction{IndexMax}{IndexMax}
		\SetKwInOut{Input}{input}\SetKwInOut{Output}{output}
		\Indm  
		\Input{Graph $\graph = (\mathbf{X},\mathbf{E})$}
		\Output{Converged messages $\setOfMessages[n]$}
		\BlankLine
		\Indp
		initialization\;
		\BlankLine
		\For{$m \leftarrow 1$ \KwTo \NrOfMessages}
		{ 
			$\mu_{m}^{n+1}$ $\leftarrow$ \ComputeUpdate{$\setOfMessages[n]$}\; 
			$r_m \leftarrow$ $|\mu_m^{n} - \mu_m^{n+1}|$\;
			\NrUpdates(m) $\leftarrow$ $1$\;
		}
		$k\leftarrow1$\;
		\While{$n < N_{BP}$ \Aand $\max |\setOfMessages[n] - \setOfMessages[n+1]| > \epsilon$}
		{ 
			$m \leftarrow \argmax_m \vm{r}^{n}$\;
			$\mu_{m}^{n}$ $\leftarrow$ $\mu_{m}^{n+1}$ \;
			\NrUpdates(m) $\leftarrow$ \NrUpdates(m) + 1 \;
			\For{$j \leftarrow 1$ \KwTo \NrOfMessages}
			{
				$\mu_{j}^{n+1} \leftarrow$ \ComputeUpdate{$\setOfMessages[n]$}\;
				$r_j \leftarrow \frac{|\mu_{j}^{n+1} - \mu_{j}^{n}|}{\textnormal{NrUpdates(j)}}$
			}   
			$n = n + 1$
		}
	}
\end{algorithm}

\emptydoublepage
\newchapter{Guide to Quotes}{quotes}

The attentive reader has probably recognized that this thesis features an opening quote for every chapter. 
It was important to me to connect with all the quotes on a personal level (be it by one's art or by a general appreciation of a person's life) 
Yet, this was not enough:
every quote must also have a close connection to the respective chapter.
This connection, however, may not always be as obvious to the reader as it was to me.
It is precisely for this reason that I will now reveal the perceived meaning of each quote, and how I see it connect to the content of the chapter.\\

Chapter~\ref{chp:intro}: We quote the commonly used English translation \emph{``Uncertainty is an uncomfortable position. But certainty is an absurd one.''} of the original line \emph{``Le doute n'est pas une état bien agréable, mais l'assurance est un état ridicule''}~\cite[p.703]{voltaire}.
This statement from Voltaire elegantly connects to the central concept of this thesis, probabilistic graphical model, as it can be interpreted as a praise for probability.\\

Chapter~\ref{chp:background}: 
\emph{``A mind is like a parachute. It doesn't work if it is not open.''} This quote by Frank Zappa seems to align nicely with a background-chapter that brings different scientific fields together into one powerful concept.
In particular, since it was Frank Zappa's ability to draw inspiration from an unimaginable wide range of musical genres that finally coalesced together and created some stunning albums. \\

Chapter~\ref{chp:bp}: Ludwig Boltzmann played a formative role in the foundational years of statistical physics and laid out the groundwork for many concepts discussed in this chapter.
Beyond that, there is undeniable truth in his statement \emph{``The stars bend like slaves to laws not decreed for them by human intelligence, but gleaned from them.''} quoted from~\cite{greenstein1991science}.
One could argue that this whole thesis adheres to this statement as we put considerable effort into  carefully studying the behavior of belief propagation to develop a clearer picture of its underlying working-principles.\\

Chapter~\ref{chp:solving}: \emph{``Inside a broken clock; Splashing the wine; With all the rain dogs; Taxi, we'd rather walk...''}~\cite{waits1985rain}.
The song figuratively refers to the rain dogs as the straying dogs that are seemingly lost after the rain has washed away all their scent.
The underlying meaning is elegantly encompassed by just another quote of Tom Waits that states: \textit{``We are buried beneath the weight of information, which is being confused with knowledge; quantity is being confused with abundance and wealth with happiness.''}.
Although arguable more general and touching a couple of interesting points of today's society, Tom Waits' statement reflects the experiences one makes when acquainting oneself with the subtleties of a new field.
When embarking into uncharted territory and trying to grasp all impressions one can only sympathize with the rain dogs that must feel similar.\\

Chapter~\ref{chp:solutionsBP}: 
\emph{``Travel makes one modest. You see what a tiny place you occupy in the world.''} \cite{flaubert1996}.
Flaubert urges the need for always taking a look at the surrounding to put things into perspective.
BP often behaves like a creature of habit, once it decides for one specific fixed point it neglects all others.
In this chapter we break up with this habit and take a global perspective on the solution space, essentially traveling through the solution space.\\

Chapter~\ref{chp:selfguided}: 
\emph{``The most significant dimension of freedom is the freedom from one's own ego - in other words, from the feeling that I am the center of everything.''}
Voytek Kurtyka is an outstanding mountaineer that played an important role in bringing the modern alpine-style climbing into the greater ranges.
More importantly, he was always honest about his inner struggle between the desire to nourish one's ego and the awareness that there is no inbred reason to take oneself to serious.
One could argue that our proposed algorithm, self-guided BP, undergoes a similar struggle.
At the beginning, all variables take up a purely ego-centered standpoint and no interactions between them take place. 
Only incrementally, as the couplings are incorporate, the variables contribute to the joint assignment of the model and put their own constraints behind.\\

Chapter~\ref{chp:accuracyBP}: 
\emph{``To the wise, life is a problem; to the fool, a solution.''}
Wise and fool are maybe strong words;
but one could have been satisfied with the preceding chapters and could have neglected the unanswered questions lingering around instead of being curious and trying to gain some further insights.
On a side-node, Marcus Aurelius was a great proponent of the philosophy of stoicism, a philosophy that is arguable beneficial when working on a thesis for many years.\\

Chapter~\ref{chp:discussion}: 
\emph{``There must be some way out of here; Said the joker to the thief``} \cite{dylan1967watchtower}.
This song puts you in the midst of a discussion, with the joker struggling with his purpose.
Even though it was inevitably to experience some struggles in the process of writing everything up, there is no easy way out and, ultimately, it is just a question of getting things done before one can appreciate the completed task.

\emptydoublepage

\bibliographystyle{alpha}
\bibliography{thesis_arxiv}


\FloatBarrier\label{end-of-document}
\end{document}